\documentclass{article}

\usepackage{microtype}
\usepackage{graphicx}
\usepackage{booktabs} 
\usepackage{caption,subcaption}
\usepackage{float}

\usepackage[loading]{tracefnt}
\usepackage{etoolbox}

\newtoggle{update}
\togglefalse{update}

\newcommand*{\update}[2]{\iftoggle{update}{\textcolor{red}{#1}\footnote{#2}}{#1}}

\usepackage{hyperref}
\usepackage[accepted]{icml2024}

\usepackage{amsmath}
\usepackage{amssymb}
\usepackage{mathtools}
\usepackage{amsthm}

\usepackage[capitalize,noabbrev]{cleveref}

\newcommand*{\N}{\mathbb{N}}

\newcommand*{\R}{\mathbb{R}}

\newcommand*{\cD}{\mathcal{D}}

\newcommand*{\cN}{\mathcal{N}}

\newcommand*{\bd}{d} 

\newcommand*{\fnn}{f_{\mathrm{nn}}}
\newcommand*{\unn}{u_{\mathrm{nn}}}
\newcommand*{\fntk}{f_{\mathrm{ntk}}}
\newcommand*{\untk}{u_{\mathrm{ntk}}}

\newcommand*{\E}{\mathop{\mathbb{E}}}
\newcommand*{\GP}{\mathcal{GP}}
\newcommand*{\defeq}{\coloneqq}
\newcommand*{\ind}[1]{\mathbb{I}_{[{#1}]}}

\newcommand*{\WO}[1]{W_{0}^{(#1)}}
\newcommand*{\WA}[1]{W_{\mathcal{A}}^{(#1)}}

\newcommand*{\go}[1]{g_{0}^{(#1)}}
\newcommand*{\ga}[1]{g_{\mathcal{A}}^{(#1)}}

\newcommand*{\ho}[1]{h_{0}^{(#1)}}
\newcommand*{\ha}[1]{h_{\mathcal{A}}^{(#1)}}

\newcommand*{\Var}{\mathrm{Var}}
\newcommand*{\Cov}{\mathrm{Cov}}

\DeclareMathOperator*{\argmin}{argmin}

\theoremstyle{plain}
\newtheorem{theorem}{Theorem}[section]

\newtheorem{lemma}[theorem]{Lemma}
\newtheorem{corollary}[theorem]{Corollary}
\theoremstyle{definition}
\newtheorem{definition}[theorem]{Definition}
\newtheorem{assumption}[theorem]{Assumption}
\theoremstyle{remark}
\newtheorem{remark}[theorem]{Remark}

\usepackage[disable,textsize=tiny]{todonotes}

\icmltitlerunning{An Infinite-Width Analysis on the Jacobian-Regularised Training of a Neural Network}

\begin{document}

\twocolumn[
\icmltitle{An Infinite-Width Analysis on the Jacobian-Regularised Training of a Neural Network}



\icmlsetsymbol{equal}{*}

\begin{icmlauthorlist}
  \icmlauthor{Taeyoung Kim}{sockaist}
  \icmlauthor{Hongseok Yang}{sockaist}
  \end{icmlauthorlist}
  
  \icmlaffiliation{sockaist}{School of Computing, KAIST, Daejeon, South Korea}

\icmlcorrespondingauthor{Hongseok Yang}{hongseok.yang@kaist.ac.kr}

\icmlkeywords{Jacobian Regularisation, Neural Tangent Kernel}

\vskip 0.3in
]



\printAffiliationsAndNotice{}  

\begin{abstract}
    The recent theoretical analysis of deep neural networks in their infinite-width limits has deepened our understanding of initialisation, feature learning, and training of those networks, and brought new practical techniques for finding appropriate hyperparameters, learning network weights, and performing inference. In this paper, we broaden this line of research by showing that this infinite-width analysis can be extended to the Jacobian of a deep neural network. We show that a multilayer perceptron (MLP) and its Jacobian at initialisation jointly converge to a Gaussian process (GP) as the widths of the MLP's hidden layers go to infinity and characterise this GP. We also prove that in the infinite-width limit, the evolution of the MLP under the so-called robust training (i.e., training with a regulariser on the Jacobian) is described by a linear first-order ordinary differential equation that is determined by a variant of the Neural Tangent Kernel. We experimentally show the relevance of our theoretical claims to wide finite networks, and empirically analyse the properties of kernel regression solution to obtain an insight into Jacobian regularisation.
\end{abstract}

\section{Introduction}

The recent theoretical analysis of deep neural networks in their infinite-width limits has substantially deepened our understanding of initialisation~\cite{Neal1996,lee2018deep,Matthews2018,Yang2019TensorPI,Favaro2020,LeeYYL22}, feature learning~\cite{Yang21FeatureLearning}, hyperparameter turning~\cite{Yang22Transfer}, and training of those networks~\cite{Jacot2018,Lee19NTKLinear,Yang21NTK}. 
Seemingly impossible questions, such as whether the gradient descent achieves the zero training error and where it converges eventually, are answered~\cite{Jacot2018}, and new techniques for scaling hyperparameters~\cite{Yang21FeatureLearning} or finding appropriate hyperparameters~\cite{Yang22Transfer} have been developed.
Also, the tools used in the analysis, such as Neural Tangent Kernel, turn out to be useful for analysing pruning and other optimisations for deep neural networks~\cite{Liu20pruningNTK}. 

Our goal is to extend this infinite-width analysis to the Jacobian of a deep neural network. 
By Jacobian, we mean the input-output Jacobian of a network's output. This Jacobian has information about the smoothness of the network's output and has been used to measure the robustness of a trained network against noise \cite{Peck2017} or to learn a network that achieves high accuracy on a test dataset even when examples in the dataset are corrupted with noise. 
The Jacobian of a neural network also features in the work on network testing, verification, and adversarial attack \cite{Goodfellow2014,Zhang2019,Wang2021}.

We show that at initialisation, a multilayer perceptron (MLP) and its Jacobian jointly converge to a zero-mean Gaussian process (GP) as the widths of the MLP's hidden layers go to infinity, and characterise this GP by describing its kernel inductively. 
Our result can be used to compute analytically the posterior of the MLP (viewed as a random function) conditioned on not just usual input-output pairs $(x^{(1)},y^{(1)}),\ldots,(x^{(N)},y^{(N)})$ in the training set, but also derivatives at some inputs, in the infinite-width setting.

We also analyse the training dynamics of an MLP under the so-called \emph{robust-training} objective, which contains a regulariser on the Jacobians of the MLP at each training input, in addition to the standard loss over input-output pairs in the training set.
As in the case of the standard training without the Jacobian regulariser, we show that the training can be characterised by a kernel that is defined in terms of the derivatives of the MLP output and its Jacobian with respect to the MLP parameters. 
We call this kernel Jacobian Neural Tangent Kernel (JNTK) due to its similarity with the standard Neural Tangent Kernel (NTK) and show that it becomes independent of the MLP parameters at initialisation, as the widths of the MLP go to infinity. 
Thus, the JNTK at initialisation is deterministic in the infinite-width limit. 
Then, we show the JNTK of the infinitely-wide MLP stays constant during the robust training. We identify a linear first-order ordinary differential equation (ODE) that characterises the evolution of this infinitely-wide MLP during robust training and describe the analytic solution of the ODE. 
We experimentally confirm that the conclusions of our theoretical analysis apply to finite MLPs with large but realistic widths. We also include empirical analysis of the ODE solution obtained and empirically check when our assumption is satisfied.

The rest of the paper is organised as follows. In Section~\ref{sec:setup}, we describe the MLPs and robust-learning objective used in the paper. 
We then present our theoretical results for initialisation in Section~\ref{sec:JacobianNNGP} and for training dynamics in Section~\ref{sec:JacobianNTK}, and describe preliminary experimental observations in Section~\ref{sec:experiments}. 
We conclude our paper with related works in Section~\ref{sec:conclusion}.
The missing proofs and additional experiments can be found in the supplementary material.

\section{Setup}
\label{sec:setup}
We use the following convention and notation. 
For any $n \in \N$, $[n]$ means the set $\{1,\ldots,n\}$. 
We write $x^{(1:N)}$ to denote the set $\{x^{(1)}, \ldots, x^{(N)}\}$.
We count the indices of vectors and matrices from $0$, not $1$.
For instance, for $\alpha,\beta \in [d_0]$, a vector $x \in \mathbb{R}^{d_0}$, we write $x_{\alpha-1}$ for the $\alpha$-th component of $x$, and for the matrix $M$ of size $d_0 \times d_0$, we write $M_{\alpha-1,\beta-1}$ for the $(\alpha,\beta)$-th entry of $x$. 
Note that we often use matrix of size $(1 + d_0) \times (1 + d_0)$, where $M_{\alpha,\beta}$ denotes the $(\alpha+1,\beta+1)$-th entry of $M$ for $\alpha,\beta \in [d_0]$.

We use multilayer perceptrons (MLPs) $f:\R^{d_0} \to \R^{d_{L+1}}$, which are defined as follows: 
for all inputs $x \in \R^{d_0}$ and all layers $l = 2,\ldots,L$, 
\begin{align*}
  h^{(1)}(x) \in \R^{d_1},  
  &
  h^{(1)}(x) \defeq \phi( W^{(1)}x),
  \\
  h^{(l)}(x) \in \R^{d_l},
  &
  h^{(l)}(x) \defeq \phi\bigg( \frac{1}{\sqrt{d_{l-1}}} W^{(l)} h^{(l-1)}(x) \bigg),
  \\
  f(x) \in \R^{d_{L+1}},
  &
  f(x) \defeq \frac{\kappa}{\sqrt{d_L}} W^{(L+1)} h^{(L)}(x).
\end{align*}
Here the weights $W^{(l)} \in \R^{d_{l} \times d_{l-1}}$ are initialised with independent samples from $\cN(0,1)$. 
Note that due to this randomness of $W^{(l)}$, the MLP $f$ is a random function at initialisation.
Following \cite{Arora2019OnEC}, we introduce a fixed non-trainable $\kappa > 0$, which controls the scale of the output of $f$
and thus the randomness (or variance) of $f$ at initialisation. Suppressing the initial randomness is known to improve the performance 
of a neural network in practice~\cite{Hu2020Simple}, and simplify the theoretical analysis of the network.

To simplify the presentation, we do not include the bias terms in our definition. We also assume that
$d \defeq d_1 = \ldots = d_L$ and 
$d_{L+1} = 1$.
We also include some assumptions on the activation function $\phi$ as
\begin{assumption}[Activation Assumption]
  \label{assm:activation-assumption}
  (1) All of $\phi$ and its first and second derivatives $\dot{\phi}$ and $\ddot{\phi}$ are Lipschitz continuous with the Lipschitz coefficient $M_1$, $M_2$, and $M_3$, respectively.
  (2) The activation $\phi$ is normalised such that $\E_{z \sim \cN(0, 1)} \left[ \phi(z)^2 \right] = 1$.
\end{assumption}
We denote the set of all trainable parameters by 
$\theta \defeq \{W_{ij}^{(l)}\}_{l,i,j}$,
and write $\theta_t$ for $t \ge 0$ to refer to the parameters trained by the gradient flow until time $t$.
We sometimes make the dependency on the width $d$ and the parameters $\theta$ explicit in the MLP, and write $f_d$, or $f_{d,\theta}$. 

Throughout the paper, we study the training of an MLP $f_d$ that regularises the Jacobian of $f_d$.
Formally, this means that the training optimises the following \emph{robust-training} objective for a given training set $\cD = \{(x^{(1)}, y^{(1)}),\ldots,(x^{(N)},y^{(N)})\}$:
\begin{align}
  \mathcal{L}(\theta) \defeq \frac{1}{2N}\sum_{i = 1}^N \Big( &(f_{d,\theta}(x^{(i)}) -y^{(i)})^2 \nonumber
  \\
  + &\lambda \sum_{\alpha \in [d_0]} \big( J(f_{d,\theta})(x^{(i)})_{\alpha-1} \big)^2\Big)
  \label{eqn:objective-for-robust-training}
\end{align}
where $J(f) : \R^{d_0} \to \R^{d_0}$ is the Jacobian of $f$ defined as $J(f)(x)_{\alpha-1} = \frac{\partial f(x)}{\partial x_{\alpha-1}}$, and $\lambda \in (0, 1]$ is a hyperparameter determining the importance of the Jacobian regulariser. 
By including the Jacobian regulariser, this objective encourages $f_{d,\theta}$ to change little at each training input $x^{(i)}$ when the input is perturbed, so that when a new input $x$ is close to the input $x^{(i)}$ of a training example, the learnt predictor tends to return an output close to $y^{(i)}$ of the example. 
This encouragement has often been used to learn a robust predictor in the literature~\cite{hoffman2019}.

To establish our results, we need assumptions for the training dataset $\cD$ and a future input $x^* \in \R^{d_0}$.
\begin{assumption}[Dataset Assumption]
 \label{assm:dataset-assumption}
 (1) The inputs have the unit norm: $\|x^{(i)}\| = 1$ for all $i \in [N]$, and $\|x^*\|= 1$.
 (2) The outputs are bounded: $|y^{(i)}| \le 1$ for all $i \in [N]$.
\end{assumption}

\update{This dataset assumption is used in our theoretical analysis, but the use of unit norm or unit bound is not
essential. We can weaken the assumption to the boundedness of input norms and outpus by a fixed constant $C$, 
without invalidating the theoretical results in the paper.}{Added new sentence.}

\section{Infinite-width Limit at Initialisation}
\label{sec:JacobianNNGP}

We start by describing our results on the infinite-width limit of the Jacobian of an MLP $f_{d,\theta}$ at \emph{initialisation}. 
Recall that due to the random initialisation of the parameters $\theta$, both the MLP and its Jacobian are random functions at initialisation. 
It is well-known that as $d$ tends to infinity, $f_{d,\theta}$ at initialisation converges to a zero-mean Gaussian process \cite{lee2018deep,Matthews2018}. 
The kernel of this GP is commonly called NNGP kernel and has an inductive characterisation over the depth of the MLP. 
What concerns us here is to answer whether these results extend to the Jacobian of $f_{d,\theta}$. 
Does the Jacobian also converge to a GP? 
If so, how is the limiting GP of the Jacobian related to the limiting GP of the network output? 
Also, in that case, can we characterise the kernel of the limiting GP of the Jacobian?

The next theorem summarises our answers to these questions.
\begin{theorem}[GP Convergence at Initialisation]
  \label{thm:JacobianNNGPInitialisation}
  Suppose that Assumptions~\update{\ref{assm:activation-assumption} and \ref{assm:dataset-assumption}}{Added activation assumption.} holds.
  As $d$ goes to infinity, the function 
 $x \longmapsto ( f_d(x), J( f_d)(x)_0, \ldots, J( f_d)(x)_{d_0-1})^\intercal$
   from $\R^{d_0}$ to $\R^{d_0+1}$ converges weakly in finite marginal \footnote{A random function $f$ converges weakly in finite marginal, if every finite evaluation $\{f(x^{(i)})\}_{i=1}^N$ converges weakly for any $N < \infty$.} to a zero-mean GP with the kernel $\kappa^2 \Sigma^{(L)}$, 
    which we call \emph{Jacobian NNGP kernel}, where $\Sigma^{(L)} : \R^{d_0} \times \R^{d_0} \to \R^{(d_0+1)\times (d_0 + 1)}$
  is defined inductively as follows: for all inputs $x,x' \in \R^{d_0}$, layers $l \in [L]$, and indices  $\alpha,\beta \in [d_0]$, 
  \begin{align*}
  &\begin{aligned}
  \Sigma^{(0)}(x,x')_{00} \defeq \langle x, x'\rangle,
  &\quad 
  \Sigma^{(0)}(x,x')_{0\beta} \defeq x_{\beta-1},   
  \\
  \Sigma^{(0)}(x,x')_{\alpha0} \defeq x_{\alpha-1}',
  &\quad 
  \Sigma^{(0)}(x, x')_{\alpha\beta} \defeq \ind{\alpha = \beta},
  \end{aligned} 
  \\
  &\Sigma^{(l)}(x,x')_{00} \defeq \E_{g}[\phi(g(x)_0) \cdot \phi(g(x')_0)],
  \\
  &\Sigma^{(l)}(x, x')_{\alpha\beta} \defeq \E_{g}[\dot{\phi}(g(x)_0) g(x)_\alpha \cdot\dot{\phi}(g(x')_0)g(x')_\beta],
  \\
  &\Sigma^{(l)}(x,x')_{0\beta} \defeq \E_{g}[\phi(g(x)_0) \cdot \dot{\phi}(g(x')_0) g(x')_\beta],
  \\
  &\Sigma^{(l)}(x,x')_{\alpha0} \defeq \E_{g}[\dot{\phi}(g(x)_0) g(x)_\alpha \cdot \phi(g(x')_0)],
  \end{align*}
  where $\ind{-}$ is the indicator function, the subscript ${[-]}_\alpha$ denotes the $(\alpha+1)$-th component of a vector, and the random variable $g$ in the expectations is a zero-mean GP with the kernel $\Sigma^{(l-1)}$.
\end{theorem}
The $(1,1)$-th entry of the kernel $\kappa^2\Sigma^{(L)}(x,x')$, denoted by $\kappa^2\Sigma^{(L)}(x,x')_{00}$, specifies the covariance between the outputs of the limiting MLP  at $x$ and $x'$. It is precisely the standard NNPG kernel. 
The other entries describe how the components of the Jacobian of the limiting MLP are related between themselves and also with the MLP output. 

We prove this theorem using the tensor-program framework~\cite{Yang2019TensorPI}.
Our proof translates our convergence question into the one on a tensor program and uses the so-called Master theorem for tensor programs to derive an answer to the question. 
See the supplementary material for the details.

The entries of the Jacobian NNGP kernel can be derived from the $(1,1)$-th entry $\kappa^2\Sigma^{(L)}(x,x')_{00}$ via differentiation:
\begin{theorem}
  \update{Suppose that Assumption~\ref{assm:activation-assumption} holds.}{Add the activation assumption.}
  \label{thm:JacobianNNGPCorrespondence}
  For the kernel $\Sigma^{(L)}$ in Theorem~\ref{thm:JacobianNNGPInitialisation}, the following equalities hold for all $\alpha, \beta \in [d_0]$:
  \begin{align*}
    &\Sigma^{(L)}(x,x')_{0\beta} = \frac{\partial \Sigma^{(L)}(x,x')_{00}}{\partial x'_{\beta-1}},
    \\
    &\Sigma^{(L)}(x,x')_{\alpha0} = \frac{\partial \Sigma^{(L)}(x,x')_{00}}{\partial x_{\alpha-1}},
    \\
    &\Sigma^{(L)}(x,x')_{\alpha\beta} =
    \frac{\partial^2 \Sigma^{(L)}(x,x')_{00}}{\partial x_{\alpha-1}\partial x'_{\beta-1}}.
  \end{align*}
\end{theorem}
Thus, the standard result on GPs (Section 9.4 of \cite{GPbook}) implies that the realisation of the limiting GP defines a differentiable function in the first component of its output with probability one, and the Jacobian of this function is also a GP with a kernel induced from $\Sigma^{(L+1)}(x,x')_{00}$ via differentiation.

\section{Jacobian Neural Tangent Kernel}
\label{sec:JacobianNTK}

We next analyse the training dynamics of an MLP under the robust-training objective in \eqref{eqn:objective-for-robust-training}, which regularises the Jacobian of the MLP. 
As in the case of the standard objective without the Jacobian regulariser, the training dynamics of a \emph{finite} MLP is described by a kernel induced by the MLP. 
The next definition describes this kernel.
\begin{definition}[Finite Jacobian NTK]
The \emph{finite Jacobian Neural Tangent Kernel (finite JNTK)} of a finite MLP $f_{d}$ with parameters $\theta$ is a function
$\Theta_{d,\theta} : \R^{d_0} \times \R^{d_0} \to \R^{(d_0+1)\times (d_0+1)}$
defined as follows: for all $x, x' \in \R^{d_0}$ and $\alpha,\beta \in [d_0]$, 
  \begin{align*}
    &\Theta_{d,\theta}(x, x')_{00} \defeq 
    \left\langle \frac{\partial f_d(x)}{\partial \theta},\, \frac{\partial f_d(x')}{\partial \theta}\right\rangle,
    \\
    &\Theta_{d,\theta}(x, x')_{\alpha\beta} \defeq 
    \left\langle \frac{\partial J(f_d)(x)_{\alpha-1}}{\partial \theta},\, \frac{\partial J(f_d)(x')_{\beta-1}}{\partial \theta }\right\rangle,
    \\
    &\Theta_{d,\theta}(x, x')_{0\beta} \defeq 
    \left\langle \frac{\partial f_d(x)}{\partial \theta},\, \frac{\partial J(f_d)(x')_{\beta-1}}{\partial \theta}\right\rangle,
    \\
    &\Theta_{d,\theta}(x, x')_{\alpha0} \defeq 
    \left\langle\frac{\partial J(f_d)(x)_{\alpha-1}}{\partial \theta},\, \frac{\partial f_d(x')}{\partial \theta}\right\rangle.
  \end{align*}
\end{definition}
The finite JNTK includes the standard NTK of an MLP in the $(1,1)$-th entry, as expected. 
In addition, it computes the relationship between the gradients of the $\alpha$ and $\beta$-th components of the Jacobian of the MLP, and also between the gradient of a component of the Jacobian and the gradient of the MLP itself. 

The next lemma shows that the finite JNTK determines the robust training with the objective in \eqref{eqn:objective-for-robust-training}.
\begin{lemma} \label{lem:JNTKdynamics}
Assume that the parameters of the MLP $f_d$ evolve by the continuous version of the gradient update formalised by the ODE $\frac{d \theta_t}{dt} = - \frac{\partial \mathcal{L}(\theta_t)}{\partial \theta_t}$. 
Then for $\alpha \in [d_0]$,
\begin{align}
  &\frac{d f_{d,\theta_t}(x)}{dt} \nonumber
  \\
  & \ {} = {-} \frac{1}{N}\sum_{i = 1}^N \left\langle 
\Psi(f_{d,\theta_t},x^{(i)},y^{(i)}), \Theta_{d,\theta_t}(x,x^{(i)})_0 
\right\rangle, \label{eqn:kernel-gradient-descent}
  \\
  &\frac{d J(f_{d,\theta_t})(x)_{\alpha-1}}{dt} \nonumber
  \\
  & \ {} = {-} \frac{1}{N}\sum_{i = 1}^N \left\langle 
\Psi(f_{d,\theta_t},x^{(i)},y^{(i)}),
\Theta_{d,\theta_t}(x,x^{(i)})_{\alpha-1}\right\rangle, \label{eqn:kernel-gradient-descent2}
\end{align}
where $\Theta_{d,\theta_t}(x,x^{(i)})_0$ is the $(d_0+1)$-dimensional vector $(\Theta_{d,\theta_t}(x,x^{(i)})_{00},\ldots,\Theta_{d,\theta_t}(x,x^{(i)})_{0d_0})^\intercal$ obtained from the first row of the kernel output $\Theta_{d,\theta_t}(x,x^{(i)})$, $\Theta_{d,\theta_t}(x,x^{(i)})_\alpha$ is defined similarly but from the $(\alpha+1)$-th row as $(\Theta_{d,\theta_t}(x,x^{(i)})_{\alpha0},\ldots,\Theta_{d,\theta_t}(x,x^{(i)})_{\alpha d_0})^\intercal$,
and $\Psi(f_{d,\theta_t},x^{(i)}, y^{(i)})$ is the below vector in $\mathbb{R}^{d_0+1}$ constructed as: 
\begin{align*}
  \bigg(&f_{d,\theta_t}(x^{(i)}) - y^{(i)},
  \\
  &J(f_{d,\theta_t})(x^{(i)})_{0}, 
  \ldots, J(f_{d,\theta_t})(x^{(i)})_{d_0-1}\bigg)^\intercal. 
\end{align*}
\end{lemma}
{\color{black} We note that a similar result holds for gradient descent as well; see Appendix L of the supplementary material for details.
\todo[author=HY,inline]{Please double-check the above sentence. I edited it and also moved it from somewhere in the middle of the next paragraph to here.}
}
Each summand in the ODEs in \eqref{eqn:kernel-gradient-descent} and \eqref{eqn:kernel-gradient-descent2} is the inner product of two vectors of size $d_0+1$, where only the second vector depends on the input $x$ and this dependency is via the kernel $\Theta_{d,\theta_t}$. 
A similar form appears in the analysis of the standard training via the usual NTK. 
There it is also proved that if a loss function has a subgradient \cite{Yilan2021}, as the widths of the MLP involved increase, the equation simplifies greatly: the kernel used in the input-dependent vector stops depending on the initial values of the parameters (i.e., $\theta_0$) and stays constant over time $t$, and so, the ODE on the MLP becomes a simple linear ODE that can be solved analytically. 
We prove that a similar simplification is possible for the robust training with Jacobian regulariser in the infinite-width limit. 

First, we show that the finite JNTK at initialisation becomes deterministic (i.e., it does not depend on the initial parameter values $\theta_0$) in the infinite-width limit.
\begin{theorem}[Convergence of Finite JNTK at Initialisation]
\label{thm:JacobianNTKInitialisation}
Suppose that Assumptions~\update{\ref{assm:activation-assumption} and \ref{assm:dataset-assumption}}{Added activation assumption.} holds.
As $d$ goes to infinity, $\Theta_{d,\theta_0}$ almost surely converges to a function $\kappa^2\Theta$ in finite marginal, called \emph{limiting JNTK}, that does not depend on $\theta_0$. 
Here $\Theta$ is a (deterministic) function of type $\R^{d_0} \times \R^{d_0} \to \R^{(d_0+1) \times (d_0+1)}$, and has the following form: for $l \in [L]$, $x,x' \in \R^{d_0}$, and $\alpha,\beta \in [d_0]$, 
  \begin{align*}
    \Theta(x, x')_{00} &\defeq \sum_{l=0}^{L} \big( \Sigma^{(l)}(x, x')_{00} \times \Delta^{(l)}(x,x')\big), 
    \\
    \Theta(x, x')_{0\beta} &\defeq \frac{\partial \Theta(x, x')_{00}}{\partial x_\beta'},
    \\
    \Theta(x, x')_{\alpha 0} &\defeq \frac{\partial \Theta(x, x')_{00}}{\partial x_\alpha},
    \\
    \Theta(x, x')_{\alpha \beta} &\defeq \frac{\partial^2 \Theta(x, x')_{00}}{\partial x_{\alpha}\partial x_\beta'},
  \end{align*}
  where $\Sigma^{(l)}$ is the kernel defined inductively in Theorem~\ref{thm:JacobianNNGPInitialisation},
  and $\Delta^{(l)}$ is defined with a zero-mean GP $g^{(u)}$ with the kernel $\Sigma^{(u)}(x,x')_{00}$ as follows:
  $\Delta^{(l)}(x,x') \defeq \prod_{u=l}^{L-1} \E_{g^{(u)}} [ \dot{\phi}(g^{(u)}(x)) \dot{\phi}(g^{(u)}(x'))]$.
  The explicit form of JNTK can be found in the supplementary material.
\end{theorem}

Next, we show that the finite JNTK stays constant during training in the infinite-width limit, 
under the following assumption:
\begin{assumption}[Full Rank of the Limiting JNTK] 
\label{assm:NTK-smallest-eigenvalue}
  The minimum eigenvalue of the $N(d_0 + 1) \times N(d_0 + 1)$ matrix $\Theta(x^{(1:N)}, x^{(1:N)})$ is greater than $0$.
\end{assumption}
Note that the assumption is stated with $\Theta$, which is independent of the random initialisation of parameters $\theta$ and, thus, deterministic. In practice, we can test this assumption computationally, as we did in our experiments shown later.

The following result is a counterpart of the standard result on the change of the NTK \cite{Jacot2018,Lee19NTKLinear} during training without our Jacobian regulariser. 
It lets us solve the ODEs \eqref{eqn:kernel-gradient-descent} and \eqref{eqn:kernel-gradient-descent2} and derive their analytic solutions in the infinite-width limit. 
\begin{theorem}[Constancy of Finite JNTK during Training]
  \label{thm:JacobianNTKTraining}
  Suppose that Assumptions~\update{\ref{assm:activation-assumption}, \ref{assm:dataset-assumption}, and \ref{assm:NTK-smallest-eigenvalue}}{Added activation assumption.} hold. 
        Let $\theta_t$ be the parameter of the MLP at time $t$, trained with gradient flow.
  Assume that a dataset $\cD = (x^{(1:N)},y^{(1:N)})$ is used for training.
  Then, there exists a function $F$ 
  not depending on the width and the parameters of the MLP, such that for all $\epsilon, \delta \in (0,1)$, 
  if the MLP width satisfies
  $d \ge F(x^{(1:N)}, L, \epsilon, \delta)$,
  then with probability at least $1-\delta$ over the randomness of the MLP initialisation, the finite JNTK stays near the corresponding limiting JNTK during training: for all $t \ge 0$ 
  and all $i,j \in [N]$,
  \[
    \left\| (1/\kappa^2)\Theta_{d,\theta_t}(x^{(i)}, x^{(j)}) -\Theta(x^{(i)}, x^{(j)})\right\|_F \le \epsilon.
  \]
\end{theorem}
{\color{black} We note that a similar theorem holds for gradient descent; see Appendix L of the supplementary material for details.
\todo[author=HY,inline]{I added the above sentence. I also removed Taeyoung's addition to gradient descent in the statement of Theorem 4.5, which I found confusing and not entirely consistent. I think that it is better to make a separate statement on gradient descent}}
We prove Theorem~\ref{thm:JacobianNTKTraining} in the supplementary material. 
The structure of the proof is similar to that of the standard results. 
We first show that if the parameters at time $t$ stay close to their initialisation, then the finite JNTK also does. 
Then, we prove that the gradient flow does not move the parameters far away from their initialisation. 
Unlike in the proof of the standard results, when we upper-bound the movement of the parameter matrix $W^{(l)}_t$ at each layer $l$ and time $t$ from its initialisation $W^{(l)}_0$, we consider all of $1$, $2$, and $\infty$ matrix norms; the standard proof considers only the $2$ norm. 
This more precise bookkeeping is needed because of the Jacobian regulariser in our setting. 
Now, Theorem~\ref{thm:JacobianNTKTraining} makes the ODEs in \eqref{eqn:kernel-gradient-descent} and \eqref{eqn:kernel-gradient-descent2} the linear first-order ODEs, which allow us to derive their analytic solutions at $t = \infty$ for a given initial condition.

We describe the solution of the ODE in \eqref{eqn:kernel-gradient-descent} at time $t = \infty$.
Recall that $\lambda$ is the coefficient of the gradient regulariser in the robust-training objective (Equation~\eqref{eqn:objective-for-robust-training}).
Let $\Theta_\lambda : \R^{d_0} \times \R^{d_0} \to \R^{(d_0+1) \times (d_0+1)}$ be the following function: for all $x,x' \in \R^{d_0}$, $\Theta_{\lambda}(x, x') \defeq \Lambda\Theta(x,x')\Lambda$, where $\Lambda$ is the ${(d_0+1) \times (d_0+1)}$ diagonal matrix with $\Lambda_{00} = 1$ and $\Lambda_{ii} = \sqrt{\lambda}$ for all $i \neq 0$.
Using $\Theta_\lambda$, we define key players in our theorem on the ODE solution. The first is the matrix
$\Theta_{\lambda}(x^{(1:N)}, x^{(1:N)}) \in \R^{N(d_0+1) \times N(d_0+1)}$
that is constructed by including every $\Theta_\lambda(x^{(i)},x^{(j)}) \in \R^{(d_0+1)\times(d_0+1)}$ as the $(i,j)$-th submatrix. The next players are
the vectors $\mathbf{y}^{(1:N)} \in \R^{N(d_0+1)}$ and $\Theta_\lambda(x,x^{(1:N)})_0 \in \R^{N(d_0+1)}$ defined by stacking the below vectors: for all $i \in [N]$ and $x \in \R^{d_0}$,
\begin{align*}
  \mathbf{y}^{(i)} &\defeq (y_i, 0, \ldots, 0)^\intercal \in  \mathbb{R}^{1+d_0},
  \\
  \Theta_\lambda(x, x^{(i)})_0 &\defeq (\Theta_\lambda(x^*, x^{(i)})_{00}, \ldots,\Theta_\lambda(x^*, x^{(i)})_{0d_0})^\intercal.
\end{align*}
The last player is the following real-valued function $\fntk$ on $\R^{d_0}$, which solves the infinite-width version of the ODE in \eqref{eqn:kernel-gradient-descent}  at time $t = \infty$: for $x \in \R^{d_0}$, 
\begin{align}
  \fntk(x) =
  {\Theta_{\lambda}(x, x^{(1:N)})_0^\intercal \Theta_\lambda(x^{(1:N)}, x^{(1:N)})^{-1} \mathbf{y}^{(1:N)}. \label{eqn:JNTKSolution}}
\end{align}

Using what we have defined, we present our result on the solution of the ODE in \eqref{eqn:kernel-gradient-descent} at time $t = \infty$.
\begin{theorem}[MLPs Learnt by Robust Training] \label{thm:trainingdynamics}
  Assume that Assumptions~\ref{assm:activation-assumption}, \ref{assm:dataset-assumption} and \ref{assm:NTK-smallest-eigenvalue} hold.
  Then, there exist functions $F$ and $G$ such that 
  for all $\epsilon, \delta > 0$ and $x^{*} \in \R^{d_0}$ with $\|x^*\| = 1$, and also for all $d$ and $\kappa$, if
  $d \ge F(x^{(1:N)}, x^*, L, \delta, \epsilon)$ and 
  $\kappa \le G(x^{(1:N)}, x^*, L, \delta, \epsilon)$ holds,
  the following statement holds with probability at least $1-\delta$:
  if $f_{d,\theta_\infty}(x^*)$ is the solution of the ODE in \eqref{eqn:kernel-gradient-descent} at time $t=\infty$, we have $|f_{d,\theta_\infty}(x^*) - \fntk(x^*)| \le \epsilon$ and the training loss converges to $0$ exponentially with respect to $t$.
\end{theorem}
Again, we note that a similar theorem holds for gradient descent; see Appendix L of the supplementary material for details.
This theorem shows that in the infinite-width limit, the fully-trained MLP converges to $\fntk$, which has the form of the standard kernel regressor.

For the implicit functions $F$ and $G$ appearing in the statement of the theorems, they are employed to quantify limiting arguments. 
The function $F$ makes the width $d$ of the network large enough to give infinite-width limit behaviour. 
We do not have an explicit formula for $F$ due to the dependency of our proof on the Tensor-Program framework, which only states convergence without rate.
\update{But if we ignore this dependency and estimate $F$, we have $F(x^{(1:N)},x^*,L,\delta,\eta) = O(N^2 (\log N)^{12L})$, and the resulting lower bound on $d$ coincides with the standard result on the required network width from the standard NTK theory, up to the logarithm factor. 
The function $G$ makes the scaling $\kappa$ of a network small enough that we can control the initial randomness of the network.}{Added additional sentence regarding width requirement.}
We do have the explicit representation of $G$, which can be found in the statement of 
Theorem~\ref{thm:JNTKCloseToInit2} in the supplementary material, 

\section{Experiments}
\label{sec:experiments}
We experimentally checked our results using finite MLPs with varying widths and depths and simple datasets, analysed the robustness of the kernel regression solution obtained, and empirically investigated the validity of Assumption~\ref{assm:NTK-smallest-eigenvalue}. 
In this section, we report our findings from those experiments. We repeated the experiments 10 times and plotted their results with 95\% \update{bootstrap confidence intervals}{Add that we used confidence interval with Bootstrapping.}.

\subsection{Validation of Main Results}
We tested our three results: (i) an MLP and its Jacobian jointly converge to a GP (Theorem~\ref{thm:JacobianNNGPInitialisation});
(ii) the finite JNTK of an MLP converges to a deterministic kernel at initialisation (Theorem~\ref{thm:JacobianNTKInitialisation}); (iii) 
the finite JNTK of an MLP stays constant throughout robust training with high probability, which grows to $1$ as the MLP gets wider (Theorem~\ref{thm:JacobianNTKTraining}). 

For the first two convergence results (Theorems~\ref{thm:JacobianNNGPInitialisation} and~\ref{thm:JacobianNTKInitialisation}), we used the simple synthetic dataset in $\R^4$ of size 256, which approximates the optimal $\epsilon$-net of the $S^3$ where $y^{(i)}$ is chosen arbitrarily since it does not matter for these results. 
In these tests, we tried MLPs with all the width-depth combinations in $\{2^i\}_{i=6}^{13} \times \{1, 2, 3\}$. 
This dataset is constructed to cover the possible angles between two inputs, and the standard basis.
For the third result regarding the evolution of a JNTK during robust training (Theorem~\ref{thm:JacobianNTKTraining}), we used
the Algerian forest fire dataset \cite{Abid2020} from the UCI Machine Learning repository, which contains $224$ data points with input dimension $11$.
We fixed the depth of the network as 11 which was the smallest depth that satisfies Assumption~\ref{assm:NTK-smallest-eigenvalue} with GeLU activation, for varying width $\{2^i\}_{i=6}^{13}$.
In all of these tests, we set $\kappa$ to $0.1$, the value that matters only for the third test, and made the training of the test achieve reasonable performance. 
\update{In the experiments, we used the normalised versions of the activation functions (namely, GeLU and erf) via
the multiplication of appropriate constants, in order to satisfy Assumption~\ref{assm:activation-assumption}.}{Add that we normalised the activations.}
More details on the experiments can be found in the supplementary material.

\begin{figure}
  \centering
  \includegraphics[width=1.0\linewidth]{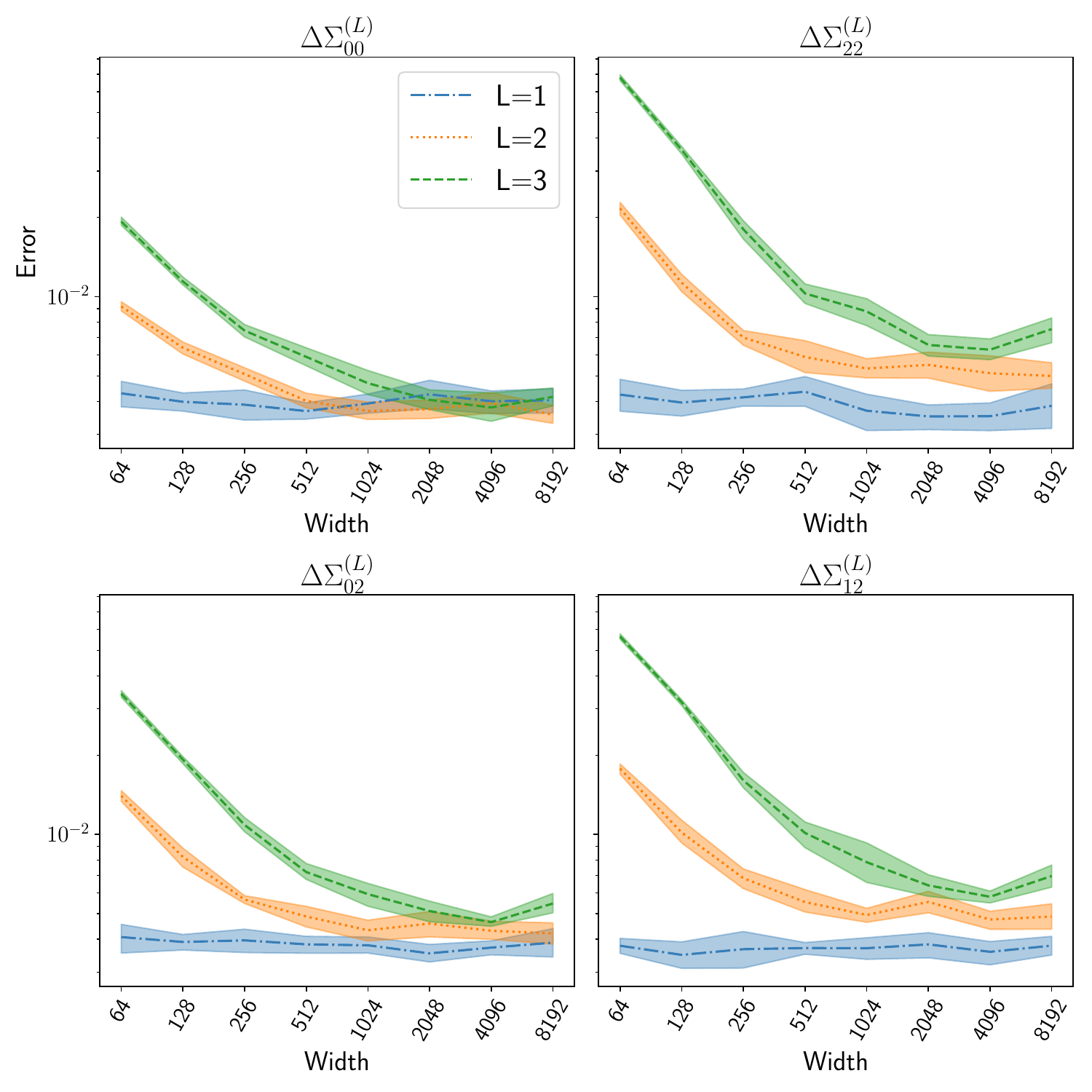}
  
  \caption{$\max$-norm distance between the Jacobian NNGP kernel and the estimate of its finite counterpart. The $x$-axis represents the width of an MLP, and the $y$-axis the $\max$-norm distance.} 
  \label{fig:NNGPConvergence}
\end{figure}

For the validation of Theorem~\ref{thm:JacobianNNGPInitialisation}, we focused on the convergence to $\Sigma^{(L)}$, the Jacobian NNGP kernel scaled down by $\kappa^2$. 
This does not directly validate the theorem since it does not verify the Gaussianity of the output. 
But this is an effective proxy and the convergence to $\Sigma^{(L)}$ is critical in the proof of Theorem~\ref{thm:JacobianNTKInitialisation}.
We estimated the finite-width counterpart of the Jacobian NNGP kernel of an MLP $f_{d,\theta}$ written simply as $f$ below: for all $i,j \in [N]$,
\[ 
  \widehat{\Sigma}^{(L)}(x^{(i)}, x^{(j)}) = \widehat{\mathrm{Cov}} \left( \begin{matrix}
    f(x^{(i)})
    \\
    J(f)(x^{(i)})
  \end{matrix}, \begin{matrix}
    f(x^{(j)})
    \\
    J(f)(x^{(j)})
  \end{matrix} \right)
\]
where $\widehat{\mathrm{Cov}}$ means a Monte-Carlo estimate of covariance with one million samples.
Then, we measured the $\max$-norm distance between the Gram matrix of the $(\alpha,\beta)$-th component of the Jacobian NNGP kernel and that of its finite estimated counterpart with $1/\kappa^2$ scaling, for all $\alpha,\beta \in \{0,1,2,3,4\}$. 
We denote the measured quantity by $\Delta \Sigma_{\alpha\beta}^{(L)} $, which is formally defined as shown below:
\begin{align*}
  \max_{i,j \in [64]} \left|\frac{1}{\kappa^2}\widehat{\Sigma}^{(L)}(x^{(i)}, x^{(j)})_{\alpha\beta} - \Sigma^{(L)}(x^{(i)}, x^{(j)})_{\alpha\beta} \right|.
\end{align*}

Figure~\ref{fig:NNGPConvergence} visualises $\Delta \Sigma_{\alpha\beta}^{(L)}$ for four choices of $(\alpha,\beta)$ representing four types of covariance: between the outputs $(00)$, between the Jacobians with respect to a fixed input coordinate $(22)$, between the output and the Jacobian $(02)$, and between the Jacobians with respect to different input coordinates $(12)$. 
\update{Note that as the width of the network increases,
$\Delta\Sigma_{\alpha\beta}^{(L)}$ decreases as predicted by our convergence result. 
In the figure, however, $\Delta\Sigma_{\alpha\beta}^{(L)}$ stays slightly above $0$.
We think that this is due to the non-negligible Monte-Carlo error in the computation of $\widehat{\Sigma}^{(L)}$,
which will decrease if we increase the number of samples.}{Add more description why error does not go to zero.} 

\begin{figure}
  \centering
  \includegraphics[width=1.0\linewidth]{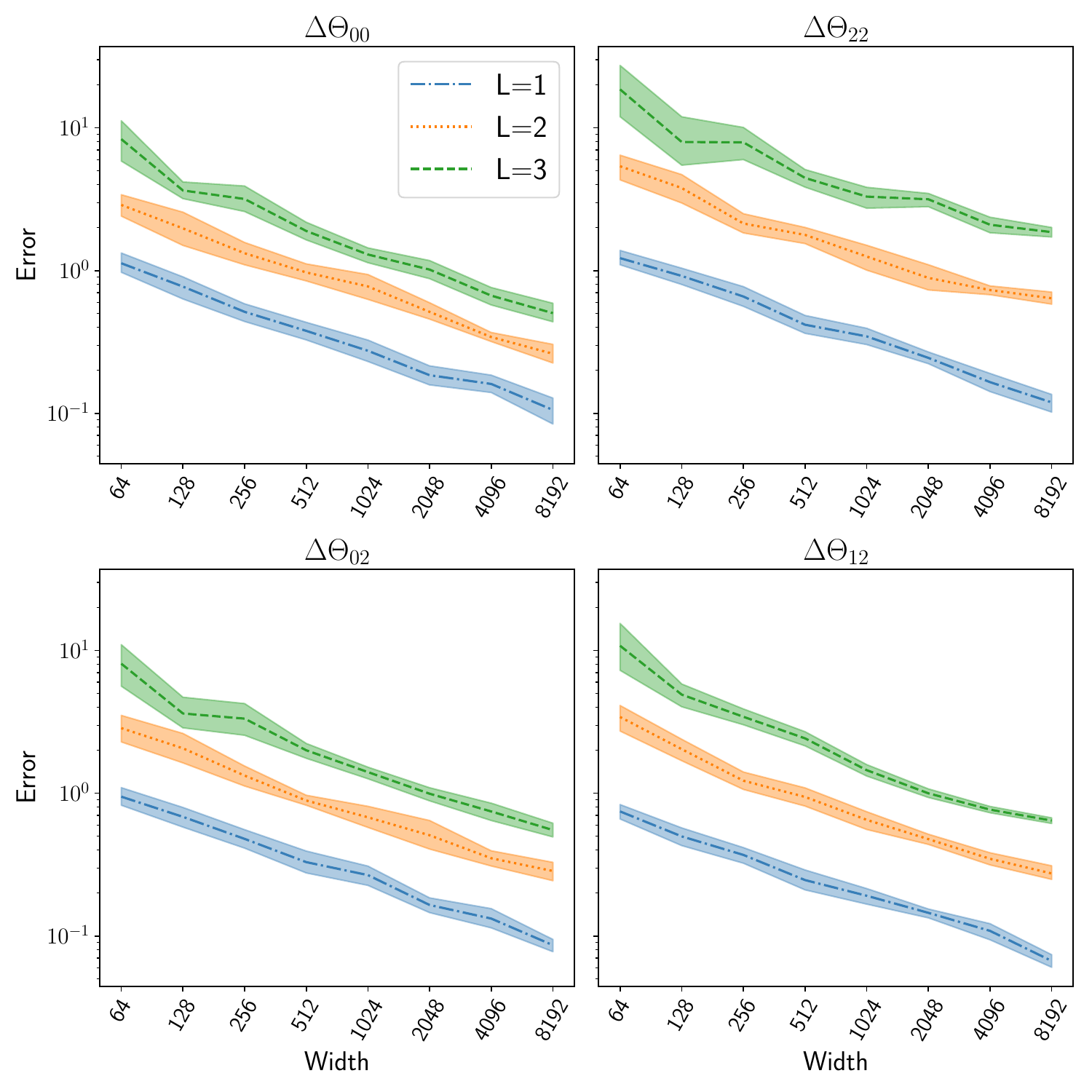}
  
  \caption{$\max$-norm distance between the $(1/\kappa^2)$-scaled finite JNTK $(1/\kappa^2)\Theta_{d,\theta_0}$ and the similarly-scaled limiting JNTK $\Theta$ at initialisation. The $x$-axis represents the width of an MLP, and the $y$-axis the $\max$-norm distance.}
  \label{fig:NTKConvergence}
\end{figure}

The validation of the next Theorem~\ref{thm:JacobianNTKInitialisation} similarly focused on the convergence to $\Theta$, the limiting JNTK scaled down by $\kappa^2$. 
We computed the $\max$-norm distance between the Gram matrix of the $(\alpha,\beta)$-th component of the $(1/\kappa^2)$-scaled finite JNTK and that of the limiting JNTK for $\alpha,\beta \in \{0,1,2,3,4\}$: 
\begin{align*}
  \max_{i,j \in [64]} \left|\frac{1}{\kappa^2}\Theta_{d,\theta_0}(x^{(i)}, x^{(j)})_{\alpha\beta} - \Theta(x^{(i)}, x^{(j)})_{\alpha\beta} \right|.
\end{align*}
Figure~\ref{fig:NTKConvergence} shows the results of these computations.
We can see that as the network width increases, the $(1/\kappa^2)$-scaled finite NTK converges to $\Theta$ at initialisation, which confirms Theorem~\ref{thm:JacobianNTKInitialisation}. 
Note that $\Delta \Theta_{22}$ is much larger than $\Delta \Theta_{00}$'s. 
This is because while both the former and the latter are defined as sums of terms, the former has four times more summands than the latter.

Finally, we checked the claim of Theorem~\ref{thm:JacobianNTKTraining} that the JNTK stays constant during robust training.
We used a binary classification dataset but regarded it as a regression dataset with targets in $\{+1,-1\}$, and trained the network under the regression loss with Jacobian regularisation coefficient $\lambda = 0.01$. 
We mimicked the gradient flow in the theorem with gradient descent with learning rate $1$, and measured the following $\max$-norm distance $\Delta \Theta_{\alpha\beta,t}$ at the training steps $t = 2^0,2^1,\ldots,2^{11}$:
for $\alpha,\beta \in \{0, 1, \ldots, 11\}$,
\[\max_{i,j \in [N]} \left|\frac{1}{\kappa^2}\Theta_{d,\theta_t}(x^{(i)}, x^{(j)})_{\alpha\beta} - \Theta(x^{(i)}, x^{(j)})_{\alpha\beta} \right|.\]

\begin{figure}
  \centering
  \includegraphics[width=1.01\linewidth]{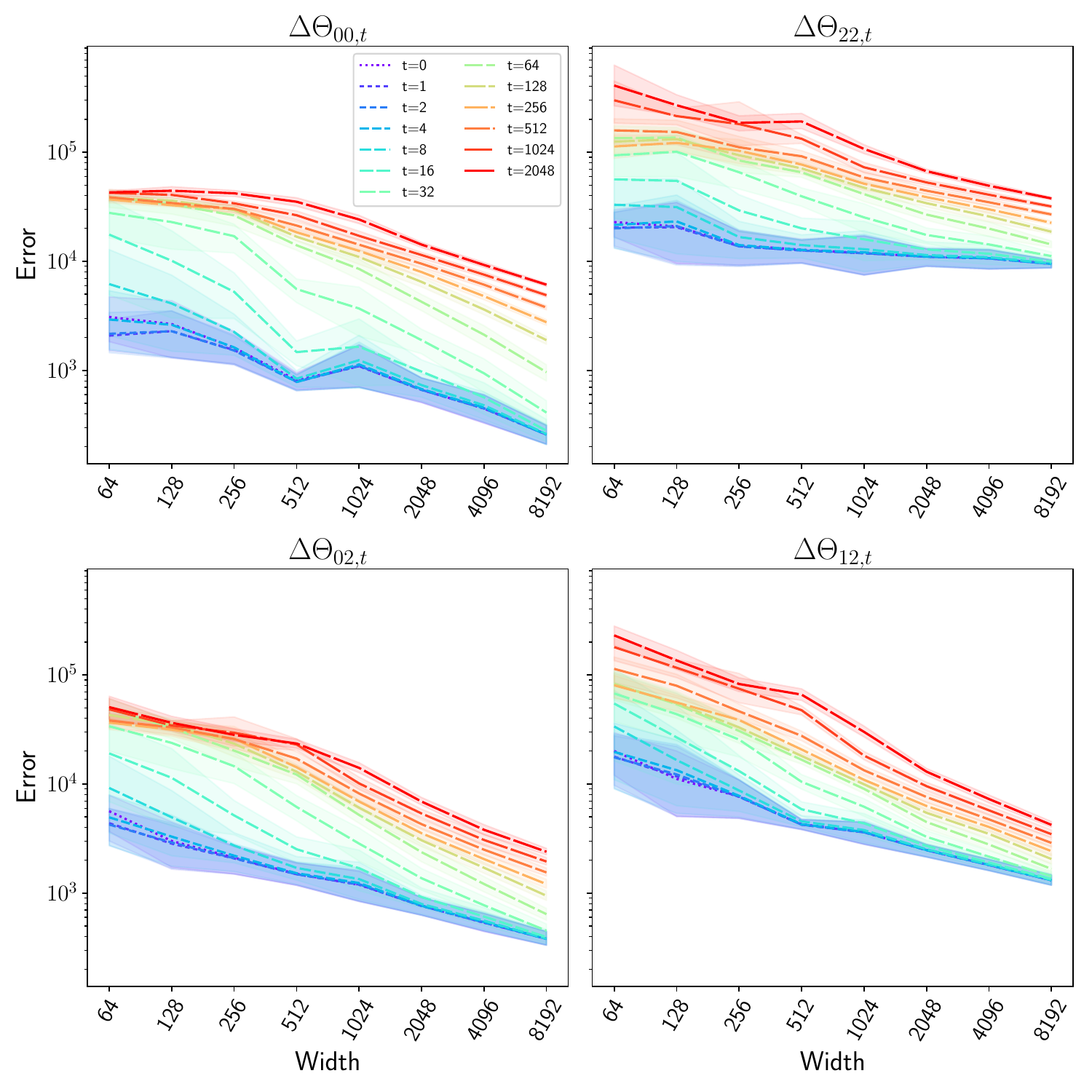}
  
  \caption{$\max$-norm distance between the $(1/\kappa^2)$-scaled finite JNTK $(1/\kappa^2)\Theta_{d,\theta_t}$ at the training step $t$, and the similarly-scaled limiting JNTK $\Theta$ at initialisation. The x-axis represents the width of an MLP, and the y-axis represents $\max$-norm distance.}
  \label{fig:NTKTrainConvergence}
\end{figure}

Figure~\ref{fig:NTKTrainConvergence} shows the measured $\max$-norm distance for the four cases of $(\alpha,\beta) \in \{(2,2),(0,2),(0,0),(1,2)\}$.
\update{Note that as the network gets wider, the distance gets smaller at all training steps $t$ checked, which
is consistent with the prediction of Theorem~\ref{thm:JacobianNTKTraining}.
While the distance at the largest width in the figure is still above 
the smallest eigenvalue of the Gram matrix of the limiting NTK and does not directly validate 
Theorem~\ref{thm:JacobianNTKTraining}, its decreasing tendency along the width suggests
that better empirical justifications of the theorem would be obtained in wider networks (which we 
could not try due to limitations of our computational resources).}{Add more description why error does not go to zero.}
Note that as the training step $t$ increases, the distance also increases, but the increment becomes smaller, which indicates that the JNTK changes less as the training progresses.

\subsection{Analysis of Kernel Regression}

\begin{figure}
  \centering
  \includegraphics[width=1.01\linewidth]{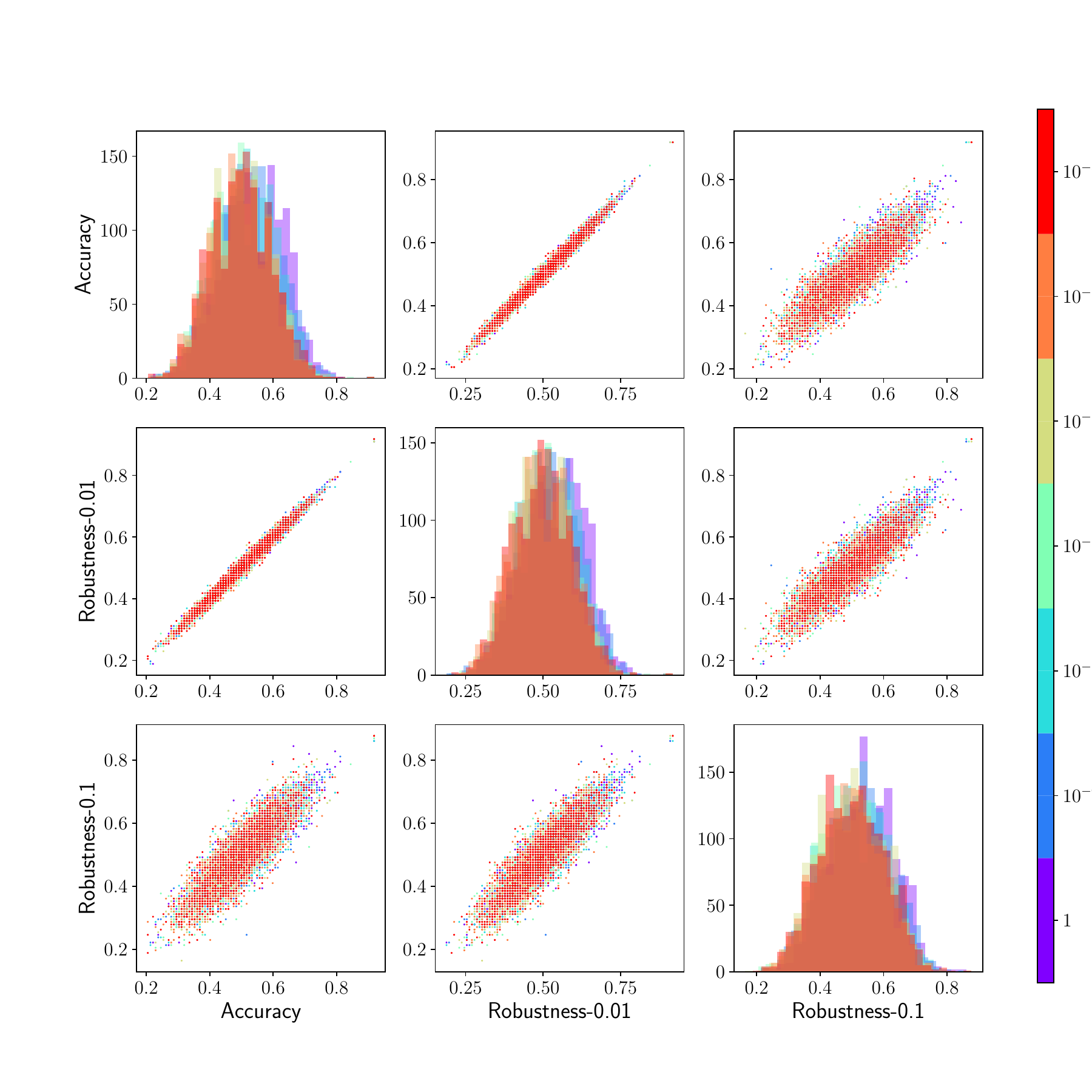}
  
  \caption{Pairplot of test accuracy, test accuracy after 0.01 perturbation, test accuracy after 0.1 perturbation of eigenfeatures of \textit{robust training}. The colours denote the coefficient of Jacobian regularisation.}
  \label{fig:regressionanalysis}
\end{figure}

As stated in Theorem~\ref{thm:trainingdynamics}, in the regime of large width and small learning rate, the neural network under robust training converges to the solution $\fntk$ of a kernel-regression problem, as the training time goes to infinity.  Following \cite{tsilivis2022what}, we analysed the precision and robustness of this kernel-regression solution by looking at the eigenvalues and eigenvectors of the matrix $\Theta_\lambda(x^{(1:N)}, x^{(1:N)})$ in the solution. 
We used an 11-layer neural network with the GeLU activation, where the depth is chosen to be the smallest among those that satisfy the Assumption~\ref{assm:NTK-smallest-eigenvalue} in our experiments. For each $i \in [N(d_0+1)]$, let
$f_i(x) \defeq \Theta_\lambda(x, x^{(1:N)})_{0}^\intercal \lambda_i^{-1} v_i v_i^\intercal \mathbf{y}^{(1:N)}$, where $\lambda_i$ and $v_i$ are the $i$-th eigenvalue (ordered from the largest) and the corresponding eigenvector of the matrix $\Theta_\lambda(x^{(1:N)}, x^{(1:N)})$. Call $f_i$ eigenfeature. Then, by Equation~\ref{eqn:JNTKSolution},  $\fntk(x) = \sum_{i \in [N(d_0+1)]} f_i(x)$. We interpret this as $\fntk$ being a linear combination of eigenfeatures $f_i$, and analyse the precision and robustness of these eigenfeatures.  

\begin{figure}
  \centering
  \includegraphics[width=1.01\linewidth]{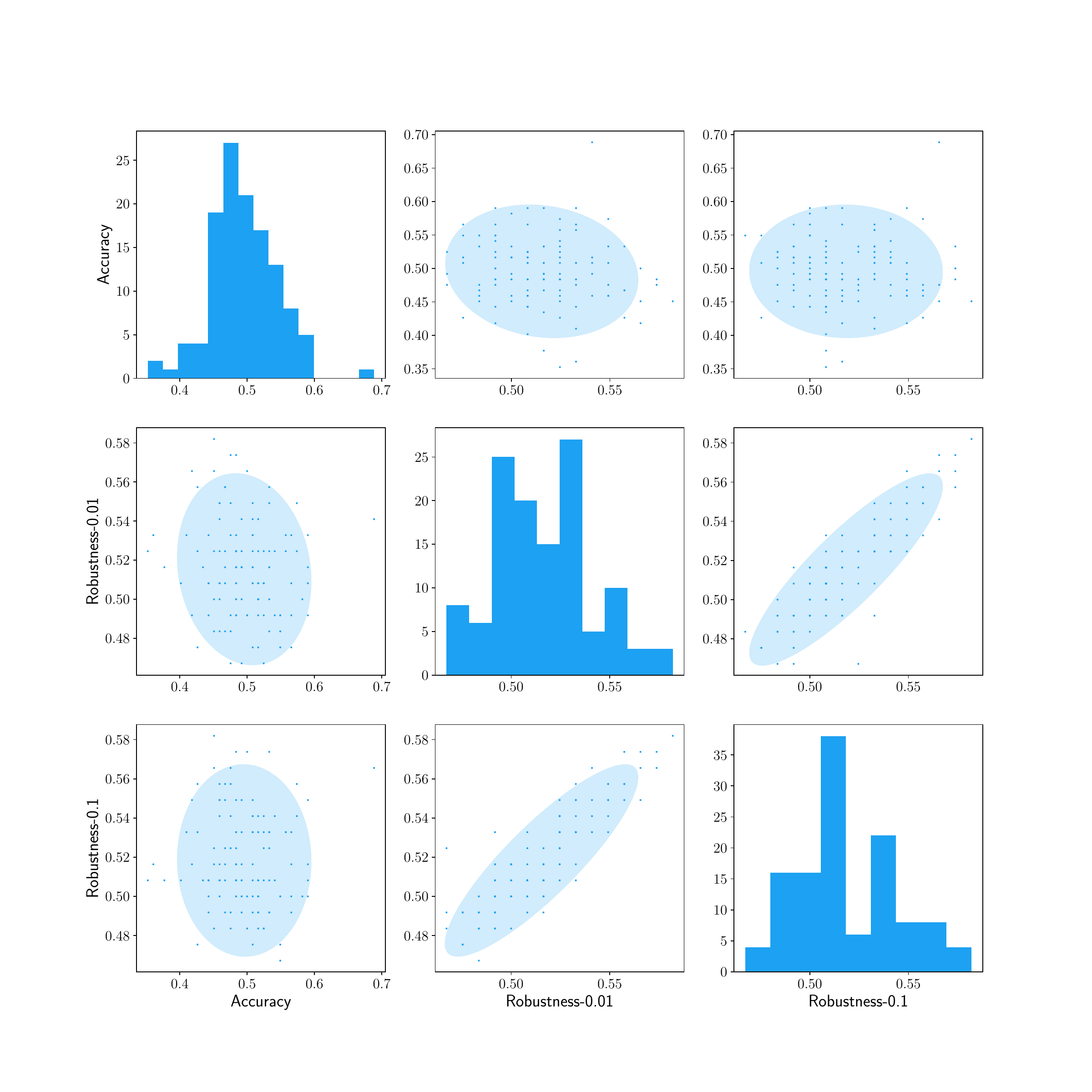}

  \caption{Pairplot of test accuracy, test accuracy after 0.01 perturbation, test accuracy after 0.1 perturbation of eigenfeatures of \textit{standard training}. The ellipses correspond to 4$\sigma$ confidence region of multivariate Gaussian distribution fitting points, to show the correlation clearly. }
  \label{fig:regressionanalysis_noJ}
\end{figure}

\begin{figure*}[t]
  \centering
  \includegraphics[width=0.9\linewidth]{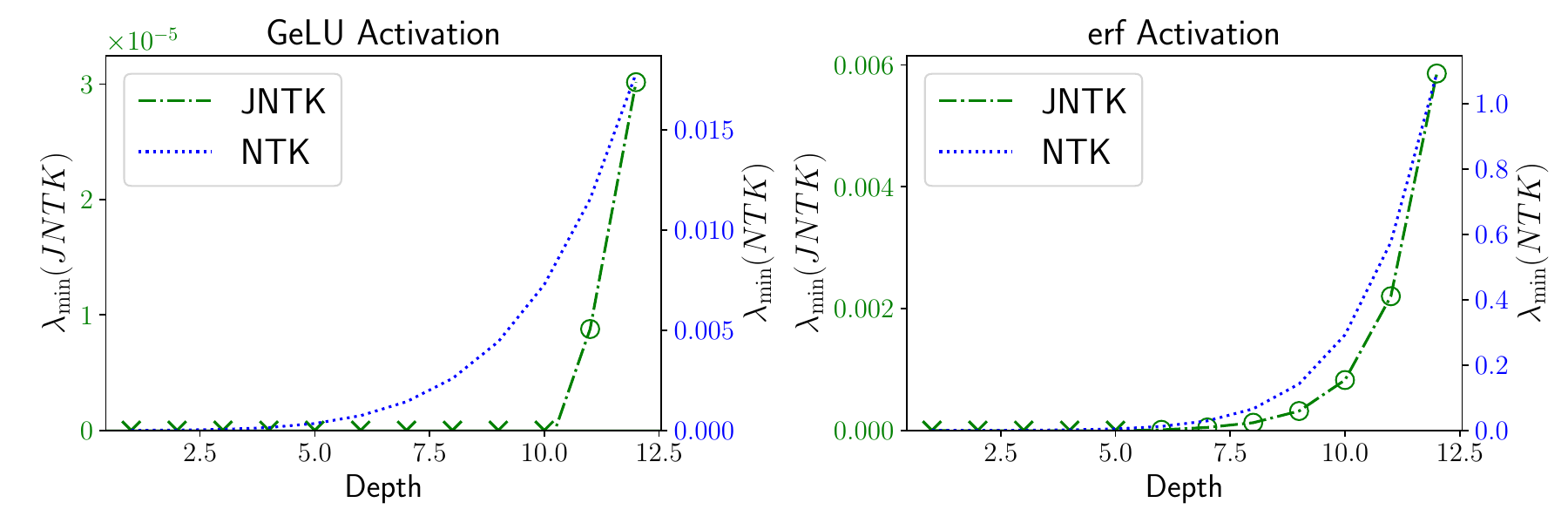}
  \caption{Smallest eigenvalue of network with varying depth with GeLU or erf activation, for both JNTK and standard NTK. {\color{black} The o-marks indicate that the nonzero minimum eigenvalue assumption is satisfied, and the x-marks indicate that the assumption is violated.}}
  \label{fig:mineig}
\end{figure*}

Figure~\ref{fig:regressionanalysis} shows the pairplot of the following three measures on these eigenfeatures $f_i$: test accuracy of $f_i$, test accuracy of $f_i$ after the $0.01$ perturbation of the training inputs, and test accuracy of $f_i$ with respect to the $0.1$ perturbation. 
The perturbation was obtained by changing the inputs via one-step projected gradient descent on the analytic solution without Jacobian regularisation, and $0.01$ and $0.1$ refer to learning rates. 
The 0.1 perturbation was enough to change the labels of some inputs, while for the 0.01 perturbation, there exists a classifier that classifies all the inputs perfectly even after this perturbation.  
In the plot, we denote eigenfeatures by their eigenvalues and use the colours to represent the Jacobian-regularisation coefficient from $\lambda = 1$ (Purple) to $\lambda = 10^{-3}$ (Red).
The result on standard training in Figure~\ref{fig:regressionanalysis_noJ} shows a distinct result, where it shows less or no correlation between accuracy and robustness.

While the empirical analysis of NTK \cite{tsilivis2022what} found, in the setting of standard training without Jacobian regularisation, eigenfeatures that are highly accurate but less robust, or highly robust but less accurate, we could not find such eigenfeatures in our setting. 
We observed that adding Jacobian regularisation promotes alignment between accuracy and robustness so that achieving high test accuracy implies satisfying robustness.
\update{We point out that Figure 2 (Right) of \citet{tsilivis2022what} shows a different result 
from our Figure~\ref{fig:regressionanalysis_noJ}. This is due to the different setups of
these two empirical analyses. Our analysis is based on the analytic solution of the kernel regression problem with
the limiting JNTK, while the analysis of \citet{tsilivis2022what} is based on the empirical NTK at finite width.}{Add comment that the experiment is not same.}

Assumption~\ref{assm:NTK-smallest-eigenvalue} is a standard assumption in the NTK analysis; it is known that the standard NTK $\Theta(x,x')_{00}$ has its smallest eigenvalue bounded away from zero if no inputs are parallel and the activation function is smooth \cite{du19c}. 
However, for our JNTK, it is not known to be true theoretically. 
Our empirical analysis suggests that whether the assumption holds or not depends on activation function, depth, and even the dataset itself. 
   
To check when Assumption~\ref{assm:NTK-smallest-eigenvalue} holds, we empirically analysed the smallest eigenvalue of the JNTK Gram matrix for different depths and activation functions. The results are summarised in Figure~\ref{fig:mineig}.
The blue-dot curves show the smallest eigenvalues of the NTK Gram matrices, which are the matrices built by only using $\Theta(x, x')_{00}$ terms. Note that the plot is always positive. 
The green curves show the smallest eigenvalues of the JNTK Gram matrices. 
Note that the smallest eigenvalues in the JNTK case are much smaller than those in the NTK case. 
Also, to be greater than $0$, it requires deeper networks, for example, $11$ in the GeLU-activation case and $6$ in the erf-activation case. 
We conjecture that this requirement on depth comes from the expanded size of our Gram matrix (that is, $(N(d_0+1))^2$ instead of $N^2$), and also from Jacobian regularisation's forcing the network to be flat around training examples and thereby reducing its flexibility. 
This conjecture is consistent with the finding of \cite{Poole16} that the complexity of the network increases exponentially as the depth increases.

\section{Related Works and Conclusion}
\label{sec:conclusion}
Jacobian regularisation was introduced in several forms including \citet{hoffman2019} and \citet{jakubovitz2018}, with different functions being differentiated. 
\update{While being simple, it has been a strong baseline for defence against adversarial attacks \cite{jakubovitz2018,liu2024}, with multiple extensions.
Also, on the theory side, the smoothness of the learnt neural network with respect to the input was shown to guarantee a small generalisation error, which partially explains the success of Jacobian regularisation~\cite{chao2021}.}{Add comments during the rebuttal that shows why Jacobian regularisation is useful.}
The reason for the success of Jacobian regularisation has not been well understood theoretically. Only recently the implicit bias of the finite difference version of Jacobian regularisation was studied for linear models \cite{karakida23a}, which favours an L1 regularisation solution.

In this paper, we provided a theoretical analysis of robust training with Jacobian regularisation from the perspective of infinite-width limits. We characterised the limiting NNGP and NTK kernels of feed-forward neural networks trained under Jacobian regularisation and showed that with high probability, such networks converge to solutions of certain kernel regression problems, as the widths of those networks are large enough. We used these theoretical results to analyse the precision and robustness of the networks trained with Jacobian regularisation, and also experimentally demonstrated the validity of these results.

\section*{Impact Statement}

This paper presents work whose goal is to advance the field of Machine Learning, especially the neural networks related to their robustness. There are many potential societal consequences of our work, none of which we feel must be specifically highlighted here.
\todo[author=TK,inline]{This year ICML requires us to write this section, but also allows us to write a simple statement as above if we do not have anything to say. I modified their sample statement slightly.}

\section*{Acknowledgements}

\update{We would like to thank Hoil Lee, Sungjin Ahn, and Juho Lee for helpful discussions, and anonymous reviewers for their useful comments, which led to the big improvement of the paper. This work was supported by the National Research Foundation of Korea (NRF) grant funded by the Korea Government (MSIT) (No. RS-2023-00279680).}{Add ack.}

\bibliography{ntk,neurips_2022}

\begin{thebibliography}{}

\bibitem[Abid and Izeboudjen, 2020]{Abid2020}
Abid, F. and Izeboudjen, N. (2020).
\newblock Predicting forest fire in algeria using data mining techniques: Case
  study of the decision tree algorithm.
\newblock In Ezziyyani, M., editor, {\em Advanced Intelligent Systems for
  Sustainable Development (AI2SD'2019)}, pages 363--370, Cham. Springer
  International Publishing.

\bibitem[Arora et~al., 2019]{Arora2019OnEC}
Arora, S., Du, S.~S., Hu, W., Li, Z., Salakhutdinov, R.~R., and Wang, R.
  (2019).
\newblock On exact computation with an infinitely wide neural net.
\newblock In Wallach, H., Larochelle, H., Beygelzimer, A., d\textquotesingle
  Alch\'{e}-Buc, F., Fox, E., and Garnett, R., editors, {\em Advances in Neural
  Information Processing Systems}, volume~32. Curran Associates, Inc.

\bibitem[Chen et~al., 2021]{Yilan2021}
Chen, Y., Huang, W., Nguyen, L., and Weng, T.-W. (2021).
\newblock On the equivalence between neural network and support vector machine.
\newblock In Ranzato, M., Beygelzimer, A., Dauphin, Y., Liang, P., and Vaughan,
  J.~W., editors, {\em Advances in Neural Information Processing Systems},
  volume~34, pages 23478--23490. Curran Associates, Inc.

\bibitem[Du et~al., 2019a]{du19c}
Du, S., Lee, J., Li, H., Wang, L., and Zhai, X. (2019a).
\newblock Gradient descent finds global minima of deep neural networks.
\newblock In Chaudhuri, K. and Salakhutdinov, R., editors, {\em Proceedings of
  the 36th International Conference on Machine Learning}, volume~97 of {\em
  Proceedings of Machine Learning Research}, pages 1675--1685. PMLR.

\bibitem[Du et~al., 2019b]{du2019}
Du, S., Lee, J., Li, H., Wang, L., and Zhai, X. (2019b).
\newblock Gradient descent finds global minima of deep neural networks.
\newblock In Chaudhuri, K. and Salakhutdinov, R., editors, {\em Proceedings of
  the 36th International Conference on Machine Learning}, volume~97 of {\em
  Proceedings of Machine Learning Research}, pages 1675--1685. PMLR.

\bibitem[Goodfellow et~al., 2015]{Goodfellow2014}
Goodfellow, I.~J., Shlens, J., and Szegedy, C. (2015).
\newblock Explaining and harnessing adversarial examples.
\newblock In Bengio, Y. and LeCun, Y., editors, {\em 3rd International
  Conference on Learning Representations, {ICLR} 2015, San Diego, CA, USA, May
  7-9, 2015, Conference Track Proceedings}.

\bibitem[Hoffman et~al., 2019]{hoffman2019}
Hoffman, J., Roberts, D.~A., and Yaida, S. (2019).
\newblock Robust learning with jacobian regularization.

\bibitem[Hu et~al., 2020]{Hu2020Simple}
Hu, W., Li, Z., and Yu, D. (2020).
\newblock Simple and effective regularization methods for training on noisily
  labeled data with generalization guarantee.
\newblock In {\em 8th International Conference on Learning Representations,
  {ICLR} 2020, Addis Ababa, Ethiopia, April 26-30, 2020}. OpenReview.net.

\bibitem[Jacot et~al., 2018]{Jacot2018}
Jacot, A., Gabriel, F., and Hongler, C. (2018).
\newblock Neural tangent kernel: Convergence and generalization in neural
  networks.
\newblock In Bengio, S., Wallach, H., Larochelle, H., Grauman, K.,
  Cesa-Bianchi, N., and Garnett, R., editors, {\em Advances in Neural
  Information Processing Systems}, volume~31. Curran Associates, Inc.

\bibitem[Jakubovitz and Giryes, 2018]{jakubovitz2018}
Jakubovitz, D. and Giryes, R. (2018).
\newblock Improving dnn robustness to adversarial attacks using jacobian
  regularization.
\newblock In Ferrari, V., Hebert, M., Sminchisescu, C., and Weiss, Y., editors,
  {\em Computer Vision -- ECCV 2018}, pages 525--541, Cham. Springer
  International Publishing.

\bibitem[Karakida et~al., 2023]{karakida23a}
Karakida, R., Takase, T., Hayase, T., and Osawa, K. (2023).
\newblock Understanding gradient regularization in deep learning: Efficient
  finite-difference computation and implicit bias.
\newblock In Krause, A., Brunskill, E., Cho, K., Engelhardt, B., Sabato, S.,
  and Scarlett, J., editors, {\em Proceedings of the 40th International
  Conference on Machine Learning}, volume 202 of {\em Proceedings of Machine
  Learning Research}, pages 15809--15827. PMLR.

\bibitem[Lee et~al., 2022]{LeeYYL22}
Lee, H., Yun, E., Yang, H., and Lee, J. (2022).
\newblock Scale mixtures of neural network gaussian processes.
\newblock In {\em The Tenth International Conference on Learning
  Representations, {ICLR} 2022, Virtual Event, April 25-29, 2022}.
  OpenReview.net.

\bibitem[Lee et~al., 2018]{lee2018deep}
Lee, J., Bahri, Y., Novak, R., Schoenholz, S.~S., Pennington, J., and
  Sohl{-}Dickstein, J. (2018).
\newblock Deep neural networks as gaussian processes.
\newblock In {\em 6th International Conference on Learning Representations,
  {ICLR} 2018, Vancouver, BC, Canada, April 30 - May 3, 2018, Conference Track
  Proceedings}. OpenReview.net.

\bibitem[Lee et~al., 2019]{Lee19NTKLinear}
Lee, J., Xiao, L., Schoenholz, S., Bahri, Y., Novak, R., Sohl-Dickstein, J.,
  and Pennington, J. (2019).
\newblock Wide neural networks of any depth evolve as linear models under
  gradient descent.
\newblock In Wallach, H., Larochelle, H., Beygelzimer, A., d\textquotesingle
  Alch\'{e}-Buc, F., Fox, E., and Garnett, R., editors, {\em Advances in Neural
  Information Processing Systems}, volume~32. Curran Associates, Inc.

\bibitem[Liu et~al., 2024]{liu2024}
Liu, D., Wu, L.~Y., Li, B., Boussaid, F., Bennamoun, M., Xie, X., and Liang, C.
  (2024).
\newblock Jacobian norm with selective input gradient regularization for
  interpretable adversarial defense.
\newblock {\em Pattern Recogn.}, 145(C).

\bibitem[Liu and Zenke, 2020]{Liu20pruningNTK}
Liu, T. and Zenke, F. (2020).
\newblock Finding trainable sparse networks through neural tangent transfer.
\newblock In {\em Proceedings of the 37th International Conference on Machine
  Learning (ICML'20)}, pages 6336--6347.

\bibitem[Lohweg, 2013]{banknote}
Lohweg, V. (2013).
\newblock {{B}anknote authentication}.
\newblock UCI Machine Learning Repository.
\newblock {DOI}: https://doi.org/10.24432/C55P57.

\bibitem[Ma and Ying, 2021]{chao2021}
Ma, C. and Ying, L. (2021).
\newblock On linear stability of sgd and input-smoothness of neural networks.
\newblock In Ranzato, M., Beygelzimer, A., Dauphin, Y., Liang, P., and Vaughan,
  J.~W., editors, {\em Advances in Neural Information Processing Systems},
  volume~34, pages 16805--16817. Curran Associates, Inc.

\bibitem[Matthews et~al., 2018]{Matthews2018}
Matthews, A.~G., Hron, J., Rowland, M., Turner, R.~E., and Ghahramani, Z.
  (2018).
\newblock Gaussian process behaviour in wide deep neural networks.
\newblock In {\em 6th International Conference on Learning Representations,
  {ICLR} 2018, Vancouver, BC, Canada, April 30 - May 3, 2018, Conference Track
  Proceedings}. OpenReview.net.

\bibitem[Neal, 1996]{Neal1996}
Neal, R.~M. (1996).
\newblock Priors for infinite networks.
\newblock In {\em Bayesian Learning for Neural Networks}, pages 29--53.
  Springer New York.

\bibitem[Nguyen et~al., 2021]{nguyen2021}
Nguyen, Q., Mondelli, M., and Montufar, G.~F. (2021).
\newblock Tight bounds on the smallest eigenvalue of the neural tangent kernel
  for deep relu networks.
\newblock In Meila, M. and Zhang, T., editors, {\em Proceedings of the 38th
  International Conference on Machine Learning}, volume 139 of {\em Proceedings
  of Machine Learning Research}, pages 8119--8129. PMLR.

\bibitem[Peck et~al., 2017]{Peck2017}
Peck, J., Roels, J., Goossens, B., and Saeys, Y. (2017).
\newblock Lower bounds on the robustness to adversarial perturbations.
\newblock In Guyon, I., Luxburg, U.~V., Bengio, S., Wallach, H., Fergus, R.,
  Vishwanathan, S., and Garnett, R., editors, {\em Advances in Neural
  Information Processing Systems}, volume~30. Curran Associates, Inc.

\bibitem[Peluchetti et~al., 2020]{Favaro2020}
Peluchetti, S., Favaro, S., and Fortini, S. (2020).
\newblock Stable behaviour of infinitely wide deep neural networks.
\newblock In Chiappa, S. and Calandra, R., editors, {\em Proceedings of the
  Twenty Third International Conference on Artificial Intelligence and
  Statistics}, volume 108 of {\em Proceedings of Machine Learning Research},
  pages 1137--1146. PMLR.

\bibitem[Poole et~al., 2016]{Poole16}
Poole, B., Lahiri, S., Raghu, M., Sohl-Dickstein, J., and Ganguli, S. (2016).
\newblock Exponential expressivity in deep neural networks through transient
  chaos.
\newblock In Lee, D., Sugiyama, M., Luxburg, U., Guyon, I., and Garnett, R.,
  editors, {\em Advances in Neural Information Processing Systems}, volume~29.
  Curran Associates, Inc.

\bibitem[Rasmussen and Williams, 2005]{GPbook}
Rasmussen, C.~E. and Williams, C. K.~I. (2005).
\newblock {\em Gaussian Processes for Machine Learning (Adaptive Computation
  and Machine Learning)}.
\newblock The MIT Press.

\bibitem[Sejnowski and Gorman, 2017]{connectionistbench}
Sejnowski, T. and Gorman, R. (2017).
\newblock {Connectionist Bench (Sonar, Mines vs. Rocks)}.
\newblock UCI Machine Learning Repository.
\newblock {DOI}: https://doi.org/10.24432/C5T01Q.

\bibitem[Tsilivis and Kempe, 2022]{tsilivis2022what}
Tsilivis, N. and Kempe, J. (2022).
\newblock What can the neural tangent kernel tell us about adversarial
  robustness?
\newblock In Koyejo, S., Mohamed, S., Agarwal, A., Belgrave, D., Cho, K., and
  Oh, A., editors, {\em Advances in Neural Information Processing Systems},
  volume~35, pages 18116--18130. Curran Associates, Inc.

\bibitem[Vershynin, 2018]{BookHDP}
Vershynin, R. (2018).
\newblock {\em High-Dimensional Probability: An Introduction with Applications
  in Data Science}.
\newblock Cambridge Series in Statistical and Probabilistic Mathematics.
  Cambridge University Press.

\bibitem[Wang et~al., 2021]{Wang2021}
Wang, J., Chen, J., Sun, Y., Ma, X., Wang, D., Sun, J., and Cheng, P. (2021).
\newblock {\em RobOT: Robustness-Oriented Testing for Deep Learning Systems},
  page 300–311.
\newblock IEEE Press.

\bibitem[Yang, 2019]{Yang2019TensorPI}
Yang, G. (2019).
\newblock Wide feedforward or recurrent neural networks of any architecture are
  gaussian processes.
\newblock In Wallach, H., Larochelle, H., Beygelzimer, A., d\textquotesingle
  Alch\'{e}-Buc, F., Fox, E., and Garnett, R., editors, {\em Advances in Neural
  Information Processing Systems}, volume~32. Curran Associates, Inc.

\bibitem[Yang, 2020]{Yang2020a}
Yang, G. (2020).
\newblock Tensor programs ii: Neural tangent kernel for any architecture.

\bibitem[Yang et~al., 2021]{Yang22Transfer}
Yang, G., Hu, E., Babuschkin, I., Sidor, S., Liu, X., Farhi, D., Ryder, N.,
  Pachocki, J., Chen, W., and Gao, J. (2021).
\newblock Tuning large neural networks via zero-shot hyperparameter transfer.
\newblock In Ranzato, M., Beygelzimer, A., Dauphin, Y., Liang, P., and Vaughan,
  J.~W., editors, {\em Advances in Neural Information Processing Systems},
  volume~34, pages 17084--17097. Curran Associates, Inc.

\bibitem[Yang and Hu, 2021]{Yang21FeatureLearning}
Yang, G. and Hu, E.~J. (2021).
\newblock Tensor programs iv: Feature learning in infinite-width neural
  networks.
\newblock In Meila, M. and Zhang, T., editors, {\em Proceedings of the 38th
  International Conference on Machine Learning}, volume 139 of {\em Proceedings
  of Machine Learning Research}, pages 11727--11737. PMLR.

\bibitem[Yang and Littwin, 2021]{Yang21NTK}
Yang, G. and Littwin, E. (2021).
\newblock Tensor programs iib: Architectural universality of neural tangent
  kernel training dynamics.
\newblock In Meila, M. and Zhang, T., editors, {\em Proceedings of the 38th
  International Conference on Machine Learning}, volume 139 of {\em Proceedings
  of Machine Learning Research}, pages 11762--11772. PMLR.

\bibitem[Zhang and Wei, 2022]{subweibull}
Zhang, H. and Wei, H. (2022).
\newblock Sharper sub-weibull concentrations.
\newblock {\em Mathematics}, 10(13).

\bibitem[Zhang et~al., 2019]{Zhang2019}
Zhang, H., Zhang, P., and Hsieh, C. (2019).
\newblock Recurjac: An efficient recursive algorithm for bounding jacobian
  matrix of neural networks and its applications.
\newblock In {\em The Thirty-Third {AAAI} Conference on Artificial
  Intelligence, {AAAI} 2019, Honolulu, Hawaii, USA, January 27 - February 1,
  2019}, pages 5757--5764. {AAAI} Press.

\end{thebibliography}

\appendix
\onecolumn
\renewcommand{\thefigure}{\Alph{section}\arabic{figure}}

\allowdisplaybreaks

\section{Experiment Details}
\label{appendix:experiment}
For the GP convergence and convergence of finite JNTK at initialisation, we utilised a synthetic dataset generated using the Fibonacci Lattice algorithm extended to the hypersphere, which approximates the 0.5 $\epsilon$-net of the $S^3$.
Regarding the evolution of the finite JNTK during robust training, we set the values of $\lambda$ to 0.01, $\kappa$ to 0.1, and the learning rate to 1. These specific values were selected to ensure that the test accuracy approached 100\% by the end of 2048 epochs.

All the natural datasets first have their entries scaled to be contained in [-1, 1], then normalised to have a unit norm to match our assumption. 
We include two additional datasets in the appendix, the banknote authentication dataset \cite{banknote} which is 4-dimensional data with 1372 data points, and the connectionist bench dataset \cite{connectionistbench} which is 60-dimensional data with 208 data points.
The banknote authentication dataset \cite{banknote} contains several data points that are nearly parallel, so the standard NTK kernel becomes nearly singular, making the comparison impossible.
To mitigate this, we filtered out some data points that are nearly parallel. 
In specific, for GeLU activation, we removed data points $x^{(i)}$ such that $|\langle x^{(i)}, x^{(j)}\rangle| > 0.99$ for some $j < i$, and 0.995 for erf activation, each resulting 186, 320 data points respectively.

For our experiments, we utilised 3 NVIDIA RTX 6000 GPUs to validate the convergence of the GP over three days. Additionally, we employed 3 NVIDIA RTX A5000 GPUs to verify the constancy of finite JNTK during training for one day. All other experiments were performed on CPUs, each completed in under an hour.

\section{Additional Experiments}
In this section, we describe additional experiments on the minimum eigenvalue assumption and also on the kernel regression solution analysis.
We repeat the same experiments in Figures 4 and 6 but with different datasets \cite{banknote,connectionistbench}. 

\begin{figure}[h]
  \centering
  \includegraphics[width=1.01\linewidth]{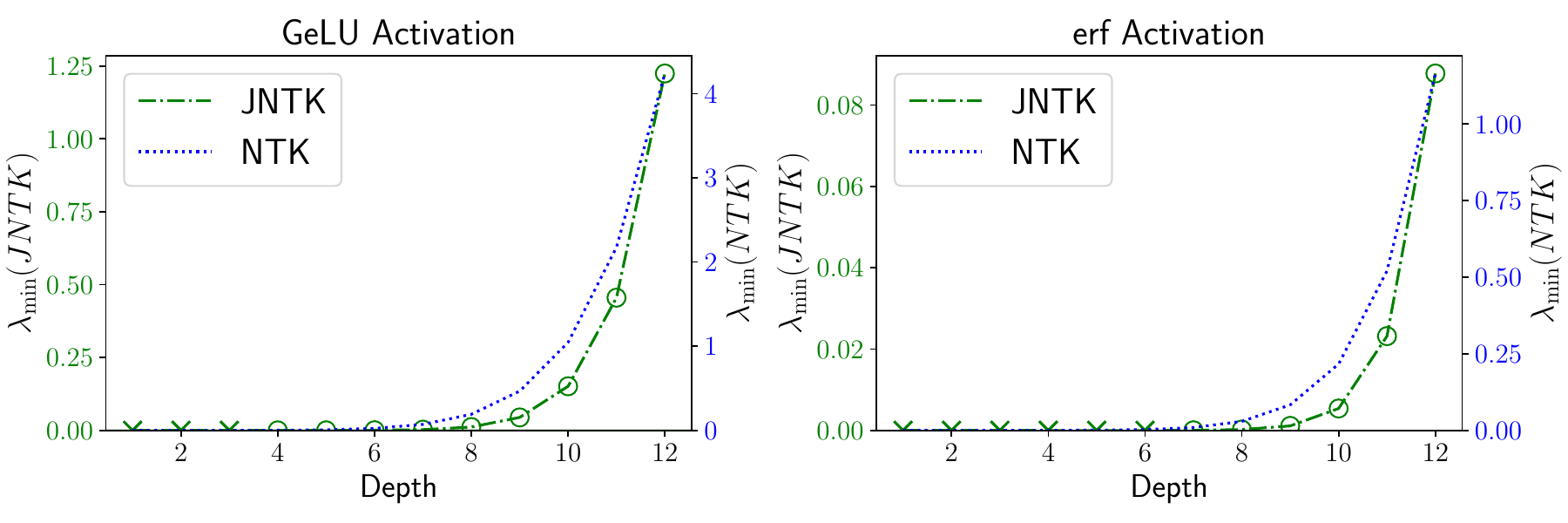}

  \caption{The minimum eigenvalues of the standard NTK kernel and the JNTK kernel under the GeLU and erf activation functions on the banknote authentication dataset. The x-axis is the depth of the network, and the y-axis is the minimum eigenvalue. The o-marks indicate that the nonzero minimum eigenvalue assumption is satisfied, and the x-marks indicate that the assumption is violated. We do not mark the points for the standard NTK kernel because the assumption is always satisfied for the standard NTK kernel.}
  \label{fig:bank_min_eig}
\end{figure}

\begin{figure}[h]
  \centering
  \includegraphics[width=1.01\linewidth]{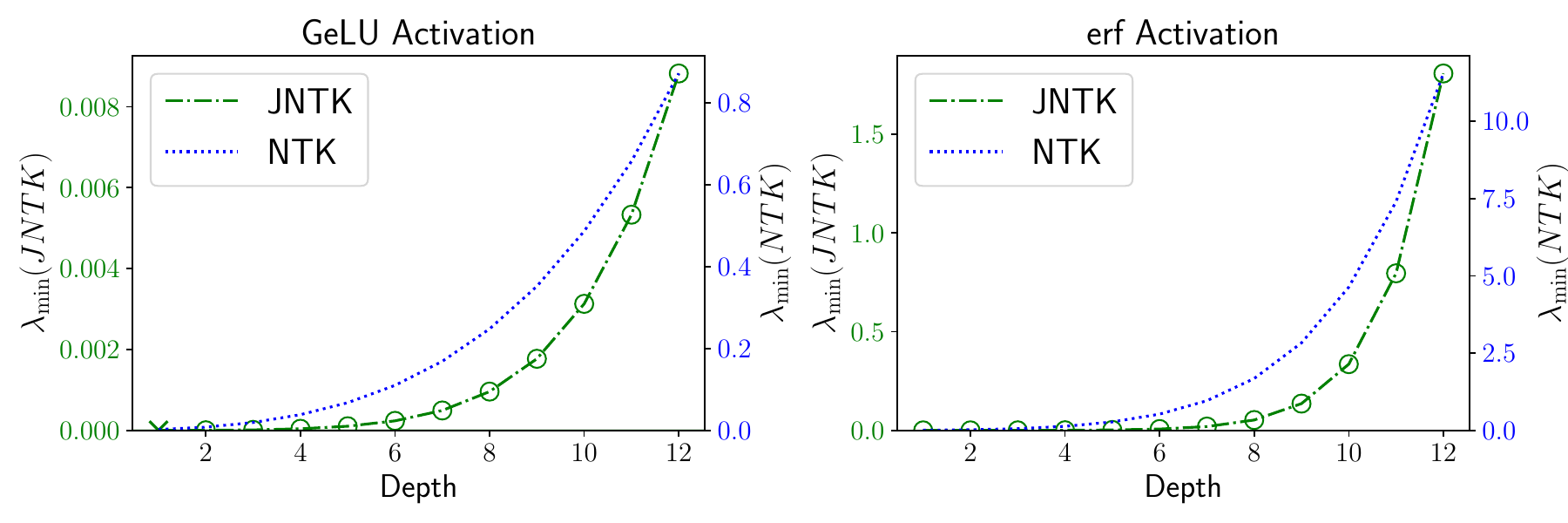} 
  \caption{The minimum eigenvalues of the standard NTK kernel and the JNTK kernel under the GeLU and erf activation functions on the connectionist bench dataset. The x-axis is the depth of the network, and the y-axis is the minimum eigenvalue. The o-marks indicate that the nonzero minimum eigenvalue assumption is satisfied, and the x-marks indicate that the assumption is violated. We do not mark the points for the standard NTK kernel because the assumption is always satisfied for the standard NTK kernel.}
  \label{fig:sonar_min_eig}
\end{figure}

Figures~\ref{fig:bank_min_eig} and~\ref{fig:sonar_min_eig} show the minimum eigenvalues of the standard NTK kernel and the JNTK kernel under the GeLU and erf activation functions on the banknote authentication dataset and the connectionist bench dataset, respectively. 
Both plots show a similar tendency as Figure 6 in the main paper. Note that in the connectionist bench dataset, the assumption is satisfied for lower-depth networks, being invalidated only for shallow GeLU networks.
This can be explained by the fact that the connectionist bench dataset has higher dimensions than the banknote authentication and forest fire datasets. 

We give the kernel regression analysis results for varying datasets and activations in Figures~\ref{fig:regression-analysis-fire-erf}, \ref{fig:regression-analysis-bank-gelu}, \ref{fig:regression-analysis-bank-erf}, \ref{fig:regression-analysis-sonar-gelu}, and \ref{fig:regression-analysis-sonar-erf}.
We use the smallest depth network that satisfies the minimum eigenvalue assumption for each dataset and activation, which can be found in Figures~\ref{fig:bank_min_eig} and~\ref{fig:sonar_min_eig}, and also in Figure 6 in the main text.

\begin{figure}[H]
  \centering
  \begin{subfigure}[b]{0.49\linewidth}
    \includegraphics[width=\linewidth]{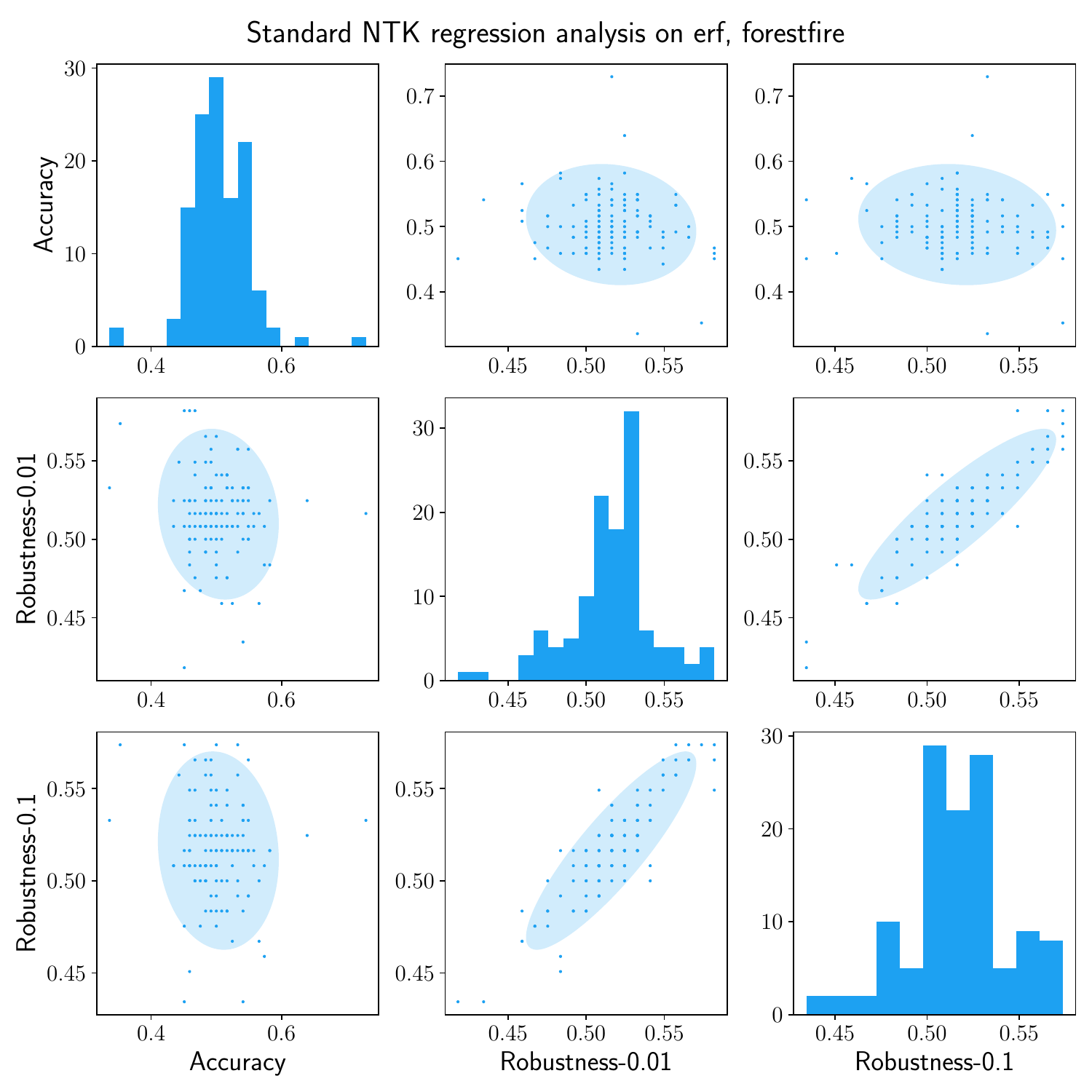}
  \end{subfigure}
  \begin{subfigure}[b]{0.49\linewidth}
    \includegraphics[width=\linewidth]{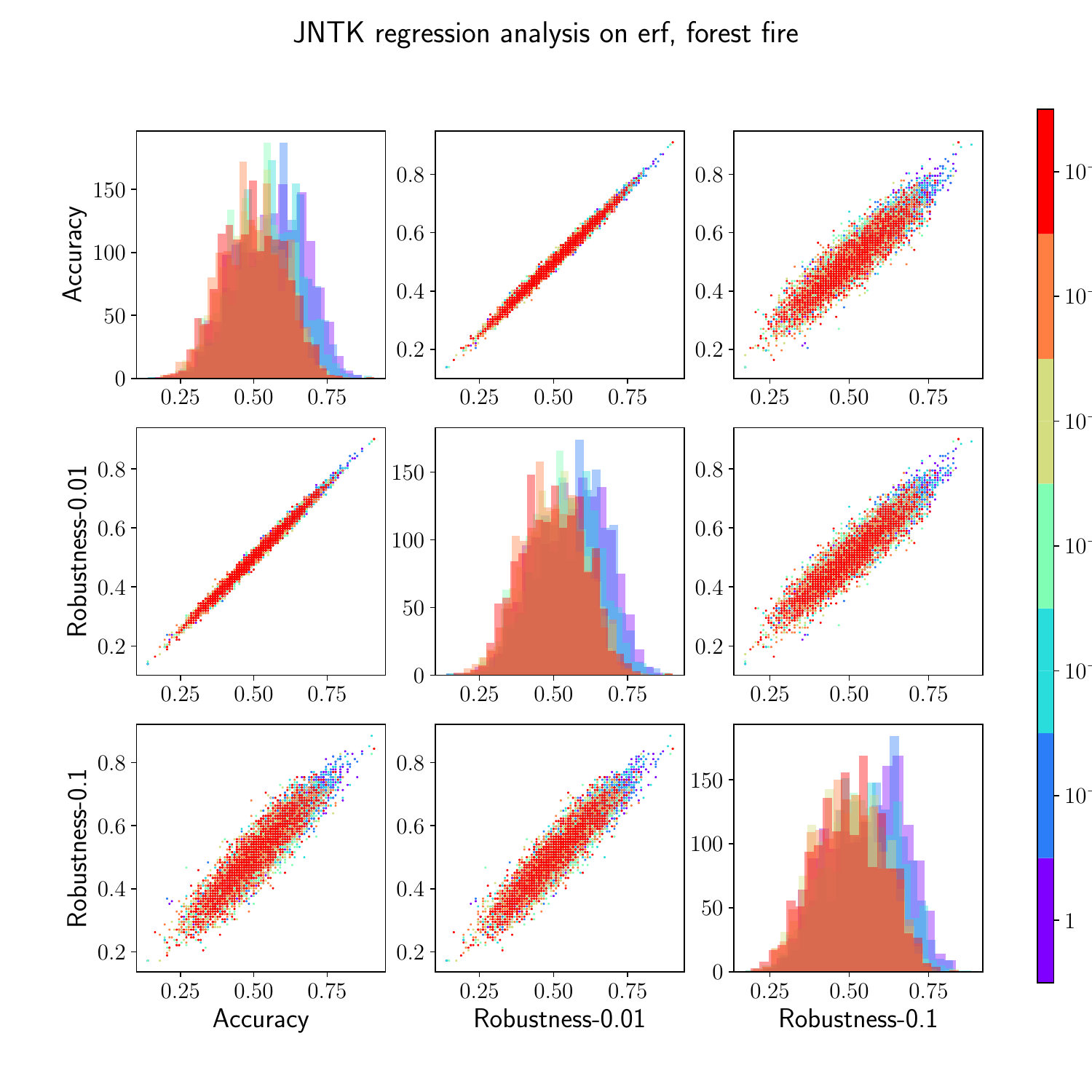}
  \end{subfigure}

  \caption{Kernel regression analysis on forest fire dataset, and erf activation. The left figure is the result without Jacobian regularisation, and the right figure is the result with Jacobian regularisation.}
  \label{fig:regression-analysis-fire-erf}
\end{figure}

\begin{figure}[H]
  \centering
  \begin{subfigure}[b]{0.49\linewidth}
    \includegraphics[width=\linewidth]{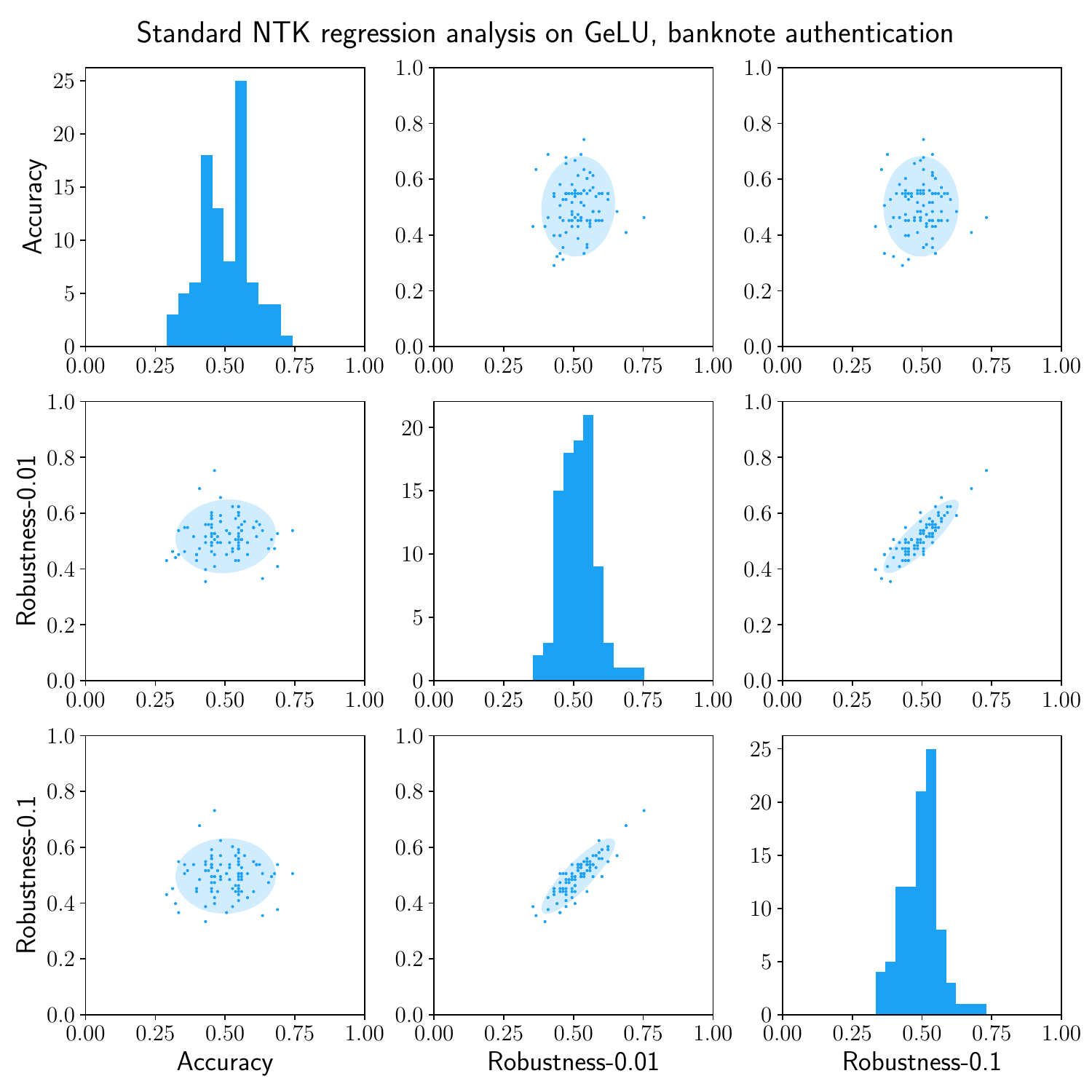}
  \end{subfigure}
  \begin{subfigure}[b]{0.49\linewidth}
    \includegraphics[width=\linewidth]{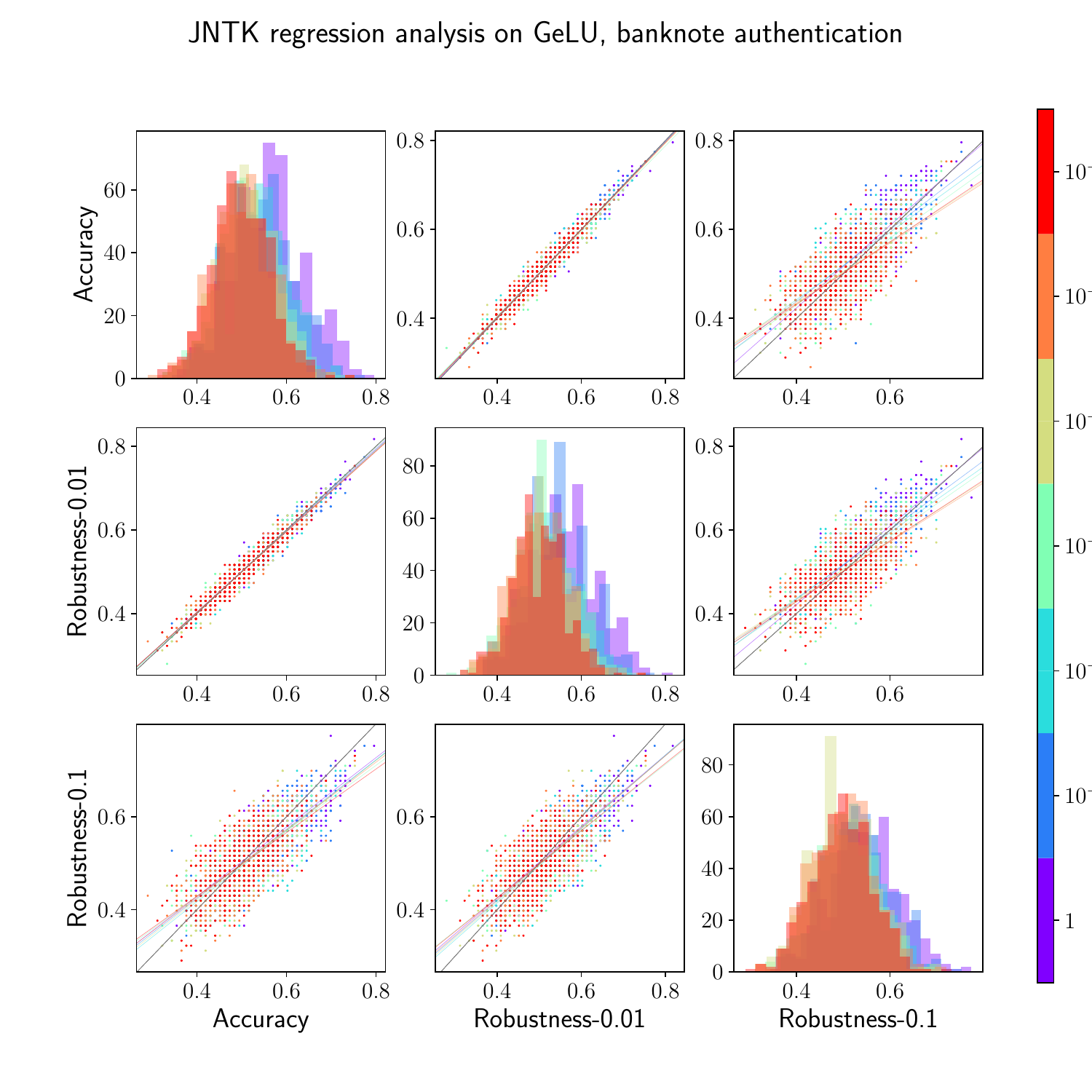}
  \end{subfigure}

  \caption{Kernel regression analysis on the banknote authentication dataset under the GeLU activation. The left figure is the result without Jacobian regularisation, and the right figure is the result with Jacobian regularisation.}
  \label{fig:regression-analysis-bank-gelu}
\end{figure}

\begin{figure}[H]
  \centering
  \begin{subfigure}[b]{0.49\linewidth}
    \includegraphics[width=\linewidth]{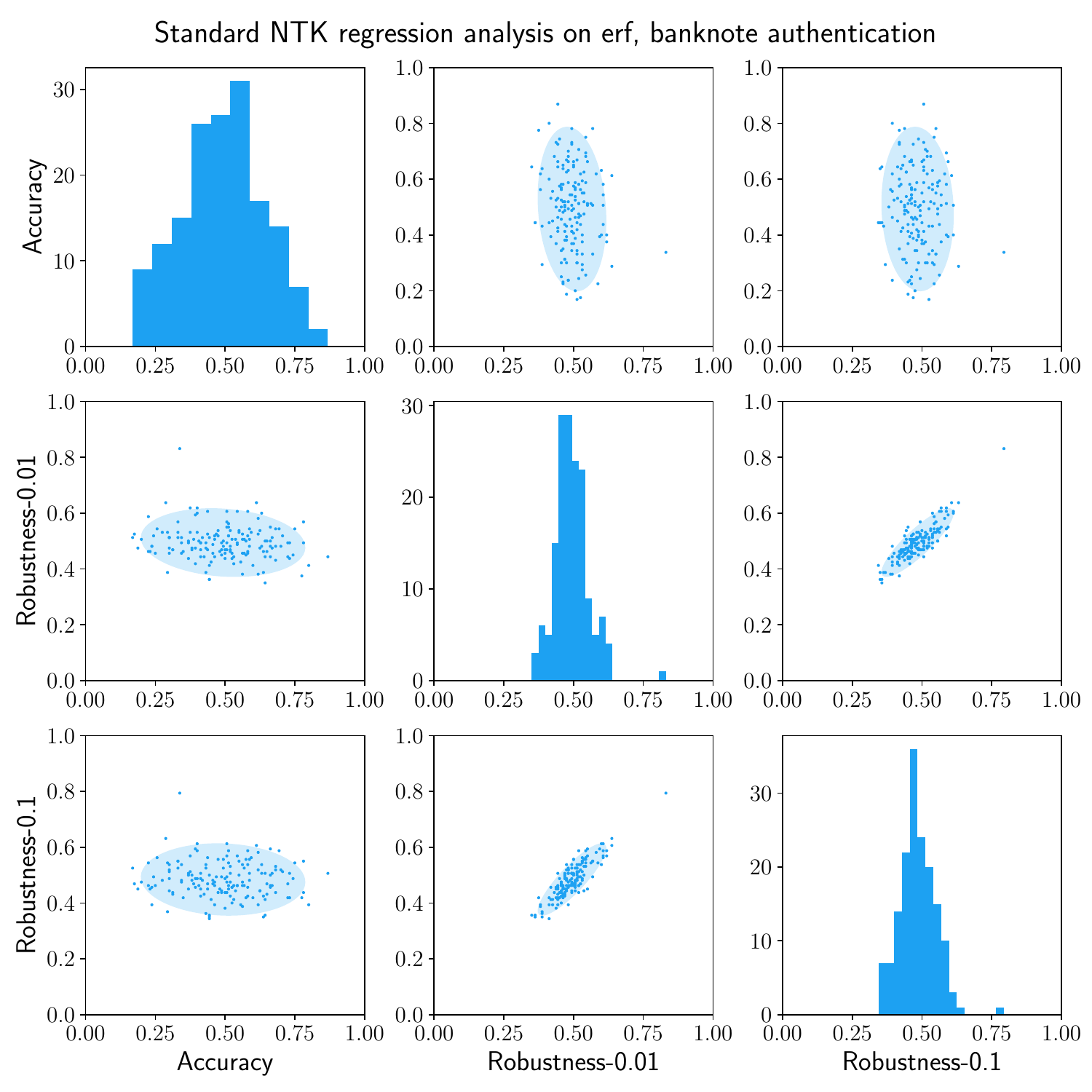}
  \end{subfigure}
  \begin{subfigure}[b]{0.49\linewidth}
    \includegraphics[width=\linewidth]{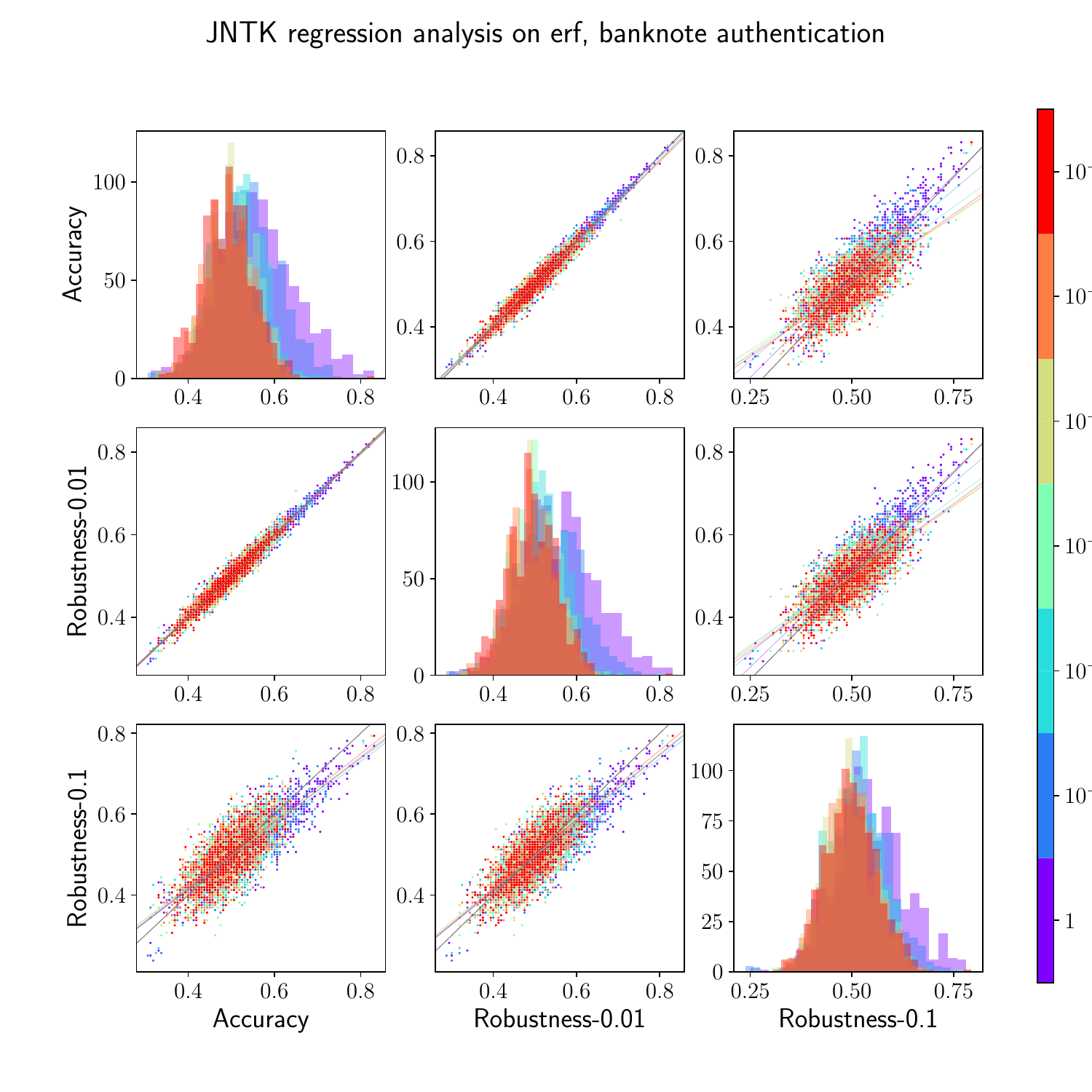}
  \end{subfigure}

  \caption{Kernel regression analysis on the banknote authentication dataset under the erf activation. The left figure is the result without Jacobian regularisation, and the right figure is the result with Jacobian regularisation.}
  \label{fig:regression-analysis-bank-erf}
\end{figure}

\begin{figure}[H]
  \centering
  \begin{subfigure}[b]{0.49\linewidth}
    \includegraphics[width=\linewidth]{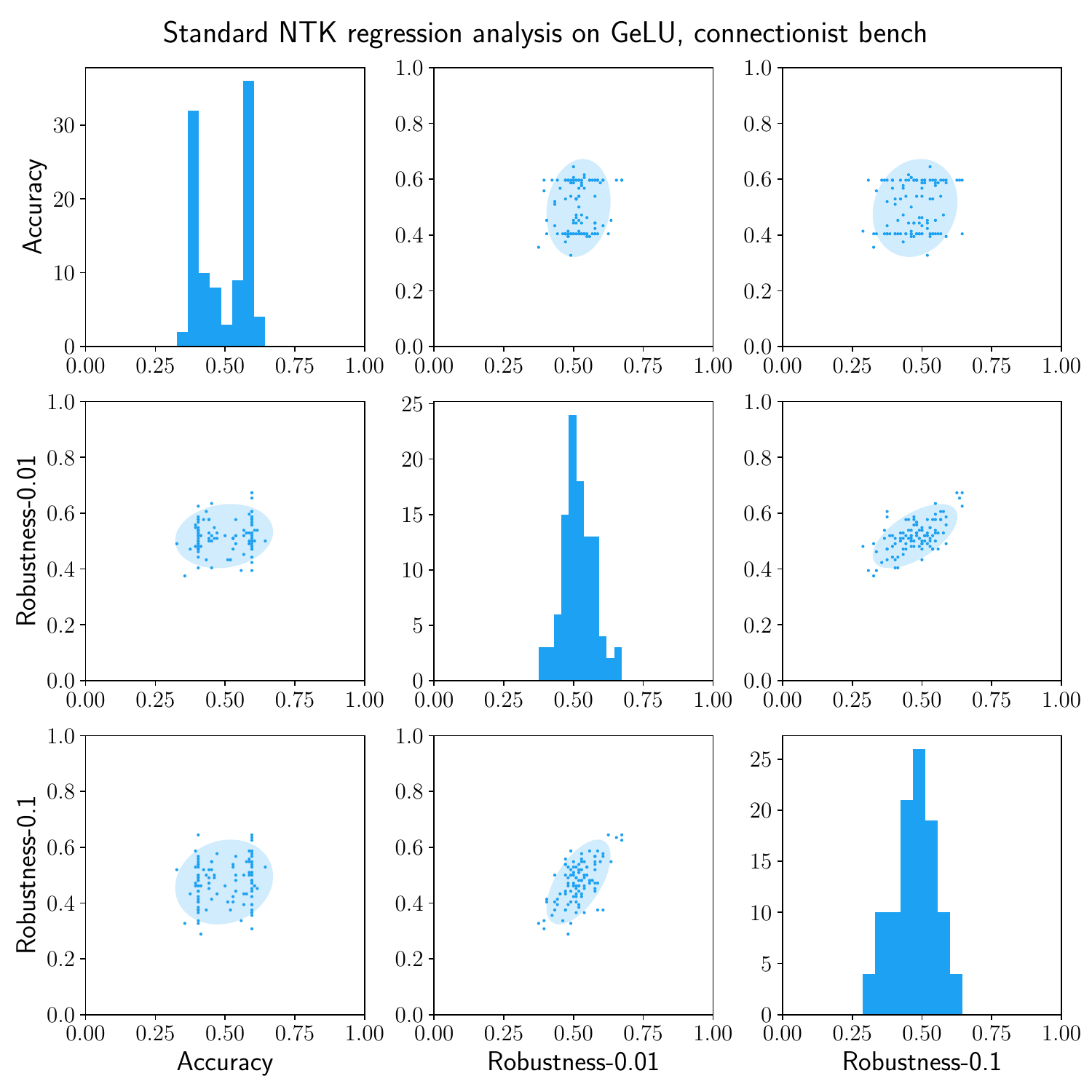} 
  \end{subfigure}
  \begin{subfigure}[b]{0.49\linewidth}
    \includegraphics[width=\linewidth]{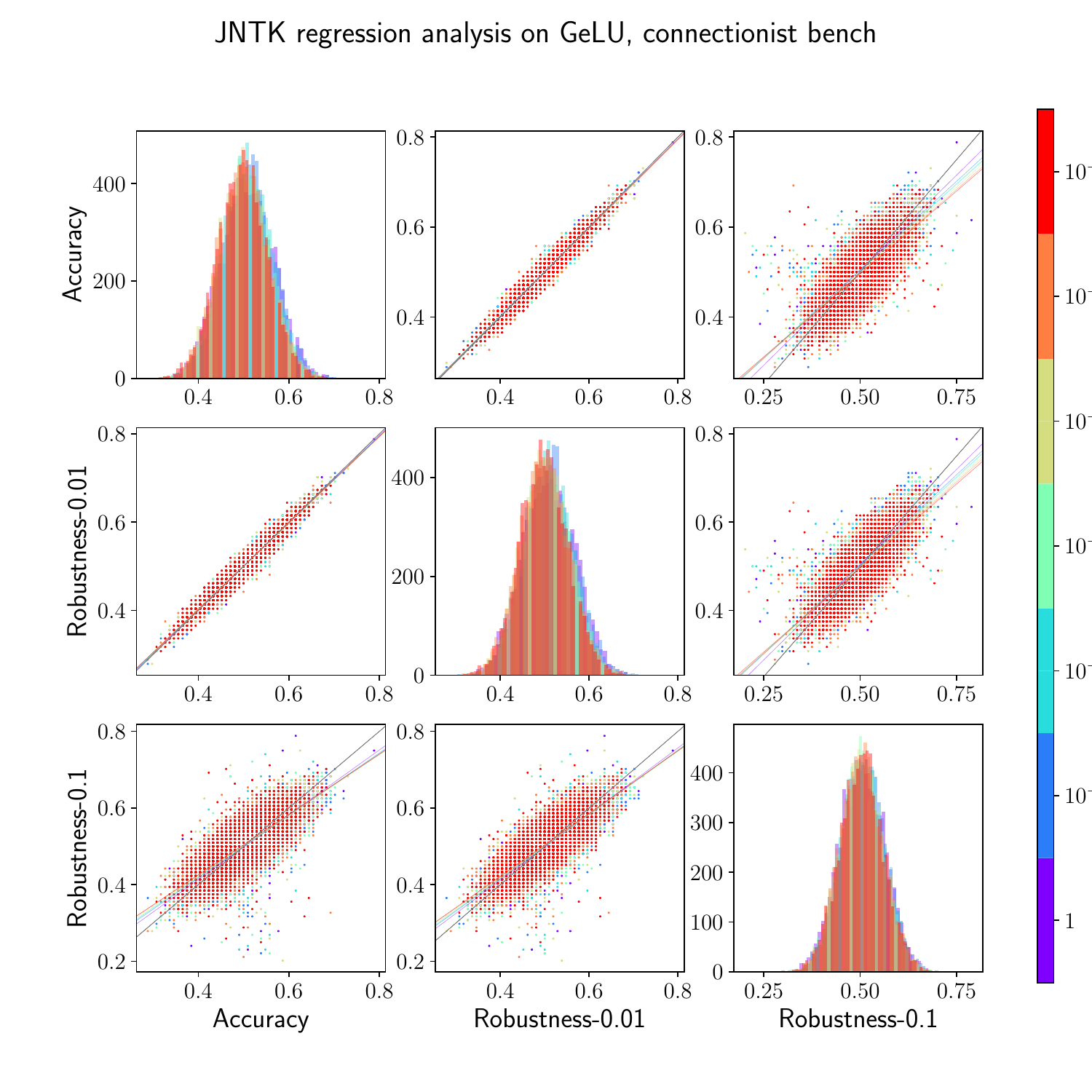} 
  \end{subfigure}

  \caption{Kernel regression analysis on the connectionist bench dataset under the GeLU activation. The left figure is the result without Jacobian regularisation, and the right figure is the result with Jacobian regularisation.}
  \label{fig:regression-analysis-sonar-gelu}
\end{figure}

\begin{figure}[H]
  \centering
  \begin{subfigure}[b]{0.49\linewidth}
    \includegraphics[width=\linewidth]{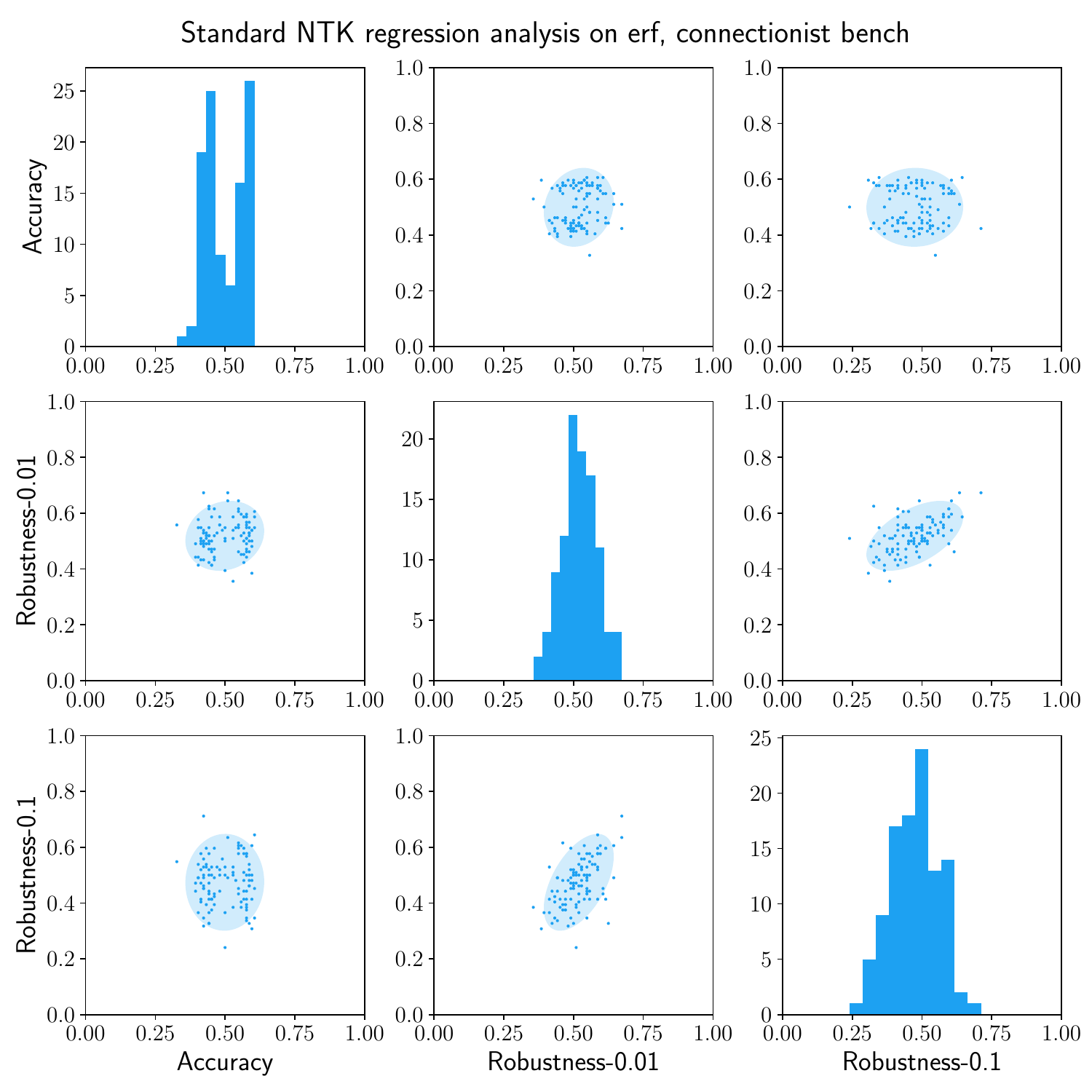} 
  \end{subfigure}
  \begin{subfigure}[b]{0.49\linewidth}
    \includegraphics[width=\linewidth]{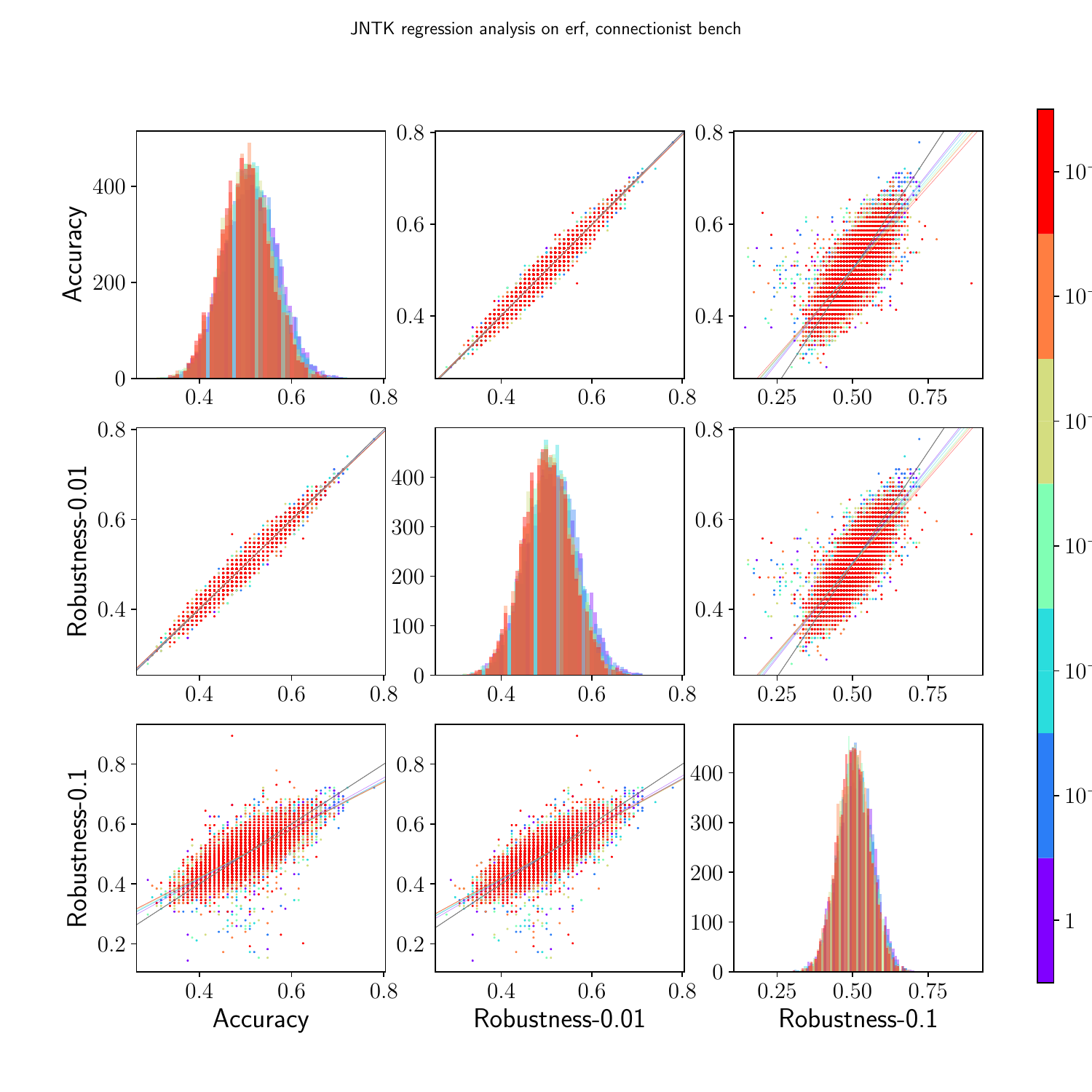}
  \end{subfigure}

  \caption{Kernel regression analysis on the connectionist bench dataset under the erf activation. The left figure is the result without Jacobian regularisation, and the right figure is the result with Jacobian regularisation.}
  \label{fig:regression-analysis-sonar-erf}
\end{figure}

\section{Additional Notations}
\label{appendix:notation}
We introduce symbols $h^{(0)}(x)$ and $g^{(l)}(x)$ for $l=1,\ldots,L+1$ and $x \in \R^{d_0}$, and let them denote the input $x$ and the pre-activation values at layers $l$. Thus, the MLP $f:\R^{d_0} \to \R^{d_{L+1}}$ is defined as follows: for all inputs $x \in \R^{d_0}$ and all layers $l = 2,\ldots,L$, 
\begin{align*}
  & h^{(0)}(x) \in \R^{d_0},  
  &&
  h^{(0)}(x) \defeq x,
  \\
  & g^{(1)}(x) \in \R^{d_1}
   &&
  g^{(1)}(x) \defeq W^{(1)}x,
  \\
  & h^{(1)}(x) \in \R^{d_1},  
  &&
  h^{(1)}(x) \defeq \phi(g^{(1)}(x)),
  \\
  & g^{(l)}(x) \in \R^{d_l},
  &&
  g^{(l)}(x) \defeq \frac{1}{\sqrt{d}} W^{(l)} h^{(l-1)}(x),
  \\
  & h^{(l)}(x) \in \R^{d_l},
  &&
  h^{(l)}(x) \defeq \phi(g^{(l)}(x)),
  \\
  & g^{(L+1)}(x) \in \R^{d_{L+1}},
  &&
  g^{(L+1)}(x) \defeq \frac{1}{\sqrt{d}} W^{(L+1)} h^{(L)}(x),
  \\
  & f(x) \in \R^{d_{L+1}},
  &&
  f(x) \defeq \kappa g^{(L+1)}(x).
\end{align*}  
Recall our assumption that $d = d_1 = \ldots = d_L$ and $d_{L+1} = 1$. 

We use the following constant $C_l$, especially for simplifying the multipliers in the neural network:
\[
  C_l \defeq \begin{cases}
    1 & l = 0,
    \\
    d & l \ge 1.
  \end{cases}
\]
We write $G_C(n)$ for the sum of the geometric series $C^0, C^1,\ldots,C^n$, i.e.,
\[
  G_C(n) \defeq \sum_{i=0}^n C^i.
\]
In the Appendix, we consider norms on random variables, such as the sub-Gaussian norm $\|\ \cdot\ \|_{SG}$ and the sub-exponential norm $\|\ \cdot\ \|_{SE}$.

For the sums $\sum_{i=a}^b$ and products $\prod_{i=a}^b$, we use empty sum / product convention when $a > b$,
\[
  \sum_{i=a}^b ({}\cdots{}) = 0, \qquad \prod_{i=a}^b ({}\cdots{}) = 1.
\]

\section{Review of the Tensor-Program Framework}\label{sec:tensorprogram}

We quickly review the tensor-program framework by Greg Yang~\cite{Yang2019TensorPI,Yang2020a}. To simplify the presentation, our review describes only a simplified version of the framework where all the hidden layers have the same width $n$. For the full version where the hidden layers have different widths (but these widths are sent to infinity with a fixed ratio), see the original papers of the framework~\cite{Yang2019TensorPI,Yang2020a}.

\emph{Tensor programs} are a particular type of straight-line programs that represent computations involving neural networks, such as forward computation and backpropagation. Each tensor program consists of three parts, namely, initialisation, main body, and output, each of which we will explain next. 

The initialisation part declares $\R^d$-valued input variables $x_1, \ldots, x_{p}$, each being initialised with i.i.d. Gaussian entries, and also $d \times d$ matrices  $M_1, \ldots, M_{q}$ where each matrix is initialised with i.i.d. Gaussian entries and independent with all the other matrices and the input variables. That is, for $i \in [p]$, $i' \in [q]$, and $k, k' \in [d]-1$,
  \begin{align*}
    x_{i} &\in \R^{d}, 
    & 
    (x_{i})_k &\stackrel{\mathrm{i.i.d.}}{\sim} \cN(0, \sigma_{i}^2), 
    \\
    M_{i'} &\in \R^{d \times d}, 
    & 
    (M_{i'})_{k,k'} &\stackrel{\mathrm{i.i.d.}}{\sim} \cN\left(0, \frac{\tau_{i'}^2}{d}\right).
  \end{align*}
  The input variables can be correlated. This correlation is described by a covariance matrix $(C_{i,i'})_{i,i' \in [p]}$ that is assumed given:
  for $i,i' \in [p]$ and $k,k' \in [d]-1$,
  \begin{equation*}
    \mathrm{Cov}((x_i)_k, (x_i')_k') 
    = 
    \begin{cases}
    C_{i,i'},  & \text{if } k = k';
    \\
    0, & \text{otherwise}.
    \end{cases}
  \end{equation*}
  
  The main body of a tensor program is a sequence of two types of assignments such that no input variables are assigned and the other variables are assigned once. The two types of assignments are as follows:
  \begin{description}
    \item[MatMul] $h = Mx$ or $h = M^\intercal x$ where $M$ is one of the matrices declared in the initialisation part, and $x$ is a variable assigned by the next NonLin-type assignment before the current MatMul-type assignment.
    \item[NonLin] $x = \phi(h_1,\ldots,h_m)$ where $h_1,\ldots,h_m$ are input variables or those assigned by the first MatMul-type assignment before the current assignment $x = \phi(h_1,\ldots,h_m)$, the $\phi : \R^m \to \R$ is a function, and this function is applied to the vectors $h_1,\ldots,h_m$ pointwise in this assignment.
  \end{description}

  The last output part of a tensor program has the form:
  \begin{align*}
    \frac{1}{d} \sum_{k=0}^{d-1} \psi((h_1)_k, \ldots, (h_m)_k) \text{ or }\left(\frac{1}{\sqrt{d}} v^\intercal h_1,\ldots,\frac{1}{\sqrt{d}} v^\intercal h_m\right)
  \end{align*}
  where $\psi$ is a function of type $\R^m \to \R$, $h_1,\ldots,h_m$ are variables, and $v$ is an input variable.

 Every variable $v$ (which can be $x$ or $h$) in a tensor program has an associated real-valued random variable $Z^v$, which is defined inductively. For input variables $x_1,\ldots,x_p$, we have real-valued random variables $Z^{x_1},\ldots, Z^{x_p} \in \R$ that are jointly distributed by the zero-mean multivariate Gaussian with the following covariance:
   \[
   \mathrm{Cov}(Z^{x_i}, Z^{x_{i'}}) = C_{i,i'}.
   \]
 For assigned variables, their random variables are defined as follows.
  \begin{description}
    \item[MatMul] For the assignment $h = M x$ or $h = M^\intercal x$, the corresponding $Z^h$ is the zero-mean Gaussian random variable that is independent with $Z^{h'}$ previously-defined by induction whenever the matrix $M'$ in $h' = M' x'$ (including $M^\intercal$) is not the same as $M$, and is correlated with $Z^{h''}$ with $h'' = M x''$ for the same $M$ as follows:
    \[
        \mathrm{Cov}(Z^{h}, Z^{h''}) = \sigma_M^2 \E[Z^{x} Z^{x''}]
    \]
    where $\sigma_M^2$ is the variance of the entries of the matrix $M$. 
    \item[NonLin] For the assignment $x = \phi(h_1,\ldots,h_m)$, the corresponding random variable $Z^x$ is defined by the following equation:
    \[
    Z^x = \phi(Z^{h_1}, \ldots, Z^{h_m}).
    \]
    Note that $h_1,\ldots,h_m$ are assigned before the assignment for $x$ so that the random variables $Z^{h_1},\ldots,Z^{h_m}$ are defined by the inductive construction.
  \end{description}

The main reason for using the tensor-program framework is the master theorem which states that as $n$ goes to infinity, the $k$-th components of variables in a tensor program jointly converge the random variables that we have just defined inductively. 

\begin{theorem}[Master Theorem] \label{thm:MasterTheoremLLN}
   If a tensor program uses only polynomially-bounded nonlinear functions $\phi$ in NonLin (i.e., there are some $C, c_1, c_2 > 0$ such that for any $u \in \R^m$, $\phi(u) \le C \lVert u \rVert^{c_1} + c_2$) and it satisfies the so-called BP-likeness (see \cite{Yang2019TensorPI,Yang2020a} for the definition), then for all polymonially-bounded $\psi$, we have the following almost-sure convergence:
\[
\frac{1}{d} \sum_{k=0}^{d-1} \psi((h_1)_k, \ldots, (h_m)_k) \stackrel{a.s.}{\to} \E[\psi(Z^{h_1}, \ldots, Z^{h_m})],
\]
as $d$ tends to infinity.
\end{theorem}

Using this Master theorem and Proposition G.4 from \cite{Yang2019TensorPI}, we can show that the output distribution converges to a Gaussian random variable.

\begin{corollary}[GP Convergence] \label{cor:MasterTheoremCLT}
  Assume that a tensor program uses only polynomially-bounded nonlinear functions $\phi$ in NonLin and it satisfies the so-called BP-likeness. Also, assume that the output of the tensor program is
\[
    \left(\frac{1}{\sqrt{d}} v^\intercal h_1,\ldots,\frac{1}{\sqrt{d}} v^\intercal h_m\right)
\]
for the input random vector $v$ such that its entries are initialised with $\cN(0, \sigma_v^2)$, the vector is not used anywhere else in the program, and it is independent of all other input variables in the program. Then, 
as $n$ tends to infinity, the output of the program converges in distribution to the following Gaussian distribution:
\[
  \left( \frac{1}{\sqrt{d}} v^\intercal h_1, \ldots, \frac{1}{\sqrt{d}} v^\intercal h_k \right) \stackrel{dist.}{\to } \cN(0, K)
\]
where
\[
  K_{ij} = \sigma_v^2 \E[Z^{h_i} Z^{h_j}].
\]
\end{corollary}

\section{Proof of Theorem \ref{thm:JacobianNNGPInitialisation}}
\label{sec:ProofJacobianNNGPInitialisation}
We express the computation of the network's outputs and Jacobians on all the training inputs as a tensor program and use the Master theorem~\ref{cor:MasterTheoremCLT} to show that these outputs and Jacobians jointly converge to Gaussian random variables.

We first define the input vectors and matrices. We have the following $(1+d_0)N$-many input vectors, which correspond to the values of the preactivation on the training inputs and the Jacobians of these preactivation:
\[
 g^{(1)}_1,\;\ldots,\;g^{(1)}_N,\; (J_1g^{(1)}_1),\;\ldots,\; (J_{d_0}g^{(1)}_N) \in \R^d
\]
These input vectors are initialised as follows:
for $k \in [d]-1$, $i, j \in [N]$, and $\alpha, \beta \in [d_0]$, 
\begin{align*}
  \mathrm{Var}\left((g_i^{(1)})_k\right) &= 1, & \mathrm{Var}\left((J_\alpha g_i^{(1)})_k\right) &= 1,
  \\
  \mathrm{Cov}\left( (g_i^{(1)})_k, (g_j^{(1)})_k \right) &= \left\langle x^{(i)}, x^{(j)} \right\rangle, & \mathrm{Cov}\left( (J_\alpha g_i^{(1)})_k, (g_j^{(1)})_k \right) &= \left\langle e_{\alpha-1}, x^{(j)} \right\rangle,
  \\
  \mathrm{Cov}\left( (g_i^{(1)})_k, (J_\beta g_j^{(1)})_k \right) &= \left\langle x^{(i)}, e_{\beta-1} \right\rangle, & \mathrm{Cov}\left( (J_\alpha g_i^{(1)})_k, (J_\beta g_j^{(1)})_k \right) &= \left\langle e_{\alpha-1}, e_{\beta-1} \right\rangle,
\end{align*}
where $e_{\alpha-1}$ is the unit vector with its $\alpha$-th component being 1 and the others being 0. Note that this initialisation corresponds to having a random matrix $W^{(1)}$ intialised with i.i.d. samples from $\mathcal{N}(0,1)$, and setting $g^{(1)}_i$ to $W^{(1)}x^{(i)}$ and $J_\alpha g_i^{(1)}$ to $\frac{\partial W^{(1)}x}{\partial x_{\alpha-1}}|_{x = x^{(i)}}$. 

\update{We also have an input vector $V^{(L+1)} \in \R^d$ whose entries are set to the i.i.d. samples from $\cN(0, \kappa^2)$. 
The vector is independent of the previously defined input vectors.}{Note that we use different notation.}

For the input matrices, we have $V^{(l)}$ for each layer $l = 2,\ldots,L$. The entries of $V^{(l)}$ are initialised with i.i.d. samples from $\cN(0, 1/d)$. 

\update{Note that we use $V$ instead of $W$ as opposed to the main text.
This is to match the syntax of the Tensor program while keeping the computation equivalent to the main text.
In specific, the computation of matrix multiplication
\[
  h^{(l)}(x) = \phi \left( \frac{1}{\sqrt{d_{l-1}}} W^{(l)} h^{(l-1)}(x) \right)
\]
is modelled by setting $V^{(l)} = 1/\sqrt{d_{l-1}} W^{(l)}$ satisfies the requirement of the Tensor program.
Similarly, the output vector $W^{(L+1)}$ is implemented by $V^{(L+1)} = \kappa W^{(L+1)}$, where we absorb the scaling factor $\kappa$ into the initialisation of $V^{(L+1)}$, and the normalisation factor $1/\sqrt{d}$ is absorbed in the result form of Corollary~\ref{cor:MasterTheoremCLT}.}{Note that we use different parameterisation to follow Tensor program.}

Next, we define the body of the program that computes the network's outputs on all the training inputs as well as the corresponding Jacobians. This body uses the following variables:
\begin{align*}
  & 
  h_1^{(1)},\; \ldots,\; h_N^{(1)},\;
  \ldots,\;
  h_1^{(L)},\; \ldots,\; h_N^{(L)} \in \R^d,
  \\
  &
  g_1^{(2)},\; \ldots,\; g_N^{(2)},\;
  \ldots,\;
  g_1^{(L)},\; \ldots,\; g_N^{(L)} \in \R^d,
  \\
  &
  (J_1h_1^{(1)}),\; \ldots,\; (J_{d_0} h_N^{(1)}),\;
  \ldots,\;
  (J_1h_1^{(L)}),\; \ldots,\; (J_{d_0} h_N^{(L)}) \in \R^d,
  \\
  &
  (J_1g_1^{(2)}),\; \ldots,\; (J_{d_0}g_N^{(2)}),\;
  \ldots,\;
  (J_1g_1^{(L)}),\; \ldots,\; (J_{d_0}g_N^{(L)}) \in \R^d.
\end{align*}
It is defined as follows:
\begin{align*}
  & h_1^{(1)} = \phi(g_1^{(1)});
  &&
    \ldots
  &&
    h_N^{(1)} = \phi(g_N^{(1)});
  \\
  & g_1^{(2)} = V^{(2)} h_1^{(1)};
  &&
    \ldots
  &&
    g_N^{(2)} = V^{(2)} h_N^{(1)};
  \\
  & h_1^{(2)} = \phi(g_1^{(2)});
  &&
    \ldots
  &&
    h_N^{(2)} = \phi(g_N^{(2)});
  \\
  & \vdots &&&& \vdots
  \\
  & g_1^{(L)} = V^{(L)} h_1^{(L-1)};
  &&
    \ldots
  &&
    g_N^{(L)} = V^{(L)} h_N^{(L-1)};
  \\
  & h_1^{(L)} = \phi(g_1^{(L)});
  &&
    \ldots
  &&
    h_N^{(L)} = \phi(g_N^{(L)});
  \\[1ex]
  & (J_1 h_1^{(1)}) = \dot{\phi}(g_1^{(1)}) \odot (J_1 g_1^{(1)});
  && 
    \ldots
  &&
    (J_{d_0} h_N^{(1)}) = \dot{\phi}(g_N^{(1)}) \odot (J_{d_0} g_N^{(1)});
  \\ 
  & (J_1 g_1^{(2)}) = V^{(2)} (J_1 h_1^{(1)});
  &&
    \ldots
  && 
    (J_{d_0} g_N^{(2)}) = V^{(2)} (J_{d_0} h_N^{(1)});
  \\
  & (J_1 h_1^{(2)}) = \dot{\phi}(g_1^{(2)}) \odot (J_1 g_1^{(2)});
  && 
    \ldots
  &&
    (J_{d_0} h_N^{(2)}) = \dot{\phi}(g_N^{(2)}) \odot (J_{d_0} g_N^{(2)});
  \\
  & \vdots &&&& \vdots
  \\
  & (J_1 g_1^{(L)}) = V^{(L)} (J_1 h_1^{(L-1)});
  &&
    \ldots
  && 
    (J_{d_0} g_N^{(L)}) = V^{(L)} (J_{d_0} h_N^{(L-1)});
  \\
  & (J_1 h_1^{(L)}) = \dot{\phi}(g_1^{(L)}) \odot (J_1 g_1^{(L)});
  && 
    \ldots
  &&
    (J_{d_0} h_N^{(L)}) = \dot{\phi}(g_N^{(L)}) \odot (J_{d_0} g_N^{(L)});
  \end{align*}

Finally, we make the program output the following $(1+d_0)N$-dimensional vector:
\begin{align*}
  &
  \Bigg\langle 
  \frac{1}{\sqrt{d}} (V^{(L+1)})^\intercal h_1^{(L)},\;
  \ldots,\; 
  \frac{1}{\sqrt{d}} (V^{(L+1)})^\intercal h_N^{(L)},\;
  \\
  & \qquad 
  \frac{1}{\sqrt{d}} (V^{(L+1)})^\intercal (J_1 h_1^{(L)}),\;
  \ldots,\;
  \frac{1}{\sqrt{d}} (V^{(L+1)})^\intercal (J_1 h_N^{(L)}),\;
  \\
  & \qquad 
  \ldots
  \\
  & \qquad
  \frac{1}{\sqrt{d}} (V^{(L+1)})^\intercal (J_{d_0} h_1^{(L)}),\;
  \ldots,\;
  \frac{1}{\sqrt{d}} (V^{(L+1)})^\intercal (J_{d_0} h_N^{(L)})
  \Bigg\rangle.
\end{align*}

Now Corollary \ref{cor:MasterTheoremCLT} implies that these random variables jointly converge in distribution to a zero-mean multivariate Gaussian distribution as $d$ tends to infinity. Furthermore, the Master theorem (i.e., Theorem~\ref{thm:MasterTheoremLLN}) gives the following inductive formula for computing the expectation in the definition of the covariance matrix in the corollary:
\begin{align*}
  \mathrm{Cov} \left( Z^{g_i^{(l)}}, Z^{g_j^{(l)}} \right) &= \E \left[ \phi\left( Z^{g_i^{(l-1)}} \right) \phi\left( Z^{g_j^{(l-1)}} \right) \right],
  \\
  \mathrm{Cov} \left( Z^{J_\alpha g_i^{(l)}}, Z^{g_j^{(l)}} \right) &= \E \left[ \dot{\phi}\left( Z^{g_i^{(l-1)}} \right) Z^{J_\alpha g_i^{(l-1)}} \phi\left( Z^{g_j^{(l-1)}} \right) \right],
  \\
  \mathrm{Cov} \left( Z^{g_i^{(l)}}, Z^{J_\beta g_j^{(l)}} \right) &= \E \left[ \phi\left( Z^{g_i^{(l-1)}} \right) \dot{\phi}\left( Z^{g_j^{(l-1)}} \right) Z^{J_\beta g_j^{(l-1)}} \right],
  \\
  \mathrm{Cov} \left( Z^{J_\alpha g_i^{(l)}}, Z^{J_\beta g_j^{(l)}} \right) &= \E \left[ \dot{\phi}\left( Z^{g_i^{(l-1)}} \right) Z^{J_\alpha g_i^{(l-1)}} \dot{\phi}\left( Z^{g_j^{(l-1)}} \right) Z^{J_\beta g_j^{(l-1)}}\right].
\end{align*}
By definition, the random vectors $(Z^{g_i^{(1)}}, Z^{J_1 g_i^{(1)}}, \ldots, Z^{J_{d_0} g_i^{(l)}}) \in \R^{1 + d_0}$ for $i \in [N]$ have the covariance given by $\Sigma^{(0)}(x^{(i)}, x^{(j)})$ in the statement of Theorem~\ref{thm:JacobianNNGPInitialisation}. 
Note that the inductive definition that we have just given is identical to the inductive definition of $\Sigma^{(l-1)}(x^{(i)}, x^{(j)})$ in Theorem~\ref{thm:JacobianNNGPInitialisation}. 
Thus, the random vectors $(Z^{g_i^{(l)}}, Z^{J_1 g_i^{(l)}}, \ldots, Z^{J_{d_0} g_i^{(l)}}) \in \R^{1 + d_0}$ for $i \in [N]$ have the covariance given by 
$\Sigma^{(l-1)}(x^{(i)}, x^{(j)})$. As a result, the covariance of the limiting distribution of the network output and its Jacobian is given by
\begin{align*}
  \mathrm{Cov}\left( f(x^{(i)}), f(x^{(j)}) \right) &= \kappa^2 \Sigma^{(L)}(x^{(i)}, x^{(j)})_{00},
  \\
  \mathrm{Cov}\left( J(f)(x^{(i)})_{\alpha-1}, f(x^{(j)}) \right) &= \kappa^2 \Sigma^{(L)}(x^{(i)}, x^{(j)})_{\alpha 0},
  \\
  \mathrm{Cov}\left( f(x^{(i)}), J(f)(x^{(j)})_{\beta-1} \right) &= \kappa^2 \Sigma^{(L)}(x^{(i)}, x^{(j)})_{0\beta}, 
  \\
  \mathrm{Cov}\left( J(f)(x^{(i)})_{\alpha-1}, J(f)(x^{(j)})_{\beta-1} \right) &= \kappa^2 \Sigma^{(L)}(x^{(i)}, x^{(j)})_{\alpha \beta}.
\end{align*}
\update{So far, we implemented the forward computation of our network as a tensor program.
To apply the Master theorem, we need to check that (1) the nonlinear functions we applied are polynomially bounded and (2) the program satisfies the BP-likeness property.
The polynomial boundedness is immediate since the function we use is either $\phi$ or $\dot{\phi}(a) \cdot b$, and the assumption that $\phi$ is Lipschitz implies that both $\phi$ and $\dot{\phi}(a) \cdot b$ are polynomially bounded.
For the BP-likeness, the simple GIA check (Condition 1 of \cite{Yang2020a}) applies to our program, since $V^{(L+1)}$ is not used in any other part of the program, except when we compute the output.}{Add the explanation of the BP-likeness.}

This proves Theorem~\ref{thm:JacobianNNGPInitialisation}.

We use this theorem to show that if we take $\kappa$ small enough and $d$ large enough, both outputs and Jacobians at initialisation are close to zero.

\begin{remark} \label{rmk:InitCloseToZero}
  By the definition of $\Sigma^L$ and the assumption on the activiation function $\phi$, for all $x \in \R^{d_0}$ with $\|x\| = 1$, the diagonal entries of $\kappa^2 \Sigma^{(L)}(x, x)$ are at most $\kappa^2(M_1^{2L}+1)$. Thus, using the tail bound for the standard normal distribution and also union bound, we can show that for all $\epsilon > 0$ and $x \in \R^{d_0}$ with $\|x\| = 1$, if $g \sim \GP(0, \kappa^2 \Sigma^{(L)})$, then
  \[
    P\left( \left\|g(x)\right\|_\infty > \epsilon \right) \le 2(1 + d_0) \exp\left( - \frac{\epsilon^2}{2\kappa^2(M_1^{2L}+1)} \right).
  \]
  Since every training input $x^{(i)}$ satisfies $\|x^{(i)}\| = 1$ by assumption, we can instantiate the above bound on those inputs and get, by union bound, that
  \begin{equation}
  \label{eqn:init-bound:limit-case}
    P\left( \forall i \in [N].\, \left\| g(x^{(i)}) \right\|_\infty \le \epsilon \right) \ge 1 - 2 N (1 + d_0) \exp\left( - \frac{\epsilon^2}{2\kappa^2(M^{2L}_1+1)} \right).
  \end{equation}

  By Theorem~\ref{thm:JacobianNNGPInitialisation}, there exists a function $F_1$ that if
  \[
    d \ge F_1\left(x^{(1:N)},\, L,\, \delta,\ \frac{\epsilon}{\kappa\sqrt{M^{2L}_1+1}}\right),
  \]
  then 
  \begin{equation}
    \label{eqn:init-bound:convergence-error}
    \left| P \left( \forall i \in [N], \alpha \in [d_0].\, \left(f_d(x^{(i)}), J(f_{d})(x^{(i)})_{\alpha-1} \in [-\epsilon, \epsilon]\right) \right) - P\left( \left\|g(x^{(i)})\right\|_\infty \le \epsilon \right) \right| \le \frac{\delta}{2}.
  \end{equation}
  From the bounds in \eqref{eqn:init-bound:limit-case} and \eqref{eqn:init-bound:convergence-error} it follows that for all $\delta \in (0,1)$, if $\kappa$ is sufficiently small and $d$ is sufficiently large, then with probability at least $1 - \delta$,
  \[
    f_d(x^{(i)}),\, J(f_{d})(x^{(i)})_{\alpha - 1} \in [-\epsilon, \epsilon]
  \]
  for all $i \in [N]$ and $\alpha \in [d_0]$.
\end{remark}

\section{Proof of Theorem~\ref{thm:JacobianNNGPCorrespondence}}
\label{sec:ProofJacobianNNGPCorrespondence}

We start with one general result that allows us to move the derivative from the outside of an expectation to the inside when the expectation is taken over a GP. 
\begin{theorem}
  \label{thm:exchange-theorem}
  Let $K:\R^{d_0} \times \R^{d_0} \to \R$ be a symmetric kernel. Consider functions
  $\psi : \R \to \R$ and $\varphi : \R^2 \to \R$ such that
  \begin{itemize}
    \item $\psi$ and $\varphi$ are polynomially bounded; and
    \item there exists a polynomially-bounded function $\dot{\psi} : \R \to \R$ that satisfies
    $\int_0^x \dot{\psi}(t) dt = \psi(x) - \psi(0)$ for all $x \in \R$.
  \end{itemize}
  Let $x,x' \in \R^{d_0}$, $\alpha,\beta \in [d_0]$, and
  \begin{align*}
  \Gamma & = 
  \begin{pmatrix}
  K(x,x) & K(x,x') & \frac{\partial K(x,y')}{\partial y'_{\alpha-1}}\Big|_{y' = x} 
  & \frac{\partial K(x,y')}{\partial y'_{\beta-1}}\Big|_{y' = x'} 
  \\
  K(x',x) & K(x',x') & \frac{\partial K(x',y')}{\partial y'_{\alpha-1}}\Big|_{y' = x} 
  & \frac{\partial K(x',y')}{\partial y'_{\beta-1}}\Big|_{y' = x'} 
  \\
  \frac{\partial K(y,x)}{\partial y_{\alpha-1}}\Big|_{y=x} & \frac{\partial K(y,x')}{\partial y_{\alpha-1}}\Big|_{y = x} 
  & \frac{\partial^2 K(y,y')}{\partial y_{\alpha-1} \partial y'_{\alpha-1}}\Big|_{(y,y') = (x, x)}
  & \frac{\partial^2 K(y,y')}{\partial y_{\alpha-1} \partial y'_{\beta-1}}\Big|_{(y,y') = (x,x')}
  \\
  \frac{\partial K(y,x)}{\partial y_{\beta-1}}\Big|_{y=x'} & \frac{\partial K(y,x')}{\partial y_{\beta-1}}\Big|_{y = x'} 
  & \frac{\partial^2 K(y,y')}{\partial y_{\beta-1} \partial y'_{\alpha-1}}\Big|_{(y,y') = (x', x)}
  & \frac{\partial^2 K(y,y')}{\partial y_{\beta-1} \partial y'_{\beta-1}}\Big|_{(y,y') = (x',x')}
  \end{pmatrix}.
  \end{align*}
  Then, 
  \begin{equation}
  \label{eqn:exchange-theorem:1}
  \frac{\partial}{\partial x'_{\beta-1}}\E_{z \sim \cN(0,\Gamma)}\left[\psi(z_1) \cdot \varphi(z_0,z_2)\right] 
  = 
  \E_{z \sim \cN(0, \Gamma)} \left[\dot{\psi}(z_1) \cdot z_3 \cdot \varphi(z_0,z_2)\right].
  \end{equation}
\end{theorem}
\begin{proof}
  Fix $x,x' \in \R^{d_0}$ and $\alpha,\beta \in [d_0]$. Let $\Gamma$ be the matrix in the statement of the theorem, and $\Sigma = \Gamma_{-3,-3} \in \R^{3 \times 3}$ the submatrix of $\Gamma$ without the last column and row. 
  Let $I_3$ and $I_4$ be the $3\times 3$ and $4\times 4$ identity matrices, respectively. 

  We will first introduce a function $H(u)$ which gives the reparameterisation of $\Sigma$ for $z_1$, defined as
  \begin{align*}
    a_0 &= \frac{\Sigma_{0, 1}}{\sqrt{\Sigma_{0, 0}}},
    \\
    a_2 &= \frac{\Sigma_{1, 2} - \Sigma_{0,1} \Sigma_{0,2} / \Sigma_{0,0}}{\sqrt{\Sigma_{2,2} - \Sigma_{0,2}^2 / \Sigma_{0,0}}},
    \\
    a_1 &= \sqrt{\Sigma_{1,1} - a_0^2 - a_1^2}, 
    \\
    H(u) &= a_0 u_0 + a_1 u_1 + a_2 u_2.
  \end{align*}
  This gives the reparameterisation as
  \[
    \left( z_0, z_1, z_2 \right) \stackrel{dist}{=} \left( \sqrt{\Sigma_{0,0}} u_0, H(u), \frac{\Sigma_{0,2}}{\sqrt{\Sigma_{0,0}}} u_0 + \sqrt{\Sigma_{2,2} - \frac{\Sigma_{0,2}^2}{\Sigma_{0,0}}} u_2 \right)
  \]
  where $z \sim \cN(0, \Sigma)$ and $u \sim \cN(0, I_3)$.
  
  Applying this reparameterisation to our expectations allows us to remove the dependency of distribution on the inputs,
  \begin{align*}
    &\E_{z \sim \cN(0, \Gamma)} \left[ \psi(z_1) \varphi(z_0, z_2) \right]
    \\
    = &\E_{z \sim \cN(0, \Sigma)} \left[ \psi(z_1) \varphi(z_0, z_2) \right] 
    \\
    = &\E_{u \sim \cN(0, I_3)} \left[ \psi(H(u)) \varphi\left(\sqrt{\Sigma_{0, 0}} u_0, \frac{\Sigma_{0,2}}{\sqrt{\Sigma_{0,0}}} u_0 + \sqrt{\Sigma_{2,2} - \frac{\Sigma_{0,2}^2}{\Sigma_{0,0}} }u_2\right) \right].
  \end{align*}

  Taking the derivative w.r.t. $x_{\beta-1}'$, gives
  \begin{align*}
    &\frac{\partial}{\partial x_{\beta-1}'} \E_{z \sim \cN(0, \Gamma)} \left[ \psi(z_1) \varphi(z_0, z_2) \right] 
    \\
    &= {}\frac{\partial}{\partial x_{\beta-1}'} \E_{u \sim \cN(0, I_3)} \left[ \psi(H(u)) \varphi\left(\sqrt{\Sigma_{0,0}} u_0, \frac{\Sigma_{0,2}}{\sqrt{\Sigma_{0,0}}} u_0 + \sqrt{\Sigma_{2,2} - \frac{\Sigma_{0,2}^2}{\Sigma_{0,0}} }u_2\right) \right] 
    \\
    & = {} \E_{u \sim \cN(0, I_3)} \left[ \frac{\partial}{\partial x_{\beta-1}'} \left(\psi(H(u)) \varphi\left(\sqrt{\Sigma_{0,0}} u_0, \frac{\Sigma_{0,2}}{\sqrt{\Sigma_{0,0}}} u_0 + \sqrt{\Sigma_{2,2} - \frac{\Sigma_{0,2}^2}{\Sigma_{0,0}} }u_2\right)\right) \right] 
    \\
    & = {} \E_{u \sim \cN(0, I_4)} \left[ H'(u) \dot{\varphi}(H(u))\varphi\left(\sqrt{\Sigma_{0,0}} u_0, \frac{\Sigma_{0,2}}{\sqrt{\Sigma_{0,0}}} u_0 + \sqrt{\Sigma_{2,2} - \frac{\Sigma_{0,2}^2}{\Sigma_{0,0}} }u_2\right) \right],
  \end{align*}
  using the fact that $\Sigma_{0,0}, \Sigma_{0,2}, \Sigma_{2,2}$ do not involve $x'$, and $H'(u)$ is defined by differentiating $H(u)$ by $x_{\beta-1}'$ and adding additional $b_3 u_3$, giving
  \begin{align*}
    H'(u) &= b_0 u_0 + b_1 u_1 + b_2 u_2 + b_3 u_3,
    \\
    b_0 &= \frac{1}{\sqrt{\Sigma_{0,0}}} \frac{\partial \Sigma_{0,1}}{\partial x_{\beta-1}'}, 
    \\
    b_2 &= \frac{1}{\sqrt{\Sigma_{2,2} - \Sigma_{0,2}^2 / \Sigma_{0,0}}} \left( \frac{\partial \Sigma_{1,2}}{\partial x_{\beta-1}'} - \frac{\Sigma_{0,2}}{\Sigma_{0,0}} \frac{\partial \Sigma_{0,1}}{\partial x_{\beta-1}'} \right), 
    \\
    b_1 &= \frac{1}{2 a_2} \left( \frac{\partial \Sigma_{1,1}}{\partial x_{\beta-1}'} - 2 a_0 b_0 - 2 a_2 b_2 \right),
    \\
    b_3 &= \sqrt{\Gamma_{3,3} - b_0^2 - b_1^2 - b_2^2}.
  \end{align*}

  Now if $u \sim \cN(0, I_4)$, the 4-dimensional random vector
  \[\left( \sqrt{\Sigma_{0,0} u_0}, H(u), \frac{\Sigma_{0,2}}{\sqrt{\Sigma_{0,0}}} u_0 + \sqrt{\Sigma_{2,2} - \frac{\Sigma_{0,2}^2}{\Sigma_{0,0}}} u_2, H'(u) \right)\]
  has the distribution $\cN(0, \Gamma)$.
  This can be verified by computing the covariance,
  \begin{align*}
    \mathrm{Cov}(\sqrt{\Sigma_{0,0}} u_0, H'(u)) &= \frac{\Sigma_{0,1}}{\partial x_{\beta-1}'}, 
    \\
    \mathrm{Cov}(H(u), H'(u)) &= \frac{1}{2} \frac{\partial \Sigma_{1,1}}{\partial x_{\beta-1}'} = \frac{\partial K(x',y')}{\partial y'_{\beta-1}}\Big|_{y' = x'}, 
    \\
    \mathrm{Cov}\left(\frac{\Sigma_{0,2}}{\sqrt{\Sigma_{0,0}}} u_0 + \sqrt{\Sigma_{2,2} - \frac{\Sigma_{0,2}^2}{\Sigma_{0,0}} }u_2 , H'(u) \right) &= \frac{\partial \Sigma_{1,2}}{\partial x_{\beta-1}'},
    \\
    \mathrm{Var}(H'(u)) &= \frac{\partial^2 K(y,y')}{\partial y_{\beta-1} \partial y'_{\beta-1}}\Big|_{(y,y') = (x',x')},
  \end{align*}
  where the second equality comes from the symmetry of $K$:
  \[\frac{\partial}{\partial x_{\beta-1}'} \Sigma_{1,1} = \left( \frac{\partial}{\partial y_{\beta-1}} K(y, x') \big|_{y=x'} + \frac{\partial}{\partial y_{\beta-1}'} K(x', y') \big|_{y'=x'} \right) = 2 \frac{\partial}{\partial y_{\beta-1}'} K(x', y') \big|_{y'=x'}.\]
  The last equality uses the independence between $u_3$ and $u_0, u_1, u_2$ and the zero mean property of $u_3$.
  
  From what we have just shown, the final result follows:
  \begin{align*}
    &\frac{\partial}{\partial x'_{\beta-1}}\E_{z \sim \cN(0,\Sigma)}\left[\psi(z_1) \cdot \varphi(z_0,z_2)\right] 
    \\
    = &\E_{u \sim \cN(0, I_4)} \left[ \dot{\psi}(H(u)) H'(u) \varphi\left( \sqrt{\Sigma_{1,1} u_0}, \frac{\Sigma_{0,2}}{\sqrt{\Sigma_{0,0}}} u_0 + \sqrt{\Sigma_{2,2} - \frac{\Sigma_{0,2}^2}{\Sigma_{0,0}} }u_2 \right) \right]
    \\
    = &\E_{z \sim \cN(0, \Gamma)} \left[\dot{\psi}(z_1) \cdot z_3 \cdot \varphi(z_0,z_2)\right].
  \end{align*}
\end{proof}

Using Theorem~\ref{thm:exchange-theorem}, we prove Theorem~\ref{thm:JacobianNNGPCorrespondence}. 
Pick 
$\alpha,\beta \in [d_0]$, and $l \in [L] \cup \{0\}$. We have to show that for all $x,x' \in \R^{d_0}$,
\begin{align}
  \Sigma^{(l)}(x, x')_{\alpha0} &= \frac{\partial}{\partial x_{\alpha-1}}\Sigma^{(l)}(x, x')_{00}, \label{eqn:equivalence-two-JacobianNNGPsa}
  \\
  \Sigma^{(l)}(x, x')_{0\beta} &= \frac{\partial}{\partial x_{\beta-1}'}\Sigma^{(l)}(x, x')_{00}, \label{eqn:equivalence-two-JacobianNNGPsb}
  \\
  \Sigma^{(l)}(x, x')_{\alpha\beta} &= \frac{\partial^2}{\partial x_{\alpha-1} \partial x_{\beta-1}'}\Sigma^{(l)}(x, x')_{00}. \label{eqn:equivalence-two-JacobianNNGPsab}
\end{align}

The proof is given by induction on $l$. When $l = 0$, all the three desired equations hold for all $x,x' \in \R^{d_0}$, since differentiation of $\langle x, x'\rangle$ gives LHS.
For the induction step, assume that the three equations \eqref{eqn:equivalence-two-JacobianNNGPsa}, \eqref{eqn:equivalence-two-JacobianNNGPsb}, and \eqref{eqn:equivalence-two-JacobianNNGPsab} hold for all $x,x' \in \R^{d_0}$ for the $l$-th layer.
We need to show that the equations also hold for all $x,x' \in \R^{d_0}$ for the $(l+1)$-th layer. 
First, we prove that the equation \eqref{eqn:equivalence-two-JacobianNNGPsb} holds for all $x,x' \in \R^{d_0}$; the equation \eqref{eqn:equivalence-two-JacobianNNGPsa} can be proved similarly. 
We define the $4\times 4$ matrix $\Gamma$ as in Theorem~\ref{thm:exchange-theorem}, with the kernel $(x, x') \mapsto \Sigma^{(l)}(x, x')_{00}$, where following equality can be proven with the induction hypothesis,
\[
  \Gamma = \begin{pmatrix}
    \Sigma^{(l)}(x, x)_{00} & \Sigma^{(l)}(x, x')_{00} & \Sigma^{(l)}(x, x)_{0\alpha} & \Sigma^{(l)}(x, x')_{0\beta} \\
    \Sigma^{(l)}(x, x')_{00} & \Sigma^{(l)}(x', x')_{00} & \Sigma^{(l)}(x, x')_{\alpha 0} & \Sigma^{(l)}(x', x')_{0\beta} \\
    \Sigma^{(l)}(x, x)_{0\alpha} & \Sigma^{(l)}(x, x')_{\alpha 0} & \Sigma^{(l)}(x, x)_{\alpha \alpha} & \Sigma^{(l)}(x, x')_{\alpha \beta} \\
    \Sigma^{(l)}(x, x')_{0\beta} & \Sigma^{(l)}(x', x')_{0\beta} & \Sigma^{(l)}(x, x')_{\alpha \beta} & \Sigma^{(l)}(x', x')_{\beta\beta}
  \end{pmatrix}.
\]
Using this equality and Theorem~\ref{thm:exchange-theorem} with $\psi(z_1) = \phi(z_1)$ and $\varphi(z_0, z_2) = \phi(z_0)$,
\begin{align*}
  \Sigma^{(l+1)}(x, x')_{0 \beta} 
  &= \E_g \left[\phi (g(x)) \dot{\phi}(g(x')) J(g)(x')_{\beta-1} \right] 
  \\
  &= \E_{u \sim \cN(0, \Gamma)}\left[ \phi(u_0) \dot{\phi}(u_1) u_3 \right]
  \\
  &= \frac{\partial}{\partial x_{\beta-1}'}\E_{u \sim \cN(0, \Gamma)} \left[\phi(u_0) \phi(u_1)\right]
  \\
  &= \frac{\partial}{\partial x_{\beta-1}'} \E_{g} \left[ \phi(g(x)) \phi(g(x')) \right]
  \\
  &= \frac{\partial}{\partial x_{\beta-1}'} \Sigma^{(l+1)}(x, x')_{00}
\end{align*}
We can prove the equation \eqref{eqn:equivalence-two-JacobianNNGPsab} with a similar method.
This time, we define the matrix $\Xi$ which is a matrix defined by permutation of the rows and columns of $\Gamma$ by the bijection that maps $(1, 2, 3, 4)$ to $(2, 1, 4, 3)$, so that the following holds:
\[
  (u_0, u_1, u_2, u_3) \stackrel{dist}{=} (w_1, w_0, w_3, w_2)
\] 
where $u \sim \cN(0, \Gamma)$ and $w \sim \cN(0, \Xi)$.
Then with Theorem~\ref{thm:exchange-theorem} with the choice of $\psi(z_1) = \phi(z_1)$ and $\varphi(z_0, z_2) = \dot{\phi}(z_0) z_2$ and the second equality \eqref{eqn:equivalence-two-JacobianNNGPsb} just proven, the third equality can be shown as
\begin{align*}
  \Sigma^{(l+1)}(x, x')_{\alpha \beta} 
  &= \E_g \left[ \dot{\phi}(g(x)) \dot{g}(x) \dot{\phi}(g(x')) \dot{g}(x') \right]
  \\
  &= \E_{w \sim \cN(0, \Xi)} \left[ \dot{\phi}(w_1) w_3 \dot{\phi}(w_0) w_2 \right]
  \\
  &= \frac{\partial}{\partial x_{\alpha-1}} \E_{w \sim \cN(0, \Xi)} \left[ \phi(w_1) \dot{\phi}(w_0) w_2 \right]
  \\
  &= \frac{\partial}{\partial x_{\alpha-1}} \E_g \left[ \phi(g(x)) \dot{\phi}(g(x')) \dot{g}(x') \right]
  \\
  &= \frac{\partial}{\partial x_{\alpha-1}} \Sigma^{(l+1)}(x, x')_{0\beta}
  \\
  &= \frac{\partial^2}{\partial x_{\alpha-1} \partial x_{\beta-1}'} \Sigma^{(l+1)}(x, x')_{\alpha \beta}.
\end{align*}

\section{Remark on Assumption~\ref{assm:NTK-smallest-eigenvalue}}
In this section, we illustrate the difficulty of finding a sufficient condition for Assumption~\ref{assm:NTK-smallest-eigenvalue} by proving that one of the standard conditions for the positive definiteness of the NTK kernel does not work for the JNTK kernel. We provide a counterexample for the condition.
To simplify the presentation, we do not impose the assumptions in Assumption~\ref{assm:activation-assumption}.
The result for the normalised case can be obtained by scaling every matrix by $\frac{1}{3}$.

We first review two results that guarantee the positive definiteness of the NTK kernel.
\begin{theorem}[Proposition F.1 of \cite{du2019}]
  Suppose activation is analytic and non-polynomial.
  Then, the standard NTK kernel
  \[
    \Theta(x^{(1:N)}, x^{(1:N)})_{00}
  \]
  has positive eigenvalues as long as no two data satisfy $x^{(i)} = c x^{(j)}$ for some $c \in \mathbb{R}$.
\end{theorem}

\begin{theorem}[Adapted from Theorem 3.2 of \cite{nguyen2021}] \label{thm:hermitedefinite}
  Suppose the Hermite expansion of activation
  \[
    \sigma(x) = \frac{1}{\sqrt{2\pi}}\sum_{r=0}^\infty \mu_r(\sigma) H_r(x) \exp(-x^2/2)
  \]
  has infinitely many nonzero $\mu_r(\sigma)$.
  Then, the standard NTK kernel
  \[
    \Theta(x^{(1:N)}, x^{(1:N)})_{00}
  \]
  has positive eigenvalues as long as no two data satisfy $x^{(i)} = c x^{(j)}$ for some $c \in \mathbb{R}$.
\end{theorem}
\begin{proof}
  We use the following inequality stated in the proof of Theorem 3.2 of \cite{nguyen2021}:
  \[
    \lambda_{\min}(\Theta(x^{(1:N)}, x^{(1:N)})_{00}) \ge \mu_r(\sigma)^2\frac{\min_{i \in [N]} \|x^{(i)}\|^{2r} - (N-1) \max_{i \neq j} |\langle x^{(i)}, x^{(j)} \rangle|^r}{\max_{i \in [N]} \|x^{(i)}\|^{2(r-1)}}.
  \]
  Then, we can choose $r$ large enough such that $\mu_r(\sigma) \neq 0$ and $\max_{i \neq j} |\langle x^{(i)}, x^{(j)} \rangle|^r < \frac{1}{N-1}$, which guarantees the positive definiteness of the NTK kernel.
\end{proof}

Our counterexample is related to Theorem~\ref{thm:hermitedefinite}. Consider neural networks with the activation function $\sigma(x) = x^2$, which is a polynomial function.
This activation satisfies the condition of Theorem~\ref{thm:hermitedefinite}, which requires that the Hermite coefficients of the activation be non-zero for infinitely many $r$.
It is because otherwise, the summation would include $\exp(-x^2/2)$ term, which is not a polynomial.

\begin{theorem}
  For any input dimension $d_0$ and square activation $\sigma(x) = x^2$, there exists dataset $\{x^{(i)}\}_{i=1}^N$ such that
  \[
    \lambda_{\min}\left( \Theta(x^{(1:N)}, x^{(1:N)}) \right) = 0
  \]
  but
  \[
    \lambda_{\min}\left( \Theta(x^{(1:N)}, x^{(1:N)})_{00} \right) > 0
  \]
  for shallow neural networks, i.e., $L = 1$.
\end{theorem}
\begin{proof}
  We first give the analytic form of the JNTK kernel in this case, which can be computed via Isserlis' theorem.
  \begin{align*}
    \Theta(x, y) &= \Sigma(x, y) + \dot{\Sigma}(x, y),
    \\
    \Sigma(x, y)_{00} &= \mathbb{E}_{w \sim \mathcal{N}(0, I_d)} \left[ \langle w, x \rangle^2 \langle w, y \rangle^2 \right]
    \\
    &= 2 \Cov(\langle w, x \rangle, \langle w, y \rangle)^2 + \Var(\langle w, x \rangle) \Var(\langle w, y \rangle)
    \\
    &= 2 \langle x, y \rangle^2 + \|x\|^2 \|y\|^2,
    \\
    \Sigma(x, y)_{\alpha 0} &= \frac{\partial}{\partial x_{\alpha-1}} \Sigma(x, y)_{00}
    \\
    &= 4 \langle x, y \rangle y_{\alpha-1} + 2 x_{\alpha-1} \|y\|^2,
    \\
    \Sigma(x, y)_{0 \beta} &= \frac{\partial}{\partial y_{\beta-1}} \Sigma(x, y)_{00}
    \\
    &= 4 \langle x, y \rangle x_{\beta-1} + 2 y_{\beta-1} \|x\|^2,
    \\
    \Sigma(x, y)_{\alpha \beta} &= \frac{\partial^2}{\partial x_{\alpha-1} \partial y_{\beta-1}} \Sigma(x, y)_{00}
    \\
    &= 4 x_{\beta - 1} y_{\alpha-1} + 4 x_{\alpha-1} y_{\beta-1} + 4 \langle x, y \rangle \mathbb{I}_{\alpha = \beta},
    \\
    \dot{\Sigma}(x, y)_{00} &= \mathbb{E}_{w \sim \mathcal{N}(0, I_d)} \left[ 4 \langle w, x \rangle \langle w, y \rangle \right]
    \\
    &= 4 \langle x, y \rangle,
    \\
    \dot{\Sigma}(x, y)_{\alpha 0} &= \frac{\partial}{\partial x_{\alpha-1}} \dot{\Sigma}(x, y)_{00}
    \\
    &= 4 y_{\alpha-1},
    \\
    \dot{\Sigma}(x, y)_{0 \beta} &= \frac{\partial}{\partial y_{\beta-1}} \dot{\Sigma}(x, y)_{00}
    \\
    &= 4 x_{\beta-1},
    \\
    \dot{\Sigma}(x, y)_{\alpha \beta} &= \frac{\partial^2}{\partial x_{\alpha-1} \partial y_{\beta-1}} \dot{\Sigma}(x, y)_{00}
    \\
    &= 4 \mathbb{I}_{\alpha = \beta}.
  \end{align*}
  In matrix form, we have
  \begin{align*}
    \Sigma(x, y) &= \begin{pmatrix}
      2 \langle x, y \rangle^2 + \|x\|^2 \|y\|^2 & 4 \langle x, y \rangle x^\intercal + 2 \|x\|^2 y^\intercal
      \\
      4 \langle x, y \rangle y + 2 \|y\|^2 x & 4 x y^\intercal + 4 y x^\intercal + 4 \langle x, y\rangle I_d
    \end{pmatrix},
    \\
    \dot{\Sigma}(x, y) &= \begin{pmatrix}
      4 \langle x, y \rangle & 4x^\intercal
      \\
      4y & 4 I_d
    \end{pmatrix}.
  \end{align*}
  We will show the result by showing
  \[
    \mathrm{rank}\ \Theta(x^{(1:N)}, x^{(1:N)}) \le \mathrm{rank}\ \Sigma(x^{(1:N)}, x^{(1:N)}) + \mathrm{rank}\ \dot{\Sigma}(x^{(1:N)}, x^{(1:N)}) < N (d_0 + 1).
  \]
  We can bound the rank of individual matrices as
  \begin{align*}
    \mathrm{rank}\ \Sigma(x^{(1:N)}, x^{(1:N)}) &\le N (d_0 + 1) - 2N + 1,
    \\
    \mathrm{rank}\ \dot{\Sigma}(x^{(1:N)}, x^{(1:N)}) &\le d_0,
  \end{align*}
  for some $x^{(1:N)}$ which will be proven in the next two lemmas that follow.
  So, setting $N > (d_0 + 1) / 2$ shows that $\Theta(x^{(1:N)}, x^{(1:N)})$ is singular, therefore $\lambda_{\min}(\Theta(x^{(1:N)}, x^{(1:N)})) = 0$.
\end{proof}

\begin{lemma} \label{lem:nngprank}
  For any $x^{(1:N)}$ such that no $x^{(i)}$ is zero vector, we have at least $2 N - 1$ linearly independent vectors that
  \[
    \Sigma(x^{(1:N)}, x^{(1:N)}) v = 0,
  \]
  so that $\mathrm{rank}\ \Sigma(x^{(1:N)}, x^{(1:N)}) \le N (d_0 + 1) - N$.
\end{lemma}
\begin{proof}
  Let's first define the following vectors:
  \[
    v_i = (\underbrace{0, \ldots, 0}_{(i - 1) \cdot (d_0 + 1)}, -2, x_{0}^{(i)}, \ldots, x_{d_0-1}^{(i)}, \underbrace{0, \ldots, 0}_{(N - i) \cdot (d_0 + 1)})^\intercal
  \]
  for $i = 1, \ldots, N$.
  Then these vectors are linearly independent, and for $i \in [N]$,
  \begin{align*}
    \left(\Sigma(x^{(1:N)}, x^{(1:N)}) v_j \right)_{i(d_0 + 1)} = &-4 \langle x^{(i)}, x^{(j)} \rangle^2 - 2 \left\|x^{(i)}\right\|^2 \left\|x^{(i)}, x^{(j)}\right\|^2
    \\
    &+ 4 \langle x^{(i)}, x^{(j)} \rangle \left( x^{(i)} \right)^\intercal x^{(j)} + 2 \left\|x^{(i)}\right\|^2 \left( x^{(j)} \right)^\intercal x^{(j)},
    \\
    \left(\Sigma(x^{(1:N)}, x^{(1:N)}) v_j\right)_{i(d_0+1)+\alpha} &= -8 \langle x^{(i)}, x^{(j)} \rangle x_{\alpha}^{(j)} - 4 \left\| x^{(j)} \right\|^2 x_{\alpha}^{(i)}
    \\
    &+ 4 x_{\alpha}^{(i)} \left( x^{(j)} \right)^\intercal x^{(j)} + 4 x_{\alpha}^{(j)} \left( x^{(i)} \right)^\intercal x^{(j)} + 4 \langle x^{(i)}, x^{(j)} \rangle x_{\alpha}^{(j)},
  \end{align*}
  showing that they all lies in the null space of $\Sigma(x^{(1:N)}, x^{(1:N)})$.

  We then define the following vectors:
  \[
    w_i = (0, -x_0^{(i)}, \ldots, -x_{d_0-1}^{(i)} \underbrace{0, \ldots, 0}_{(i - 2) \cdot (d_0 + 1)}, 0, x_{0}^{(1)}, \ldots, x_{d_0-1}^{(1)}, \underbrace{0, \ldots, 0}_{(N - i) \cdot (d_0 + 1)})^\intercal
  \]
  for $i = 2, \ldots, N$.
  
  Similarly for $i \in [N]$,
  \begin{align*}
    \left(\Sigma(x^{(1:N)}, x^{(1:N)}) w_j \right)_{i(d_0 + 1)} = &-4 \langle x^{(i)}, x^{(1)} \rangle \left(x^{(i)} \right)^\intercal x^{(j)} - 2 \left\|x^{(i)}\right\|^2 \left( x^{(1)} \right)^\intercal x^{(j)}
    \\
    &+ 4 \langle x^{(i)}, x^{(j)} \rangle \left(x^{(i)} \right)^\intercal x^{(1)} + 2 \left\| x^{(i)} \right\|^2 \left( x^{(j)}\right)^\intercal x^{(1)},
    \\
    \left(\Sigma(x^{(1:N)}, x^{(1:N)}) w_j \right)_{i(d_0+1)+\alpha} = &-4 x_\alpha^{(i)} \left( x^{(1)}\right)^\intercal x^{(j)} - 4 x_\alpha^{(1)} \left(x^{(i)}\right)^\intercal x^{(j)} - 4 \langle x^{(i)}, x^{(1)} \rangle x_\alpha^{(j)}
    \\ 
    &+ 4 x_\alpha^{(i)} \left( x^{(j)} \right)^\intercal x^{(1)} + 4 x_\alpha^{(j)} \left(x^{(i)}\right)^\intercal x^{(1)} + 4 \langle x^{(i)}, x^{(j)} \rangle x_\alpha^{(1)}
  \end{align*}
  showing that these vectors also lie in the null space of $\Sigma(x^{(1:N)}, x^{(1:N)})$.

  It is clear that all vectors $\{v_i\}_{i=1}^N \cup \{w_j\}_{j=2}^N$ are linearly independent.
  These show that the null space is at least $N + N - 1$ dimension, therefore the rank of 
  \[
    \Sigma(x^{(1:N)}, x^{(1:N)}) \le N(d_0 + 1) - 2N + 1.
  \]
\end{proof}

\begin{lemma} \label{lem:dotnngprank}
  For any $x^{(1:N)}$, we have
  \[
    \mathrm{rank}\ \dot{\Sigma}(x^{(1:N)}, x^{(1:N)}) \le d_0.
  \]
\end{lemma}
\begin{proof}
  Instead of giving the eigendecomposition, we show that there exists $d_0$ vectors that 
  \[
    \dot{\Sigma}(x^{(1:N)}, x^{(1:N)}) = \sum_{\alpha=0}^{d_0-1} \lambda_\alpha v_\alpha v_\alpha^\intercal
  \]
  for $\lambda_\alpha \ge 0$ and $v_\alpha \in \mathbb{R}^{N(d_0 + 1)}$.
  
  We define these vectors as
  \begin{align*}
    v_\alpha &= (x_\alpha^{(1)}, \underbrace{0, \ldots, 0}_{\alpha}, 1, \underbrace{0, \ldots, 0}_{d_0 - \alpha - 1}, x_\alpha^{(2)}, \underbrace{0, \ldots, 0}_{\alpha}, 1, \underbrace{0, \ldots, 0}_{d_0 - \alpha - 1}, \ldots, x_\alpha^{(N)}, \underbrace{0, \ldots, 0}_{\alpha}, 1, \underbrace{0, \ldots, 0}_{d_0 - \alpha - 1})^\intercal
  \end{align*}
  Then, we have
  \begin{align*}
    \left( v_\gamma v_\gamma^\intercal \right)_{i(d_0+1), j(d_0 + 1)} &= x_\gamma^{(i)} x_\gamma^{(j)},
    \\
    \left( v_\gamma v_\gamma^\intercal \right)_{i(d_0+1)+\alpha, j(d_0+1)} &= x_\gamma^{(j)} \mathbb{I}_{\alpha = \gamma},
    \\
    \left( v_\gamma v_\gamma^\intercal \right)_{i(d_0+1), j(d_0+1)+\beta} &= x_\gamma^{(i)} \mathbb{I}_{\beta = \gamma},
    \\
    \left( v_\gamma v_\gamma^\intercal \right)_{i(d_0+1)+\alpha, j(d_0+1)+\beta} &= \mathbb{I}_{\alpha = \gamma} \mathbb{I}_{\beta = \gamma}
  \end{align*}
  which shows that taking $\lambda_\gamma = 4$ proves the identity.
\end{proof}

\begin{remark}
  We note that the bound of Lemma~\ref{lem:nngprank} is not tight.
  If we computationally test this rank condition, the rank of $\Sigma(x^{(1:N)}, x^{(1:N)})$ follows
  \[
    \mathrm{rank}\ \Sigma(x^{(1:N)}, x^{(1:N)}) = \sum_{i=0}^{N-1} \max(d_0 - i, 0) = \begin{cases}
      \frac{N(2d_0 - N + 1)}{2} &\text{if }d_0 > N
      \\
      \frac{d_0 (d_0 + 1)}{2} & \text{if }d_0 \le N
    \end{cases}
  \]
  for non-parallel inputs $x^{(1:N)}$.

  Conversely, the bound of Lemma~\ref{lem:dotnngprank} is tight, i.e.,
  \[
    \mathrm{rank}\ \dot{\Sigma}(x^{(1:N)}, x^{(1:N)}) = d_0
  \]
  for non-parallel inputs $x^{(1:N)}$.
\end{remark}

\section{Proof of Theorem~\ref{thm:JacobianNTKInitialisation}}
\label{sec:ProofJacobianNTKInitialisation}

Before giving the proof, we first state the complete form of $\Theta$ in Theorem~\ref{thm:JacobianNTKInitialisation}, which is the $1/\kappa^2$-scaled version of the limiting JNTK:
\begin{align*}
  \Theta(x, x')_{00} 
  &\defeq \sum_{l=0}^{L} \left( \Sigma^{(l)}(x, x')_{00} \: \prod_{u=l}^{L-1} \Gamma^{(u)}(x, x')_{00}\right),
  \\
  \Theta(x, x')_{\alpha 0} 
  &\defeq \sum_{l=0}^L \Bigg( \Sigma^{(l)}(x, x')_{\alpha 0} \: \prod_{u=l}^{L-1} \Gamma^{(u)}(x, x')_{00}
  \\
  &+ \Sigma^{(l)}(x, x')_{00} \: \sum_{u_\alpha=l}^{L-1} \Gamma^{(u_\alpha)}(x, x')_{\alpha 0} \: \prod_{\substack{u=l \\ u \neq u_\alpha}}^{L-1} \Gamma^{(u)}(x, x')_{00} \Bigg) 
  \\
  \Theta(x, x')_{0\beta} 
  &\defeq \sum_{l=0}^L \Bigg( \Sigma^{(l)}(x, x')_{0 \beta} \: \prod_{u=l}^{L-1} \Gamma^{(u)}(x, x')_{00}
  \\
  & + \Sigma^{(l)}(x, x')_{00} \: \sum_{u_\beta=l}^{L-1} \Gamma^{(u_\beta)}(x, x')_{0\beta} \:\prod_{\substack{u=l \\ u \neq u_\beta}}^{L-1} \Gamma^{(u)}(x, x')_{00} \Bigg)
  \\
  \Theta(x, x')_{\alpha \beta} 
  &\defeq \sum_{l=0}^L \Bigg( \Sigma^{(l)}(x, x')_{\alpha\beta} \: \prod_{u=l}^{L-1} \Gamma^{(u)}(x, x')_{00}
  \\
  & + \Sigma^{(l)}(x, x')_{\alpha 0} \: \sum_{u_\beta=l}^{L-1} \Gamma^{(u_\beta)}(x, x')_{0\beta}\: \prod_{\substack{u=l \\ u \neq u_\beta}}^{L-1} \Gamma^{(u)}(x, x')_{00}
  \\
  & + \Sigma^{(l)}(x, x')_{0 \beta} \: \sum_{u_\alpha=l}^{L-1} \Gamma^{(u_\alpha)}(x, x')_{\alpha 0}\: \prod_{\substack{u=l \\ u \neq u_\alpha}}^{L-1} \Gamma^{(u)}(x, x')_{00}
  \\
  & + \Sigma^{(l)}(x, x')_{00} \: \sum_{u_{\alpha\beta}=l}^{L-1} \Gamma^{(u_{\alpha\beta})}(x, x')_{\alpha \beta} \:\prod_{\substack{u=l\\u\neq u_{\alpha\beta}}}^{L-1} \Gamma^{(u)}(x, x')_{00}
  \\
  & + \Sigma^{(l)}(x, x')_{00} \: \sum_{\substack{u_\alpha, u_\beta = l\\ u_\alpha \neq u_\beta}}^{L-1} \Gamma^{(u_\alpha)}(x, x')_{\alpha 0} \: \Gamma^{(u_\beta)}(x, x')_{0\beta} \:\prod_{\substack{u=l\\u\neq u_{\alpha} \\ u\neq u_\beta}}^{L-1} \Gamma^{(u)}(x, x')_{00}
  \Bigg),
\end{align*}
where $\Gamma^{(u)} : \R^{d_0} \times \R^{d_0} \to \R^{(1 + d_0) \times (1 + d_0)}$ is defined as follows: for $g^{(u)} \sim \GP(0, \Sigma^{(u)})$,
\begin{align*}
  \Gamma^{(u)}(x, x')_{00} &\defeq \E_{g^{(u)}} \left[ \dot{\phi}(g^{(u)}(x)_0) \dot{\phi}(g^{(u)}(x')_0) \right],
  \\
  \Gamma^{(u)}(x, x')_{\alpha 0} &\defeq \E_{g^{(u)}} \left[ \ddot{\phi}(g^{(u)}(x)_0) \dot{\phi}(g^{(u)}(x')_0) g^{(u)}(x)_\alpha \right],
  \\
  \Gamma^{(u)}(x, x')_{0\beta} &\defeq \E_{g^{(u)}} \left[ \dot{\phi}(g^{(u)}(x)_0) \ddot{\phi}(g^{(u)}(x')_0) g^{(u)}(x')_\beta \right],
  \\
  \Gamma^{(u)}(x, x')_{\alpha \beta} &\defeq \E_{g^{(u)}} \left[ \ddot{\phi}(g^{(u)}(x)_0) \ddot{\phi}(g^{(u)}(x')_0) g^{(u)}(x)_\alpha g^{(u)}(x')_\beta \right].
\end{align*} 

We prove Theorem~\ref{thm:JacobianNTKInitialisation} by using the proof strategy for Theorem~\ref{thm:JacobianNNGPInitialisation} in Appendix~\ref{sec:ProofJacobianNNGPInitialisation} again. That is, we express the computation of backpropagated gradient as a tensor program, and apply Theorem~\ref{thm:MasterTheoremLLN} to show the claimed convergence in Theorem~\ref{thm:JacobianNTKInitialisation}.

Recall the definition of the finite JNTK: for $\alpha, \beta \in [d_0]$ and $x,x' \in \R^{d_0}$,
\begin{align*}
  \Theta_{d,\theta}(x,x')_{00} &= \left\langle \frac{\partial f(x)}{\partial \theta}, \frac{\partial f(x')}{\partial \theta} \right\rangle,
  \\
  \Theta_{d,\theta}(x,x')_{\alpha 0} &= \left\langle \frac{\partial J(f)(x)_{\alpha-1}}{\partial \theta}, \frac{\partial f(x')}{\partial \theta} \right\rangle,
  \\
  \Theta_{d,\theta}(x,x')_{0\beta} &= \left\langle \frac{\partial f(x)}{\partial \theta}, \frac{\partial J(f)(x')_{\beta-1}}{\partial \theta} \right\rangle,
  \\
  \Theta_{d,\theta}(x,x')_{\alpha\beta} &= \left\langle \frac{\partial J(f)(x)_{\alpha-1}}{\partial \theta}, \frac{\partial J(f)(x')_{\beta-1}}{\partial \theta} \right\rangle.
\end{align*}
We rephrase this finite JNTK such that the contribution of each layer is explicit: 
\begin{align}
\nonumber
  \Theta_{d,\theta}(x,x')_{00} &= \sum_{l=1}^{L+1} \left\langle \frac{\partial f(x)}{\partial W^{(l)}}, \frac{\partial f(x')}{\partial W^{(l)}} \right\rangle,
  \\
  \nonumber
  \Theta_{d,\theta}(x,x')_{\alpha 0} &= \sum_{l=1}^{L+1}\left\langle \frac{\partial J(f)(x)_{\alpha-1}}{\partial W^{(l)}}, \frac{\partial f(x')}{\partial W^{(l)}} \right\rangle,
  \\
  \nonumber
  \Theta_{d,\theta}(x,x')_{0\beta} &= \sum_{l=1}^{L+1}\left\langle \frac{\partial f(x)}{\partial W^{(l)}}, \frac{\partial J(f)(x')_{\beta-1}}{\partial W^{(l)}} \right\rangle,
  \\
  \label{eqn:JNTK-layerwise}
  \Theta_{d,\theta}(x,x')_{\alpha\beta} &= \sum_{l=1}^{L+1}\left\langle \frac{\partial J(f)(x)_{\alpha-1}}{\partial W^{(l)}}, \frac{\partial J(f)(x')_{\beta-1}}{\partial W^{(l)}} \right\rangle.
\end{align}
By the definitions of the MLP $f$ and its components $g^{(l)}$ and $h^{(l)}$ in Appendix~\ref{appendix:notation},
for every $l = 2, \ldots, L+1$, we have
\begin{align*}
  \frac{\partial f(x)}{\partial W^{(l)}} &= \frac{1}{\sqrt{d}} \frac{\partial f(x)}{\partial g^{(l)}(x)} \left(h^{(l-1)}(x)\right)^\intercal, 
  \\
  \frac{\partial J(f)(x)_{\alpha-1}}{\partial W^{(l)}} &= \frac{1}{\sqrt{d}} \frac{\partial J(f)(x)_{\alpha-1}}{\partial g^{(l)}(x)} \left(h^{(l-1)}(x)\right)^\intercal 
  + \frac{1}{\sqrt{d}}\frac{\partial J(f)(x)_{\alpha-1}}{\partial (\partial g^{(l)}(x) / \partial x_{\alpha-1})} \left(  \frac{\partial h^{(l-1)}(x)}{\partial x_{\alpha-1}}\right)^\intercal.
\end{align*}
The derivatives with respect to $W^{(1)}$ have similar forms shown below:
\begin{align*}
  \frac{\partial f(x)}{\partial W^{(1)}} & = \frac{\partial f(x)}{\partial g^{(1)}(x)} \left(x \right)^\intercal, \\
  \frac{\partial J(f)(x)}{\partial W^{(1)}} & = \frac{\partial J(f)(x)_{\alpha-1}}{\partial g^{(1)}(x)} \left(x\right)^\intercal + \frac{\partial J(f)(x)_{\alpha-1}}{\partial (\partial g^{(1)}(x) / \partial x_{\alpha-1})}\left( e_{\alpha-1} \right)^\intercal.
\end{align*}
We plug in these characterisations of the derivatives in the layer-wise description of the finite JNTK in \eqref{eqn:JNTK-layerwise},
and simplify the results using $\langle v_1 w_1^\intercal, v_2 w_2^\intercal \rangle = \langle v_1, v_2\rangle \cdot \langle w_1, w_2\rangle$. Recall that $C_l = d$ if $l = 2, \ldots, L+1$ and $C_1=1$.
\begin{align}
  &\Theta_{d,\theta}(x, x')_{00} 
  \\
  \quad &= \sum_{l=1}^{L+1} \frac{1}{C_{l-1}}\left\langle h^{(l-1)}(x), h^{(l-1)}(x') \right\rangle \left\langle \frac{\partial f(x)}{\partial g^{(l)}(x)}, \frac{\partial f(x')}{\partial g^{(l)}(x')} \right\rangle, \nonumber
  \\
  &\Theta_{d,\theta}(x, x')_{\alpha 0}
  \\
  \quad &= \sum_{l=1}^{L+1} \frac{1}{C_{l-1}}\left\langle h^{(l-1)}(x), h^{(l-1)}(x')\right\rangle\left\langle \frac{\partial J(f)(x)_{\alpha-1}}{\partial g^{(l)}(x)} , \frac{\partial f(x')}{\partial g^{(l)}(x')} \right\rangle \nonumber
  \\
  \quad &+ \sum_{l=1}^{L+1} \frac{1}{C_{l-1}} \left\langle \frac{\partial h^{(l-1)}(x)}{\partial x_{\alpha-1}}, h^{(l-1)}(x') \right\rangle \left\langle  \frac{\partial J(f)(x)_{\alpha-1}}{\partial (\partial g^{(l)}(x) / \partial x_{\alpha-1})}, \frac{\partial f(x')}{\partial g^{(l)}(x')}  \nonumber\right\rangle,
  \\
  &\Theta_{d,\theta}(x, x')_{0 \beta}
  \\
  &= \sum_{l=1}^{L+1} \frac{1}{C_{l-1}}\left\langle h^{(l-1)}(x), h^{(l-1)}(x')\right\rangle\left\langle \frac{\partial f(x)_{\alpha-1}}{\partial g^{(l)}(x)} , \frac{\partial J(f)(x')_{\beta-1}}{\partial g^{(l)}(x')} \right\rangle \nonumber
  \\
  \quad &+ \sum_{l=1}^{L+1} \frac{1}{C_{l-1}} \left\langle h^{(l-1)}(x), \frac{\partial h^{(l-1)}(x')}{\partial x_{\beta-1}'} \right\rangle \left\langle  \frac{\partial f(x)}{\partial g^{(l)}(x)}, \frac{\partial J(f)(x')_{\beta-1}}{\partial (\partial g^{(l)}(x') / \partial x_{\beta-1}')} \right\rangle, \nonumber
  \\
  &\Theta_{d,\theta}(x, x')_{\alpha \beta}
  \\
  \quad &= \sum_{l=1}^{L+1} \frac{1}{C_{l-1}} \left\langle h^{(l-1)}(x), h^{(l-1)}(x') \right\rangle \left\langle \frac{\partial J(f)(x)_{\alpha-1}}{\partial g^{(l)}(x)}, \frac{\partial J(f)(x')_{\beta-1}}{\partial g^{(l)}(x')} \right\rangle \nonumber
  \\
  \quad &+ \sum_{l=1}^{L+1} \frac{1}{C_{l-1}} \left \langle \frac{\partial h^{(l-1)}(x)}{\partial x_{\alpha-1}},h^{(l-1)}(x') \right \rangle \left \langle\frac{\partial J(f)(x)_{\alpha-1}}{\partial (\partial g^{(l)}(x) / \partial x_{\alpha-1})}, \frac{\partial J(f)(x')_{\beta-1}}{\partial g^{(l)}(x')}\right \rangle \nonumber
  \\
  \quad &+ \sum_{l=1}^{L+1} \frac{1}{C_{l-1}} \left \langle h^{(l-1)}(x), \frac{\partial h^{(l-1)}(x')}{\partial x_{\beta-1}'}\right \rangle \left \langle\frac{\partial J(f)(x)_{\alpha-1}}{\partial g^{(l)}(x)}, \frac{\partial J(f)(x')_{\beta-1}}{\partial (\partial g^{(l)}(x') / \partial x_{\beta-1}')}\right \rangle \nonumber
  \\
  \quad &+ \sum_{l=1}^{L+1} \frac{1}{C_{l-1}} \left \langle\frac{\partial h^{(l-1)}(x)}{\partial x_{\alpha-1}}, \frac{\partial h^{(l-1)}(x')}{\partial x_{\beta-1}'}\right \rangle \left \langle\frac{\partial J(f)(x)_{\alpha-1}}{\partial (\partial g^{(l)}(x) / \partial x_{\alpha-1})}, \frac{\partial J(f)(x')_{\beta-1}}{\partial (\partial g^{(l)}(x') / \partial x_{\beta-1}')}\right \rangle. \label{eqn:JNTKLayerWise}
\end{align}
So it is enough to show that all the inner products here converge to constants. Then, we have the convergence of the finite JNTK at initialisation. The tensor program in Appendix~\ref{sec:ProofJacobianNNGPInitialisation} already computes both components
of the first inner product of each summand. Applying the Master theorem to these inner products normalised by $C_{l-1}$ gives the almost-sure convergence of these normalised inner products  to $\Sigma^{(l-1)}(x, x')_{00}, \Sigma^{(l-1)}(x, x')_{\alpha 0}, \Sigma^{(l-1)}(x, x')_{0\beta}, \Sigma^{(l-1)}(x, x')_{\alpha \beta}$. Now it remains to show the convergence of the second inner product of each summand. We do this by expressing the backpropagated gradients in a tensor program.

We first include a new random vector for the initialisation, $V^{(L+1)} \in \R^d$, which is independent of all other vectors, and its entries are i.i.d. $\cN(0, \kappa^2)$.
The body of this program first includes the variables that appeared in Section~\ref{sec:ProofJacobianNNGPInitialisation}, and additionally the following variables:
\begin{align*}
  D h_1^{(1)}, \ldots, D h_N^{(1)}, \ldots, D h_1^{(L)}, \ldots, D h_N^{(L)} \in \R^d,
  \\
  D g_1^{(1)}, \ldots, D g_N^{(1)}, \ldots, D g_1^{(L)}, \ldots, D g_N^{(L)} \in \R^d,
  \\
  D_\alpha J_\alpha h_1^{(1)}, \ldots, D_\alpha J_\alpha h_N^{(1)}, \ldots, D_\alpha J_\alpha h_1^{(L)}, \ldots, D_\alpha J_\alpha h_N^{(L)} \in \R^d,
  \\
  D_\alpha J_\alpha g_1^{(1)}, \ldots, D_\alpha J_\alpha g_N^{(1)}, \ldots, D_\alpha J_\alpha g_1^{(L)}, \ldots, D_\alpha J_\alpha g_N^{(L)} \in \R^d,
  \\
  D_\alpha h_1^{(1)}, \ldots, D_\alpha h_N^{(1)}, \ldots, D_\alpha h_1^{(L)}, \ldots, D_\alpha h_N^{(L)} \in \R^d,
  \\
  D_\alpha g_1^{(1)}, \ldots, D_\alpha g_N^{(1)}, \ldots, D_\alpha g_1^{(L)}, \ldots, D_\alpha g_N^{(L)} \in \R^d.
\end{align*}
These variables are defined as follows:
\begin{align*}
  D h_1^{(L)} &\defeq V^{(L+1)}, & &\cdots & D h_N^{(L)} &\defeq V^{(L+1)},
  \\
  D g_1^{(L)} &\defeq \dot{\phi} \left(g_1^{(L)}\right) \odot D h_1^{(L)}, & &\cdots & D g_N^{(L)} &\defeq \dot{\phi} \left(g_N^{(L)}\right) \odot D h_N^{(L)},
  \\
  D h_1^{(L-1)} &\defeq \left( V^{(L)} \right)^\intercal D g_1^{(L)}, & &\cdots & D h_N^{(L-1)} &\defeq \left( V^{(L)} \right)^\intercal D g_N^{(L)},
  \\
  &\vdots & & \vdots &\vdots
  \\
  D g_1^{(1)} &\defeq \dot{\phi} \left( g_1^{(1)} \right) \odot D h_1^{(1)}, & &\cdots & D g_N^{(1)} &\defeq \dot{\phi} \left( g_N^{(1)} \right) \odot D h_N^{(1)}. 
\end{align*}

These random vectors correspond to the backpropagated gradient, but we need the scaling as
\begin{align*}
  D h_i^{(l)} &= \sqrt{d} \frac{\partial f(x^{(i)})}{\partial h^{(l)}(x^{(i)})}, & D g_i^{(l)} &= \sqrt{d} \frac{\partial f(x^{(i)})}{\partial g^{(l)}(x^{(i)})}
\end{align*}
where this $\sqrt{d}$ scaling comes from the scaling at output $\frac{1}{\sqrt{d}} \left( V^{(L+1)} \right)^\intercal h_i^{(L)}$.

Similarly, we define the backpropagation of the Jacobian gradient, defined as
\begin{align*}
  D_\alpha J_\alpha h_1^{(L)} &\defeq V^{(L+1)}, & &\cdots & D_\alpha J_\alpha h_N^{(L)} &\defeq V^{(L+1)},
  \\
  D_\alpha J_\alpha g_1^{(L)} &\defeq \dot{\phi} \left(g_1^{(L)}\right) \odot D_\alpha J_\alpha h_1^{(L)}, & &\cdots & D_\alpha J_\alpha g_N^{(L)} &\defeq \dot{\phi} \left(g_N^{(L)}\right) \odot D_\alpha J_\alpha h_N^{(L)},
  \\
  D_\alpha J_\alpha h_1^{(L-1)} &\defeq \left( V^{(L)} \right)^\intercal D_\alpha J_\alpha g_1^{(L)}, & &\cdots & D_\alpha J_\alpha h_N^{(L-1)} &\defeq \left( V^{(L)} \right)^\intercal D_\alpha J_\alpha g_N^{(L)},
  \\
  &\vdots & & \vdots &\vdots
  \\
  D_\alpha J_\alpha g_1^{(1)} &\defeq \dot{\phi} \left( g_1^{(1)} \right) \odot D_\alpha J_\alpha h_1^{(1)}, & &\cdots & D_\alpha J_\alpha g_N^{(1)} &\defeq \dot{\phi} \left( g_N^{(1)} \right) \odot D_\alpha J_\alpha h_N^{(1)},
\end{align*}
and
\begin{align*}
  D_\alpha h_1^{(L)} &\defeq \mathbf{0}, 
  \\
  &\vdots
  \\
  D_\alpha h_N^{(L)} &\defeq \mathbf{0},
  \\
  D_\alpha g_1^{(L)} &\defeq \ddot{\phi}\left( g_1^{(L)} \right) \odot J_\alpha g_1^{(L)} \odot D_\alpha J_\alpha h_1^{(L)} + \dot{\phi} \left( g_1^{(L)} \right) \odot D_\alpha h_1^{(L)}, 
  \\
  &\vdots 
  \\
  D_\alpha g_N^{(L)} &\defeq \ddot{\phi}\left( g_N^{(L)} \right) \odot J_\alpha g_N^{(L)} \odot D_\alpha J_\alpha h_N^{(L)} + \dot{\phi} \left( g_N^{(L)} \right) \odot D_\alpha h_N^{(L)},
  \\
  D_\alpha h_1^{(L-1)} &\defeq \left( V^{(L)} \right)^\intercal D_\alpha g_1^{(L)},
  \\
  &\vdots
  \\
  D_\alpha h_N^{(L-1)} &\defeq \left( V^{(L)} \right)^\intercal D_\alpha g_N^{(L)},
  \\
  &\vdots
  \\
  D_\alpha g_1^{(1)} &\defeq \ddot{\phi}\left( g_1^{(1)} \right) \odot J_\alpha g_1^{(1)} \odot D_\alpha J_\alpha h_1^{(1)} + \dot{\phi} \left( g_1^{(1)} \right) \odot D_\alpha h_1^{(1)}, 
  \\
  &\vdots 
  \\
  D_\alpha g_N^{(1)} &\defeq \ddot{\phi}\left( g_N^{(1)} \right) \odot J_\alpha g_N^{(1)} \odot D_\alpha J_\alpha h_N^{(1)} + \dot{\phi} \left( g_N^{(1)} \right) \odot D_\alpha h_N^{(1)}.
\end{align*}
Again, these random vectors correspond to the backpropagated gradient with scaling, as
\begin{align*}
  D_\alpha J_\alpha g_i^{(l)} &= \sqrt{d}\frac{\partial J(f)(x^{(i)})_{\alpha-1}}{\partial \left( \frac{\partial g^{(l)}(x^{(i)})}{\partial x_{\alpha-1}^{(i)}} \right)}, & D_\alpha J_\alpha h_i^{(l)} &= \sqrt{d}\frac{\partial J(f)(x^{(i)})_{\alpha-1}}{\partial \left( \frac{\partial h^{(l)}(x^{(i)})}{\partial x_{\alpha-1}^{(i)}} \right)},
  \\
  D_\alpha g_i^{(l)} &= \sqrt{d}\frac{\partial J(f)(x^{(i)})_{\alpha-1}}{\partial g^{(l)}(x^{(i)})}, & D_\alpha h_i^{(l)} &= \sqrt{d}\frac{\partial J(f)(x^{(i)})_{\alpha-1}}{\partial h^{(l)}(x^{(i)})}.
\end{align*}
\update{To apply the Master theorem, we need to check that (1) the nonlinear functions we applied are polynomially bounded and (2) the program satisfies the BP-likeness property.
We have three non-linear functions to check for polynomial boundedness: $\phi$, $\dot{\phi}(a) \cdot b$, and $\ddot{\phi}(a) \cdot b \cdot c + \dot{\phi}(a) + d$.
All of these functions are polynomially bounded, since they are linear combinations of products of $\phi, \dot{\phi}, \ddot{\phi}$ to the linear functions, where $\phi, \dot{\phi}, \ddot{\phi}$ is Lipschitz by assumption. 
For the BP-likeness, the simple GIA check (Condition 1 of \cite{Yang2020a}) again applies to our program, since $V^{(L+1)}$ is not used in any other part of the program, except when we compute the output.}{Add explanation.}

The outputs of the program whose convergence is guaranteed by Theorem~\ref{thm:MasterTheoremLLN} are followings: for $i, j \in [N]$, $\alpha, \beta \in [d_0]$ and $l \in [N]$,
\begin{align}
  \left\langle \frac{\partial f(x^{(i)})}{\partial g^{(l)}(x^{(i)})}, \frac{\partial f(x^{(j)})}{\partial g^{(l)}(x^{(j)})} \right\rangle &\to \E \left[ Z^{D g_i^{(l)}} Z^{D g_j^{(l)}} \right], \label{eqn:LLNDgDg}
  \\
  \left\langle \frac{\partial J(f)(x^{(i)})_{\alpha-1}}{\partial g^{(l)}(x^{(i)})} , \frac{\partial f(x^{(j)})}{\partial g^{(l)}(x^{(j)})} \right\rangle &\to \E \left[ Z^{D_\alpha g_i^{(l)}} Z^{D g_j^{(l)}} \right], \label{eqn:LLNDagDg}
  \\
  \left\langle  \frac{\partial J(f)(x^{(i)})_{\alpha-1}}{\partial (\partial g^{(l)}(x^{(i)}) / \partial x_{\alpha-1}^{(i)})}, \frac{\partial f(x^{(j)})}{\partial g^{(l)}(x^{(j)})} \right\rangle &\to \E \left[ Z^{D_\alpha J_\alpha g_i^{(l)}} Z^{g_j^{(l)}} \right], \label{eqn:LLNDaJagDg}
  \\
  \left\langle \frac{\partial f(x^{(i)})_{\alpha-1}}{\partial g^{(l)}(x^{(i)})} , \frac{\partial J(f)(x^{(j)})_{\beta-1}}{\partial g^{(l)}(x^{(j)})} \right\rangle &\to \E \left[ Z^{D g_i^{(l)}} Z^{D_\beta g_j^{(l)}} \right], \label{eqn:LLNDgDbg}
  \\
  \left\langle  \frac{\partial f(x^{(i)})}{\partial g^{(l)}(x^{(i)})}, \frac{\partial J(f)(x^{(j)})_{\beta-1}}{\partial (\partial g^{(l)}(x^{(j)}) / \partial x_{\beta-1}^{(j)})} \right\rangle &\to \E \left[ Z^{D g_i^{(l)}} Z^{D_\beta J_\beta g_j^{(l)}} \right], \label{eqn:LLNDgDbJbg}
  \\
  \left\langle \frac{\partial J(f)(x^{(i)})_{\alpha-1}}{\partial g^{(l)}(x^{(i)})}, \frac{\partial J(f)(x^{(j)})_{\beta-1}}{\partial g^{(l)}(x^{(j)})} \right\rangle &\to \E \left[ Z^{D_\alpha g_i^{(l)}} Z^{D_\beta g_j^{(l)}} \right], \label{eqn:LLNDagDbg}
  \\
  \left \langle\frac{\partial J(f)(x^{(i)})_{\alpha-1}}{\partial (\partial g^{(l)}(x^{(i)}) / \partial x_{\alpha-1}^{(i)})}, \frac{\partial J(f)(x^{(j)})_{\beta-1}}{\partial g^{(l)}(x^{(j)})}\right \rangle &\to \E \left[ Z^{D_\alpha J_\alpha g_i^{(l)}} Z^{D_\beta g_j^{(l)}} \right], \label{eqn:LLNDaJagDbg}
  \\
  \left \langle\frac{\partial J(f)(x^{(i)})_{\alpha-1}}{\partial g^{(l)}(x^{(i)})}, \frac{\partial J(f)(x^{(j)})_{\beta-1}}{\partial (\partial g^{(l)}(x^{(j)}) / \partial x_{\beta-1}^{(j)})}\right \rangle &\to \E \left[ Z^{D_\alpha g_i^{(l)}} Z^{D_\beta J_\beta g_j^{(l)}} \right], \label{eqn:LLNDagDbJbg}
  \\
  \left \langle\frac{\partial J(f)(x^{(i)})_{\alpha-1}}{\partial (\partial g^{(l)}(x^{(i)}) / \partial x_{\alpha-1}^{(i)})}, \frac{\partial J(f)(x^{(j)})_{\beta-1}}{\partial (\partial g^{(l)}(x^{(j)}) / \partial x_{\beta-1}^{(j)})}\right \rangle &\to \E \left[ Z^{D_\alpha J_\alpha g_i^{(l)}} Z^{D_\beta J_\beta g_j^{(l)}} \right]. \label{eqn:LLNDaJagDbJbg}
\end{align}
Our final computation is evaluating these nine expectations, which can be done recursively.
Before proceeding further, let's focus on two variables, $D_\alpha J_\alpha g_i^{(l)}$ and $D g_i^{(l)}$.
Their initialisation $D_\alpha J_\alpha h_i^{(L)}$ and $D h_i^{(L)}$ are the same, and the recursive definition also matches, so they are equal always.
This implies that
\begin{align*}
  D g_i^{(l)} &= D_\alpha J_\alpha g_i^{(l)}, & D g_j^{(l)} &= D_\beta J_\beta g_j^{(l)},
  \\
  D h_i^{(l)} &= D_\alpha J_\alpha h_i^{(l)}, & D h_j^{(l)} &= D_\beta J_\beta h_j^{(l)}.
\end{align*}
This implies that most of the expectations are equal, thereby we only need to compute three expectations, \eqref{eqn:LLNDgDg}, \eqref{eqn:LLNDagDg}, and \eqref{eqn:LLNDagDbg}.
All other expectations can be reduced to these three expectations, by equivalence we just shown or symmetry.

Before giving the proof of all equations, we first handle simple cases.
Equation~\eqref{eqn:LLNDgDg} is a term that also appears in the standard NTK, and can be computed with the following recurrence relation,
\begin{align}
  \E \left[ Z^{D h_i^{(L)}} Z^{D h_j^{(L)}} \right] &= \kappa^2, \label{eqn:ZhL}
  \\
  \E \left[ Z^{D g_i^{(l)}} Z^{D g_j^{(l)}} \right] &= \E \left[ \dot{\phi} \left( Z^{g_i^{(l)}}\right) Z^{D h_i^{(l)}}\dot{\phi} \left( Z^{g_j^{(l)}}\right) Z^{D h_j^{(l)}} \right] \nonumber
  \\
  &= \E \left[ \dot{\phi} \left( Z^{g_i^{(l)}} \right) \dot{\phi} \left( Z^{g_j^{(l)}} \right) \right] \times \E \left[ Z^{D h_i^{(l)}} Z^{D h_j^{(l)}} \right], \label{eqn:Zfromhtog}
  \\
  \E \left[ Z^{D h_i^{(l-1)}} Z^{D h_j^{(l-1)}} \right] &= \E \left[ Z^{D g_i^{(l)}} Z^{D g_j^{(l)}} \right] \label{eqn:Zfromgtoh}
\end{align}
where we used the independency of $Z$-variables $Z^{g_i^{(l)}}$ and $Z^{Dh_i^{(l)}}$, this gives
\[
  \E \left[ Z^{D g_i^{(l)}} Z^{D g_j^{(l)}} \right] = \Gamma^{(l)}(x^{(i)}, x^{(j)})_{00} \E \left[ Z^{D g_i^{(l+1)}} Z^{D g_j^{(l+1)}} \right],
\]
and resolving recurrence relations shows
\[
  \E \left[ Z^{D g_i^{(l)}} Z^{D g_j^{(l)}} \right] = \kappa^2 \prod_{u=l}^{L-1} \Gamma^{(u)}(x^{(i)}, x^{(j)})_{00}.
\]

Now we can extend this approach to all the variables. 
The first steps for \eqref{eqn:ZhL} are 
\begin{align*}
  \E \left[ Z^{D h_i^{(L)}} Z^{D h_j^{(L)}} \right] &= \kappa^2,
  \\
  \E \left[ Z^{D_\alpha h_i^{(L)}} Z^{D h_j^{(L)}} \right] &= 0,
  \\
  \E \left[ Z^{D_\alpha J_\alpha h_i^{(L)}} Z^{D h_j^{(L)}} \right] &= \kappa^2,
  \\
  \E \left[ Z^{D_\alpha h_i^{(L)}} Z^{D_\beta h_j^{(L)}} \right] &= 0,
  \\
  \E \left[ Z^{D_\alpha J_\alpha h_i^{(L)}} Z^{D_\beta h_j^{(L)}} \right] &= 0,
  \\
  \E \left[ Z^{D_\alpha J_\alpha h_i^{(L)}} Z^{D_\beta J_\beta h_j^{(L)}} \right] &= \kappa^2.
\end{align*}
For the recursion step for the activation application corresponding to \eqref{eqn:Zfromhtog}, we have
\begin{align}
  &\E \left[ Z^{D g_i^{(l)}} Z^{D g_j^{(l)}} \right] \nonumber
  \\
  = &\E \left[ \dot{\phi} \left( Z^{g_i^{(l)}}\right) Z^{D h_i^{(l)}}\dot{\phi} \left( Z^{g_j^{(l)}}\right) Z^{D h_j^{(l)}} \right] \nonumber
  \\
  = &\E \left[ \dot{\phi} \left( Z^{g_i^{(l)}} \right) \dot{\phi} \left( Z^{g_j^{(l)}} \right) \right] \E \left[ Z^{D h_i^{(l)}} Z^{D h_j^{(l)}} \right]\nonumber
  \\
  = &\Gamma^{(l)}(x^{(i)}, x^{(j)})_{00} \E \left[ Z^{D h_i^{(l)}} Z^{D h_j^{(l)}} \right],
  \\
  &\E \left[ Z^{D_\alpha g_i^{(l)}} Z^{ D g_j^{(l)}} \right] \nonumber
  \\
  = &\E \left[ \left( \ddot{\phi} \left(Z^{g_i^{(l)}}\right) \cdot Z^{J_\alpha g_i^{(l)}} \cdot Z^{D_\alpha J_\alpha h_i^{(l)}} + \dot{\phi}\left( Z^{g_i^{(l)}} \right) \cdot Z^{D_\alpha h_i^{(l)}}\right) \dot{\phi} \left( Z^{g_j^{(l)}}\right) Z^{D h_j^{(l)}} \right]\nonumber
  \\
  = &\E \left[ \ddot{\phi} \left(Z^{g_i^{(l)}}\right) \cdot Z^{J_\alpha g_i^{(l)}} \cdot \dot{\phi} \left( Z^{g_j^{(l)}}\right) \right] \E \left[ Z^{D_\alpha J_\alpha h_i^{(l)}} Z^{D h_j^{(l)}} \right] \nonumber
  \\
  &\qquad + \E \left[ \dot{\phi}\left( Z^{g_i^{(l)}} \right) \dot{\phi} \left( Z^{g_j^{(l)}}\right) \right] \E \left[ Z^{D_\alpha h_i^{(l)}} Z^{D h_j^{(l)}} \right]
  \\
  = &\Gamma^{(l)}(x^{(i)}, x^{(j)})_{\alpha 0} \E \left[ Z^{D_\alpha J_\alpha h_i^{(l)}} Z^{D h_j^{(l)}} \right] + \Gamma^{(l)}(x^{(i)}, x^{(j)})_{00} \E \left[ Z^{D_\alpha h_i^{(l)}} Z^{D h_j^{(l)}} \right], \nonumber
  \\
  &\E \left[ Z^{D_\alpha g_i^{(l)}} Z^{D_\beta g_j^{(l)}} \right] \nonumber
  \\
  = &\E\bigg[ \left( \ddot{\phi} \left(Z^{g_i^{(l)}}\right) \cdot Z^{J_\alpha g_i^{(l)}} \cdot Z^{D_\alpha J_\alpha h_i^{(l)}} + \dot{\phi}\left( Z^{g_i^{(l)}} \right) \cdot Z^{D_\alpha h_i^{(l)}}\right) \nonumber
  \\
  &\qquad \times \left( \ddot{\phi} \left(Z^{g_j^{(l)}}\right) \cdot Z^{J_\beta g_j^{(l)}} \cdot Z^{D_\beta J_\beta h_j^{(l)}} + \dot{\phi}\left( Z^{g_j^{(l)}} \right) \cdot Z^{D_\beta h_j^{(l)}}\right) \bigg]\nonumber
  \\
  = &\E \left[ \ddot{\phi} \left(Z^{g_i^{(l)}}\right) Z^{J_\alpha g_i^{(l)}} \ddot{\phi} \left(Z^{g_j^{(l)}}\right) Z^{J_\beta g_j^{(l)}}\right] \E \left[ Z^{D_\alpha J_\alpha h_i^{(l)}} Z^{D_\beta J_\beta h_j^{(l)}} \right] \nonumber
  \\ 
  &\quad + \E \left[ \ddot{\phi} \left(Z^{g_i^{(l)}}\right) Z^{J_\alpha g_i^{(l)}} \dot{\phi}\left( Z^{g_j^{(l)}} \right) \right] \E \left[ Z^{D_\alpha J_\alpha h_i^{(l)}} Z^{D_\beta h_j^{(l)}}\right]\nonumber
  \\
  &\quad + \E \left[ \dot{\phi}\left( Z^{g_i^{(l)}} \right)\ddot{\phi} \left(Z^{g_j^{(l)}}\right) Z^{J_\beta g_j^{(l)}} \right] \E \left[ Z^{D_\alpha h_i^{(l)}} Z^{D_\beta J_\beta h_j^{(l)}} \right]\nonumber
  \\
  &\quad + \E \left[ \dot{\phi}\left( Z^{g_i^{(l)}} \right) \dot{\phi}\left( Z^{g_j^{(l)}} \right) \right] \E \left[ Z^{D_\alpha h_i^{(l)}} Z^{D_\beta h_j^{(l)}} \right]\nonumber
  \\
  = &\Gamma^{(l)}(x^{(i)}, x^{(j)})_{\alpha\beta} \E \left[ Z^{D_\alpha J_\alpha h_i^{(l)}} Z^{D_\beta J_\beta h_j^{(l)}} \right] + \Gamma^{(l)}(x^{(i)}, x^{(j)})_{\alpha 0} \E \left[ Z^{D_\alpha J_\alpha h_i^{(l)}} Z^{D_\beta h_j^{(l)}}\right]\nonumber
  \\
  &\quad + \Gamma^{(l)}(x^{(i)}, x^{(j)})_{0\beta} \E \left[ Z^{D_\alpha h_i^{(l)}} Z^{D_\beta J_\beta h_j^{(l)}} \right] + \Gamma^{(l)}(x^{(i)}, x^{(j)})_{00} \E \left[ Z^{D_\alpha h_i^{(l)}} Z^{D_\beta h_j^{(l)}} \right].
\end{align}
Finally for the recursion step corresponding to matrix multiplication \eqref{eqn:Zfromgtoh}, we have
\begin{align*}
  \E \left[ Z^{D h_i^{(l-1)}} Z^{D h_j^{(l-1)}} \right] &= \E \left[ Z^{D g_i^{(l)}} Z^{D g_j^{(l)}} \right],
  \\
  \E \left[ Z^{D_\alpha h_i^{(l-1)}} Z^{D h_j^{(l-1)}} \right] &= \E \left[ Z^{D_\alpha g_i^{(l)}} Z^{D g_j^{(l)}} \right],
  \\
  \E \left[ Z^{D_\alpha h_i^{(l-1)}} Z^{D_\beta h_j^{(l-1)}} \right] &= \E \left[ Z^{D_\alpha g_i^{(l)}} Z^{D_\beta g_j^{(l)}} \right].
\end{align*}
Our final task is solving the recurrence relations,
\begin{align*}
  &\E \left[ Z^{D_\alpha g_i^{(l)}} Z^{D g_j^{(l)}} \right] 
  \\
  = &\kappa^2 \Gamma^{(l)}(x^{(i)}, x^{(j)})_{\alpha 0} \prod_{u=l+1}^{L-1} \Gamma^{(u)}(x^{(i)}, x^{(j)})_{00} + \Gamma^{(l)}(x^{(i)}, x^{(j)})_{00} \E \left[ Z^{D_\alpha h_i^{(l)}} Z^{D h_j^{(l)}} \right]
  \\
  = &\kappa^2 \sum_{u_\alpha=l}^{L-1} \Gamma^{(u_\alpha)}(x^{(i)}, x^{(j)})_{\alpha 0} \prod_{\substack{u=l \\ u \neq u_\alpha}}^{L-1} \Gamma^{(u)} (x^{(i)}, x^{(j)})_{00}
\end{align*}
and
\begin{align*}
  &\E \left[ Z^{D_\alpha g_i^{(l)}} Z^{D_\beta g_j^{(l)}} \right] 
  \\
  &= \Gamma^{(l)}(x^{(i)}, x^{(j)})_{\alpha\beta} \E \left[ Z^{D_\alpha J_\alpha h_i^{(l)}} Z^{D_\beta J_\beta h_j^{(l)}} \right] + \Gamma^{(l)}(x^{(i)}, x^{(j)})_{\alpha 0} \E \left[ Z^{D_\alpha J_\alpha h_i^{(l)}} Z^{D_\beta h_j^{(l)}}\right]
  \\
  &\quad + \Gamma^{(l)}(x^{(i)}, x^{(j)})_{0\beta} \E \left[ Z^{D_\alpha h_i^{(l)}} Z^{D_\beta J_\beta h_j^{(l)}} \right] + \Gamma^{(l)}(x^{(i)}, x^{(j)})_{00} \E \left[ Z^{D_\alpha h_i^{(l)}} Z^{D_\beta h_j^{(l)}} \right]
  \\
  &= \kappa^2 \Gamma^{(l)}(x^{(i)}, x^{(j)})_{\alpha\beta} \prod_{u=l+1}^{L-1} \Gamma^{(u)}(x^{(i)}, x^{(j)})_{00}
  \\
  &\quad + \kappa^2 \Gamma^{(l)}(x^{(i)}, x^{(j)})_{\alpha 0} \sum_{u_\beta=l+1}^{L-1} \Gamma^{(u_\beta)}(x^{(i)}, x^{(j)})_{0 \beta} \prod_{\substack{u=l \\ u \neq u_\beta}}^{L-1} \Gamma^{(u)} (x^{(i)}, x^{(j)})_{00}
  \\
  &\quad + \kappa^2 \Gamma^{(l)}(x^{(i)}, x^{(j)})_{0 \beta} \sum_{u_\alpha=l+1}^{L-1} \Gamma^{(u_\alpha)}(x^{(i)}, x^{(j)})_{\alpha 0} \prod_{\substack{u=l \\ u \neq u_\alpha}}^{L-1} \Gamma^{(u)} (x^{(i)}, x^{(j)})_{00}
  \\
  &\quad + \Gamma^{(l)}(x^{(i)}, x^{(j)})_{00} \E \left[ Z^{D_\alpha g_i^{(l+1)}} Z^{D_\beta g_j^{(l+1)}} \right]
  \\
  &= \kappa^2 \sum_{u' = l}^{L-1} \Gamma^{(u')}(x^{(i)}, x^{(j)})_{\alpha\beta} \prod_{\substack{u=l \\ u \neq l'}}^{L-1} \Gamma^{(u)}(x^{(i)}, x^{(j)})_{00}
  \\
  &\quad + \kappa^2 \sum_{\substack{u_\alpha=l\\u_\beta=l\\u_\alpha \neq u_\beta}} \Gamma^{(u_\alpha)}(x^{(i)}, x^{(j)})_{\alpha 0} \Gamma^{(u_\beta)}(x^{(i)}, x^{(j)})_{0 \beta} \prod_{\substack{u=l\\u \neq u_\alpha \\ u\neq u_\beta}}^{L-1} \Gamma^{(u)}(x^{(i)}, x^{(j)})_{00}.
\end{align*}

Together with some trivial computations:
\begin{align*}
  \left\langle h^{(0)}(x^{(i)}), h^{(0)}(x^{(j)})\right\rangle &= \langle x^{(i)}, x^{(j)} \rangle,
  \\
  \left\langle \frac{\partial h^{(0)}(x^{(i)})}{\partial x_{\alpha-1}^{(i)}}, h^{(0)}(x^{(j)})\right\rangle &= \langle e_{\alpha-1}, x^{(j)} \rangle,
  \\
  \left\langle h^{(0)}(x^{(i)}), \frac{\partial h^{(0)}(x^{(j)})}{\partial x_{\beta-1}^{(j)}}\right\rangle &= \langle x^{(i)}, e_{\beta-1} \rangle,
  \\
  \left\langle \frac{\partial h^{(0)}(x^{(i)})}{\partial x_{\alpha-1}^{(i)}}, \frac{\partial h^{(0)}(x^{(j)})}{\partial x_{\beta-1}^{(j)}}\right\rangle &= \langle e_{\alpha-1}, e_{\beta-1} \rangle,
  \\
  \left\langle \frac{\partial f(x)}{\partial g^{(L+1)}(x)}, \frac{\partial f(x)}{\partial g^{(L+1)}(x)}\right\rangle &= \kappa^2
\end{align*}
we can show the convergence of every inner product in the layer-wise description of the finite JNTK, which proves the desired convergence of the finite JNTK.

To utilise this theorem, we need a version that gives a convergence rate.
Since the Tensor Program framework does not give an explicit convergence rate, we will use an implicit function to give a convergence rate.
\begin{remark} \label{rmk:JNTKInit2}
  There exists function $F_2(\cdot, \cdot, \cdot, \cdot)$, so that if the width of network $d$ is large enough, 
  \[
    d \ge F_2\left( x^{(1:N)}, L, \delta, \epsilon \right),
  \]
  then the finite JNTK and the limiting JNTK satisfy
  \[
    \left\| \Theta_{d, \theta_0}(x^{(i)}, x^{(j)}) - \kappa^2 \Theta(x^{(i)}, x^{(j)}) \right\|_F \le \kappa^2 \epsilon
  \]
  for all $i, j \in [N]$, with probability at least $1 - \delta$.
\end{remark}

\section{Finite JNTK After Perturbation} \label{sec:NTKPerturbProof}

\begin{theorem} \label{thm:JNTKCloseToInit}
  Let $\theta_0 = \{\WO{l}\}_{l=1}^{L+1}$ be the weights of the $L$-layer neural network, initialised with $\cN(0,1)$.

  Let $\omega$ be small enough so that it satisfies
  \[
    \omega \le O\left( (\log d)^{-2L} \right).  
  \]
  We also assume that the width is large enough to satisfy
  \[
    d \ge \Omega(\exp(L) N d_0 / \delta).
  \]
  
  Now let $\theta_\mathcal{A} = \{\WA{l}\}_{l=1}^{L+1}$ be the `trained' network, which is close to the initialisation: 
  \begin{align*}
    \left\| \WA{l} - \WO{l}\right\| &\le \omega \sqrt{d},
    \\
    \left\| \WA{l} - \WO{l} \right\|_\infty &\le \begin{cases}
      \omega \log d & \text{ if } l = 1,
      \\
      \omega \sqrt{d} \log d & \text{ if }2 \le l \le L,
      \\
      \infty & \text{ if } l = L+1,
    \end{cases}
    \\
    \left\| \WA{l} - \WO{l} \right\|_1 &\le \begin{cases}
      \infty & \text{ if } l =1,
      \\
      \omega \sqrt{d} \log d & \text{ if }2 \le l \le L,
      \\
      \omega \log d & \text{ if }l = L+1.
    \end{cases}
  \end{align*}

  Then with probability at least $1 - \delta$ over the randomness of $\theta_0$, for any choice of $\theta_\mathcal{A}$ satisfying the assumption, the entries of the finite JNTK satisfy
  \[
    \left\| \Theta_{d, \theta_0}(x^{(i)}, x^{(j)}) - \Theta_{d, \theta_\mathcal{A}}(x^{(i)}, x^{(j)}) \right\| \le \kappa^2 \omega \exp(O(L)) \log d.
  \]
  for all $i, j \in [N]$.
\end{theorem}

In the statement of this Theorem, we introduced two weights $\WO{l}$ and $\WA{l}$ each representing the initialised weight at time \textbf{0} and the \textbf{A}dversarially perturbated weights.
Using these matrices defines two separate forward computations, which we write as
\begin{align*}
  \go{1}(x) &\defeq \WO{1} x, & \ga{1}(x) &\defeq \WA{1} x,
  \\
  \ho{1}(x) &\defeq \phi \left( \go{1}(x) \right), & \ha{1}(x) &\defeq \phi\left( \ga{1}(x) \right),
  \\
  \go{l}(x) &\defeq \frac{1}{\sqrt{d}} \WO{l} \ho{l-1} (x), & \ga{l}(x) &\defeq \frac{1}{\sqrt{d}} \WA{l} \ha{l-1}(x),
  \\
  \ho{l}(x) &\defeq \phi\left( \go{l}(x) \right), & \ha{l}(x) &\defeq \phi \left( \ga{l}(x) \right),
  \\
  f_0(x) &\defeq \frac{\kappa}{\sqrt{d_L}} \WO{L+1} \ho{L}(x), & f_\mathcal{A}(x) &\defeq \frac{\kappa}{\sqrt{d_L}} \WA{L+1} \ha{L}(x).
\end{align*}
We often omit function argument $x$ when it is clear from the context.

Before stating the required lemmas and proofs, we first give the brief proof idea of Theoerm 9.
By the layer-wise decomposition of the finite JNTK in \eqref{eqn:JNTKLayerWise}, it is enough to show that each layer-wise part of the finite JNTK is close to its initialisation.
To show this, we can show it by showing that the inner products appearing in each summand are close to their initialisation after perturbation, and at initialisation, the inner products are bounded with high probability.
Again, to show that the inner products after perturbation stay constant, we need to show at initialisation the norm of vectors are bounded with high probability, and their difference after perturbation is small.
Overall, we need to show the following informal lemmas, the activations and their norms are bounded with high probability:
\begin{equation}
  \frac{1}{\sqrt{C_{l-1}}}\left\|\ho{l-1}(x^{(i)})\right\|, \frac{1}{\sqrt{C_{l-1}}}\left\|\frac{\partial \ho{l-1}(x^{(i)})}{\partial x_{\alpha-1}^{(i)}} \right\| \le \tilde{O}(1), \label{eqn:InformalForwardBound}
\end{equation}
the backpropagated gradients of pre-activations are bounded with high probability:
\begin{equation}
  \left\|\frac{\partial f_0(x^{(i)})}{\partial \go{l}(x^{(i)})}\right\|, \left\| \frac{\partial J(f_0)(x^{(i)})_{\alpha-1}}{\partial \go{l}(x^{(i)})} \right\| \le \kappa \tilde{O}(1), \label{eqn:InformalBackwardBound}
\end{equation}
the activations stay close to their initialisation:
\begin{equation}
  \frac{1}{\sqrt{C_{l-1}}}\left\|\ho{l-1}(x^{(i)}) - \ha{l-1}(x^{(i)})\right\|,\frac{1}{\sqrt{C_{l-1}}} \left\|\frac{\partial \ho{l-1}(x^{(i)})}{\partial x_{\alpha-1}^{(i)}} - \frac{\partial \ha{l-1}(x^{(i)})}{\partial x_{\alpha-1}^{(i)}}\right\| \le \omega \tilde{O}(1), \label{eqn:InformalForwardConstant}
\end{equation}
and the backpropagated gradients stay close to their initialisation:
\begin{equation}
  \left\|\frac{\partial f_0(x^{(i)})}{\partial \go{l}(x^{(i)})} - \frac{\partial f_{\mathcal{A}}(x^{(i)})}{\partial \ga{l}(x^{(i)})}\right\|, \left\| \frac{\partial J(f_0)(x^{(i)})_{\alpha-1}}{\partial \go{l}(x^{(i)})} - \frac{\partial J(f_{\mathcal{A}})(x^{(i)})_{\alpha-1}}{\partial \ga{l}(x^{(i)})}\right\| \le \kappa \omega \tilde{O}(1). \label{eqn:InformalBackwardConstant}
\end{equation}
Here we used $\tilde{O}(1)$ to denote that there is implicit dependency over the logarithmic factor of $d$.

\subsection{Bounds at initialisation} \label{subsec:boundinit}
In this section, we will prove bounds holding at initialisation, especially the lemmas formalising Equation~\ref{eqn:InformalForwardBound} and \ref{eqn:InformalBackwardBound}.
All the lemmas here are stated for single data $x$, but we can apply the Union-bound to data indices, showing that the results hold for all $x^{(i)}$ with probability at least $1 - N \cdot \delta$.
Similarly for those applied to the Jacobian values, we can apply the Union-bound to the input dimension indices, showing equivalent results with probability at least $1 - N d_0 \delta$.
\begin{lemma} \label{lem:ForwardNormInit}
  Fix the fail probability $\delta > 0$, and the network depth $L$. 
  If the network width satisfies both $d \ge \Omega(\log(1 / \delta))$ and $d \ge \exp(\Omega(L))$, then 
  \[
    \left\| \go{l}(x) \right\| \le 2\sqrt{d},\qquad \left\| \ho{l}(x) \right\| \le 2 \sqrt{d},
  \]
  holds for all $l \in [L]$ with probability at least $1 - \delta$.
\end{lemma}
\begin{proof}
  We prove this lemma by induction on the conditioned activations.
  Our inductive assumption is
  \[
    \left| \frac{\|\ho{l}\|^2}{d} - 1 \right| \le \frac{G_l(\beta)}{2G_L(\beta)} \text{ with probability at least } 1 - \frac{l \cdot\delta}{L}
  \]
  for some $\beta > 0$ that will be computed later.
  
  To prove the induction step for this assumption, we will use the following claim:
  
  \textbf{Claim : } \textit{If $v \in \R^{d'}$ satisfy 
  \[
    \left|\|v\|^2 - 1 \right| \le \epsilon
  \]
  with $\epsilon < 1/2$, then for $W \in \R^{d \times d'}$ and $W_{ij} \sim \cN(0, 1)$, 
  \[
    \left| \frac{\left\| \phi(W v) \right\|^2}{d} - 1 \right| \le \frac{1}{2 G_\beta(L)} + \beta \epsilon
  \]
  for some $\beta$, with probability at least $1 - \frac{\delta}{L}$. }

  For the proof of claim, we first apply the triangle inequality to decompose the difference as
  \begin{align*}
    &\left| \frac{\left\| \phi(W v) \right\|^2}{d} - 1 \right|
    \\
    \le &\left| \frac{\left\| \phi(W v) \right\|^2}{d} - \E \left[ \phi(\|v\|Z)^2 \right] \right| + \left| \E \left[ \phi(\|v\|Z)^2 \right] - 1 \right|
  \end{align*}
  where we can rewrite the first term as concentration inequality,
  \[
    \left| \frac{1}{d} \sum_{i=1}^d \phi(\|v\|Z_i)^2 - \E \left[ \phi(\|v\|Z)^2 \right] \right|.
  \]
  The random variable here is sub-exponential, whose sub-exponential norm is bounded as
  \begin{align}
    \|\phi(\|v\| Z)^2\|_{\mathrm{SE}} &\le \|\phi(\|v\|Z)\|_{\mathrm{SG}}^2 \nonumber
    \\
    &\le M_1^2 \|\|v\|Z\|_{\mathrm{SG}}^2 \nonumber
    \\
    &\le M_1^2 \|v\|^2 \|Z\|_{\mathrm{SG}}^2 \nonumber
    \\
    &\le \frac{4 M_1^2}{\ln 2} \label{eqn:SENorm}. 
  \end{align}
  where $Z \sim \cN(0, 1)$. 
  By Bernstein's inequality (Corollary 2.8.3 of \cite{BookHDP}), we have
  \begin{align*}
    \left| \frac{1}{d} \sum_{i=1}^d \phi(\|v\|Z_i)^2 - \E \left[ \phi(\|v\|Z)^2 \right] \right| \le \frac{1}{2 G_\beta(L)}
  \end{align*}
  with a probability of at least
  \[
    1 - 2 \exp\left( -c \max \left( \frac{1}{4 K^2 G_\beta(L)^2}, \frac{1}{2 K G_\beta(L)} \right) \cdot d\right)
  \]
  where $K$ is the sub-exponential norm computed in \eqref{eqn:SENorm}.

  For the second term, we can observe that this is similar to local Lipschitz-ness, which is clearer if we rewrite it as
  \[
    \left| \E \left[ \phi(\|v\|Z)^2 \right] - 1 \right|
    = \left| \E \left[ \phi(\|v\|Z)^2 \right] - \E \left[ \phi(1 \cdot Z)^2 \right] \right|.
  \]
  Which can be bounded as
  \begin{align*}
    &\left| \E \left[ \phi(\|v\|Z)^2 \right] - \E \left[ \phi(Z)^2 \right] \right|
    \\
    \le &\left| \E \left[ \phi(\|v\|Z)^2 - \phi(Z)^2 \right] \right|
    \\
    \le &\E \left| \phi(\|v\|Z)^2 - \phi(Z)^2 \right|
    \\
    \le &\E \left|\phi(\|v\| Z) - \phi(Z)\right| \left|\phi(\|v\|Z) + \phi(Z)\right|
    \\
    \le &M_1 |\|v\| - 1| \E\left|\phi(\|v\|Z) + \phi(Z)\right|
    \\
    \le &M_1 |\|v\| - 1| \E \left| \phi(0) + (\phi(\|v\|Z) - \phi(0)) + \phi(0) + (\phi(Z) - \phi(0)) \right|
    \\
    \le &M_1 |\|v\| - 1| \E \left| 2\phi(0) + M_1 (\|v\| + 1)Z \right|
    \\
    = &M_1 |\|v\| - 1| \left( 2 \phi(0) + M_1(\|v\| + 1) \sqrt{\frac{2}{\pi}} \right)
    \\
    \le &M_1 \left( 2 \phi(0) + M_1 \frac{5}{\sqrt{2\pi}} \right) |\|v\|^2 - 1| 
  \end{align*}
  where we used the assumption that $\frac{1}{2} \le \|v\| \le \frac{3}{2}$ in the last inequality. 
  Combining these two inequalities, we prove the claim
  \[
    \left| \frac{\|\phi(Wv)\|^2}{d} \right| \le \frac{1}{2 G_\beta(L)} + \beta \epsilon
  \]
  with $\beta = M_1 \left( 2 \phi(0) + M_1 \frac{5}{\sqrt{2\pi}} \right)$.

  At layer one, we can show that
  \[
    \left| \frac{\|\ho{1}\|^2}{d} - 1 \right| \le \frac{1}{2 G_\beta(L)}
  \]
  with probability at least $1 - \delta / L$, which proves the inductive assumption's base case $l = 1$.
  
  The inductive step is similar, inductive assumption implies the condition of the claim, which then gives
  \[
    \left| \frac{\|\ho{l}\|^2}{d} - 1 \right| \le \frac{1}{2 G_\beta(L)}+ \beta \left| \frac{\|\ho{l-1}\|^2}{d} - 1 \right| = \frac{1 + \beta G_\beta(l-1)}{2 G_\beta(L)} = \frac{G_\beta(l)}{2 G_\beta(L)}
  \] 
  with probability at least $1 - l \delta / L$ by Union-bound on the previous layer's inductive assumption and the claim's event.
\end{proof}

\begin{lemma} \label{lem:ForwardJacobNormInit}
  Fix the fail probability $\delta > 0$, and the network depth $L$. 
  If the network width satisfies $d \ge \Omega(\log (L / \delta))$, then
  \[
    \left\| \frac{\partial \go{l}(x)}{\partial x_{\alpha-1}} \right\| \le C^l M_1^{l-1} \sqrt{d},\qquad \left\| \frac{\partial \ho{l}(x)}{\partial x_{\alpha-1}} \right\| \le (C M_1)^l \sqrt{d},
  \]
  holds for all $l \in [L]$ and $\alpha \in [d_0]$ with probability at least $1 - \delta$.
\end{lemma}
\begin{proof}
  Let's recall the recursive definition of Jacobian first:
  \[
    \frac{\partial \ho{l}(x)}{\partial x_{\alpha-1}} = \dot{\phi} \left( \frac{1}{\sqrt{d}}\WO{l-1} \ho{l-1}(x) \right) \odot \left( \frac{1}{\sqrt{d}}\WO{l-1} \frac{\partial \ho{l-1}(x)}{\partial x_{\alpha-1}} \right).
  \]
  Since $\phi$ is $M_1$-Lipschitz, we can bound as
  \[
    \left\|\frac{\partial \ho{l}(x)}{\partial x_{\alpha-1}} \right\| \le M_1 \left\|\frac{1}{\sqrt{d}} \WO{l-1} \right\| \left\| \frac{\partial \ho{l-1}(x)}{\partial x_{\alpha-1}} \right\|.
  \]
  To bound the operator norm of the random matrix, we use the standard result (Theorem 4.4.5 of \cite{BookHDP}) which shows that
  \[
    \left\|\frac{1}{\sqrt{d}} \WO{l-1} \right\| \le C \left( 2 + \frac{t}{\sqrt{d}}\right)
  \]
  with probability at least $ 1 - 2 \exp(-t^2)$. 
  Plugging in $t = \sqrt{d}$, we need to set $1 - 2 \exp(-d) \ge 1 - \delta / L$, which is achieved by our assumption on $d$. 

  We have 
  \[
    \left\|\frac{\partial \ho{1}(x)}{\partial x_{\alpha-1}}\right\| \le 3C M_1 \sqrt{d}
  \]
  with probability at least $ 1 - \delta / L$ in a similar way as the recursive step.
  Then applying the recursive steps proves the result.
\end{proof}

We can derive similar bounds for the backpropagated gradients.
\begin{lemma} \label{lem:BackwardNormInit}
  Fix the fail probability $\delta > 0$, and the network depth $L$. 
  If the network width satisfy $d \ge \Omega(\log (L / \delta))$, then
  \[
    \left\| \frac{\partial f_0(x)}{\partial \go{l}(x)} \right\| \le \kappa (C M_1)^{L+1-l}, \qquad \left\|\frac{\partial f_0(x)}{\partial \ho{l}(x)} \right\| \le \kappa C^{L-l} M_1^{L+1-l},
  \]
  holds for all $l \in [L]$ with probability at least $ 1 - \delta$.
\end{lemma}
\begin{proof}
  The proof idea of this lemma is similar to the proof of Lemma~\ref{lem:ForwardJacobNormInit}.
  Here we can recursively bound the backpropagated gradients, 
  \begin{align*}
    \left\| \frac{\partial f_0(x)}{\partial \go{l}(x)} \right\| &= \left\| \dot{\phi}\left( \WO{l-1} \ho{l-1}(x) \right) \odot \left( \frac{1}{\sqrt{d}}\left( \WO{l+1} \right)^\intercal \frac{\partial f_0(x)}{\partial \go{l+1}(x)} \right) \right\|
    \\
    &\le M_1 \left\| \frac{1}{\sqrt{d}}\WO{l+1} \right\| \left\| \frac{\partial f_0(x)}{\partial \go{l+1}(x)} \right\|.
  \end{align*}
  Similar to the Lemma~\ref{lem:ForwardJacobNormInit}, we obtain 
  \[
    \left\|\frac{1}{\sqrt{d}} \WO{l+1} \right\| \le 3C
  \]
  with probability at least $1 - 2 \exp(-d) \ge 1 - \delta / L$.

  For the initial case, we have
  \[
    \left\|\frac{\partial f_0(x)}{\partial \go{L}(x)}\right\| \le \kappa 3 M_1 C
  \]
  with probability at least $1 - \delta / L$, and applying the recursive steps proves our result.
\end{proof}

\begin{lemma} \label{lem:BackwardJacobNormInit}
  Fix the fail probability $\delta > 0$, and the network depth $L$. 
  If the network width satisfy $d \ge \Omega(\log (L / \delta))$, then
  \begin{align*}
    \left\| \frac{\partial J(f_0)(x)_{\alpha-1}}{\partial \go{l}(x)} \right\| &\le \kappa M_2 C' (C M_1)^{L-1} \sum_{i=0}^{L-l} (C M_1)^i,
    \\
    \left\| \frac{\partial J(f_0)(x)_{\alpha-1}}{\partial \ho{l}(x)} \right\| &\le \kappa M_2 C' C^l M_1^{L-1} \sum_{i=0}^{L-l-1} (C M_1)^i,
  \end{align*}
  holds for all $l \in [L]$ and $\alpha \in [d_0]$ with probability at least $1 - \delta$.
\end{lemma}
\begin{proof}
  Let's first unfold the definition,
  \begin{align*}
    \frac{\partial J(f_0)(x)_{\alpha-1}}{\partial \go{l}(x)} &= \ddot{\phi}\left( \go{l}(x) \right) \odot \frac{\partial \go{l}(x)}{\partial x_{\alpha-1}} \odot \frac{\partial f_0(x)}{\partial \ho{l}(x)}
    \\
    &+ \dot{\phi} \left( \go{l} \right) \odot \frac{\partial J(f_0)(x)_{\alpha-1}}{\partial \ho{l}(x)}
    \\
    &= \ddot{\phi}\left( \go{l}(x) \right) \odot \left( \frac{1}{\sqrt{d}} \WO{l} \frac{\partial \ho{l-1}(x)}{\partial x_{\alpha-1}}\right) \odot \left(\frac{1}{\sqrt{d}} \left(\WO{l+1}\right)^\intercal \frac{\partial f_0(x)}{\partial \go{l+1}(x)}\right)
    \\
    &+ \dot{\phi} \left( \go{l} \right) \odot \left(\frac{1}{\sqrt{d}} \left(\WO{l+1}\right)^\intercal\frac{\partial J(f_0)(x)_{\alpha-1}}{\partial \go{l+1}(x)}\right)
  \end{align*}
  To bound this vector's norm, we will analyse each term,
  \begin{align*}
    \left\|\frac{\partial J(f_0)(x)_{\alpha-1}}{\partial \go{l}(x)}\right\| &\le M_2 \left\| \left(\frac{1}{\sqrt{d}} \WO{l} \frac{\partial \ho{l-1}(x)}{\partial x_{\alpha-1}} \right)\odot \left( \frac{1}{\sqrt{d}} \left( \WO{l+1} \right)^\intercal \frac{\partial f_0(x)}{\partial \go{l+1}(x)} \right)  \right\|
    \\
    &+ M_1 \left\| \frac{1}{\sqrt{d}} \left( \WO{l+1} \right)^\intercal \frac{\partial J(f_0)(x)_{\alpha-1}}{\partial \go{l+1}(x)} \right\|.
  \end{align*}
  
  For the first term, we use Lemma~\ref{lem:ForwardJacobNormInit} and \ref{lem:BackwardNormInit}, which allow us to argue that
  \begin{align*}
    \left\| \frac{\partial \ho{l-1}(x)}{\partial x_{\alpha-1}} \right\| \le (C M_1)^{l-1} \sqrt{d},\qquad \left\| \frac{\partial f_0(x)}{\partial \go{l+1}(x)} \right\| \le \kappa M_2 C' (C M_1)^{L-1} \sum_{i=0}^{L-l} (C M_1)^i
  \end{align*}
  with probability at least $1 - \delta / 2$.
  Recall that these two lemmas assume the conditioning on the event 
  \[
    \left\|\WO{l}\right\| \le C\sqrt{d}
  \]
  for all $l = 1, \ldots, L+1$.
  We can implement this conditioning via multiplying additional random variable $\alpha_l \le 1$, whose existence is guaranteed by stochastic dominance of $\|\WO{l}\|$ over conditioned random variable $\|\WO{l}\| \le C\sqrt{d}$.
  \[
    \frac{1}{\sqrt{d}} \WO{l} \frac{\partial \ho{l-1}(x)}{\partial x_{\alpha-1}} \bigg| \|\WO{l}\| \le C \sqrt{d} \stackrel{\text{dist.}}{=} \frac{\alpha_l}{\sqrt{d}} \WO{l} \left(\frac{\partial \ho{l-1}(x)}{\partial x_{\alpha-1}} \bigg| \|\WO{l}\| \le C \sqrt{d} \right)
  \]
  where $X | E$ is the conditional distribution of $X$ on the event $E$, and $\stackrel{dist.}{=}$ means that these two random vectors are distributionally equal.
  The formal proof of this argument is given in \ref{lem:stocdomofcondition}.

  Now to bound the first term, we can proceed as
  \begin{align*}
    &\left\|\left( \frac{1}{\sqrt{d}} \WO{l} \frac{\partial \ho{l-1}(x)}{\partial x_{\alpha-1}}\right) \odot \left(\frac{1}{\sqrt{d}} \left(\WO{l+1}\right)^\intercal \frac{\partial f_0(x)}{\partial \go{l+1}(x)}\right)\right\| \bigg| \forall l. \|\WO{l}\| \le C \sqrt{d}
    \\
    \stackrel{d}{=} & \Bigg\|\left( \frac{\alpha_l}{\sqrt{d}} \WO{l} \left(\frac{\partial \ho{l-1}(x)}{\partial x_{\alpha-1}} \bigg| \forall l. \|\WO{l}\| \le C \sqrt{d}\right)\right) 
    \\
    &\qquad \odot \left(\frac{\alpha_{l+1}}{\sqrt{d}} \left(\WO{l+1}\right)^\intercal \left(\frac{\partial f_0(x)}{\partial \go{l+1}(x)} \forall l. \|\WO{l}\| \le C \sqrt{d}\right)\right)\Bigg\|
    \\
    \le &\kappa (C M_1)^{L-1} \sqrt{d} \left\|\left( \frac{\alpha_l}{\sqrt{d}} \WO{l} v\right) \odot \left(\frac{\alpha_{l+1}}{\sqrt{d}} \left(\WO{l+1}\right)^\intercal w\right)\right\|
    \\
    \le &\kappa (C M_1)^{L-1} \left\|\frac{1}{\sqrt{d}}\left(\WO{l} v\right) \odot \left(\left(\WO{l+1}\right)^\intercal w\right)\right\|
  \end{align*}
  where $v$ and $w$ follows uniform distribution over $S^{d-1}$. 
  
  So to bound this vector's norm, we should bound the following quantity,
  \[
    \frac{1}{d} \sum_{i=1}^d Z_i^2 W_i^2
  \]
  where $Z_i, W_i \stackrel{i.i.d.}{\sim} \cN(0, 1)$.
  This distribution is known to be sub-Weibull with parameter $\theta = 1/2$, \cite{subweibull} with the concentration bound (Proposition 3 of \cite{subweibull})
  \[
    \frac{1}{d} \sum_{i=1}^d Z_i^2 W_i^2 \le C_{1,1/2} +  C_{2,1/2}
  \]
  for some absolute constants $C_{1,1/2}, C_{2,1/2} > 0$, with probability at least $1 - \exp(-d)$.
  
  In summary, we obtain
  \[
    \left\|\left( \frac{1}{\sqrt{d}} \WO{l} \frac{\partial \ho{l-1}(x)}{\partial x_{\alpha-1}}\right) \odot \left(\frac{1}{\sqrt{d}} \left(\WO{l+1}\right)^\intercal \frac{\partial f_0(x)}{\partial \go{l+1}(x)}\right)\right\| \le \kappa C' (C M_1)^{L-1}
  \]
  with probability at least $1 - \exp(-d) / (2L)$.

  The second term can be bounded as
  \begin{align*}
    &\left\| \dot{\phi} \left( \go{l} \right) \odot \left(\frac{1}{\sqrt{d}} \left(\WO{l+1}\right)^\intercal\frac{\partial J(f_0)(x)_{\alpha-1}}{\partial \go{l+1}(x)}\right) \right\| \bigg| \forall l. \|\WO{l}\| \le C \sqrt{d}
    \\
    \le &M_1 \left\| \left(\frac{1}{\sqrt{d}} \left(\WO{l+1}\right)^\intercal\frac{\partial J(f_0)(x)_{\alpha-1}}{\partial \go{l+1}(x)}\right) \right\| \bigg| \forall l. \|\WO{l}\| \le C \sqrt{d}
    \\
    \le &M_1 C \left\| \left(\frac{\partial J(f_0)(x)_{\alpha-1}}{\partial \go{l+1}(x)} \bigg| \forall l. \|\WO{l}\| \le C\sqrt{d}\right)\right\|.
  \end{align*}
  Summing up, including the fact that
  \[
    \left\|\frac{\partial J(f_0)(x)_{\alpha-1}}{\partial g_0^{(L)}(x)}\right\| \le \kappa C 
  \]
  with probability at least $1 - \exp(-d) / (2L)$, we result
  \[
    \left\|\frac{\partial J(f_0)(x)_{\alpha-1}}{\partial g_0^{(l)}(x)}\right\| \le \kappa M_2 C' (C M_1)^{L-1} \sum_{i=0}^{L-l} (C M_1)^i.
  \]
\end{proof}

\begin{lemma} \label{lem:stocdomofcondition}
  Suppose some random variable $X$ with $X > 0$ almost surely.
  For any $M > 0$, there exists some random variable $Z \le 1$ that satisfies
  \[
    Z \cdot X \stackrel{d}{=} X | X \le M.
  \]
\end{lemma}
\begin{proof}
  We prove this via stochastic dominance.
  
  We can show that $X$ (first order) stochastically dominates $X | X \le M$, by showing that for any non-decreasing $u$,
  \begin{align*}
    \E [u(X)] &= P(X > M) \E [u(X) | X > M] + P(X \le M) \E [u(X) | X \le M]
    \\
    &\ge P(X > M) u(M) + P(X \le M) \E [u(X) | X \le M]
    \\
    &\ge P(X > M) \E [u(X) | X \le M] + P(X \le M) \E [u(X) | X \le M]
    \\
    &= \E[ u(X) | X \le M].
  \end{align*}
  Now from the property of stochastic dominance, this implies that there exists some random variable $W \ge 0$ such that
  \[
    X = (X | X \le M) + W
  \]
  which then after rearranging, we obtain
  \[
    X | X \le M \stackrel{d}{=} \frac{X - W}{X} \cdot X
  \]
  with $Z = (X - W) / X \le 1$ which is guaranteed by $X > 0$ and $W \ge 0$.
\end{proof}

We also include a corollary that proves similar results for $\infty$-norms. 
These results are in general not required for standard NTK analysis, but for JNTK, we need to have a tight bound of element-wise product $\|v \odot w\|$. Without such a bound, we end up with an additional $\sqrt{d}$ factor.
\begin{corollary} \label{lem:supnormbound}
  Fix the fail probability $\delta > 0$, and the network depth $L$. 
  If the network depth satisfy both $d \ge \Omega(1/\delta)$ and $d \ge \exp(\Omega(L))$, then
  \begin{align*}
    \left\| \go{l}(x) \right\|_\infty &\le 2 \sqrt{\log d} 
    \\
    \left\| \frac{\partial \go{l}(x)}{\partial x_{\alpha-1}} \right\|_\infty &\le O(\exp(L) \sqrt{\log d}),
    \\
    \left\| \frac{\partial f_0(x)}{\partial \ho{l}(x)} \right\|_\infty &\le \frac{\kappa \sqrt{\log d}}{\sqrt{d}}O(\exp(L)),
    \\
    \left\| \frac{\partial J(f_0)(x)_{\alpha-1}}{\partial \ho{l}} \right\|_\infty &\le \frac{\kappa \sqrt{\log d}}{\sqrt{d}} \exp(O(L)) ,
  \end{align*}
  holds for all $l \in [L]$ and $\alpha \in [d_0]$ with probability at least $1- \delta$.
\end{corollary}
\begin{proof}
  Note that all these three random vectors are generated by matrix multiplication with Gaussian matrix, 
  \begin{align*}
    \frac{\partial \go{l}(x)}{\partial x_{\alpha-1}} &= \frac{1}{\sqrt{d}} \WO{l} \frac{\partial \ho{l-1}(x)}{\partial x_{\alpha-1}},
    \\
    \frac{\partial f_0(x)}{\partial ho{l}(x)} &= \frac{1}{\sqrt{d}} \left( \WO{l+1} \right)^\intercal \frac{\partial f_0(x)}{\partial \go{l+1}(x)},
    \\
    \frac{\partial J(f_0)(x)_{\alpha-1}}{\partial \ho{l}(x)} &= \frac{1}{\sqrt{d}} \left( \WO{l+1} \right)^\intercal \frac{\partial J(f_0)(x)_{\alpha-1}}{\partial \go{l+1}(x)}.
  \end{align*}
  This allows us to argue that each coordinate of the vectors we are curious about is conditionally Gaussian, where their variance is given by the norm of the vectors in RHS.
  
  As we've done in the proof of Lemma~\ref{lem:BackwardJacobNormInit}, the boundedness of the vector norms requires the conditioning of weight matrices, to $\|\WO{l}\| \le C \sqrt{d}$ for some $C > 0$, and applying Lemma~\ref{lem:stocdomofcondition} similarly show that this conditioning makes the norm smaller, so we can safely ignore them.

  So, all three problems collapse to the following probabilistic bound,
  \[
    P\left( \max_i |Z_i| \le \sqrt{\log d} + t \right) \ge 1 - 2 \exp(-t^2)
  \]
  where $Z_i \sim \cN(0, 1)$. Plugging $t = \sqrt{\log d}$, we obtain probability bound $1 - \frac{2}{d}$.
\end{proof}

Before ending this section, we will give intuitive big-O notation-based results, which say that if the network width satisfy 
\[
  d \ge \Omega(\log(L d_0 / \delta))
\]
then 
\begin{align}
  \left\| \go{l}(x) \right\| &\le 2\sqrt{d}, & \left\| \ho{l}(x) \right\| &\le 2\sqrt{d}, \label{eqn:initbegin}
  \\
  \left\|\frac{\partial \go{l}(x)}{\partial x_{\alpha-1}} \right\| &\le O(\exp(L)) \sqrt{d}, & \left\|\frac{\partial \ho{l}(x)}{\partial x_{\alpha-1}} \right\| &\le O(\exp(L)) \sqrt{d},
  \\
  \left\| \frac{\partial f_0(x)}{\partial \go{l}(x)} \right\| &\le \kappa O(\exp(L)), & \left\| \frac{\partial f_0(x)}{\partial \ho{l}(x)} \right\| &\le \kappa O(\exp(L)), 
  \\
  \left\| \frac{\partial J(f_0)(x)_{\alpha-1}}{\partial \go{l}(x)} \right\| &\le \kappa \exp(O(L)), & \left\| \frac{\partial J(f_0)(x)_{\alpha-1}}{\partial \ho{l}(x)} \right\| &\le \kappa \exp(O(L)), \label{eqn:initend}
\end{align}
for all $l \in [L]$ and $\alpha \in [d_0]$ simultaneously, with probability at least $1 - \delta$.

\subsection{Bounds after perturbation} \label{subsec:boundperturb}
We now prove that the vectors that appeared in previous sections are stable after perturbation of the weight matrices.
We consider $\sqrt{d} \omega$ perturbation over the network weights, which models the perturbation due to the gradient.
We additionally assume the $\omega \sqrt{d} \log d$ perturbation over the $\|\WA{l} - \WO{l}\|_\infty$ and $\|\WA{l} - \WO{l}\|_1$, which are the maximum of column sum and row sums, respectively.

We will later let the perturbation $\omega$ be small enough to prove the global convergence, so throughout this section, we will assume that $\omega \le 1$ to write $\omega^p \le \omega$.

Since it is hard to interpret the order of quantifiers in the lemmas, we will mention here explicitly, which applies to all lemmas in Section~\ref{subsec:boundperturb}.
We first specify the fail probability and network depth. 
Then there exists a bound on $\omega$ and $d$, which implies that the results hold for all choices of $\WA{l}$ satisfying the condition on $\omega$, for all $l \in [L]$ and $\alpha \in [d_0]$ simultaneously, with probability at least $1- \delta$.
This result holds for any choice of $x$, so to extend this proof to all training data $x^{(1:N)}$, we can apply the Union-bound to results weakened fail probability guarantee $1 - N \delta$.
\begin{lemma} \label{lem:perturbact}
  Fix the fail probability $\delta > 0$, and the network depth $L$.
  Suppose that $\|\WA{l} - \WO{l} \|\le \omega \sqrt{d}$ for all $l \in [L+1]$.
  If $\omega \le 1$, and the network width satisfies both $d \ge \Omega(\log(1 / \delta))$ and $d \ge \exp(\Omega(L))$, then
  \[
    \|\ha{l} - \ho{l}\| \le 2 \sum_{i=0}^{l-1} ((C+1)M_1)^{i} \omega \sqrt{d} = O(\exp(L) \omega \sqrt{d}).
  \]
  holds for all $l \in [L]$ with probability at least $1 - \delta$.
\end{lemma}
\begin{proof}
  This is proven by recursive stability analysis, at layer $l = 1$,
  \[
    \left\|\ga{1} - \go{1}\right\| \le \omega\sqrt{d}  \|x\| = \omega\sqrt{d},
  \]
  and for the latter layers,
  \begin{align*}
    \left\| \ha{l} - \ho{l} \right\| &\le M_1 \left\| \ga{l} - \go{l} \right\|,
    \\
    \left\|\ga{l} - \go{l}\right\| &\le \left\| \ga{l} - \frac{1}{\sqrt{d}} \WA{l} \ho{l-1} \right\| + \left\| \frac{1}{\sqrt{d}} \WA{l} \ho{l-1} - \go{l} \right\|
    \\
    &\le \left\|\frac{1}{\sqrt{d}} \WA{l}\right\| \left\| \ha{l-1} - \ho{l-1} \right\| + \omega \|\ho{l-1}\|
    \\
    &\le (C + \omega) \left\| \ha{l-1} - \ho{l-1} \right\| + 2\omega \sqrt{d}
  \end{align*}
  where we applied Lemma~\ref{lem:ForwardNormInit} which holds with probability at least $1 - \delta$, and norm bound on the $\WO{l}$ as we've done in the proof of Lemma~\ref{lem:BackwardNormInit}.
  
  Summing up, we results
  \[
    \left\| \ha{l} - \ho{l} \right\| \le ((C+1) M_1)^{l-1} \omega \sqrt{d} + \sum_{i=1}^{l-1} 2 M_1^{l-i} (C+1)^{l-i-1} \omega \sqrt{d} 
  \]
\end{proof}

We also prove a similar version for the $L_\infty$-norm.
\begin{lemma} \label{lem:perturbactinfty}
  Fix the fail probability $\delta > 0$, and the network depth $L$. 
  Suppose that $\|\WA{l} - \WO{l} \|\le \omega \sqrt{d}$ for all $l \in [L+1]$ and $\|\WA{l} - \WO{l}\|_\infty \le \omega \sqrt{d} \log d$ for all $2 \le l \le L$ and $\|\WA{1} - \WO{1}\|_\infty \le \omega \log d$.
  If $\omega \le 1$, and the network width satisfies $d \ge L / \delta$, then
  \[
    \|\ha{l} - \ho{l}\|_\infty \le \omega \exp(O(L)) (\log d)^L.
  \]
  holds for all $l \in [L]$ with probability at least $1 - \delta$.
\end{lemma}
\begin{proof}
  This has a similar recursive analysis, as
  \begin{align*}
    \left\| \ga{1} - \go{1} \right\|_\infty &\le \omega \log d,
    \\
    \left\| \ha{l} - \ho{l} \right\|_\infty &\le M_1 \left\| \ga{l} - \go{l} \right\|_\infty,
    \\
    \left\| \ga{l} - \go{l} \right\|_\infty &\le \left\| \frac{1}{\sqrt{d}} \WA{l} \right\|_\infty \left\| \ha{l-1} - \ho{l-1} \right\|_\infty + \frac{1}{\sqrt{d}} \left\| \WA{l} - \WO{l} \right\|_\infty \left\| \ho{l-1} \right\|_\infty
    \\
    &\le C \log d \left\| \ha{l-1} - \ho{l-1} \right\|_\infty + C\omega \log d.
  \end{align*}
\end{proof}

\begin{lemma} \label{lem:perturbjacobian}
  Fix the fail probability $\delta > 0$ and the network depth $L$.
  Suppose that $\|\WA{l} - \WO{l} \|\le \omega \sqrt{d}$ for all $l \in [L+1]$.
  If $\omega \le 1$, and the network width satisfy both $d \ge \Omega(\log (1 / \delta))$ and $d \ge \exp(\Omega(L))$, then
  \[
    \left\|\frac{\partial \ha{l}}{\partial x_{\alpha-1}} - \frac{\partial \ho{l}}{\partial x_{\alpha-1}}\right\| \le ((C+1)M_1)^{l-1} \omega \sqrt{d \log d} + \sum_{i=1}^{l-1} (1+C)^{l-i-1} M_1^{l-i} (C M_1)^{i-1} \omega \sqrt{d}
  \]
  holds for all $l \in [L+1]$ and $\alpha \in [d_0]$ with probability at least $1 - \delta$.
\end{lemma}
\begin{proof}
  By using Lemma~\ref{lem:ForwardJacobNormInit} and probabilistic bounds on $\WO{l}$, we obtain
  \begin{align*}
    &\left\|\frac{\partial \ga{1}(x)}{\partial x_{\alpha-1}} - \frac{\partial \go{1}(x)}{\partial x_{\alpha-1}}\right\| 
    \\
    &\quad \le \omega\sqrt{d} \|e_{\alpha-1}\| = \omega\sqrt{d},
    \\
    &\left\| \frac{\partial \ha{l}(x)}{\partial x_{\alpha-1}} - \frac{\partial \ho{l}(x)}{\partial x_{\alpha-1}} \right\| 
    \\
    &\quad \le \left\| \dot{\sigma}\left( \ga{l}(x) \right) \odot \left( \frac{\partial \ga{l}(x)}{\partial x_{\alpha-1}} - \frac{\partial \go{l}(x)}{\partial x_{\alpha-1}} \right) \right\| + \left\| \left( \dot{\sigma} \left( \ga{l}(x) \right) - \dot{\sigma} \left( \go{l}(x) \right) \right) \odot \frac{\partial \go{l}(x)}{\partial x_{\alpha-1}} \right\|
    \\
    &\quad \le M_1 \left\| \left( \frac{\partial \ga{l}(x)}{\partial x_{\alpha-1}} - \frac{\partial \go{l}(x)}{\partial x_{\alpha-1}} \right) \right\| + M_2 \left\| \left( \ga{l}(x) - \go{l}(x) \right) \odot \frac{\partial \go{l}(x)}{\partial x_{\alpha-1}} \right\| 
    \\
    &\quad \le M_1 \left\| \left( \frac{\partial \ga{l}(x)}{\partial x_{\alpha-1}} - \frac{\partial \go{l}(x)}{\partial x_{\alpha-1}} \right) \right\| + M_2 \left\| \left( \ga{l}(x) - \go{l}(x) \right) \right\| O(\exp(L) \sqrt{\log d})
    \\
    &\quad \le M_1 \left\| \left( \frac{\partial \ga{l}(x)}{\partial x_{\alpha-1}} - \frac{\partial \go{l}(x)}{\partial x_{\alpha-1}} \right) \right\| + M_2 O\left( \exp(L) \sqrt{d \log d} \right)
    \\
    &\left\|\frac{\partial \ga{l}(x)}{\partial x_{\alpha-1}} - \frac{\partial \go{l}(x)}{\partial x_{\alpha-1}}\right\| 
    \\
    &\quad \le (C + \omega) \left\| \frac{\partial \ha{l-1}(x)}{\partial x_{\alpha-1}} - \frac{\partial \ho{l-1}(x)}{\partial x_{\alpha-1}} \right\| + (C M_1)^{l-1} \omega \sqrt{d}.
  \end{align*}
  Here, we applied Lemma~\ref{lem:supnormbound} to handle the bound for the Hadamard product, which requires $d \ge \Omega(L/\delta)$ assumption.

  Summing up, we have
  \begin{align*}
    \left\|\frac{\partial \ha{l}}{\partial x_{\alpha-1}} - \frac{\partial \ho{l}}{\partial x_{\alpha-1}}\right\| \le \exp(O(L)) \omega \sqrt{d \log d}.
  \end{align*}
\end{proof}

\begin{lemma}
  Fix the fail probability $\delta > 0$ and the network depth $L$.
  Suppose that $\|\WA{l} - \WO{l} \|\le \omega \sqrt{d}$ for all $l$, $\|\WA{l} - \WO{l}\|_\infty \le \omega \sqrt{d} \log d$ for all $2 \le l \le L$ and $\|\WA{1} - \WO{1}\|_\infty \le \omega \log d$.
  If $\omega \le 1$, and the network width satisfies $d \ge L/\delta$, then 
  \[
    \left\|\frac{\partial \ha{l}}{\partial x_{\alpha-1}} - \frac{\partial \ho{l}}{\partial x_{\alpha-1}}\right\|_\infty \le \omega \exp(O(L)) (\log d)^{2L}
  \]
  holds for all $l$ and $\alpha \in [d_0]$ with probability at least $1 - \delta$.
\end{lemma}
\begin{proof}
  We can prove this by simple analysis,
  \begin{align*}
    &\left\|\frac{\partial \ga{1}(x)}{\partial x_{\alpha-1}} - \frac{\partial \go{1}(x)}{\partial x_{\alpha-1}}\right\|_\infty 
    \\
    &\quad \le \omega \log d,
    \\
    &\left\| \frac{\partial \ha{l}(x)}{\partial x_{\alpha-1}} - \frac{\partial \ho{l}(x)}{\partial x_{\alpha-1}} \right\|_\infty 
    \\
    &\quad \le \left\| \dot{\sigma}\left( \ga{l}(x) \right) \odot \left( \frac{\partial \ga{l}(x)}{\partial x_{\alpha-1}} - \frac{\partial \go{l}(x)}{\partial x_{\alpha-1}} \right) \right\|_\infty 
    \\
    &\qquad + \left\| \left( \dot{\sigma} \left( \ga{l}(x) \right) - \dot{\sigma} \left( \go{l}(x) \right) \right) \odot \frac{\partial \go{l}(x)}{\partial x_{\alpha-1}} \right\|_\infty
    \\
    &\le M_1 \left\| \frac{\partial \ga{l}(x)}{\partial x_{\alpha-1}} - \frac{\partial \go{l}(x)}{\partial x_{\alpha-1}} \right\|_\infty + M_2 \left\| \ga{l}(x) - \go{l}(x) \right\|_\infty \left\| \frac{\partial \go{l}(x)}{\partial x_{\alpha-1}} \right\|_\infty
    \\
    &\le M_1 \left\| \frac{\partial \ga{l}(x)}{\partial x_{\alpha-1}} - \frac{\partial \go{l}(x)}{\partial x_{\alpha-1}} \right\|_\infty + M_2 \omega \exp(O(L)) (\log d)^L \sqrt{\log d}
    \\
    &\left\|\frac{\partial \ga{l}(x)}{\partial x_{\alpha-1}} - \frac{\partial \go{l}(x)}{\partial x_{\alpha-1}}\right\|_\infty 
    \\
    &\quad \le \left\| \frac{1}{\sqrt{d}} \WA{l} \right\|_\infty \left\| \frac{\partial \ha{l-1}(x)}{\partial x_{\alpha-1}} - \frac{\partial \ho{l-1}(x)}{\partial x_{\alpha-1}} \right\|_\infty + \frac{1}{\sqrt{d}} \left\| \WA{l} - \WO{l} \right\|_\infty \left\| \frac{\partial \ho{l-1}(x)}{\partial x_{\alpha-1}} \right\|_\infty
    \\
    &\le C \log d \left\| \frac{\partial \ha{l-1}(x)}{\partial x_{\alpha-1}} - \frac{\partial \ho{l-1}(x)}{\partial x_{\alpha-1}} \right\|_\infty + \omega (\log d)^{3/2} \exp(O(L))
  \end{align*}
\end{proof}

\begin{lemma} \label{lem:perturbback}
  Fix the fail probability $\delta > 0$ and the network depth $L$.
  Suppose that $\|\WA{l} - \WO{l} \| \le \omega \sqrt{d}$ for all $l \in [L+1]$.
  If $\omega \le 1$, and the network width satisfies both $d \ge \Omega(1/\delta)$ and $d \ge \exp(L)$, then
  \[
    \left\|\frac{\partial f_{\mathcal{A}}(x)}{\partial \ga{l}} - \frac{\partial f_{0}(x)}{\partial \go{l}}\right\| \le \kappa \omega \exp(O(L)) O(\sqrt{\log d})
  \]
  holds for all $l \in [L]$ with probability at least $1- \delta$.
\end{lemma}
\begin{proof}
  We prove this by a similar approach to Lemma~\ref{lem:perturbjacobian},
  \begin{align*}
    &\left\| \frac{\partial f_\mathcal{A}(x)}{\partial \ha{L}(x)} - \frac{\partial f_0(x)}{\partial \ho{L}(x)} \right\| \le \kappa \omega,
    \\
    &\left\| \frac{\partial f_\mathcal{A}(x)}{\partial \ga{l}(x)} - \frac{\partial f_0(x)}{\partial \go{l}}(x) \right\| 
    \\
    &\quad \le \left\|\left( \frac{\partial f_{\mathcal{A}}(x)}{\partial \ha{l}(x)} - \frac{\partial f_0(x)}{\partial \ho{l}(x)} \right) \odot \dot{\sigma} \left( \ga{l}(x) \right)\right\| 
    + \left\|\left( \dot{\sigma} \left( \ga{l}(x) \right) - \dot{\sigma} \left( \go{l}(x) \right) \right) \odot \frac{\partial f_0(x)}{\partial \ho{l}(x)} \right\|
    \\
    &\quad \le M_1 \left\| \frac{\partial f_{\mathcal{A}}(x)}{\partial \ha{l}(x)} - \frac{\partial f_0(x)}{\partial \ho{l}(x)} \right\| + M_2 \left\| \ga{l}(x) - \go{l}(x) \right\| \frac{\kappa}{\sqrt{d}} O\left( \exp(L) \sqrt{\log d} \right)
    \\
    &\quad \le M_1 \left\| \frac{\partial f_{\mathcal{A}}(x)}{\partial \ha{l}(x)} - \frac{\partial f_0(x)}{\partial \ho{l}(x)} \right\| + \kappa M_2 \omega O\left( \exp(L) \sqrt{\log d} \right)
    \\
    &\left\| \frac{\partial f_\mathcal{A}(x)}{\partial \ha{l}(x)} - \frac{\partial f_0(x)}{\partial \ho{l}}(x) \right\| 
    \\
    &\quad \le 3C \left\| \frac{\partial f_{\mathcal{A}}(x)}{\partial \ga{l+1}(x)} - \frac{\partial f_0(x)}{\partial \go{l+1}}(x) \right\|.
  \end{align*}
\end{proof}

\begin{lemma}
  Fix the fail probability $\delta > 0$ and the network depth $L$.
  Suppose that $\|\WA{l} - \WO{l} \| \le \omega \sqrt{d}$ for all $l \in [L+1]$, $\|\WA{l} - \WO{l}\|_1 \le \omega \sqrt{d} \log d$ for $2 \le l \le L$, and $\|\WA{L+1} - \WO{L+1}\|_1 \le \omega \log d$.
  If $\omega \le 1$ and the network width satisfies $d \ge \Omega(L/\delta)$, then
  \[
    \left\|\frac{\partial f_{\mathcal{A}}(x)}{\partial \ga{l}} - \frac{\partial f_{0}(x)}{\partial \go{l}}\right\|_\infty \le \kappa \omega \exp(O(L)) \frac{(\log d)^{2L}}{\sqrt{d}}
  \]
  holds for all $l \in [L]$ with probability at least $1 - \delta$.
\end{lemma}
\begin{proof}
  \begin{align*}
    &\left\| \frac{\partial f_{\mathcal{A}}(x)}{\partial \ha{L}(x)} - \frac{\partial f_0(x)}{\partial \ho{L}(x)}\right\|_\infty \le \kappa \omega \frac{\log d}{\sqrt{d}},
    \\
    &\left\| \frac{\partial f_{\mathcal{A}}(x)}{\partial \ga{l}(x)} - \frac{\partial f_0(x)}{\partial \go{l}(x)}\right\|_\infty 
    \\ 
    &\quad \le \left\| \left( \frac{\partial f_{\mathcal{A}}(x)}{\partial \ha{l}(x)} - \frac{\partial f_0(x)}{\partial \ho{l}(x)} \right) \odot \dot{\sigma} \left( \ga{l}(x) \right)\right\|_\infty 
    \\
    &\qquad + \left\| \left( \dot{\sigma} \left( \ga{l}(x) \right) - \dot{\sigma} \left( \go{l}(x) \right) \right) \odot \frac{\partial f_0(x)}{\partial \ho{l}(x)} \right\|_\infty
    \\
    &\quad \le M_1 \left\| \left( \frac{\partial f_{\mathcal{A}}(x)}{\partial \ha{l}(x)} - \frac{\partial f_0(x)}{\partial \ho{l}(x)} \right) \right\|_\infty + M_2 \left\| \ga{l}(x) - \go{l}(x)  \right\|_\infty \left\| \frac{\partial f_0(x)}{\partial \ho{l}(x)} \right\|_\infty
    \\
    &\quad \le M_1 \left\| \left( \frac{\partial f_{\mathcal{A}}(x)}{\partial \ha{l}(x)} - \frac{\partial f_0(x)}{\partial \ho{l}(x)} \right) \right\|_\infty + M_2 \kappa \omega \exp(O(L)) (\log d)^L \frac{\sqrt{\log d}}{\sqrt{d}},
    \\
    &\left\| \frac{\partial f_{\mathcal{A}}(x)}{\partial \ha{l}(x)} - \frac{\partial f_0(x)}{\partial \ho{l}(x)} \right\|_\infty 
    \\
    &\quad \le \frac{1}{\sqrt{d}}\left\| \left(\WA{l+1} - \WO{l+1}\right)^\intercal \right\|_\infty \left\| \frac{\partial f_{\mathcal{A}}(x)}{\partial \ga{l+1}(x)} \right\|_\infty 
    \\
    &\qquad + \left\| \frac{1}{\sqrt{d}} \WO{l+1} \right\|_\infty \left\| \frac{\partial f_{\mathcal{A}}(x)}{\partial \ga{l+1}(x)}  - \frac{\partial f_{0}(x)}{\partial \go{l+1}(x)}\right\|_\infty 
    \\
    &\quad \le \kappa \omega \log d \frac{\sqrt{\log d}}{\sqrt{d}}\exp(O(L)) + C \log d \left\| \frac{\partial f_{\mathcal{A}}(x)}{\partial \ga{l+1}(x)}  - \frac{\partial f_{0}(x)}{\partial \go{l+1}(x)}\right\|_\infty 
  \end{align*}
\end{proof}

\begin{lemma} \label{lem:perturbjacobback}
  Fix the fail probability $\delta > 0$ and the network depth $L$.
  Suppose that $\|\WA{l} - \WO{l} \|\le \omega \sqrt{d}$ for all $l \in [L+1]$, $\|\WA{l} - \WO{l}\|_\infty \le \omega \sqrt{d} \log d$ for all $2 \le l \le L$, and $\|\WA{1} - \WO{l}\|_\infty \le \omega \log d$.
  If $\omega \le (\log d)^{-2L}$ and the network width satisfies both $d \ge \Omega(1 /\delta)$ and $d \ge \exp(\Omega(L))$, then
  \[
    \left\|\frac{\partial J(f_{\mathcal{A}})(x)_{\alpha-1}}{\partial \ga{l}} - \frac{\partial J(f_{0})(x)_{\alpha-1}}{\partial \go{l}}\right\| \le \kappa \omega \exp(O(L)) \log d
  \]
  holds for all $l \in [L]$ and $\alpha \in [d_0]$ with probability at least $1- \delta$.
\end{lemma}
\begin{proof}
  We first unroll the definitions,
  \begin{align}
    &\frac{\partial J(f_\mathcal{A})(x)_{\alpha-1}}{\partial \ha{L}(x)} - \frac{\partial J(f_0)(x)_{\alpha-1}}{\partial \ho{L}(x)} \nonumber
    \\
    = &\mathbf{0}, \nonumber
    \\
    &\frac{\partial J(f_{\mathcal{A}})(x)_{\alpha-1}}{\partial \ha{l}(x)} - \frac{\partial J(f_0)(x)_{\alpha-1}}{\partial \ho{l}(x)} \nonumber
    \\
    = &\frac{1}{\sqrt{d}}\left( \WA{l+1} - \WO{l+1} \right)^\intercal \frac{\partial J(f_{\mathcal{A}})(x)_{\alpha-1}}{\partial \ga{l+1}(x)} \nonumber 
    \\ &+ \frac{1}{\sqrt{d}} \left( \WO{l+1} \right)^\intercal \left( \frac{\partial J(f_\mathcal{A})(x)_{\alpha-1}}{\partial \ga{l+1}(x)} - \frac{\partial J(f_0)(x)_{\alpha-1}}{\partial \go{l+1}(x)}\right), \nonumber
    \\
    &\frac{\partial J(f_{\mathcal{A}})(x)_{\alpha-1}}{\partial \ga{l}(x)} - \frac{\partial J(f_0)(x)_{\alpha-1}}{\partial \go{l}(x)} \nonumber
    \\
    = &\left( \frac{\partial f_{\mathcal{A}}(x)_{\alpha-1}}{\partial \ha{l}(x)} - \frac{\partial f_0(x)}{\partial \ho{l}(x)} \right) \odot \frac{\partial \go{l}(x)}{\partial x_{\alpha-1}} \odot \ddot{\sigma} \left( \ga{l} \right) \label{eqn:eqn28}
    \\
    &+ \left( \frac{\partial f_{\mathcal{A}}(x)_{\alpha-1}}{\partial \ha{l}(x)} - \frac{\partial f_0(x)}{\partial \ho{l}(x)} \right) \odot \left(\frac{\partial \ga{l}(x)}{\partial x_{\alpha-1}} - \frac{\partial \go{l}(x)}{\partial x_{\alpha-1}}\right) \odot \ddot{\sigma} \left( \ga{l} \right) \label{eqn:eqn29}
    \\
    &+ \frac{\partial f_0(x)}{\partial \ho{l}(x)} \odot \left(\frac{\partial \ga{l}(x)}{\partial x_{\alpha-1}} - \frac{\partial \go{l}(x)}{\partial x_{\alpha-1}}\right)  \odot \ddot{\sigma} \left( \ga{l} \right) \label{eqn:eqn30}
    \\
    &+ \frac{\partial f_0(x)}{\partial \ho{l}(x)} \odot \frac{\partial \go{l}(x)}{\partial x_{\alpha-1}} \odot \left( \ddot{\sigma}\left( \ga{l}\right) - \ddot{\sigma} \left(\go{l}\right) \right) \label{eqn:eqn31}
    \\
    &+ \left( \frac{\partial J(f_{\mathcal{A}})(x)_{\alpha-1}}{\partial \ha{l}} - \frac{\partial J(f_0)(x)_{\alpha-1}}{\partial \ho{l}} \right) \odot \dot{\sigma} \left( \ga{l} \right) \label{eqn:eqn32}
    \\
    &+ \frac{\partial J(f_0)(x)_{\alpha-1}}{\partial \ho{l}} \odot \left( \dot{\sigma} \left(\ga{l}\right) - \dot{\sigma} \left(\go{l}\right) \right). \label{eqn:eqn33}
  \end{align}
  Then we can derive 
  \begin{align*}
    \left\| \eqref{eqn:eqn28} \right\| &\le M_2 \kappa \omega \exp(O(L)) \log d, 
    \\
    \left\| \eqref{eqn:eqn29} \right\| &\le M_2 \kappa \omega^2 \exp(O(L)) (\log d)^{2L+1/2}, 
    \\
    \left\| \eqref{eqn:eqn30} \right\| &\le M_2 \kappa \omega \exp(O(L)) \log d, 
    \\
    \left\| \eqref{eqn:eqn31} \right\| &\le M_3 \kappa \omega O(\exp(L)) \log d, 
    \\
    \left\| \eqref{eqn:eqn32} \right\| &\le M_1 \left\| \frac{\partial J(f_{\mathcal{A}})(x)_{\alpha-1}}{\partial \ha{l}} - \frac{\partial J(f_0)(x)_{\alpha-1}}{\partial \ho{l}} \right\|, 
    \\
    \left\| \eqref{eqn:eqn33} \right\| &\le M_2 \kappa \omega O(\exp(L)) \sqrt{\log d}. 
  \end{align*}
  Summing up the results gives our final result.
\end{proof}

\begin{lemma}
  Fix the fail probability $\delta > 0$ and the network depth $L$.
  Suppose that $\|\WA{l} - \WO{l} \|\le \omega \sqrt{d}$ for all $l \in [L+1]$, $\|\WA{l} - \WO{l}\|_1 \le \omega \sqrt{d} \log d$ for $2 \le l \le L$, and $\|\WA{L+1} - \WO{L+1}\|_1 \le \omega \log d$.
  If $\omega \le (\log d)^{-L}$ and the network width satisfies $d \ge \Omega(L/\delta)$, then
  \[
    \left\|\frac{\partial J(f_{\mathcal{A}})(x)_{\alpha-1}}{\partial \ga{l}} - \frac{\partial J(f_{0})(x)_{\alpha-1}}{\partial \go{l}}\right\|_\infty \le \kappa \omega \exp(O(L)) \frac{(\log d)^{3L+1}}{\sqrt{d}}
  \]
  holds for all $l \in [L]$ and $\alpha \in [d_0]$ with probability at least $1 - \delta$.
\end{lemma}
\begin{proof}
  We have a similar decomposition as the previous lemma, which shows 
  \begin{align*}
    &\left\| \frac{\partial J(f_{\mathcal{A}})(x)_{\alpha-1}}{\partial \ha{l}(x)} - \frac{\partial J(f_0)(x)_{\alpha-1}}{\partial \ho{l}(x)} \right\|_\infty 
    \\
    &\quad \le \kappa \omega \frac{(\log d)^{3/2}}{\sqrt{d}} \exp(O(L)) + C \log d \left\| \frac{\partial J(f_\mathcal{A})(x)_{\alpha-1}}{\partial \ga{l+1}(x)} - \frac{\partial J(f_0)(x)_{\alpha-1}}{\partial \go{l+1}(x)} \right\|_\infty,
    \\
    &\left\| \ref{eqn:eqn28} \right\|_\infty \le M_2 \kappa \omega \exp(O(L)) \frac{(\log d)^{2L+1/2}}{\sqrt{d}},
    \\
    &\left\| \ref{eqn:eqn29} \right\|_\infty \le M_2 \kappa \omega^2 \exp(O(L)) \frac{(\log d)^{4L}}{\sqrt{d}},
    \\
    &\left\| \ref{eqn:eqn30} \right\|_\infty \le M_2 \kappa \omega \frac{(\log d)^{2L+1/2}}{\sqrt{d}} \exp(O(L)),
    \\
    &\left\| \ref{eqn:eqn31} \right\|_\infty \le M_3 \kappa \omega \frac{(\log d)^{L+1}}{\sqrt{d}} \exp(O(L)),
    \\
    &\left\| \ref{eqn:eqn32} \right\|_\infty \le M_1 \left\| \frac{\partial J(f_{\mathcal{A}})(x)_{\alpha-1}}{\partial \ha{l}} - \frac{\partial J(f_0)(x)_{\alpha-1}}{\partial \ho{l}} \right\|_\infty,
    \\
    &\left\| \ref{eqn:eqn33} \right\|_\infty \le M_2 \kappa \omega \frac{(\log d)^{L+1/2}}{\sqrt{d}} \exp(O(L)).
  \end{align*}
\end{proof}

As we've done in the previous section, we also give an asymptotic summary of the results.
If the width $d \ge \Omega(L \log(d_0)/\delta)$, and the weights after perturbation satisfy $\|\WA{l} - \WO{l}\| \le \omega \sqrt{d}$ and $\|\WA{l} - \WO{l}\|_\infty \le \omega \log d$, we have 
\begin{align}
  \left\|\ha{l} - \ho{l}\right\| &\le \omega \sqrt{d} O(\exp(L)), \label{eqn:perturbbegin}
  \\
  \left\| \frac{\partial \ha{l}(x)}{\partial x_{\alpha-1}} - \frac{\ho{l}(x)}{\partial x_{\alpha-1}} \right\| &\le \omega \sqrt{d \log d} \exp(O(L)),
  \\
  \left\| \frac{\partial f_{\mathcal{A}}(x)}{\partial \ga{l}(x)} - \frac{\partial f_0(x)}{\partial \go{l}(x)} \right\| &\le \kappa \omega \exp(O(L)) \sqrt{\log d},
  \\
  \left\| \frac{\partial J(f_{\mathcal{A}})(x)}{\partial \ga{l}(x)} - \frac{\partial J(f_0)(x)_{\alpha-1}}{\partial \go{l}(x)} \right\| &\le \kappa \omega \log d \exp(O(L)) \label{eqn:perturbend}
\end{align}
for all $l \in [L]$ and $\alpha \in [d_0]$ simultaneously, with probability at least $1 - \delta$.

\subsection{Proof of Theorem~\ref{thm:JNTKCloseToInit}}

Now we are ready to prove Theorem~\ref{thm:JNTKCloseToInit}.
\begin{proof}
  We first recall the decomposition of finite JNTK to its layer-wise contributions, 
  \begin{align*}
    &\Theta_{d,\theta_0}(x, x')_{00} 
    \\ 
    &\quad = \sum_{l=1}^{L+1} \frac{1}{C_{l-1}}\left\langle \ho{l-1}(x), \ho{l-1}(x') \right\rangle \left\langle \frac{\partial f_0(x)}{\partial \go{l}(x)}, \frac{\partial f_0(x')}{\partial \go{l}(x')} \right\rangle,
    \\
    &\Theta_{d,\theta}(x, x')_{\alpha 0} 
    \\
    & \quad = \sum_{l=1}^{L+1} \frac{1}{C_{l-1}}\left\langle \ho{l-1}(x), \ho{l-1}(x')\right\rangle\left\langle \frac{\partial J(f_0)(x)_{\alpha-1}}{\partial \go{l}(x)} , \frac{\partial f_0(x')}{\partial \go{l}(x')} \right\rangle
    \\
    &\quad + \sum_{l=1}^{L+1} \frac{1}{C_{l-1}} \left\langle \frac{\partial \ho{l-1}(x)}{\partial x_{\alpha-1}}, \ho{l-1}(x') \right\rangle \left\langle \frac{\partial f_0(x)}{\partial \go{l}(x)}, \frac{\partial f_0(x')}{\partial \go{l}(x')}  \right\rangle,
    \\
    &\Theta_{d,\theta}(x, x')_{0 \beta} 
    \\
    & \quad = \sum_{l=1}^{L+1} \frac{1}{C_{l-1}}\left\langle \ho{l-1}(x), \ho{l-1}(x')\right\rangle\left\langle \frac{\partial f_0(x)}{\partial \go{l}(x)} , \frac{\partial J(f_0)(x')_{\beta-1}}{\partial \go{l}(x')} \right\rangle 
    \\
    &\quad + \sum_{l=1}^{L+1} \frac{1}{C_{l-1}} \left\langle \ho{l-1}(x), \frac{\partial \ho{l-1}(x')}{\partial x_{\beta-1}'} \right\rangle \left\langle  \frac{\partial f_0(x)}{\partial \go{l}(x)}, \frac{\partial f_0(x')}{\partial \go{l}(x')} \right\rangle, 
    \\
    &\Theta_{d,\theta}(x, x')_{\alpha \beta} 
    \\
    &\quad = \sum_{l=1}^{L+1} \frac{1}{C_{l-1}} \left\langle \ho{l-1}(x), \ho{l-1}(x') \right\rangle \left\langle \frac{\partial J(f_0)(x)_{\alpha-1}}{\partial \go{l}(x)}, \frac{\partial J(f_0)(x')_{\beta-1}}{\partial \go{l}(x')} \right\rangle 
    \\
    &\quad + \sum_{l=1}^{L+1} \frac{1}{C_{l-1}} \left \langle \frac{\partial \ho{l-1}(x)}{\partial x_{\alpha-1}},\ho{l-1}(x') \right \rangle \left \langle\frac{\partial f_0(x)}{\partial \go{l}(x)}, \frac{\partial J(f_0)(x')_{\beta-1}}{\partial \go{l}(x')}\right \rangle 
    \\
    &\quad + \sum_{l=1}^{L+1} \frac{1}{C_{l-1}} \left \langle \ho{l-1}(x), \frac{\partial \ho{l-1}(x')}{\partial x_{\beta-1}'}\right \rangle \left \langle\frac{\partial J(f_0)(x)_{\alpha-1}}{\partial \go{l}(x)}, \frac{\partial f_0(x')}{\partial \go{l}(x')}\right \rangle 
    \\
    &\quad + \sum_{l=1}^{L+1} \frac{1}{C_{l-1}} \left \langle\frac{\partial \ho{l-1}(x)}{\partial x_{\alpha-1}}, \frac{\partial \ho{l-1}(x')}{\partial x_{\beta-1}'}\right \rangle \left \langle\frac{\partial f_0(x)}{\partial \go{l}(x)}, \frac{\partial f_0(x')}{\partial \go{l}(x')}\right \rangle. 
  \end{align*}
  We need to bound its difference $\Theta_{d, \theta_{\mathcal{A}}}(x, x')$ where $\theta_{\mathcal{A}} = \{W_\mathcal{A}^{(l)}\}_{l=1}^{L+1}$ assuming that weights are close to their initialisation.

  We will bound each summand in the above decomposition, which adds $\times (L+1)$ multiplicative coefficient to the bound of finite JNTK difference.
  By triangle inequality, we can bound each summands in the first term.
  \begin{align*}
    &\Bigg|\frac{1}{C_{l-1}}\left\langle \ho{l-1}(x), \ho{l-1}(x') \right\rangle \left\langle \frac{\partial f_0(x)}{\partial \go{l}(x)}, \frac{\partial f_0(x')}{\partial \go{l}(x')} \right\rangle 
    \\
    &\quad - \frac{1}{C_{l-1}}\left\langle \ha{l-1}(x), \ha{l-1}(x') \right\rangle \left\langle \frac{\partial f_\mathcal{A}(x)}{\partial \ga{l}(x)}, \frac{\partial f_\mathcal{A}(x')}{\partial \ga{l}(x')} \right\rangle\Bigg|
    \\ 
    \le &\Bigg|\frac{1}{C_{l-1}}\left\langle \ho{l-1}(x), \ho{l-1}(x') \right\rangle \left\langle \frac{\partial f_0(x)}{\partial \go{l}(x)}, \frac{\partial f_0(x')}{\partial \go{l}(x')} \right\rangle 
    \\
    &\qquad - \frac{1}{C_{l-1}}\left\langle \ha{l-1}(x), \ho{l-1}(x') \right\rangle \left\langle \frac{\partial f_0(x)}{\partial \go{l}(x)}, \frac{\partial f_0(x')}{\partial \go{l}(x')} \right\rangle\Bigg|
    \\
    &+ \Bigg|\frac{1}{C_{l-1}}\left\langle \ha{l-1}(x), \ho{l-1}(x') \right\rangle \left\langle \frac{\partial f_0(x)}{\partial \go{l}(x)}, \frac{\partial f_0(x')}{\partial \go{l}(x')} \right\rangle 
    \\
    &\qquad - \frac{1}{C_{l-1}}\left\langle \ha{l-1}(x), \ha{l-1}(x') \right\rangle \left\langle \frac{\partial f_0(x)}{\partial \go{l}(x)}, \frac{\partial f_0(x')}{\partial \go{l}(x')} \right\rangle\Bigg|
    \\
    &+ \Bigg|\frac{1}{C_{l-1}}\left\langle \ha{l-1}(x), \ha{l-1}(x') \right\rangle \left\langle \frac{\partial f_0(x)}{\partial \go{l}(x)}, \frac{\partial f_0(x')}{\partial \go{l}(x')} \right\rangle 
    \\
    &\qquad - \frac{1}{C_{l-1}}\left\langle \ha{l-1}(x), \ha{l-1}(x') \right\rangle \left\langle \frac{\partial f_\mathcal{A}(x)}{\partial \ga{l}(x)}, \frac{\partial f_0(x')}{\partial \go{l}(x')} \right\rangle\Bigg|
    \\
    &+ \Bigg|\frac{1}{C_{l-1}}\left\langle \ha{l-1}(x), \ha{l-1}(x') \right\rangle \left\langle \frac{\partial f_\mathcal{A}(x)}{\partial \ga{l}(x)}, \frac{\partial f_0(x')}{\partial \go{l}(x')} \right\rangle 
    \\
    &\qquad - \frac{1}{C_{l-1}}\left\langle \ha{l-1}(x), \ha{l-1}(x') \right\rangle \left\langle \frac{\partial f_\mathcal{A}(x)}{\partial \ga{l}(x)}, \frac{\partial f_\mathcal{A}(x')}{\partial \ga{l}(x')} \right\rangle\Bigg|
    \\
    \le & \frac{1}{C_{l-1}}\left\| \ho{l-1}(x) - \ha{l-1}(x) \right\| \left\| \ho{l-1}(x') \right\| \left\| \frac{\partial f_0(x)}{\partial \go{l}(x)} \right\| \left\| \frac{\partial f_0(x')}{\partial \go{l}(x')} \right\|
    \\
    &+ \frac{1}{C_{l-1}} \left\| \ha{l-1}(x) \right\| \left\| \ho{l-1}(x') - \ha{l-1}(x') \right\| \left\| \frac{\partial f_0(x)}{\partial \go{l}(x)} \right\| \left\| \frac{\partial f_0(x')}{\partial \go{l}(x')} \right\|
    \\
    &+ \frac{1}{C_{l-1}} \left\| \ha{l-1}(x) \right\| \left\| \ha{l-1}(x') \right\| \left\| \frac{\partial f_0(x)}{\partial \go{l}(x)} - \frac{\partial f_\mathcal{A}(x)}{\partial \ga{l}(x)} \right\| \left\| \frac{\partial f_0(x')}{\partial \go{l}(x')} \right\|
    \\
    &+ \frac{1}{C_{l-1}} \left\| \ha{l-1}(x) \right\| \left\| \ha{l-1}(x') \right\| \left\| \frac{\partial f_{\mathcal{A}}(x)}{\partial \ga{l}(x)} \right\| \left\| \frac{\partial f_0(x')}{\partial \go{l}(x')} - \frac{\partial f_\mathcal{A}(x')}{\partial \ga{l}(x')} \right\|,
  \end{align*}
  which rewrites the difference by the terms we've bounded in Appendices~\ref{subsec:boundinit} and \ref{subsec:boundperturb}.

  We can bound the rest of the terms as follows:
  \begin{align*}
    &\Bigg| \frac{1}{C_{l-1}} \left\langle \ho{l-1}(x), \ho{l-1}(x') \right\rangle \left\langle \frac{\partial J(f_0)(x)_{\alpha-1}}{\partial \go{l}(x)}, \frac{\partial f_0(x')}{\partial \go{l}(x')} \right\rangle
    \\
    &\quad - \frac{1}{C_{l-1}} \left\langle \ha{l-1}(x), \ha{l-1}(x') \right\rangle \left\langle \frac{\partial J(f_\mathcal{A})(x)_{\alpha-1}}{\partial \ga{l}(x)}, \frac{\partial f_\mathcal{A}(x')}{\partial \ga{l}(x')} \right\rangle \Bigg|
    \\
    \le & \frac{1}{C_{l-1}}\left\| \ho{l-1}(x) - \ha{l-1}(x) \right\| \left\| \ho{l-1}(x') \right\| \left\| \frac{\partial J(f_0)(x)_{\alpha-1}}{\partial \go{l}(x)} \right\| \left\| \frac{\partial f_0(x')}{\partial \go{l}(x')} \right\|
    \\
    &+ \frac{1}{C_{l-1}} \left\| \ha{l-1}(x) \right\| \left\| \ho{l-1}(x') - \ha{l-1}(x') \right\| \left\| \frac{\partial J(f_0)(x)_{\alpha-1}}{\partial \go{l}(x)} \right\| \left\| \frac{\partial f_0(x')}{\partial \go{l}(x')} \right\|
    \\
    &+ \frac{1}{C_{l-1}} \left\| \ha{l-1}(x) \right\| \left\| \ha{l-1}(x') \right\| \left\| \frac{\partial J(f_0)(x)_{\alpha-1}}{\partial \go{l}(x)} - \frac{\partial J(f_\mathcal{A})(x)_{\alpha-1}}{\partial \ga{l}(x)} \right\| \left\| \frac{\partial f_0(x')}{\partial \go{l}(x')} \right\|
    \\
    &+ \frac{1}{C_{l-1}} \left\| \ha{l-1}(x) \right\| \left\| \ha{l-1}(x') \right\| \left\| \frac{\partial J(f_{\mathcal{A}})(x)_{\alpha-1}}{\partial \ga{l}(x)} \right\| \left\| \frac{\partial f_0(x')}{\partial \go{l}(x')} - \frac{\partial f_\mathcal{A}(x')}{\partial \ga{l}(x')} \right\|,
    \\
    &\Bigg| \frac{1}{C_{l-1}} \left\langle \frac{\partial \ho{l-1}(x)}{\partial x_{\alpha-1}}, \ho{l-1}(x') \right\rangle \left\langle \frac{\partial f_0(x)}{\partial \go{l}(x)}, \frac{\partial f_0(x')}{\partial \go{l}(x')}  \right\rangle
    \\
    &\quad - \frac{1}{C_{l-1}} \left\langle \frac{\partial \ha{l-1}(x)}{\partial x_{\alpha-1}}, \ha{l-1}(x') \right\rangle \left\langle \frac{\partial f_\mathcal{A}(x)}{\partial \ga{l}(x)}, \frac{\partial f_\mathcal{A}(x')}{\partial \ga{l}(x')}  \right\rangle \Bigg|
    \\
    \le & \frac{1}{C_{l-1}}\left\| \frac{\partial \ho{l-1}(x)}{\partial x_{\alpha-1}} - \frac{\partial \ha{l-1}(x)}{\partial x_{\alpha-1}} \right\| \left\| \ho{l-1}(x') \right\| \left\| \frac{\partial f_0(x)}{\partial \go{l}(x)} \right\| \left\| \frac{\partial f_0(x')}{\partial \go{l}(x')} \right\|
    \\
    &+ \frac{1}{C_{l-1}} \left\| \frac{\partial \ha{l-1}(x)}{\partial x_{\alpha-1}} \right\| \left\| \ho{l-1}(x') - \ha{l-1}(x') \right\| \left\| \frac{\partial f_0(x)}{\partial \go{l}(x)} \right\| \left\| \frac{\partial f_0(x')}{\partial \go{l}(x')} \right\|
    \\
    &+ \frac{1}{C_{l-1}} \left\| \frac{\partial \ha{l-1}(x)}{\partial x_{\alpha-1}} \right\| \left\| \ha{l-1}(x') \right\| \left\| \frac{\partial f_0(x)}{\partial \go{l}(x)} - \frac{\partial f_\mathcal{A}(x)}{\partial \ga{l}(x)} \right\| \left\| \frac{\partial f_0(x')}{\partial \go{l}(x')} \right\|
    \\
    &+ \frac{1}{C_{l-1}} \left\| \frac{\partial \ha{l-1}(x)}{\partial x_{\alpha-1}} \right\| \left\| \ha{l-1}(x') \right\| \left\| \frac{\partial f_{\mathcal{A}}(x)}{\partial \ga{l}(x)} \right\| \left\| \frac{\partial f_0(x')}{\partial \go{l}(x')} - \frac{\partial f_\mathcal{A}(x')}{\partial \ga{l}(x')} \right\|,
    \\
    &\Bigg| \frac{1}{C_{l-1}}\left\langle \ho{l-1}(x), \ho{l-1}(x')\right\rangle\left\langle \frac{\partial f_0(x)}{\partial \go{l}(x)} , \frac{\partial J(f_0)(x')_{\beta-1}}{\partial \go{l}(x')} \right\rangle
    \\
    &\quad - \frac{1}{C_{l-1}}\left\langle \ha{l-1}(x), \ha{l-1}(x')\right\rangle\left\langle \frac{\partial f_\mathcal{A}(x)}{\partial \ga{l}(x)} , \frac{\partial J(f_\mathcal{A})(x')_{\beta-1}}{\partial \ga{l}(x')} \right\rangle \Bigg|
    \\
    \le & \frac{1}{C_{l-1}}\left\| \ho{l-1}(x) - \ha{l-1}(x) \right\| \left\| \ho{l-1}(x') \right\| \left\| \frac{\partial f_0(x)}{\partial \go{l}(x)} \right\| \left\| \frac{\partial J(f_0)(x')_{\beta-1}}{\partial \go{l}(x')} \right\|
    \\
    &+ \frac{1}{C_{l-1}} \left\| \ha{l-1}(x) \right\| \left\| \ho{l-1}(x') - \ha{l-1}(x') \right\| \left\| \frac{\partial f_0(x)}{\partial \go{l}(x)} \right\| \left\| \frac{\partial J(f_0)(x')_{\beta-1}}{\partial \go{l}(x')} \right\|
    \\
    &+ \frac{1}{C_{l-1}} \left\| \ha{l-1}(x) \right\| \left\| \ha{l-1}(x') \right\| \left\| \frac{\partial f_0(x)}{\partial \go{l}(x)} - \frac{\partial f_\mathcal{A}(x)}{\partial \ga{l}(x)} \right\| \left\| \frac{\partial J(f_0)(x')_{\beta-1}}{\partial \go{l}(x')} \right\|
    \\
    &+ \frac{1}{C_{l-1}} \left\| \ha{l-1}(x) \right\| \left\| \ha{l-1}(x') \right\| \left\| \frac{\partial f_{\mathcal{A}}(x)}{\partial \ga{l}(x)} \right\| \left\| \frac{\partial J(f_0)(x')_{\beta-1}}{\partial \go{l}(x')} - \frac{\partial J(f_{\mathcal{A}})(x')_{\beta-1}}{\partial \ga{l}(x')} \right\|,
    \\
    &\Bigg| \frac{1}{C_{l-1}} \left\langle \ho{l-1}(x), \frac{\partial \ho{l-1}(x')}{\partial x_{\beta-1}'} \right\rangle \left\langle  \frac{\partial f_0(x)}{\partial \go{l}(x)}, \frac{\partial f_0(x')}{\partial \go{l}(x')} \right\rangle
    \\
    &\quad - \frac{1}{C_{l-1}} \left\langle \ha{l-1}(x), \frac{\partial \ha{l-1}(x')}{\partial x_{\beta-1}'} \right\rangle \left\langle  \frac{\partial f_\mathcal{A}(x)}{\partial \ga{l}(x)}, \frac{\partial f_\mathcal{A}(x')}{\partial \ga{l}(x')} \right\rangle\Bigg|
    \\
    &\le \frac{1}{C_{l-1}}\left\| \ho{l-1}(x) - \ha{l-1}(x) \right\| \left\| \frac{\partial \ho{l-1}(x')}{\partial x_\beta'} \right\| \left\| \frac{\partial f_0(x)}{\partial \go{l}(x)} \right\| \left\| \frac{\partial f_0(x')}{\partial \go{l}(x')} \right\|
    \\
    &+ \frac{1}{C_{l-1}} \left\| \ha{l-1}(x) \right\| \left\| \frac{\partial \ho{l-1}(x')}{\partial x_\beta'} - \frac{\partial \ha{l-1}(x')}{\partial x_\beta'} \right\| \left\| \frac{\partial f_0(x)}{\partial \go{l}(x)} \right\| \left\| \frac{\partial f_0(x')}{\partial \go{l}(x')} \right\|
    \\
    &+ \frac{1}{C_{l-1}} \left\| \ha{l-1}(x) \right\| \left\| \frac{\partial \ha{l-1}(x')}{\partial x_\beta'} \right\| \left\| \frac{\partial f_0(x)}{\partial \go{l}(x)} - \frac{\partial f_\mathcal{A}(x)}{\partial \ga{l}(x)} \right\| \left\| \frac{\partial f_0(x')}{\partial \go{l}(x')} \right\|
    \\
    &+ \frac{1}{C_{l-1}} \left\| \ha{l-1}(x) \right\| \left\| \frac{\partial \ha{l-1}(x')}{\partial x_\beta'} \right\| \left\| \frac{\partial f_{\mathcal{A}}(x)}{\partial \ga{l}(x)} \right\| \left\| \frac{\partial f_0(x')}{\partial \go{l}(x')} - \frac{\partial f_\mathcal{A}(x')}{\partial \ga{l}(x')} \right\|,
    \\
    &\Bigg| \frac{1}{C_{l-1}} \left\langle \ho{l-1}(x), \ho{l-1}(x') \right\rangle \left\langle \frac{\partial J(f_0)(x)_{\alpha-1}}{\partial \go{l}(x)}, \frac{\partial J(f_0)(x')_{\beta-1}}{\partial \go{l}(x')} \right\rangle 
    \\
    &\quad - \frac{1}{C_{l-1}} \left\langle \ha{l-1}(x), \ha{l-1}(x') \right\rangle \left\langle \frac{\partial J(f_\mathcal{A})(x)_{\alpha-1}}{\partial \ga{l}(x)}, \frac{\partial J(f_\mathcal{A})(x')_{\beta-1}}{\partial \ga{l}(x')} \right\rangle \Bigg|
    \\
    &\le \frac{1}{C_{l-1}}\left\| \ho{l-1}(x) - \ha{l-1}(x) \right\| \left\| \ho{l-1}(x') \right\| \left\| \frac{\partial J(f_0)(x)_{\alpha-1}}{\partial \go{l}(x)} \right\| \left\| \frac{\partial J(f_0)(x')_{\beta-1}}{\partial \go{l}(x')} \right\|
    \\
    &+ \frac{1}{C_{l-1}} \left\| \ha{l-1}(x) \right\| \left\| \ho{l-1}(x') - \ha{l-1}(x') \right\| \left\| \frac{\partial J(f_0)(x)_{\alpha-1}}{\partial \go{l}(x)} \right\| \left\| \frac{\partial J(f_0)(x')_{\beta-1}}{\partial \go{l}(x')} \right\|
    \\
    &+ \frac{1}{C_{l-1}} \left\| \ha{l-1}(x) \right\| \left\| \ha{l-1}(x') \right\| \left\| \frac{\partial J(f_0)(x)_{\alpha-1}}{\partial \go{l}(x)} - \frac{\partial J(f_\mathcal{A})(x)_{\alpha-1}}{\partial \ga{l}(x)} \right\| \left\| \frac{\partial J(f_0)(x')_{\beta-1}}{\partial \go{l}(x')} \right\|
    \\
    &+ \frac{1}{C_{l-1}} \left\| \ha{l-1}(x) \right\| \left\| \ha{l-1}(x') \right\| \left\| \frac{\partial J(f_{\mathcal{A}})(x)_{\alpha-1}}{\partial \ga{l}(x)} \right\| \left\| \frac{\partial J(f_0)(x')_{\beta-1}}{\partial \go{l}(x')} - \frac{\partial J(f_\mathcal{A})(x')_{\beta-1}}{\partial \ga{l}(x')} \right\|, 
    \\
    &\Bigg| \frac{1}{C_{l-1}} \left \langle \frac{\partial \ho{l-1}(x)}{\partial x_{\alpha-1}},\ho{l-1}(x') \right \rangle \left \langle\frac{\partial f_0(x)}{\partial \go{l}(x)}, \frac{\partial J(f_0)(x')_{\beta-1}}{\partial \go{l}(x')}\right \rangle
    \\
    &\quad - \frac{1}{C_{l-1}} \left \langle \frac{\partial \ha{l-1}(x)}{\partial x_{\alpha-1}},\ha{l-1}(x') \right \rangle \left \langle\frac{\partial f_\mathcal{A}(x)}{\partial \ga{l}(x)}, \frac{\partial J(f_\mathcal{A})(x')_{\beta-1}}{\partial \ga{l}(x')}\right \rangle \Bigg|
    \\
    &\le \frac{1}{C_{l-1}}\left\| \frac{\partial \ho{l-1}(x)}{\partial x_{\alpha-1}} - \frac{\partial \ha{l-1}(x)}{\partial x_{\alpha-1}} \right\| \left\| \ho{l-1}(x') \right\| \left\| \frac{\partial f_0(x)}{\partial \go{l}(x)} \right\| \left\| \frac{\partial J(f_0)(x')_{\beta-1}}{\partial \go{l}(x')} \right\|
    \\
    &+ \frac{1}{C_{l-1}} \left\| \frac{\partial \ha{l-1}(x)}{\partial x_{\alpha-1}} \right\| \left\| \ho{l-1}(x') - \ha{l-1}(x') \right\| \left\| \frac{\partial f_0(x)}{\partial \go{l}(x)} \right\| \left\| \frac{\partial J(f_0)(x')_{\beta-1}}{\partial \go{l}(x')} \right\|
    \\
    &+ \frac{1}{C_{l-1}} \left\| \frac{\partial \ha{l-1}(x)}{\partial x_{\alpha-1}} \right\| \left\| \ha{l-1}(x') \right\| \left\| \frac{\partial f_0(x)}{\partial \go{l}(x)} - \frac{\partial f_\mathcal{A}(x)}{\partial \ga{l}(x)} \right\| \left\| \frac{\partial J(f_0)(x')_{\beta-1}}{\partial \go{l}(x')} \right\|
    \\
    &+ \frac{1}{C_{l-1}} \left\| \frac{\partial \ha{l-1}(x)}{\partial x_{\alpha-1}} \right\| \left\| \ha{l-1}(x') \right\| \left\| \frac{\partial f_{\mathcal{A}}(x)}{\partial \ga{l}(x)} \right\| \left\| \frac{\partial J(f_0)(x')_{\beta-1}}{\partial \go{l}(x')} -\frac{\partial J(f_\mathcal{A})(x')_{\beta-1}}{\partial \ga{l}(x')} \right\|, 
    \\
    &\Bigg| \frac{1}{C_{l-1}} \left \langle \ho{l-1}(x), \frac{\partial \ho{l-1}(x')}{\partial x_{\beta-1}'}\right \rangle \left \langle\frac{\partial J(f_0)(x)_{\alpha-1}}{\partial \go{l}(x)}, \frac{\partial f_0(x')}{\partial \go{l}(x')}\right \rangle
    \\
    &\quad - \frac{1}{C_{l-1}} \left \langle \ha{l-1}(x), \frac{\partial \ha{l-1}(x')}{\partial x_{\beta-1}'}\right \rangle \left \langle\frac{\partial J(f_\mathcal{A})(x)_{\alpha-1}}{\partial \ga{l}(x)}, \frac{\partial f_\mathcal{A}(x')}{\partial \ga{l}(x')}\right \rangle \Bigg|
    \\
    &\le \frac{1}{C_{l-1}}\left\| \ho{l-1}(x) - \ha{l-1}(x) \right\| \left\| \frac{\partial \ho{l-1}(x')}{\partial x_\beta'} \right\| \left\| \frac{\partial J(f_0)(x)_{\alpha-1}}{\partial \go{l}(x)} \right\| \left\| \frac{\partial f_0(x')}{\partial \go{l}(x')} \right\|
    \\
    &+ \frac{1}{C_{l-1}} \left\| \ha{l-1}(x) \right\| \left\| \frac{\partial \ho{l-1}(x')}{\partial x_\beta'} - \frac{\partial \ha{l-1}(x')}{\partial x_\beta'} \right\| \left\| \frac{\partial J(f_0)(x)_{\alpha-1}}{\partial \go{l}(x)} \right\| \left\| \frac{\partial f_0(x')}{\partial \go{l}(x')} \right\|
    \\
    &+ \frac{1}{C_{l-1}} \left\| \ha{l-1}(x) \right\| \left\| \frac{\partial \ha{l-1}(x')}{\partial x_\beta'} \right\| \left\| \frac{\partial J(f_0)(x)_{\alpha-1}}{\partial \go{l}(x)} - \frac{\partial J(f_\mathcal{A})(x)_{\alpha-1}}{\partial \ga{l}(x)} \right\| \left\| \frac{\partial f_0(x')}{\partial \go{l}(x')} \right\|
    \\
    &+ \frac{1}{C_{l-1}} \left\| \ha{l-1}(x) \right\| \left\| \frac{\partial \ha{l-1}(x')}{\partial x_\beta'} \right\| \left\| \frac{\partial J(f_\mathcal{A})(x)_{\alpha-1}}{\partial \ga{l}(x)} \right\| \left\| \frac{\partial f_0(x')}{\partial \go{l}(x')} - \frac{\partial f_\mathcal{A}(x')}{\partial \ga{l}(x')} \right\|,
    \\
    &\Bigg| \frac{1}{C_{l-1}} \left \langle\frac{\partial \ho{l-1}(x)}{\partial x_{\alpha-1}}, \frac{\partial \ho{l-1}(x')}{\partial x_{\beta-1}'}\right \rangle \left \langle\frac{\partial f_0(x)}{\partial \go{l}(x)}, \frac{\partial f_0(x')}{\partial \go{l}(x')}\right \rangle
    \\
    &\quad - \frac{1}{C_{l-1}} \left \langle\frac{\partial \ha{l-1}(x)}{\partial x_{\alpha-1}}, \frac{\partial \ha{l-1}(x')}{\partial x_{\beta-1}'}\right \rangle \left \langle\frac{\partial f_\mathcal{A}(x)}{\partial \ga{l}(x)}, \frac{\partial f_\mathcal{A}(x')}{\partial \ga{l}(x')}\right \rangle \Bigg|
    \\
    &\le \frac{1}{C_{l-1}}\left\| \frac{\partial \ho{l-1}(x)}{\partial x_{\alpha-1}} - \frac{\partial \ha{l-1}(x)}{\partial x_{\alpha-1}} \right\| \left\| \frac{\partial \ho{l-1}(x')}{\partial x_\beta'} \right\| \left\| \frac{\partial f_0(x)}{\partial \go{l}(x)} \right\| \left\| \frac{\partial f_0(x')}{\partial \go{l}(x')} \right\|
    \\
    &+ \frac{1}{C_{l-1}} \left\| \frac{\partial \ha{l-1}(x)}{\partial x_{\alpha-1}} \right\| \left\| \frac{\partial \ho{l-1}(x')}{\partial x_\beta'} - \frac{\partial \ha{l-1}(x')}{\partial x_\beta'} \right\| \left\| \frac{\partial f_0(x)}{\partial \go{l}(x)} \right\| \left\| \frac{\partial f_0(x')}{\partial \go{l}(x')} \right\|
    \\
    &+ \frac{1}{C_{l-1}} \left\| \frac{\partial \ha{l-1}(x)}{\partial x_{\alpha-1}} \right\| \left\| \frac{\partial \ha{l-1}(x')}{\partial x_\beta'} \right\| \left\| \frac{\partial f_0(x)}{\partial \go{l}(x)} - \frac{\partial f_\mathcal{A}(x)}{\partial \ga{l}(x)} \right\| \left\| \frac{\partial f_0(x')}{\partial \go{l}(x')} \right\|
    \\
    &+ \frac{1}{C_{l-1}} \left\| \frac{\partial \ha{l-1}(x)}{\partial x_{\alpha-1}} \right\| \left\| \frac{\partial \ha{l-1}(x')}{\partial x_\beta'} \right\| \left\| \frac{\partial f_{\mathcal{A}}(x)}{\partial \ga{l}(x)} \right\| \left\| \frac{\partial f_0(x')}{\partial \go{l}(x')} - \frac{\partial f_\mathcal{A}(x')}{\partial \ga{l}(x')} \right\|.
  \end{align*}
  
  We will use the norm bounds for the activations and gradients after perturbation, which is asymptotically the same as norms before perturbations with assumptions on $\omega$.
  \begin{align*}
    \left\|\ha{l}\right\| &\le 2\sqrt{d} + \omega \sqrt{d} O(\exp(L))
    \\
    &\le O(\exp(L)) \sqrt{d},
    \\
    \left\| \frac{\partial \ha{l}(x)}{\partial x_{\alpha-1}} \right\| &\le O(\exp(L)) \sqrt{d} + \omega \sqrt{d} O(\exp(L))
    \\
    &\le O(\exp(L)) \sqrt{d},
    \\
    \left\| \frac{\partial f_{\mathcal{A}}(x)}{\partial \ga{l}(x)} \right\| &\le \kappa O(\exp(L)) + \kappa \omega \exp(O(L)) \sqrt{\log d}
    \\
    &\le \kappa \exp(O(L)),
    \\
    \left\| \frac{\partial J(f_{\mathcal{A}})(x)_{\alpha-1}}{\partial \ga{l}(x)} \right\| &\le \kappa \exp(O(L)) + \kappa \omega \log d \exp(O(L)) + \kappa^2 \omega^2 \exp(O(L)) (\log d)^{3/2}
    \\
    &\le \kappa \exp(O(L)).
  \end{align*}
  
  Combining these with Equation~\eqref{eqn:initbegin}-\eqref{eqn:initend} and \eqref{eqn:perturbbegin}-\eqref{eqn:perturbend}, we obtain
  \begin{align*}
    \left|\Theta_{d,\theta_\mathcal{A}}(x, x')_{ij} - \Theta_{d,\theta_0}(x, x')_{ij} \right| &\le \kappa^2 \omega \exp(O(L)) \log d 
  \end{align*}
  for any $i, j \in \{0\} \cup [d_0]$.
  We finally compute the bound on Frobenius norm as follows:
  \[
    \left\| \Theta_{d,\theta_\mathcal{A}}(x, x') - \Theta_{d,\theta_0}(x, x') \right\|_F \le \kappa^2 \omega d_0 \exp(O(L)) \log d.
  \]
\end{proof}

\section{Proof of Theorem~\ref{thm:JacobianNTKTraining}}
We now prove that the finite JNTK stays constant during the gradient flow.
We consider the weight to evolve along the gradient flow of the objective $\mathcal{L}$,
\[
  \dot{\theta}_t = - \frac{\partial \mathcal{L}(\theta_t)}{\theta_t}
\]
where we write each layer's weights as $\theta_t = \{W^{(l)}(t)\}_{l=1}^{L+1}$.

As stated in Assumption~\ref{assm:NTK-smallest-eigenvalue}, we will assume that the smallest eigenvalue of the matrix $\Theta(x^{(1:N)}, x^{(1:N)})$ is greater than 0, with its value $\lambda_0$. 
The proof of this theorem is based on two lemmas, whose condition and the results are reversed.

\begin{lemma} \label{lem:linearconvergence}
  First, assume that the width network $d$ is large enough to satisfy 
  \[
    d \ge F_1\left( x^{(1:N)}, L, \delta, \frac{\epsilon}{\kappa} \right), d \ge F_2\left( x^{(1:N)}, L, \delta, \epsilon \right)
  \]
  where $F_1$ is defined in Remark~\ref{rmk:InitCloseToZero} and the $F_2$ is defined in Remark~\ref{rmk:JNTKInit2}.

  Further assume that the network width $d$ and $\omega$ satisfy
  \begin{align*}
    d &\ge \Omega(N d_0 \exp(L) / \delta),
    \\
    \omega &\le O \left( \frac{\lambda_0 \lambda}{N \exp(\Omega(L)) \log d} \right),
    \\
    \omega &\le (\log d)^{-2L}
  \end{align*}
  Then with probability at least $1 - \delta$ over the random initialisation, the following holds for all $t_0$. 
  If all of the matrices do not change a lot for all $0 \le t < t_0$, i.e., 
  \begin{align*}
    \left\| W^{(l)}(t) - W^{(l)}(0) \right\| &\le \omega \sqrt{d},
    \\
    \left\| W^{(l)}(t) - W^{(l)}(0) \right\|_\infty &\le \begin{cases}
      \infty & \text{ if }l = L+1,
      \\
      \omega \sqrt{d} \log d & \text{ if } 2 \le l \le L,
      \\
      \omega \log d & \text{ if } l = 1,
    \end{cases}
    \\
    \left\| W^{(l)}(t) - W^{(l)}(0) \right\|_1 &\le \begin{cases}
      \infty & \text{ if } l = 1
      \\
      \omega \sqrt{d} \log d & \text{ if } 2 \le l \le L,
      \\
      \omega \log d & \text{ if }l = L+1,
    \end{cases}
  \end{align*}
  then we have linear convergence for $0 \le t \le t_0$, i.e., 
  \[
    \mathcal{L}(\theta_t) \le \mathcal{L}(\theta_0) \exp\left( - \frac{\kappa^2 \lambda \lambda_0 t}{N} \right).
  \]
\end{lemma}
\begin{proof}
  By the assumption on $d$, with Theorem~\ref{thm:JacobianNTKInitialisation}, we can assume that the finite JNTK at initialisation is close to its limit, i.e., 
  \[
    \left\|\Theta_{d,\theta_0}(x^{(1:N)}, x^{(1:N)}) - \kappa^2 \Theta(x^{(1:N)}, x^{(1:N)}) \right\|_F \le \frac{\kappa^2 \lambda \lambda_0}{4}
  \]
  with probability at least $1 - \delta/2$.

  Also from our assumption and Theorem~\ref{thm:JacobianNTKTraining}, we can show that the finite JNTK stays close to its initialisation,
  \[
    \left\|\Theta_{d,\theta_t}(x^{(1:N)}, x^{(1:N)})- \Theta_{d,\theta_0}(x^{(1:N)}, x^{(1:N)})\right\|_F \le \frac{\kappa^2 \lambda \lambda_0}{4}
  \]
  with probability at least $1 - \delta/2$.

  Combining these two, we can show that the finite JNTK during training is close to its limit,
  \[
    \left\|\Theta_{d,\theta_t}(x^{(1:N)}, x^{(1:N)})- \kappa^2 \Theta(x^{(1:N)}, x^{(1:N)})\right\|_2 \le \frac{\kappa^2 \lambda  \lambda_0}{2}
  \]
  with probability at least $1 - \delta$.
  From the assumption $\lambda < 1$, this also applies to the JNTK with multipliers,
  \[
    \left\|\Lambda^{\oplus N} \Theta_{d,\theta_t}(x^{(1:N)}, x^{(1:N)}) \Lambda^{\oplus N} - \kappa^2 \Theta_\lambda(x^{(1:N)}, x^{(1:N)})\right\|_2 \le \frac{\kappa^2 \lambda \lambda_0}{2}.
  \]
  Now using the bound on the smallest eigenvalue, the JNTK with $\lambda$ multiplier has the smallest eigenvalue bounded below as
  \[
    \lambda_{\min}\left( \kappa^2 \Theta_\lambda(x^{(1:N)}, x^{(1:N)})\right) \ge \kappa^2 \left(\min_{i \in [N(d_0+1)]} \Lambda_{i-1,i-1}^{\oplus N} \right)^2 \lambda_{\min} \left( \Theta(x^{(1:N)}, x^{(1:N)})\right) = \kappa^2 \lambda \lambda_0,
  \]
  using the fact that for two PSD matrix $P, Q \in \R^{n \times n}$, the following inequality holds: 
  \[
    \lambda_{\min} (P \odot Q) \ge \min_{i \in [n]} P_{i-1,i-1} \lambda_{\min} (Q).
  \]
  This allows the application of Weyl's inequality to show that the minimum eigenvalue during the training is also bounded from below,
  \[
    \lambda_{\min} \left( \Lambda^{\oplus N} \Theta_{d,\theta_t}(x^{(1:N)}, x^{(1:N)}) \Lambda^{\oplus N} \right) \ge \frac{\kappa^2 \lambda \lambda_0}{2}.
  \]
  Now let's define a vector $u_t \in \R^{N(d_0+1)}$ which is defined by stacking 
  \[
    \left[ f_{d,\theta_t}(x^{(i)}), \sqrt{\lambda}  J(f_{d,\theta_t})(x^{(i)})_0, \cdots, \sqrt{\lambda} J(f_{d,\theta_t})(x^{(i)})_{d_0-1} \right].
  \]
  Then, the loss function can be written in terms of $\mathcal{L}(\theta_t) = \frac{1}{2N} \|u_t - \mathbf{y}^{(1:N)}\|_2^2$.
  Now the time derivative of the objective function can be written as
  \begin{equation}
    \frac{d}{dt} \mathcal{L}(\theta_t) 
    = \frac{d}{dt} \frac{1}{2N} \left\| u_t - \mathbf{y}^{(1:N)} \right\|_2^2 
    = \frac{1}{N} \left\langle u_t - \mathbf{y}^{(1:N)}, \dot{u}_t \right\rangle \label{eqn:timederivative}
  \end{equation}
  Now from Lemma~\ref{lem:JNTKdynamics}, the time derivative of a function is written by $\dot{u}_t = - \frac{1}{N} \Lambda^{\oplus N} \Theta_{d,\theta_t}(x^{(1:N)}, x^{(1:N)}) (u_t - \mathbf{y}^{(1:N)})$, so substituting this in \eqref{eqn:timederivative} gives
  \begin{align*}
    \frac{d}{dt} \mathcal{L}(\theta_t) &= - \frac{1}{N^2} \left\langle u_t - \mathbf{y}^{(1:N)}, \Lambda^{\oplus N} \Theta_{d,\theta_t}(x^{(1:N)}, x^{(1:N)}) \Lambda^{\oplus N} (u_t - \mathbf{y}^{(1:N)}) \right\rangle
    \\
    &= - \frac{1}{N^2} (u_t - \mathbf{y}^{(1:N)})^\intercal \left( \Lambda^{\oplus N} \Theta_{d,\theta_t}(x^{(1:N)}, x^{(1:N)}) \Lambda^{\oplus N}\right) (u_t - \mathbf{y}^{(1:N)})
    \\
    &\le - \frac{ \kappa^2 \lambda \lambda_0 }{2N^2} \|u_t - \mathbf{y}^{(1:N)}\|_2^2
    \\
    &\le - \frac{\kappa^2 \lambda \lambda_0}{N} \mathcal{L}(\theta_t).
  \end{align*}
  Then by Gronwall's inequality, we prove the linear convergence,
  \[
    \mathcal{L}(\theta_t) \le \mathcal{L}(\theta_0) \exp \left( - \frac{\kappa^2 \lambda \lambda_0 t}{N} \right).
  \]
\end{proof}

\begin{lemma} \label{lem:WeightStayStill}
  Fix the fail probability $\delta > 0$ and network depth $L$. Suppose that the network width $d$ and $\omega$ satisfies
  \begin{align*}
    d &\ge \Omega(N d_0 \exp(L) / \delta),
    \\
    d &\ge \Omega(\omega^{-2}),
    \\
    \omega &\le O \left( \frac{\lambda_0 \lambda}{N \exp(\Omega(L)) (\log d)^{3L}} \right),
    \\
    \omega &\le (\log d)^{-2L}.
  \end{align*}
  Also suppose that at initialisation, all outputs and their Jacobian are bounded by 1, i.e., for all $i \in [N]$,
  \[
    \left|f_{\bd,\theta_0}(x^{(i)}) \right|, \left\| J(f_{\bd,\theta_0})(x^{(i)}) \right\|_\infty \le 1.
  \]
  With probability at least $1 - \delta$ over the random initialisation, the following holds for all $t_0$ and for any $l$,
  if for all $0 \le t \le t_0$,
  \begin{itemize}
    \item Linear convergence happens:
    \[
      \mathcal{L}(\theta_t) \le \mathcal{L}(\theta_0) \exp \left( - \frac{\kappa^2 \lambda \lambda_0 t}{N} \right).
    \]
    \item All other matrices stay close to their initialisation:
    \begin{align*}
      \left\| W^{(l')}(t) - W^{(l')}(0) \right\| &\le \omega \sqrt{d},
      \\
      \left\| W^{(l')}(t) - W^{(l')}(0) \right\|_\infty &\le \begin{cases}
        \infty & \text{ if }l' = L+1,
        \\
        \omega \sqrt{d} \log d & \text{ if }2 \le l' \le L,
        \\
        \omega \log d & \text{ if }l' = 1,
      \end{cases}
      \\
      \left\| W^{(l')}(t) - W^{(l')}(0) \right\|_1 &\le \begin{cases}
        \infty & \text{ if } l' = 1,
        \\
        \omega \sqrt{d} \log d & \text{ if }2 \le l' \le L,
        \\
        \omega \log d & \text{ if }l' = L+1,
      \end{cases}
    \end{align*}
    for all $l' \in [L+1] \setminus \{l\}$.
  \end{itemize}
  then the chosen layer's weight also stays close to its initialisation,
  \begin{align*}
    \left\| W^{(l)}(t) - W^{(l)}(0) \right\| &\le \omega \sqrt{d},
    \\
    \left\| W^{(l)}(t) - W^{(l)}(0) \right\|_\infty &\le \begin{cases}
      \infty & \text{ if }l = L+1,
      \\
      \omega \sqrt{d} \log d & \text{ if }2 \le l \le L,
      \\
      \omega \log d & \text{ if }l = 1,
    \end{cases}
    \\
    \left\| W^{(l)}(t) - W^{(l)}(0) \right\|_1 &\le \begin{cases}
      \infty & \text{ if }l = 1,
      \\
      \omega \sqrt{d} \log d & \text{ if }2 \le l \le L,
      \\
      \omega \log d & \text{ if }l = L+1.
    \end{cases}
  \end{align*}
\end{lemma}
\begin{proof}
  We will first prove the bound for the Frobenius norm.

  Let $C_0, C_1, C_2$ be some constant. 
  \begin{align*}
    &\left\| W^{(l)}(t_0) - W^{(l)}(0) \right\|
    \\
    = &\left\| \int_0^{t_0} \frac{d W^{(l)}(t)}{dt} dt\right\|
    \\
    = &\left\| \int_0^{t_0} \frac{\partial \mathcal{L}(\theta_t)}{\partial W^{(l)}(t)} dt \right\|
    \\
    \le &\int_0^{t_0} \Bigg\|\frac{1}{N} \sum_{i=1}^N \Bigg((f_{d,\theta_t}(x^{(i)}) - y^{(i)}) \frac{\partial f_{d,\theta_t}(x^{(i)})}{\partial W^{(l)}(t)} 
    \\
    &\qquad \qquad \qquad + \lambda \sum_{\alpha=1}^{d_0} J(f_{d,\theta_t})(x^{(i)})_{\alpha-1} \frac{\partial J(f_{d,\theta_t})(x^{(i)})_{\alpha-1}}{\partial W^{(l)}(t)} \Bigg) \Bigg\| dt
  \end{align*}
  We can define matrix $M \in \R^{(d_{l-1} d_l) \times N(d_0+1)}$ by stacking all flattened gradients w.r.t. matrix $W^{(l)}$.
  Then the Frobenius norm can be rewritten and bounded as 
  \begin{align*}
    &\frac{1}{N}\left\| M \mathbf{y}^{(1:N)}\right\| \le \frac{1}{N} \|M\| \|\mathbf{y}^{(1:N)}\| \le \frac{\|\mathbf{y}^{(1:N)}\|}{N} \sqrt{\sum_{k=1}^{N(1+d_0)} \|M_{-, k}\|^2} 
    \\
    \le &\frac{\|\mathbf{y}^{(1:N)}\|}{N} \sqrt{N(1+d_0) \max_k \|M_{-,k}\|^2} = \frac{\sqrt{1 + d_0}}{\sqrt{N}} \max_k \|M_{-,i}\| \|\mathbf{y}^{(1:N)}\|
  \end{align*}
  and the last term has the following form when written with the original notations,
  \[
    \sqrt{\frac{1 + d_0}{N}} \max_{\substack{i \in [N] \\ \alpha \in [d_0]}} \left\{ \left\| \frac{\partial f_{d,\theta_t}(x^{(i)}, \theta_t)}{\partial W^{(l)}(t)} \right\|, \left\| \frac{\partial J(f_{d,\theta_t})(x^{(i)})_{\alpha-1}}{\partial W^{(l)}(t)} \right\| \right\} \sqrt{\mathcal{L}(\theta_t)}
  \]
  and using $\lambda < 1$, we can finally bound this by
  \[
    \sqrt{1 + d_0} \max_{\substack{i \in [N] \\ \alpha \in [d_0]}} \left\{ \left\| \frac{\partial f_{d,\theta_t}(x^{(i)})}{\partial W^{(l)}(t)} \right\|, \left\| \frac{\partial J(f_{d,\theta_t})(x^{(i)})_{\alpha-1}}{\partial W^{(l)}(t)} \right\| \right\} \sqrt{\mathcal{L}(
      \theta_t)}.
  \]
  Then we can bound the integral using this as
  \begin{align}
    \le & \sqrt{1 + d_0} \int_0^{t_0} \max_{i,\alpha} \left\{ \left\| \frac{\partial f_{d,\theta_t}(x^{(i)})}{\partial W^{(l)}(t)} \right\|, \left\| \frac{\partial J(f_{d,\theta_t})(x^{(i)})_{\alpha-1}}{\partial W^{(l)}(t)} \right\| \right\} \sqrt{\mathcal{L}(\theta_t)} dt \nonumber
    \\
    \le &\sqrt{1 + d_0} \max_{\substack{i \in [N] \\ 0 \le \tau \le t_0}} \left\{ \left\| \frac{\partial f_{d,\theta_\tau}(x^{(i)})}{\partial W^{(l)}(\tau)} \right\|, \left\| \frac{\partial J(f_{d,\theta_\tau})(x^{(i)})_{\alpha-1}}{\partial W^{(l)}(\tau)} \right\| \right\} \int_0^{t_0} \sqrt{\mathcal{L}(\theta_t)} dt \nonumber
    \\
    \le &\sqrt{1 + d_0} \max_{\substack{i \in [N] \\ 0 \le \tau \le t_0}} \left\{ \left\| \frac{\partial f_{d,\theta_\tau}(x^{(i)})}{\partial W^{(l)}(\tau)} \right\|, \left\| \frac{\partial J(f_{d,\theta_\tau})(x^{(i)})_{\alpha-1}}{\partial W^{(l)}(\tau)} \right\| \right\} \int_0^{t_0} \sqrt{\mathcal{L}(\theta_0)} \exp \left( -\frac{\kappa^2 \lambda \lambda_0 t}{2N} \right) dt \nonumber
    \\
    \le &\frac{2\sqrt{(1 + d_0) \mathcal{L}(\theta_0)}}{N \kappa^2 \lambda \lambda_0} \max_{\substack{i \in [N] \\ 0 \le \tau \le t_0}} \left\{ \left\| \frac{\partial f_{d,\theta_\tau}(x^{(i)})}{\partial W^{(l)}(\tau)} \right\|, \left\| \frac{\partial J(f_{d,\theta_\tau})(x^{(i)})_{\alpha-1}}{\partial W^{(l)}(\tau)} \right\| \right\} \label{eqn:boundformatrix}
  \end{align}
  so it is enough to bound each term at maximum.
  To use our previous analysis, let's instead bound these terms:
  \begin{align*}
    \left\| \frac{\partial f_{d,\theta_\tau}(x^{(i)})}{\partial W^{(l)}(\tau)} \right\| &\le \left\| \frac{\partial f_{d,\theta_\tau}(x^{(i)})}{\partial W^{(l)}(0)} \right\| + \left\| \frac{\partial f_{d,\theta_\tau}(x^{(i)})}{\partial W^{(l)}(\tau)} - \frac{\partial f_{d,\theta_0}(x^{(i)})}{\partial W^{(l)}(0)} \right\|,
    \\
    \left\| \frac{\partial J(f_{d,\theta_\tau})(x^{(i)})_{\alpha-1}}{\partial W^{(l)}(\tau)} \right\| &\le \left\| \frac{\partial J(f_{d,\theta_\tau})(x^{(i)})_{\alpha-1}}{\partial W^{(l)}(0)} \right\| 
    \\
    &\qquad + \left\|\frac{\partial J(f_{d,\theta_\tau})(x^{(i)})_{\alpha-1}}{\partial W^{(l)}(\tau)} - \frac{\partial J(f_{d,\theta_0})(x^{(i)})_{\alpha-1}}{\partial W^{(l)}(0)} \right\|.
  \end{align*}
  For the gradients at initialisation, we can bound them as
  \begin{align*}
    \left\|\frac{\partial f_{d,\theta_\tau}(x^{(i)})}{\partial W^{(l)}(0)}  \right\| &\le \frac{1}{\sqrt{C_{l-1}}}\left\| h_0^{(l-1)}(x^{(i)}) \right\| \left\| \frac{\partial f_{d,\theta_0}(x^{(i)})}{\partial g_0^{(l)}(x^{(i)})} \right\| \le \kappa O(\exp(L)).
  \end{align*}
  from Lemma~\ref{lem:ForwardNormInit} and \ref{lem:BackwardNormInit}.
  Similarly, we can bound the gradient of the directional derivative as
  \begin{align}
    &\left\| \frac{\partial J(f_{d,\theta_0})(x^{(i)})_{\alpha-1}}{\partial W^{(l)}(0)} \right\| \nonumber
    \\
    \le &\frac{1}{\sqrt{C_{l-1}}} \left\| \frac{\partial h_0^{(l-1)}(x^{(i)})}{\partial x_{\alpha-1}^{(i)}} \right\| \left\| \frac{\partial f_{d,\theta_0}(x^{(i)})}{\partial g_0^{(l)}(x^{(i)})} \right\| \nonumber 
    \\
    &\qquad + \frac{1}{\sqrt{C_{l-1}}} \left\| h_0^{(l-1)}(x^{(i)}) \right\| \left\| \frac{\partial J(f_{d,\theta_0})(x^{(i)})_{\alpha-1}}{\partial g_0^{(l)}(x^{(i)})} \right\| \nonumber
    \\
    &\le \kappa \exp(O(L)), \label{eqn:jacobbound}
  \end{align}
  using Lemma~\ref{lem:ForwardJacobNormInit} and \ref{lem:BackwardJacobNormInit}.

  We can do similar analysis for the $p$-norm for $p = 1, \infty$, as
  \begin{align*}
    &\left\| W^{(l)}(t_0) - W^{(l)}(0) \right\|_p
    \\
    \le &\int_0^{t_0} \Bigg\| \frac{1}{N} \sum_{i=1}^N \Bigg((f_{d,\theta_t}(x^{(i)}) - y^{(i)}) \frac{\partial f_{d,\theta_t}(x^{(i)})}{\partial W^{(l)}(t)} 
    \\
    &\qquad \qquad \qquad + \lambda \sum_{\alpha=1}^{d_0} J(f_{d,\theta_t})(x^{(i)})_{\alpha-1} \frac{\partial J(f_{d,\theta_t})(x^{(i)})_{\alpha-1}}{\partial W^{(l)}(t)} \Bigg) \Bigg\|_p dt,
  \end{align*}
  and applying the following inequality
  \begin{align*}
    &\Bigg\| \frac{1}{N} \sum_{i=1}^N \Bigg((f_{d,\theta_t}(x^{(i)}) - y^{(i)}) \frac{\partial f_{d,\theta_t}(x^{(i)})}{\partial W^{(l)}(t)} 
    \\
    &\qquad \qquad \quad + \lambda \sum_{\alpha=1}^{d_0} J(f_{d,\theta_t})(x^{(i)})_{\alpha-1} \frac{\partial J(f_{d,\theta_t})(x^{(i)})_{\alpha-1}}{\partial W^{(l)}(t)} \Bigg) \Bigg\|_p
    \\
    \le &\frac{1}{N} \sum_{i=1}^N \Bigg( | f_{d,\theta_t}(x^{(i)}) - y^{(i)} | \left\| \frac{\partial f_{d,\theta_t}(x^{(i)})}{\partial W^{(l)}(t)} \right\|_p 
    \\
    &\qquad \qquad + \lambda \sum_{\alpha=1}^{d_0} |J(f_{d,\theta_t})(x^{(i)})_{\alpha-1}| \left\| \frac{\partial J(f_{d,\theta_t})(x^{(i)})_{\alpha-1}}{\partial W^{(l)}(t)} \right\|_p \Bigg)
    \\
    \le &\sqrt{\mathcal{L}(\theta_t)} \max_{\substack{i \in [N] \\ \alpha \in [d_0]}} \left\{ \left\| \frac{\partial f_{d,\theta_t}(x^{(i)})}{\partial W^{(l)}(t)} \right\|_p, \left\| \frac{\partial J(f_{d,\theta_t})(x^{(i)})_{\alpha-1}}{\partial W^{(l)}(t)} \right\|_p \right\}.
  \end{align*}

  Again, we will decompose these matrix $p$-norms with triangle inequality, which requires us to bound the following four quantities, 
  \begin{align*}
    \left\| \frac{\partial f_{d,\theta_0}(x^{(i)})}{\partial W^{(l)}(0)} \right\|_p, & \left\| \frac{\partial f_{d,\theta_\tau}(x^{(i)})}{\partial W^{(l)}(\tau)} - \frac{\partial f_{d,\theta_0}(x^{(i)})}{\partial W^{(l)}(0)} \right\|_p,
    \\
    \left\| \frac{\partial J(f_{d,\theta_0})(x^{(i)})_{\alpha-1}}{\partial W^{(l)}(0)} \right\|_p, & \left\| \frac{\partial J(f_{d,\theta_\tau})(x^{(i)})_{\alpha-1}}{\partial W^{(l)}(\tau)} - \frac{\partial J(f_{d,\theta_0})(x^{(i)})_{\alpha-1}}{\partial W^{(l)}(0)} \right\|_p.
  \end{align*}
  This time, the bounds need to consider the corner cases, $l = L+1$ for $p = 1$ and $l = 1$ for $p = \infty$.
  We will first bound the bounds at initialisations for $p = 1$, for $1 \le l \le L$,
  \begin{align*}
    \left\| \frac{\partial f_{d,\theta_0}(x^{(i)})}{\partial W^{(l)}(0)} \right\|_1 
    &\le \frac{1}{\sqrt{C_{l-1}}} \left\| \frac{\partial f_{d,\theta_0}(x^{(i)})}{\partial \go{l}} \right\|_1 \left\| \ho{l-1}\right\|_\infty
    \\
    &\le \left\| \frac{\partial f_{d,\theta_0}(x^{(i)})}{\partial \go{l}} \right\| \left\| \ho{l-1}\right\|_\infty
    \\
    &\le \kappa \sqrt{\log d} \exp(O(L)),
    \\
    \left\| \frac{\partial f_{d,\theta_0}(x^{(i)})}{\partial W^{(L+1)}(0)} \right\|_1
    &\le \frac{1}{\sqrt{d}} \left\| \frac{\partial f_{d,\theta_0}(x^{(i)})}{\partial \go{L+1}} \right\| \left\| \ho{L} \right\|_\infty
    \\
    &\le \kappa \exp(O(L)) \sqrt{\frac{\log d}{d}}.
  \end{align*}
  The $p = \infty$ case is similar, for $2 \le l \le L+1$,
  \begin{align*}
    \left\|\frac{\partial f_{d,\theta_0}(x^{(i)})}{\partial W^{(l)}(0)} \right\|_\infty
    &= \left\| \frac{1}{\sqrt{C_{l-1}}}\frac{\partial f_{d,\theta_0}(x^{(i)})}{\partial \go{l}} \left(\ho{l-1}\right)^\intercal \right\|_\infty
    \\
    &\le \frac{1}{\sqrt{C_{l-1}}} \left\| \frac{\partial f_{d,\theta_0}(x^{(i)})}{\partial \go{l}} \right\|_\infty \left\| \ho{l-1} \right\|_1
    \\
    &\le \left\| \frac{\partial f_{d,\theta_0}(x^{(i)})}{\partial \go{l}} \right\|_\infty \left\| \ho{l-1} \right\|
    \\
    &\le 2\sqrt{d} \cdot \frac{\kappa \sqrt{\log d}}{\sqrt{d}} O(\exp(L))
    \\
    &\le \kappa \sqrt{\log d} \exp(O(L)),
    \\
    \left\| \frac{\partial f_{d,\theta_0}(x^{(i)})}{\partial W^{(1)}(0)} \right\|_\infty
    &\le \left\| \frac{\partial f_{d,\theta_0}(x^{(i)})}{\partial \go{1}} \right\|_\infty \left\| x^{(i)} \right\|_1
    \\
    &\le \kappa \sqrt{\frac{\log d}{d}} \exp(O(L)).
  \end{align*}
  and similarly, the gradient of Jacobians w.r.t weights can be bounded as
  \begin{align*}
    &\left\|\frac{\partial J(f_{d,\theta_0})(x^{(i)})_{\alpha-1}}{\partial W^{(l)}(0)} \right\|_\infty
    \\
    \le &\kappa O((\log d)^{3/2} \exp(L)).
  \end{align*}

  For the second term, we will apply a contradiction argument.
  To apply Theorem~\ref{thm:JacobianNTKTraining}, we require the $W^{(l)}$ to be close to its initialisation also, but the assumptions do not contain it.
  To address this problem, let's define the first violation time $t'$ as
  \[
    t' \defeq \argmin_{0 \le t \le t_0} \left\{ \begin{aligned}
      \left\| W^{(l)}(t) - W^{(l)}(0) \right\| &\ge \omega \sqrt{d},
      \\
      \left\|W^{(l)}(t) - W^{(l)}(0)\right\|_\infty &\ge \begin{cases} \infty & \text{ if } l = L+1,\\\omega \sqrt{d} \log d & \text{ if } 2 \le l \le L, \\ \omega \log d & \text{ if }l = 1, \end{cases}
      \\
      \left\|W^{(l)}(t) - W^{(l)}(0)\right\|_1 &\ge \begin{cases} \infty & \text{ if }l = 1,\\ \omega \sqrt{d} \log d & \text{ if } 2 \le l \le L, \\ \omega \log d & \text{ if }l = L+1, \end{cases}
    \end{aligned} \right\}.
  \]
  Let's suppose that the first condition is violated.

  Then we can see that for all times $0 \le \tau < t'$, all weights stay close to their initialisation: for all $l \in [L+1]$, 
  \begin{align*}
    \left\| W^{(l)}(t) - W^{(l)}(0) \right\| &\le \omega \sqrt{d},
    \\
    \left\|W^{(l)}(t) - W^{(l)}(0)\right\|_\infty &\le \begin{cases} \omega \sqrt{d} \log d & \text{ if } l \ge 2, \\ \omega \log d & \text{ if }l = 1, \end{cases}
    \\
    \left\|W^{(l)}(t) - W^{(l)}(0)\right\|_1 &\le \begin{cases} \omega \sqrt{d} \log d & \text{ if } l \ge L, \\ \omega \log d & \text{ if }l = L+1. \end{cases}
  \end{align*}
  By Lemma~\ref{lem:perturbact}, \ref{lem:perturbback}, \ref{lem:perturbjacobian} \ref{lem:perturbjacobback}, we can show that for all $i \in [N]$, $\alpha \in [d_0]$, and $l \in [L+1]$,
  \begin{align*}
    &\left\|\frac{\partial f_{d,\theta_{t'}}(x^{(i)})}{\partial W^{(l)}(t')} - \frac{\partial f_{d,\theta_0}(x^{(i)})}{\partial W^{(l)}(0)}\right\|
    \\
    \le &\left\|\frac{1}{\sqrt{C_{l-1}}} \left(\frac{\partial f_{d,\theta_{t'}}(x^{(i)}) }{\partial h_{t'}^{(l)}(x^{(i)})} - \frac{\partial f_{d,\theta_{0}}(x^{(i)}) }{\partial \ho{l}(x^{(i)})}\right) \left( g_{t'}^{(l-1)}(x^{(i)}) \right)^\intercal \right\| 
    \\
    &\quad + \left\|\frac{1}{\sqrt{C_{l-1}}} \frac{\partial f_{d,\theta_{0}}(x^{(i)}) }{\partial \ho{l}(x^{(i)})} \left( g_{t'}^{(l-1)}(x^{(i)}) - \go{l-1}(x^{(i)}) \right)^\intercal \right\|
    \\
    \le & \frac{1}{\sqrt{C_{l-1}}}\left\| \frac{\partial f_{d,\theta_{t'}}(x^{(i)}) }{\partial h_{t'}^{(l)}(x^{(i)})} - \frac{\partial f_{d,\theta_{0}}(x^{(i)}) }{\partial \ho{l}(x^{(i)})}\right\| \left\| g_{t'}^{(l-1)}(x^{(i)}) \right\| 
    \\
    &\quad + \frac{1}{\sqrt{C_{l-1}}}\left\| \frac{\partial f_{d,\theta_{0}}(x^{(i)}) }{\partial \ho{l}(x^{(i)})} \right\| \left\| g_{t'}^{(l-1)}(x^{(i)}) - \go{l-1}(x^{(i)})  \right\|
    \\
    \le &\kappa \omega \exp(O(L)),
    \\
    &\left\|\frac{\partial J(f_{d,\theta_{t'}})(x^{(i)})_{\alpha-1}}{\partial W^{(l)}(t')} - \frac{\partial J(f_{d,\theta_0})(x^{(i)})_{\alpha-1}}{\partial W^{(l)}(0)}\right\|
    \\
    \le &\frac{1}{\sqrt{C_{l-1}}}\left\|  \frac{\partial J(f_{d,\theta_{t'}})(x^{(i)})_{\alpha-1}}{\partial h_{t'}^{(l)}(x^{(i)})} \left( g_{t'}^{(l-1)}(x^{(i)}) \right)^\intercal - \frac{\partial J(f_{d,\theta_0})(x^{(i)})_{\alpha-1}}{\partial \ho{l}(x^{(i)})} \left(\go{l-1}(x^{(i)})\right)^\intercal \right\| \\
    &+ \frac{1}{\sqrt{C_{l-1}}} \left\| \frac{\partial f_{d,\theta_{t'}}(x^{(i)})}{\partial h_{t'}^{(l)}(x^{(i)})} \left( \frac{\partial g_{t'}^{(l-1)}(x^{(i)})}{\partial x_{\alpha-1}^{(i)}} \right)^\intercal  - \frac{\partial f_{d,\theta_0}(x^{(i)})}{\partial \ho{l}(x^{(i)})} \left(  \frac{\partial \go{l-1}(x^{(i)})}{ \partial x_{\alpha-1}^{(i)} } \right)^\intercal \right\|
    \\
    \le &\kappa \omega \exp(O(L)) \log d. 
  \end{align*}

  By our assumption on $\omega$, we can show that all these terms are also $\kappa \exp(O(L))$, thereby bounding the entire term by
  \[
    \left\| W^{(l)}(t') - W^{(l)}(0) \right\| \le O\left(\frac{ \sqrt{(1 + d_0) \mathcal{L}(\theta_0)}}{N \kappa \lambda \lambda_0} \right)\exp(O(L)).
  \]
  From our assumption that all outputs of the function and their Jacobian are in $[-1, 1]$ with high probability, we can assume that $\mathcal{L}(\theta_0) \le 2 (1 + d_0)$.

  However, from our assumption of $d \ge O(1/\omega^2)$, by increasing $d$ large enough, the entire term becomes smaller than $\omega \sqrt{d}$.

  This contradicts our assumption, which is the existence of the violation point $t'$, showing that such a violation never happens.

  Therefore for every $0 \le t < t_0$,
  \[
    \left\| W^{(l)}(t) - W^{(l)}(0) \right\| \le \sqrt{d} \omega.
  \]

  Now let's suppose that the second condition is violated. We will similarly prove this as we've done for the operator norm.

  Using a similar approach as the operator norm, we obtain the following bounds for $2 \le l \le L$,
  \begin{align*}
    &\left\| \frac{\partial f_{d,\theta_{t'}}(x^{(i)})}{\partial W^{(l)}(t')} - \frac{\partial f_{d,\theta_0}(x^{(i)})}{\partial W^{(l)}(0)} \right\|_\infty
    \\
    \le &\left\|\frac{1}{\sqrt{C_{l-1}}} \left(\frac{\partial f_{d,\theta_{t'}}(x^{(i)}) }{\partial h_{t'}^{(l)}(x^{(i)})} - \frac{\partial f_{d,\theta_{0}}(x^{(i)}) }{\partial \ho{l}(x^{(i)})}\right) \left( g_{t'}^{(l-1)}(x^{(i)}) \right)^\intercal \right\|_\infty 
    \\
    &\quad + \left\|\frac{1}{\sqrt{C_{l-1}}} \frac{\partial f_{d,\theta_{0}}(x^{(i)}) }{\partial \ho{l}(x^{(i)})} \left( g_{t'}^{(l-1)}(x^{(i)}) - \go{l-1}(x^{(i)}) \right)^\intercal \right\|_\infty
    \\
    \le &\frac{1}{\sqrt{C_{l-1}}} \left\| \frac{f_{d,\theta_{t'}}(x^{(i)})}{\partial h_{t'}^{(l)}(x^{(i)})} - \frac{f_{d,\theta_{0}}(x^{(i)})}{\partial h_{0}^{(l)}(x^{(i)})} \right\|_\infty \left\| g_{t'}^{(l-1)}(x^{(i)}) \right\|_1
    \\
    &\quad + \frac{1}{\sqrt{C_{l-1}}} \left\| \frac{\partial f_{d,\theta_0}(x^{(i)})}{\partial h_0^{(l)}(x^{(i)})} \right\|_\infty \left\| g_{t'}^{(l-1)}(x^{(i)}) - g_0^{(l-1)}(x^{(i)}) \right\|_1 
    \\
    \le & \kappa \omega \exp(O(L)) (\log d)^{2L} + \kappa \omega \exp(O(L)) \sqrt{\log d},
    \\
    &\left\| \frac{\partial f_{d,\theta_{t'}}(x^{(i)})}{\partial W^{(1)}(t')} - \frac{\partial f_{d,\theta_0}(x^{(i)})}{\partial W^{(1)}(0)} \right\|_\infty
    \\
    \le &\left\| \left( \frac{\partial f_{d,\theta_{t'}}(x^{(i)})}{\partial h_{t'}^{(1)}(x^{(i)})} - \frac{\partial f_{d,\theta_{0}}(x^{(i)})}{\partial h_{0}^{(1)}(x^{(i)})} \right) \left(x^{(i)}\right)^\intercal \right\|_\infty
    \\
    \le &\kappa \omega \exp(O(L)) \frac{(\log d)^{2L}}{\sqrt{d}},
    \\
    &\left\| \frac{\partial J(f_{d,\theta_{t'}})(x^{(i)})_{\alpha-1}}{\partial W^{(l)}(t')} - \frac{\partial J(f_{d,\theta_{0}})(x^{(i)})_{\alpha-1}}{\partial W^{(l)}(0)} \right\|_\infty
    \\
    \le &\frac{1}{\sqrt{C_{l-1}}} \left\| \frac{\partial f_{d,\theta_{t'}}(x^{(i)})}{\partial h_{t'}^{(l)}(x^{(i)})} \left( \frac{\partial g_{t'}^{(l-1)}(x^{(i)})}{\partial x_{\alpha-1}^{(i)}} \right)^\intercal - \frac{\partial f_{d,\theta_{0}}(x^{(i)})}{\partial h_{0}^{(l)}(x^{(i)})} \left( \frac{\partial g_{0}^{(l-1)}(x^{(i)})}{\partial x_{\alpha-1}^{(i)}} \right)^\intercal \right\|_\infty 
    \\
    &+ \frac{1}{\sqrt{C_{l-1}}}\left\| \frac{\partial J(f_{d,\theta_{t'}})(x^{(i)})_{\alpha-1}}{\partial h_{t'}^{(l)}(x^{(i)})} \left( g_{t'}^{(l-1)}(x^{(i)}) \right)^\intercal - \frac{\partial J(f_{d,\theta_{0}})(x^{(i)})_{\alpha-1}}{\partial h_{0}^{(l)}(x^{(i)})} \left( g_{0}^{(l-1)}(x^{(i)}) \right)^\intercal \right\|_\infty
    \\
    \le &\kappa \omega \exp(O(L)) (\log d)^{2L} + \kappa \omega \exp(O(L)) (\log d)^{3L+1},
    \\
    &\left\| \frac{\partial J(f_{d,\theta_{t'}})(x^{(i)})_{\alpha-1}}{\partial W^{(1)}(t')} - \frac{\partial J(f_{d,\theta_{0}})(x^{(i)})_{\alpha-1}}{\partial W^{(1)}(0)} \right\|_\infty
    \\
    \le &\left\| \left(\frac{\partial f_{d,\theta_{t'}}(x^{(i)})}{\partial h_{t'}^{(1)}(x^{(i)})} - \frac{\partial f_{d,\theta_{0}}(x^{(i)})}{\partial h_{0}^{(l)}(x^{(i)})}\right) e_{\alpha-1}^\intercal \right\|_\infty 
    \\
    &\quad + \left\| \left(\frac{\partial J(f_{d,\theta_{t'}})(x^{(i)})_{\alpha-1}}{\partial h_{t'}^{(l)}(x^{(i)})}  - \frac{\partial J(f_{d,\theta_{0}})(x^{(i)})_{\alpha-1}}{\partial h_{0}^{(l)}(x^{(i)})} \right) \left( x^{(i)} \right)^\intercal \right\|_\infty
    \\
    \le &\left\| \left(\frac{\partial f_{d,\theta_{t'}}(x^{(i)})}{\partial h_{t'}^{(1)}(x^{(i)})} - \frac{\partial f_{d,\theta_{0}}(x^{(i)})}{\partial h_{0}^{(l)}(x^{(i)})}\right)  \right\|_\infty 
    \\
    &\quad + \left\| \left(\frac{\partial J(f_{d,\theta_{t'}})(x^{(i)})_{\alpha-1}}{\partial h_{t'}^{(l)}(x^{(i)})}  - \frac{\partial J(f_{d,\theta_{0}})(x^{(i)})_{\alpha-1}}{\partial h_{0}^{(l)}(x^{(i)})} \right) \right\|_\infty
    \\
    \le &\kappa \omega \exp(O(L)) \frac{(\log d)^{3L+1}}{\sqrt{d}}.
  \end{align*}
  As we've done in the operator norm cases, we can set $\omega \le (\log d)^{-3L}$ so that the norms at time $t=0$ dominate the perturbed norm, showing that
  \begin{align*}
    \left\|W^{(l)}(t') - W^{(l)}(0) \right\|_\infty &\le O\left( \frac{d_0}{N \kappa \lambda \lambda_0} \right) \exp(O(L)) \log d,
    \\
    \left\|W^{(1)}(t') - W^{(1)}(0) \right\|_\infty &\le O\left( \frac{d_0}{N \kappa \lambda \lambda_0} \right) \exp(O(L)) \frac{\log d}{\sqrt{d}},
  \end{align*}
  and setting $d \ge \Omega(1 / \omega^2)$ makes these terms to be bounded by $\omega \sqrt{d}\log d, \omega \log d$, respectively. 
  This results in a contradiction, showing that the second condition can not be violated.

  Now we finally consider the 1-norm. For $2 \le l \le L$,
  \begin{align*}
    &\left\| \frac{\partial f_{d,\theta_{t'}}(x^{(i)})}{\partial W^{(l)}(t')} - \frac{\partial f_{d,\theta_0}(x^{(i)})}{\partial W^{(l)}(0)} \right\|_1
    \\
    \le &\left\|\frac{1}{\sqrt{C_{l-1}}} \left(\frac{\partial f_{d,\theta_{t'}}(x^{(i)}) }{\partial h_{t'}^{(l)}(x^{(i)})} - \frac{\partial f_{d,\theta_{0}}(x^{(i)}) }{\partial \ho{l}(x^{(i)})}\right) \left( g_{t'}^{(l-1)}(x^{(i)}) \right)^\intercal \right\|_1 
    \\
    &\quad + \left\|\frac{1}{\sqrt{C_{l-1}}} \frac{\partial f_{d,\theta_{0}}(x^{(i)}) }{\partial \ho{l}(x^{(i)})} \left( g_{t'}^{(l-1)}(x^{(i)}) - \go{l-1}(x^{(i)}) \right)^\intercal \right\|_1
    \\
    \le &\frac{1}{\sqrt{C_{l-1}}} \left\| \frac{f_{d,\theta_{t'}}(x^{(i)})}{\partial h_{t'}^{(l)}(x^{(i)})} - \frac{f_{d,\theta_{0}}(x^{(i)})}{\partial h_{0}^{(l)}(x^{(i)})} \right\|_1 \left\| g_{t'}^{(l-1)}(x^{(i)}) \right\|_\infty
    \\
    &\quad + \frac{1}{\sqrt{C_{l-1}}} \left\| \frac{\partial f_{d,\theta_0}(x^{(i)})}{\partial h_0^{(l)}(x^{(i)})} \right\|_1 \left\| g_{t'}^{(l-1)}(x^{(i)}) - g_0^{(l-1)}(x^{(i)}) \right\|_\infty
    \\
    \le &\kappa \omega \exp(O(L)) \log d + \kappa \omega \exp(O(L)) (\log d)^L,
    \\
    &\left\| \frac{\partial f_{d,\theta_{t'}}(x^{(i)})}{\partial W^{(L+1)}(t')} - \frac{\partial f_{d,\theta_0}(x^{(i)})}{\partial W^{(L+1)}(0)} \right\|_1
    \\
    \le &\frac{1}{\sqrt{C_{l-1}}} \left\| \frac{\partial f_{d,\theta_0}(x^{(i)})}{\partial h_0^{(L+1)}(x^{(i)})} \right\|_1\left\| g_{t'}^{(L)}(x^{(i)}) - g_0^{(L)}(x^{(i)}) \right\|_\infty
    \\
    \le &\kappa \exp(O(L)) \frac{(\log d)^L}{\sqrt{d}},
    \\
    &\left\|\frac{\partial J(f_{d,\theta_{t'}})(x^{(i)})_{\alpha-1}}{\partial W^{(l)}(t')} - \frac{\partial J(f_{d,\theta_{0}})(x^{(i)})_{\alpha-1}}{\partial W^{(l)}(0)}\right\|_1
    \\
    \le &\frac{1}{\sqrt{C_{l-1}}} \left\| \frac{\partial f_{d,\theta_{t'}}(x^{(i)})}{\partial h_{t'}^{(l)}(x^{(i)})} \left( \frac{\partial g_{t'}^{(l-1)}(x^{(i)})}{\partial x_{\alpha-1}^{(i)}} \right)^\intercal - \frac{\partial f_{d,\theta_{0}}(x^{(i)})}{\partial h_{0}^{(l)}(x^{(i)})} \left( \frac{\partial g_{0}^{(l-1)}(x^{(i)})}{\partial x_{\alpha-1}^{(i)}} \right)^\intercal \right\|_1
    \\
    &+ \frac{1}{\sqrt{C_{l-1}}}\left\| \frac{\partial J(f_{d,\theta_{t'}})(x^{(i)})_{\alpha-1}}{\partial h_{t'}^{(l)}(x^{(i)})} \left( g_{t'}^{(l-1)}(x^{(i)}) \right)^\intercal - \frac{\partial J(f_{d,\theta_{0}})(x^{(i)})_{\alpha-1}}{\partial h_{0}^{(l)}(x^{(i)})} \left( g_{0}^{(l-1)}(x^{(i)}) \right)^\intercal \right\|_1
    \\
    \le &\kappa \omega \exp(O(L)) (\log d)^{2L}+ \kappa \omega \exp(O(L)) (\log d)^{L},
    \\
    &\left\|\frac{\partial J(f_{d,\theta_{t'}})(x^{(i)})_{\alpha-1}}{\partial W^{(L+1)}(t')} - \frac{\partial J(f_{d,\theta_{0}})(x^{(i)})_{\alpha-1}}{\partial W^{(L+1)}(0)}\right\|_1
    \\
    \le &\frac{1}{\sqrt{d}} \left\| \frac{\partial f_{d,\theta_{t'}}(x^{(i)})}{\partial h_{t'}^{(L+1)}(x^{(i)})} \left( \frac{\partial g_{t'}^{(L)}(x^{(i)})}{\partial x_{\alpha-1}^{(i)}} \right)^\intercal - \frac{\partial f_{d,\theta_{0}}(x^{(i)})}{\partial h_{0}^{(L+1)}(x^{(i)})} \left( \frac{\partial g_{0}^{(L)}(x^{(i)})}{\partial x_{\alpha-1}^{(i)}} \right)^\intercal \right\|_1
    \\
    &+ \frac{1}{\sqrt{d}}\left\| \frac{\partial J(f_{d,\theta_{t'}})(x^{(i)})_{\alpha-1}}{\partial h_{t'}^{(L+1)}(x^{(i)})} \left( g_{t'}^{(L)}(x^{(i)}) \right)^\intercal - \frac{\partial J(f_{d,\theta_{0}})(x^{(i)})_{\alpha-1}}{\partial h_{0}^{(L+1)}(x^{(i)})} \left( g_{0}^{(L)}(x^{(i)}) \right)^\intercal \right\|_1
    \\
    \le &\kappa \omega \exp(O(L))\frac{(\log d)^{2L}}{\sqrt{d}} + \kappa \omega \exp(O(L)) \frac{(\log d)^L}{\sqrt{d}}.
  \end{align*}
  Again setting $\omega \le (\log d)^{-3L}$ allow us to bound the entires term by norm at $t=0$, showing that
  \begin{align*}
    \left\| W^{(l)}(t') - W^{(l)}(0) \right\|_1 &\le O\left( \frac{d_0}{N \kappa \lambda \lambda_0} \right) \exp(O(L)) \log d,
    \\
    \left\| W^{(L+1)}(t') - W^{(L+1)}(0) \right\|_1 &\le O\left( \frac{d_0}{N \kappa \lambda \lambda_0} \right) \exp(O(L)) \frac{\log d}{\sqrt{d}}.
  \end{align*}
  Setting $d \ge \Omega(1/\omega^2)$ makes these terms to be bounded by $\omega \sqrt{d} \log d, \omega \log d$, respectively. 

  All three conditions can not happen, and this proves our result.
\end{proof}

Now we are ready to prove Theorem~\ref{thm:JacobianNTKTraining}. 
We first give a detailed statement of this theorem.
\begin{theorem} \label{thm:JNTKCloseToInit2}
  Fix the fail probability $\delta$, and error $\epsilon$. 
  Suppose that Assumption~\ref{assm:NTK-smallest-eigenvalue} holds with the smallest eigenvalue $\lambda_0$.

  Suppose that the width $d$, the coefficient $\kappa$, and the threshold $\omega$ satisfy
  \begin{align*}
    \kappa &\le \frac{1}{2 \log (8N(1+d_0) / \delta)},
    \\
    d &\ge F_1\left( x^{(1:N)}, L, \delta, \frac{\epsilon}{\kappa} \right),
    \\
    d &\ge F_2\left( x^{(1:N)}, L, \delta, \epsilon \right),
    \\
    d &\ge \Omega(N d_0 \exp(L) / \delta),
    \\
    d &\ge \Omega(\omega^{-2}),
    \\
    \omega &\le O\left( \frac{\lambda_0 \lambda}{ N \exp(\Omega(L)) (\log d)^{3L}} \right),
    \\
    \omega &\le (\log d)^{-3L}
  \end{align*}
  
  Then with probability at least $1 - \delta$, for all $t \ge 0$,
  \[
    \left\| \Theta_{d, \theta_t}(x^{(i)}, x^{(j)}) - \kappa^2 \Theta(x^{(i)}, x^{(j)}) \right\|_F \le \kappa^2 \epsilon.
  \]
  for all $i, j \in [N]$.
  
  Moreover, linear convergence happens:
  \[
    \mathcal{L}(\theta_t) \le \exp \left( -\frac{\kappa^2 \lambda \lambda_0 t}{N} \right) \mathcal{L}(\theta_0).
  \]
\end{theorem}
\begin{proof}
  From the assumption of $\kappa$, we can show that 
  \[
    P\left( \forall i \in [N]. \left\|g(x^{(i)})\right\|_\infty \le 1 \right) \ge 1 - \frac{\delta}{4}
  \]
  where $g \sim \cN(0, \kappa^2 \Sigma^{(L)})$, and together with the first assumption on $d$, the network outputs and their Jacobian are bounded by 1 at initialisation with probability at least $1 - \delta/2$.
  
  We will first prove that the linear convergence and the weights stay close to their initialisation.
  This can be proven by contradiction argument with Lemma~\ref{lem:WeightStayStill} and \ref{lem:linearconvergence}.

  Suppose that either linear convergence does not happen,
  \[
    \mathcal{L}(\theta_t) > \mathcal{L}(\theta_0) \exp \left( - \frac{\kappa^2 \lambda \lambda_0 t}{N} \right)
  \]
  or one of the weights becomes far from its initialisation,
  \begin{align*}
    \left\| W^{(l)}(t) - W^{(l)}(0) \right\| &> \omega \sqrt{d},
    \\
    \left\|W^{(l)}(t) - W^{(l)}(0)\right\|_\infty &> \begin{cases} \omega \sqrt{d} \log d & \text{ if } l \ge 2, \\ \omega \log d & \text{ if }l = 1, \end{cases}
    \\
    \left\|W^{(l)}(t) - W^{(l)}(0)\right\|_1 &> \begin{cases} \omega \sqrt{d} \log d & \text{ if } l \ge L, \\ \omega \log d & \text{ if }l = L+1. \end{cases}
  \end{align*}
  Write $\tau$ as the first violation time that this happens, 
  \[
    \tau = \argmin_{0 \le t} \left\{t : \text{Linear Convergence fail or weight is far from init.}\right\}.
  \]
  Then by definition, the linear convergence is satisfied, and the weights are close to its initialisation for all $0 \le t < \tau$.
  Now Lemma~\ref{lem:linearconvergence} shows that at time $\tau$, the linear convergence happens, and similarly, Lemma~\ref{lem:WeightStayStill} show that the weights at time $\tau$ also stay close to its initialisation.
  This is a contradiction, which shows that our assumption, of the existence of a violation point is false.

  So now we can assume that
  \begin{align*}
    \left\| W^{(l)}(t) - W^{(l)}(0) \right\| &\le \omega \sqrt{d},
    \\
    \left\|W^{(l)}(t) - W^{(l)}(0)\right\|_\infty &\le \begin{cases} \omega \sqrt{d} \log d & \text{ if } l \ge 2, \\ \omega \log d & \text{ if }l = 1, \end{cases}
    \\
    \left\|W^{(l)}(t) - W^{(l)}(0)\right\|_1 &\le \begin{cases} \omega \sqrt{d} \log d & \text{ if } l \ge L, \\ \omega \log d & \text{ if }l = L+1. \end{cases}
  \end{align*}
  for all $l \in [L+1]$ and $t \ge 0$.
  Now by Theorem~\ref{thm:JNTKCloseToInit2}, we can show that 
  \[
    \left\| \Theta_{d, \theta_t}(x^{(i)}, x^{(j)}) - \kappa^2 \Theta(x^{(i)}, x^{(j)}) \right\|_F \le \kappa^2 \epsilon
  \]
  for all time $t \ge 0$, $i, j \in [N]$.
\end{proof}

\section{Proof of Theorem~\ref{thm:trainingdynamics}}
\label{sec:ProofTrainingDynamics}

Before presenting the proof, we need several additional definitions to be used in the proof.
We define the feature function of limiting JNTK $\phi(x)$, which satisfies
\[
  \phi(x)^\intercal \phi(x') = \Theta(x, x')
\]
with $\phi_0(x), \ldots, \phi_{d_0}(t)$ the columns of this feature function.

We then define the limiting JNTK regressor as
\[
  \fntk(x; \vartheta) \defeq \kappa \Lambda \phi(x)^\intercal \vartheta \in \R^{d_0 + 1}
\]
for trainable $\vartheta$, and define $\untk(\vartheta) \in \R^{N(1+d_0)}$ by stacking $\fntk(x^{(i)}; \vartheta)$ for $i \in [N]$.

If we consider the training of $\vartheta$ under the same objective (1), which is reformulated as
\[
  \mathcal{L}(\vartheta) \defeq \frac{1}{2N} \left\| \unn(\vartheta) - \mathbf{y} \right\|_2^2
\]
when initialised with $\vartheta_0 = 0$, the solution is given by
\[
  \vartheta_\infty = \argmin_{\vartheta} \left\|\vartheta\right\|_2^2, \text{ such that } \unn(\vartheta) = \mathbf{y}.
\]
Note that this optimal parameter $\vartheta_\infty$ gives the solution $\fntk$ defined (5).

We also define $\fnn$ similarly,
\[
  \fnn(x; \theta) \defeq \left( f_{d,\theta}(x; \theta), \sqrt{\lambda} J(f_{d,\theta})(x; \theta)_0, \ldots, \sqrt{\lambda} J(f_{d,\theta})(x; \theta)_{d_0-1} \right)^\intercal \in \R^{1+d_0}
\]
and $\unn(\theta) \in \R^{N(d_0+1)}$ by stacking $\fnn(x^{(i)}; \theta)$ for $i \in [N]$.

Now to track the dynamics of the limiting JNTK regressor, we will take the time derivative.
\begin{align*}
  \frac{d \fntk(x; \vartheta_t)}{dt} &= \left( \frac{\partial \fntk(x; \vartheta_t)}{\partial \vartheta_t} \right)^\intercal \left( \frac{d \vartheta_t}{dt} \right)
  \\
  &= \left\langle \frac{\partial \fntk(x; \vartheta)}{\partial \vartheta_t}, -\frac{\partial \mathcal{L}(\vartheta_t)}{\partial \vartheta_t}\right\rangle
  \\
  &= \left\langle \kappa \phi(t), - \frac{\partial \mathcal{L}(\vartheta_t)}{\partial \vartheta_t}\right\rangle
  \\
  &= - \frac{1}{N} \left\langle \kappa \phi(t), \sum_{i=1}^N  \left(\fntk(x^{(i)}; \vartheta_t)  - \mathbf{y}^{(i)}\right)^\intercal \frac{\partial \fntk(x^{(i)}; \vartheta_t)}{\partial \vartheta_t} \right\rangle
  \\
  &= - \frac{1}{N} \left\langle\kappa \phi(t), \sum_{i=1}^N  \left(\fntk(x^{(i)}; \vartheta_t)- \mathbf{y}^{(i)}\right)^\intercal \kappa \phi(x^{(i)}) \right\rangle
  \\
  &= - \frac{\kappa^2}{N} \sum_{i=1}^N \Theta_\lambda(x, x^{(i)}) \left( \fntk(x^{(i)}; \vartheta_t) - \mathbf{y}^{(i)} \right).
\end{align*}
Then we can see that these dynamics are nearly identical to NN's dynamics, except that they use different kernels.
\begin{align*}
  \frac{d \fnn(x; \theta_t)}{dt} &= \left( \frac{\partial \fnn(x; \theta_t)}{\partial \theta_t} \right)^\intercal \left( \frac{d \theta_t}{dt} \right)
  \\
  &= \left\langle \frac{\partial \fnn(x; \theta)}{\partial \theta_t}, -\frac{\partial \mathcal{L}(\theta_t)}{\partial \theta_t}\right\rangle
  \\
  &= - \frac{1}{N} \left\langle \frac{\partial \fnn(x; \theta)}{\partial \theta_t}, \sum_{i=1}^N  \left(\fnn(x^{(i)}; \vartheta_t) - y^{(i)}\right)^\intercal \frac{\partial \fnn(x^{(i)}; \theta_t)}{\partial \theta_t} \right\rangle
  \\
  &= - \frac{1}{N} \sum_{i=1}^N \Lambda \Theta_{d, \theta_t}(x, x^{(i)}) \Lambda \left( \fnn(x^{(i)}; \theta_t)  - \mathbf{y}^{(i)} \right).
\end{align*}
So to analyse the difference between limiting JNTK regressor and NN, we can bound them using the integral form
\begin{align*}
  &\left| f_{d,\theta_\infty}(x^*) - \fntk(x^*) \right|
  \\
  \le &\left| f_{d,\theta_0}(x^*) - \fntk(x^*; \vartheta_0)_0\right| + \left| \int_0^\infty \left( \frac{d \fnn(x^*; \theta_t)_0}{dt} - \frac{d \fntk(x^*; \vartheta_t)_0}{dt} \right) \right|
\end{align*}
First, since we initialised $\vartheta_0 = 0$, the first term is $|f_{d,\theta_0}(x^*)|$ which is bounded by $\epsilon_0$ with our assumption, so it is enough to show that the integral term is smaller than $\epsilon_0$.

Using our previous expansion of time derivative, we can rewrite the integral as
\begin{align*}
  &\left| \int_0^\infty \left( \frac{d \fnn(x^*; \theta_t)_0}{dt} - \frac{d \fntk(x^*; \vartheta_t)_0}{dt} \right) \right|
  \\
  = &\bigg| - \frac{1}{N} \sum_{i=1}^N \int_0^\infty \bigg(\Lambda \Theta_{d,\theta_t}(x^*, x^{(i)}) \Lambda \left( \fnn(x^{(i)}; \theta_t) - \mathbf{y}^{(i)} \right) 
  \\
  &\qquad \qquad \quad \quad - \kappa^2 \Theta_\lambda(x^*, x^{(i)}) \left( \fntk(x^{(i)}; \vartheta_t) - \mathbf{y}^{(i)} \right)\bigg)_0 dt \bigg|
  \\
  \le &\frac{1}{N} \left| \sum_{i=1}^N \int_0^\infty \left( \left(\Lambda \Theta_{d,\theta_t}(x^*, x^{(i)}) \Lambda - \kappa^2 \Theta_\lambda(x^*, x^{(i)})\right) \left( \fnn(x^{(i)}; \theta_t)- \mathbf{y}^{(i)} \right) \right)_0 dt \right|
  \\
  &\quad + \frac{1}{N} \left| \sum_{i=1}^N \int_0^\infty \left(\kappa^2 \Theta_\lambda(x^*, x^{(i)})\left( \fnn(x^{(i)}; \theta_t) - \fntk(x^{(i)}; \vartheta_t) \right) \right)_0 dt \right|.
\end{align*}
For the first term, we can bound it through 
\begin{align*}
  &\frac{1}{N} \left| \sum_{i=1}^N \int_0^\infty \left( \left(\Lambda \Theta_{d,\theta_t}(x^*, x^{(i)}) \Lambda - \kappa^2 \Theta_\lambda(x^*, x^{(i)})\right) \left( \fnn(x^{(i)}; \theta_t) - \mathbf{y}^{(i)} \right) \right)_0 dt \right|
  \\
  \le &\frac{1}{N} \max_{\tau \ge 0} \left( \sum_{i=1}^N \left\| \Lambda \Theta_{d,\theta_\tau}(x^*, x^{(i)}) \Lambda - \kappa^2 \Theta_\lambda (x^*, x^{(i)}) \right\|_F \right) \int_0^\infty \left\| \unn(\theta_t) - \mathbf{y} \right\| dt.
\end{align*}
By Theorem~\ref{thm:JNTKCloseToInit2}, we can let $d$ large enough so that 
\[
  \left\| \Lambda \Theta_{d,\theta_\tau}(x^*, x^{(i)}) \Lambda - \kappa^2 \Theta_\lambda (x^*, x^{(i)}) \right\|_F \le \kappa^2 \epsilon_\Theta
\]
for all $\tau \ge 0$ and $i \in [N]$, for some $\epsilon_\Theta$ that will be choosen later.
And from the same Theorem, we can also show that the integral term is bounded, as follows:
\begin{align*}
  \int_0^\infty \left\| \unn(\theta_t) - \mathbf{y} \right\| dt 
  &= \int_0^\infty \sqrt{2N \mathcal{L}(\theta_t)} dt
  \\
  &\le \sqrt{2N \mathcal{L}(\theta_0)} \int_0^\infty \exp\left(- \frac{\kappa^2 \lambda \lambda_0 t}{2N}\right) dt
  \\
  &\le \frac{\sqrt{8N^3 \mathcal{L}(\theta_0)}}{\kappa^2 \lambda \lambda_0}.
\end{align*}
From Remark~\ref{rmk:InitCloseToZero}, we can assume that all function outputs and their Jacobian are in $[-1, 1]$, which allows us to bound $\sqrt{\mathcal{L}(\theta_0)} \le \sqrt{2 (1+d_0)}$.

For the second term, we will first divide this integral into two intervals $[0, t_0]$ and $[t_0, \infty]$,
\begin{align*}
  &\frac{1}{N} \left| \left(\sum_{i=1}^N \int_0^\infty \kappa^2 \Theta_\lambda(x^*, x^{(i)}) \left( \fnn(x^{(i)}; \theta_t) - \fntk(x^{(i)}; \vartheta_t) \right) \right)_0 dt \right|
  \\
  \le &\frac{1}{N} \left(\sum_{i=1}^N \left\| \kappa^2 \Theta_\lambda(x^*, x^{(i)}) \right\|_F \right) \int_0^\infty \left\| \unn(\theta_t) - \untk(\vartheta_t) \right\| dt
  \\
  = &\frac{1}{N} \left(\sum_{i=1}^N \left\| \kappa^2 \Theta_\lambda(x^*, x^{(i)}) \right\|_F \right) \left(\int_{t_0}^\infty \left\| \unn(\theta_t) - \untk(\vartheta_t) \right\| dt + \int_{0}^{t_0} \left\| \unn(\theta_t) - \untk(\vartheta_t) \right\| dt\right)
\end{align*}
for some $t_0 \ge 0$, so to prove that this term converges to zero, we should show that both integrals converge to zero.

The integral from $t_0$ to $\infty$ can be bounded with linear convergence of both predictors:
\begin{align*}
  &\int_{t_0}^\infty \left\| \unn(\theta_t) - \untk(\vartheta_t) \right\| dt
  \\
  \le &\int_{t_0}^\infty \left\| \unn(\theta_t) - \mathbf{y} \right\| + \left\|\untk(\vartheta_t) - \mathbf{y}\right\| dt
  \\
  = &\int_{t_0}^\infty \sqrt{2N \mathcal{L}(\theta_t)} + \sqrt{2N \mathcal{L}(\vartheta_t)} dt
  \\
  \le &\int_{t_0}^\infty \sqrt{2N \mathcal{L}(\theta_{t_0})} \exp\left( - \frac{\kappa^2 \lambda \lambda_0 (t - t_0)}{2N} \right) + \sqrt{2N \mathcal{L}(\vartheta_{t_0})} \exp\left( - \frac{\kappa^2 \lambda \lambda_0 (t - t_0)}{2N } \right) dt
  \\
  \le & \left( \left(\sqrt{2N \mathcal{L}(\theta_0)} + \sqrt{2N \mathcal{L}(\vartheta_0)}\right) \exp\left( - \frac{\kappa^2 \lambda \lambda_0 t_0}{2N} \right) \right)\frac{2N}{\kappa^2 \lambda \lambda_0}
  \\
  \le &\left( 4\sqrt{N(1 + d_0)} \exp\left( - \frac{\kappa^2 \lambda \lambda_0 t_0}{2N} \right) \right)\frac{2N}{\kappa^2 \lambda \lambda_0}
\end{align*}
where we used Remark~\ref{rmk:InitCloseToZero} to show that $\sqrt{\mathcal{L}(\theta_0)} \le \sqrt{2(d_0 + 1)}$ and $\vartheta_0 = 0$ to show that $\sqrt{\mathcal{L}(\vartheta_0)} \le \sqrt{2(1 + d_0)}$.
Then setting $t_0$ 
\[
  t_0 = \frac{2N}{\kappa^2\lambda \lambda_0} \log \left( \frac{8\sqrt{N^3 (1+d_0)}}{\lambda \lambda_0 \epsilon_\Theta} \right)
\]
allow us to bound this integral by $\epsilon_\Theta / \kappa^2$.
Combining with $\left\|\kappa^2 \Theta_\lambda (x^*, x^{(i)})\right\|_F \le \kappa^2 (L+1) (1 + d_0)$ gives resulting bound $\epsilon_\Theta (L+1) (1 + d_0)$.

For the rest of the integral, we can expand the difference in the integral form
\begin{align*}
  &\unn(\theta_t) - \untk(\vartheta_t)
  \\
  = &\unn(\theta_0) + \int_0^t \frac{d \unn(\theta_\tau)}{d\tau} - \frac{d \untk(\vartheta_\tau)}{d\tau } d\tau.
\end{align*}
We can continue rewriting the integral as
\begin{align*}
  &\frac{d \unn(\theta_\tau) - \untk(\vartheta_\tau)}{d\tau}
  \\
  = & - \Lambda^{\oplus N} \Theta_{d,\theta_\tau}(x^{(1:N)}, x^{(1:N)}) \Lambda^{\oplus N} (\unn(\theta_\tau) - \mathbf{y}) 
  \\
  &\quad + \kappa^2 \Theta_\lambda(x^{(1:N)}, x^{(1:N)}) (\untk(\vartheta_\tau) - \mathbf{y})
  \\
  = &\left( - \Lambda^{\oplus N} \Theta_{d,\theta_\tau}(x^{(1:N)}, x^{(1:N)}) \Lambda^{\oplus N} + \kappa^2 \Theta_\lambda (x^{(1:N)}, x^{(1:N)}) \right)(\unn(\theta_\tau) - \mathbf{y})
  \\
  &\quad - \kappa^2 \Theta_\lambda (x^{(1:N)}, x^{(1:N)}) \left(\unn(\theta_\tau) - \untk(\vartheta_\tau)\right).
\end{align*}
From Assumption~\ref{assm:NTK-smallest-eigenvalue}, $\Theta_\lambda(x^{(1:N)}, x^{(1:N)})$ is positive definite, so the second term makes the $\|\unn(\theta_\tau) - \untk(\vartheta_\tau)\|$ decrease. 
So to bound the norm we can ignore it, and focus on the first term only.
\begin{align*}
  &\left\| \unn(\theta_t) - \untk(\vartheta_t) \right\| 
  \\
  \le &\Bigg\| \unn(\theta_0) 
  \\
  &\quad + \int_0^t \left( - \Lambda^{\oplus N} \Theta_{d,\theta_\tau}(x^{(1:N)}, x^{(1:N)}) \Lambda^{\oplus N} + \kappa^2 \Theta_\lambda (x^{(1:N)}, x^{(1:N)})\right) (\unn(\theta_\tau) - \mathbf{y})  d\tau \Bigg\|
  \\
  \le &\epsilon_0 \sqrt{N(1 + d_0)} 
  \\
  &\quad + \int_0^t \left\|- \Lambda^{\oplus N} \Theta_{d,\theta_\tau}(x^{(1:N)}, x^{(1:N)}) \Lambda^{\oplus N} + \kappa^2 \Theta_\lambda (x^{(1:N)}, x^{(1:N)}) \right\|_F \|\unn(\theta_\tau) - \mathbf{y}\| d\tau 
  \\
  \le &\epsilon_0 \sqrt{N(1 + d_0)} 
  \\
  &\quad + \max_{0 \le \tau \le t} \left\|- \Lambda^{\oplus N} \Theta_{d,\theta_\tau}(x^{(1:N)}, x^{(1:N)}) \Lambda^{\oplus N}+ \kappa^2  \Theta_\lambda (x^{(1:N)}, x^{(1:N)})\right\|_F \int_0^t \|\unn(\theta_\tau) - \mathbf{y}\| d\tau
  \\
  \le &\epsilon_0 \sqrt{N(1 + d_0)} + N \kappa^2 \epsilon_\Theta \int_0^t \|\unn(\theta_\tau) - \mathbf{y}\| d\tau
  \\
  \le &\epsilon_0 \sqrt{N(1 + d_0)} + N \kappa^2 \epsilon_\Theta \int_0^\infty \sqrt{2N \mathcal{L}(\theta_\tau)} d\tau
  \\
  \le &\epsilon_0 \sqrt{N(1 + d_0)} + 2N \kappa^2 \epsilon_\Theta \sqrt{2N \mathcal{L}(\theta_0)} \int_0^\infty \exp\left( - \frac{\kappa^2 \lambda \lambda_0 \tau}{2N} \right) d\tau
  \\
  \le &\epsilon_0 \sqrt{N(1 + d_0)} + \frac{8N^2 \sqrt{N(1 + d_0)} \epsilon_\Theta}{\lambda \lambda_0} .
\end{align*}
Integrating it gives 
\begin{align*}
  &\int_0^{t_0} \left\| \unn(\theta_t) - \untk(\vartheta_t) \right\| dt
  \\
  \le &\left(\epsilon_0 \sqrt{N(1+d_0)} + \frac{8N^2 \sqrt{N(1 + d_0)} \epsilon_\Theta}{\lambda \lambda_0}\right) t_0
  \\
  = &\left(\epsilon_0 \sqrt{N(1+d_0)} + \frac{8N^2 \sqrt{N(1 + d_0)} \epsilon_\Theta}{\lambda \lambda_0}\right) \times \left( \frac{2N}{\kappa^2 \lambda \lambda_0} \log \left( \frac{8\sqrt{N^3(1+d_0)}}{\lambda \lambda_0 \epsilon_\Theta} \right) \right).
\end{align*}
We finally multiply $\kappa^2 (L+1) \cdot (1+d_0)$ as we've done in the first integral.

Summing up, the bound is
\begin{align*}
  &\left| f_{d,\theta_\infty}(x^*) - \fntk(x^*)_0 \right|
  \\
  \le & \epsilon_0 + \frac{4 \epsilon_\Theta \sqrt{N^3 (1 + d_0)}}{\lambda \lambda_0} + (L+1) \cdot (d_0 + 1) \cdot \epsilon_\Theta
  \\ 
  &+ (L+1) \cdot (d_0 + 1) \cdot \left(\epsilon_0 \sqrt{N(1+d_0)} + \frac{8N^2 \sqrt{N(1 + d_0)} \epsilon_\Theta}{\lambda \lambda_0}\right) \frac{2N}{\lambda \lambda_0} \log \frac{8\sqrt{N^3(1+d_0)}}{\lambda \lambda_0 \epsilon_\Theta} .
\end{align*}

\section{Proofs of Theorems~\ref{thm:JacobianNTKTraining} and \ref{thm:trainingdynamics} for Gradient Descent}

Consider the loss evaluation at $k+1$ time step with abusing notation of $u_{nn}(\theta_{k+1})$, then
\begin{align*}
  &\|\unn(\theta_{k+1}) - \mathbf{y}\|_2^2
  \\
  = &\|\mathbf{y} - \unn(\theta_k) + \unn(\theta_k) - \unn(\theta_{k+1})\|_2^2
  \\
  = &\|\mathbf{y} - \unn(\theta_k)\|_2^2 - 2 \langle \mathbf{y} - \unn(\theta_k), \unn(\theta_{k+1}) - \unn(\theta_k) \rangle + \|\unn(\theta_{k+1}) - \unn(\theta_k)\|_2^2
\end{align*}
so to quantify convergence, we should analyse
\[
  2 \langle \mathbf{y} - \unn(\theta_k), \unn(\theta_{k+1}) - \unn(\theta_k) \rangle,\qquad  \|\unn(\theta_{k+1}) - \unn(\theta_k)\|_2^2.
\]
Now by the Taylor expansion, we can first rewrite $\unn(\theta_{k+1}) - \unn(\theta_k)$, which is also a gradient descent analogue of Lemma~\ref{lem:JNTKdynamics},
\begin{align*}
  &\unn(\theta_{k+1}) - \unn(\theta_k)
  \\
  = &\unn(\theta_k - \eta \nabla_{\theta_k} \mathcal{L}(\theta_k)) - \unn(\theta_k)
  \\
  = &- \int_0^\eta \left\langle \nabla_{\theta_k} \mathcal{L}(\theta_k), \frac{\partial \unn(\theta_{k,s})}{\partial \theta_{k,s}} \right\rangle ds
  \\
  = &- \frac{1}{N}\int_{0}^\eta (\unn(\theta_k) - \mathbf{y})^\intercal \Lambda^{\oplus N}\Theta_{d,\theta_k,\theta_{k,s}}(x^{(1:N)}, x^{(1:N)})\Lambda^{\oplus N} ds
\end{align*}
where we introduced two extensions of notations:
\[
  \theta_{k,s} = \theta_k - s \nabla_{\theta_k} \mathcal{L}(\theta_k)
\]
and
\begin{align*}
  \Theta_{d,\theta_k,\theta_{k,s}}(x, x')_{00} &= \left\langle \frac{\partial f_{d,\theta_k}(x)}{\partial \theta_k}, \frac{\partial f_{d,\theta_{k,s}}(x')}{\partial \theta_{k,s}} \right\rangle,
  \\
  \Theta_{d,\theta_k,\theta_{k,s}}(x, x')_{\alpha 0} &= \left\langle \frac{\partial J(f_{d,\theta_k})(x)_{\alpha-1}}{\partial \theta_k}, \frac{\partial f_{d,\theta_{k,s}}(x')}{\partial \theta_{k,s}} \right\rangle,
  \\
  \Theta_{d,\theta_k,\theta_{k,s}}(x, x')_{0 \beta} &= \left\langle \frac{\partial f_{d,\theta_k}(x)}{\partial \theta_k}, \frac{\partial J(f_{d,\theta_{k,s}})(x')_{\beta-1}}{\partial \theta_{k,s}} \right\rangle,
  \\
  \Theta_{d,\theta_k,\theta_{k,s}}(x, x')_{\alpha \beta} &= \left\langle \frac{\partial J(f_{d,\theta_k})(x)_{\alpha-1}}{\partial \theta_k}, \frac{\partial J(f_{d,\theta_{k,s}})(x')_{\beta-1}}{\partial \theta_{k,s}} \right\rangle
\end{align*}
so that $\Theta_{d,\theta_k,\theta_{k,s}}(x^{(1:N)}, x^{(1:N)})$ is defined in the same manner of $\Theta_{d,\theta}$.

If we can show that this empirical JNTK also satisfies the following for all $k, s \in [0, \eta]$:
\[
  \lambda_{\min}(\Theta_{d,\theta_k,\theta_{k,s}}(x^{(1:N)}, x^{(1:N)})) \ge \frac{\kappa^2}{2} \lambda_0
\]
we can prove that
\begin{align*}
  &\langle \mathbf{y} - \unn(\theta_k), \unn(\theta_{k+1}) - \unn(\theta_k) \rangle
  \\
  = &\frac{1}{N}\int_0^\eta (\unn(\theta_k) - \mathbf{y})^\intercal \Lambda^{\oplus N} \Theta_{d,\theta_k,\theta_{k,s}}(x^{(1:N)}, x^{(1:N)}) \Lambda^{\oplus N} (\unn(\theta_k) - \mathbf{y}) ds
  \\
  \ge &\frac{1}{N}\int_0^\eta (\unn(\theta_k) - \mathbf{y})^\intercal \lambda_{\min}(\Lambda^{\oplus N} \Theta_{d,\theta_k,\theta_{k,s}}(x^{(1:N)}, x^{(1:N)}) \Lambda^{\oplus N}) (\unn(\theta_k) - \mathbf{y}) ds
  \\
  \ge & \frac{\eta \kappa^2 \lambda \lambda_0}{2N} \|\unn(\theta_k) - \mathbf{y}\|_2^2.
\end{align*}

So if we can show that the third term, $\|\unn(\theta_{k+1}) - \unn(\theta_k)\|_2^2$ is small enough so that it can't cancel out the second term, we can show that the loss decreases every step by multiplicative factor $(1 - \eta \kappa^2 \lambda \lambda_0 / 2)$.
Our proof is done by the induction on two properties, (1) the weights stay close to their initialisation, and (2) the term $\|\unn(\theta_{k+1}) - \unn(\theta_k)\|^2$ has roughly order $\eta^2 \kappa^4$, which then shows the linear convergence.

We first introduce the $Q(k)$ indexed by $k \in \mathbb{Z}_{\ge 0}$, which is defined as
\begin{align*}
  Q(k) = \frac{1}{2}(d_0 + 1) \sum_{i=0}^{k-1} \left( 1 - \frac{\eta \kappa^2 \lambda \lambda_0}{4N} \right)^{i/2}
\end{align*}
which will upper bound $\sum_{i=0}^{k-1} \sqrt{\mathcal{L}(\theta_{i})}$.
Note that this quantity itself is again bounded globally, by
\begin{align*}
  Q(k) &\le Q(\infty)
  \\
  &= \frac{d_0 + 1}{2} \sum_{i=0}^\infty \left( 1 - \frac{\eta \kappa^2 \lambda \lambda_0}{4N} \right)^{i/2}
  \\
  &= \frac{d_0 + 1}{2 (1 - \sqrt{1 - \eta \kappa^2 \lambda \lambda_0 / 4N})}.
\end{align*}
We then define $\omega(k)$ as
\[
  \omega(k) = \eta \kappa \exp(O(L)) \sqrt{N d_0} Q(k) \frac{(\log d)^{3L}}{\sqrt{d}} 
\]
which similarly bounded by 
\[
  \omega(k) \le \omega(\infty) = \eta \kappa \exp(O(L)) \sqrt{N d_0} Q(\infty) \frac{(\log d)^{3L}}{\sqrt{d}}.
\]
Then we further bound this quantity by $\omega$ which will be used as in the gradient flow case, 
\begin{align*}
  \omega(\infty) &\le \kappa \exp(O(L)) \sqrt{N d_0^3}  \frac{(\log d)^{3L}}{\sqrt{d}} \frac{\eta}{1 - \sqrt{1 - \eta \kappa^2 \lambda \lambda_0 / 4N}}
  \\
  &\le \kappa \exp(O(L)) \sqrt{N d_0^3}  \frac{(\log d)^{3L}}{\sqrt{d}} \frac{8N}{\kappa^2 \lambda \lambda_0}
  \\
  &= \underbrace{\frac{\exp(O(L)) (N d_0)^{3/2} (\log d)^{3L}}{\kappa \lambda \lambda_0 \sqrt{d}}}_{\omega \defeq}.
\end{align*}
where the inequality holds since $\eta / (1 - \sqrt{1 - t \eta})$ monotonically decreases at $\eta \ge 0$, with its maximum $2 / t$ attained at $\eta = 0$.

We first state the lemma that will be used in both induction hypotheses.
\begin{lemma} \label{lem:weightdiff}
  Fix the fail probability $\delta > 0$ and network depth $L$. 
  Suppose that the network width $d$ and $\omega(\infty)$ satisfies
  \begin{align*}
    d &\ge \Omega(N d_0 \exp(L) / \delta),
    \\
    d &\ge \Omega(\omega^{-2}),
    \\
    \omega &\le O\left( \frac{\lambda_0 \lambda}{N \exp(\Omega(L)) (\log d)^{3L}} \right),
    \\
    \omega &\le (\log d)^{-2L}.
  \end{align*}

  With probability at least $1-\delta$ over the random initialisation, the following holds for all $k_0$ and any $l$. 
  If for all $0 \le k < k_0$, all the weights stay close to their initialisation, i.e.,
  \begin{align*}
    \left\|W^{(l)}(k) - W^{(l)}(0)\right\| &\le \omega(k) \sqrt{d},
    \\
    \left\|W^{(l)}(k) - W^{(l)}(0)\right\|_\infty &\le \begin{cases}
      \infty & \text{if } l = L+1,
      \\
      \omega(k) \sqrt{d} \log d & \text{if } 2 \le l \le L,
      \\
      \omega(k) \log d & \text{if } l = 1,
    \end{cases}
    \\
    \left\|W^{(l)}(k) - W^{(l)}(0)\right\|_1 &\le \begin{cases}
      \infty & \text{if } l = 1,
      \\
      \omega(k) \sqrt{d} \log d & \text{if } 2 \le l \le L,
      \\
      \omega(k) \log d & \text{if } l = L+1,
    \end{cases}
  \end{align*}
  then the weight update is bounded by
  \begin{align*}
    \left\|W^{(1)}(k+1) - W^{(1)}(k)\right\| &\le \eta \kappa \exp(O(L)) \sqrt{N d_0 \mathcal{L}(\theta_k)} \log d,
    \\
    \left\|W^{(l)}(k+1) - W^{(l)}(k)\right\| &\le \eta \kappa \exp(O(L)) \sqrt{N d_0 \mathcal{L}(\theta_k)} \log d,
    \\
    \left\|W^{(L+1)}(k+1) - W^{(L+1)}(k)\right\| &\le \eta \kappa \exp(O(L)) \sqrt{N d_0 \mathcal{L}(\theta_k)} \log d,
    \\
    \left\|W^{(1)}(k+1) - W^{(1)}(k)\right\|_\infty &\le \eta \kappa \exp(O(L)) \sqrt{N d_0 \mathcal{L}(\theta_k)} \frac{(\log d)^{2L}}{\sqrt{d}},
    \\
    \left\|W^{(l)}(k+1) - W^{(l)}(k)\right\|_\infty &\le \eta \kappa \exp(O(L)) \sqrt{N d_0 \mathcal{L}(\theta_k)} (\log d)^{3L+1},
    \\
    \left\|W^{(l)}(k+1) - W^{(l)}(k)\right\|_1 &\le \eta \kappa \exp(O(L)) \sqrt{N d_0 \mathcal{L}(\theta_k)} (\log d)^{2L},
    \\
    \left\|W^{(L+1)}(k+1) - W^{(L+1)}(k)\right\|_1 &\le \eta \kappa \exp(O(L)) \sqrt{N d_0 \mathcal{L}(\theta_k)} \frac{(\log d)^{2L}}{\sqrt{d}}
  \end{align*}
  for $2 \le l \le L$.
\end{lemma}
\begin{proof}
  We can unfold the weight update rule to get
  \begin{align*}
    &W^{(l)}(k+1) - W^{(l)}(k) 
    \\
    = &- \eta \frac{\partial \mathcal{L}(\theta_{k})}{\partial W^{(l)}(k)} 
    \\
    = &- \frac{\eta}{N} \sum_{i=1}^N \left( (f_{\theta_k}(x^{(i)}) - y^{(i)}) \frac{\partial f_{\theta_k}(x^{(i)})}{\partial W^{(l)}(k)} + \lambda \sum_{\alpha=1}^{d_0} J(f_{\theta_k})(x^{(i)})_{\alpha-1} \frac{\partial J(f_{\theta_k})(x^{(i)})_{\alpha-1}}{\partial W^{(l)}(k)} \right)
  \end{align*}
  which then shows
  \begin{align*}
    &\left\|W^{(l)}(k+1) - W^{(l)}(k)\right\|_p 
    \\\le &\frac{\eta}{N} \sum_{i=1}^N \left( \left| f_{\theta_k}(x^{(i)}) - y^{(i)} \right| + \sqrt{\lambda} \sum_{\alpha=1}^{d_0} \left|J(f_{\theta_k})(x^{(i)})_{\alpha-1} \right| \right) \max_{i \in [N], \alpha \in [d_0]} \left\{\left\| \frac{\partial f_{\theta_k}(x^{(i)})}{\partial W^{(l)}(k)} \right\|_p, \left\| \frac{\partial J(f_{\theta_k})(x^{(i)})_{\alpha-1}}{\partial W^{(l)}(k)} \right\|_p\right\}
    \\
    \le &\frac{\eta \sqrt{(d_0 + 1)} \|\unn(\theta_k) - \mathbf{y}\|}{\sqrt{N}} \max_{i \in [N], \alpha \in [d_0]} \left\{\left\| \frac{\partial f_{\theta_k}(x^{(i)})}{\partial W^{(l)}(k)} \right\|_p, \left\| \frac{\partial J(f_{\theta_k})(x^{(i)})_{\alpha-1}}{\partial W^{(l)}(k)} \right\|_p\right\}
  \end{align*}
  where we used assumption that $\lambda \le 1$.
  The matrix Jacobians can be further bounded as in the proof of Lemma~\ref{lem:WeightStayStill}: for $2 \le l \le L$,
  \begin{align*}
    \left\|\frac{\partial f_{d,\theta_k}(x^{(i)})}{\partial W^{(1)}(k)}\right\| &\le \kappa O(\exp(L)), 
    &
    \left\|\frac{\partial J(f_{d,\theta_k})(x^{(i)})_{\alpha-1}}{\partial W^{(1)}(k)}\right\| &\le \kappa \exp(O(L)) \log d,
    \\
    \left\|\frac{\partial f_{d,\theta_k}(x^{(i)})}{\partial W^{(l)}(k)}\right\| &\le \kappa O(\exp(L)),
    &
    \left\|\frac{\partial J(f_{d,\theta_k})(x^{(i)})_{\alpha-1}}{\partial W^{(l)}(k)}\right\| &\le \kappa \exp(O(L)) \log d,
    \\
    \left\|\frac{\partial f_{d,\theta_k}(x^{(i)})}{\partial W^{(L+1)}(k)}\right\| &\le \kappa O(\exp(L)), 
    &
    \left\|\frac{\partial J(f_{d,\theta_k})(x^{(i)})_{\alpha-1}}{\partial W^{(L+1)}(k)}\right\| &\le \kappa \exp(O(L)) \log d,
    \\
    \left\|\frac{\partial f_{d,\theta_k}(x^{(i)})}{\partial W^{(1)}(k)}\right\|_\infty &\le \kappa \exp(O(L)) \frac{(\log d)^{2L}}{\sqrt{d}}, 
    &
    \left\|\frac{\partial J(f_{d,\theta_k})(x^{(i)})_{\alpha-1}}{\partial W^{(1)}(k)}\right\|_\infty &\le \kappa \exp(O(L)) \frac{(\log d)^{3L+1}}{\sqrt{d}},
    \\
    \left\|\frac{\partial f_{d,\theta_k}(x^{(i)})}{\partial W^{(l)}(k)}\right\|_\infty &\le \kappa \exp(O(L)) (\log d)^{L},
    &
    \left\|\frac{\partial J(f_{d,\theta_k})(x^{(i)})_{\alpha-1}}{\partial W^{(l)}(k)}\right\|_\infty &\le \kappa \exp(O(L)) (\log d)^{3L+1},
    \\
    \left\|\frac{\partial f_{d,\theta_k}(x^{(i)})}{\partial W^{(l)}(k)}\right\|_1 &\le \kappa \exp(O(L)) (\log d)^{2L},
    &
    \left\|\frac{\partial J(f_{d,\theta_k})(x^{(i)})_{\alpha-1}}{\partial W^{(l)}(k)}\right\|_1 &\le \kappa \exp(O(L)) (\log d)^{2L},
    \\
    \left\|\frac{\partial f_{d,\theta_k}(x^{(i)})}{\partial W^{(L+1)}(k)}\right\|_1 &\le \kappa \exp(O(L)) \frac{(\log d)^{L}}{\sqrt{d}}, 
    &
    \left\|\frac{\partial J(f_{d,\theta_k})(x^{(i)})_{\alpha-1}}{\partial W^{(L+1)}(k)}\right\|_1 &\le \kappa \exp(O(L)) \frac{(\log d)^{2L}}{\sqrt{d}}.
  \end{align*}
  We obtain the following inequalities for $2 \le l \le L$:
  \begin{align*}
    \left\|W^{(1)}(k+1) - W^{(1)}(k)\right\| &\le \eta \kappa \exp(O(L)) \sqrt{N d_0 \mathcal{L}(\theta_k)} \log d,
    \\
    \left\|W^{(l)}(k+1) - W^{(l)}(k)\right\| &\le \eta \kappa \exp(O(L)) \sqrt{N d_0 \mathcal{L}(\theta_k)} \log d,
    \\
    \left\|W^{(L+1)}(k+1) - W^{(L+1)}(k)\right\| &\le \eta \kappa \exp(O(L)) \sqrt{N d_0 \mathcal{L}(\theta_k)} \log d,
    \\
    \left\|W^{(1)}(k+1) - W^{(1)}(k)\right\|_\infty &\le \eta \kappa \exp(O(L)) \sqrt{N d_0 \mathcal{L}(\theta_k)} \frac{(\log d)^{2L}}{\sqrt{d}},
    \\
    \left\|W^{(l)}(k+1) - W^{(l)}(k)\right\|_\infty &\le \eta \kappa \exp(O(L)) \sqrt{N d_0 \mathcal{L}(\theta_k)} (\log d)^{3L+1},
    \\
    \left\|W^{(l)}(k+1) - W^{(l)}(k)\right\|_1 &\le \eta \kappa \exp(O(L)) \sqrt{N d_0 \mathcal{L}(\theta_k)} (\log d)^{2L},
    \\
    \left\|W^{(L+1)}(k+1) - W^{(L+1)}(k)\right\|_1 &\le \eta \kappa \exp(O(L)) \sqrt{N d_0 \mathcal{L}(\theta_k)} \frac{(\log d)^{2L}}{\sqrt{d}}.
  \end{align*}
\end{proof}

\begin{lemma} \label{lem:weightinductiongd}
  Fix the fail probability $\delta > 0$ and network depth $L$. 
  Suppose that the network width $d$ and $\omega(\infty)$ satisfies
  \begin{align*}
    d &\ge \Omega(N d_0 \exp(L) / \delta),
    \\
    d &\ge \Omega(\omega^{-2}),
    \\
    \omega &\le O\left( \frac{\lambda_0 \lambda}{N \exp(\Omega(L)) (\log d)^{3L}} \right),
    \\
    \omega &\le (\log d)^{-2L}.
  \end{align*}
  Also assume that the $\kappa$ is small, and network is intialised with $\mathcal{L}(\theta_0) \le d_0 + 1$. 

  With probability at least $1-\delta$ over the random initialisation, the following holds for all $k_0$ and any $l$. 
  If for all $0 \le k < k_0$, all the weights stay close to their initialisation, i.e.,
  \begin{align*}
    \left\|W^{(l)}(k) - W^{(l)}(0)\right\| &\le \omega(k) \sqrt{d},
    \\
    \left\|W^{(l)}(k) - W^{(l)}(0)\right\|_\infty &\le \begin{cases}
      \infty & \text{if } l = L+1,
      \\
      \omega(k) \sqrt{d} \log d & \text{if } 2 \le l \le L,
      \\
      \omega(k) \log d & \text{if } l = 1,
    \end{cases}
    \\
    \left\|W^{(l)}(k) - W^{(l)}(0)\right\|_1 &\le \begin{cases}
      \infty & \text{if } l = 1,
      \\
      \omega(k) \sqrt{d} \log d & \text{if } 2 \le l \le L,
      \\
      \omega(k) \log d & \text{if } l = L+1,
    \end{cases}
  \end{align*}
  and linear convergence, i.e., 
  \[
    \mathcal{L}(\theta_{k}) \le \mathcal{L}(\theta_0) \left( 1 - \frac{\eta \kappa^2 \lambda \lambda_0}{4} \right)^{k_0}
  \]
  holds, then 
  \begin{align*}
    \left\|W^{(l)}(k_0) - W^{(l)}(0)\right\| &\le \omega(k_0) \sqrt{d},
    \\
    \left\|W^{(l)}(k_0) - W^{(l)}(0)\right\|_\infty &\le \begin{cases}
      \infty & \text{if } l = L+1,
      \\
      \omega(k_0) \sqrt{d} \log d & \text{if } 2 \le l \le L,
      \\
      \omega(k_0) \log d & \text{if } l = 1,
    \end{cases}
    \\
    \left\|W^{(l)}(k_0) - W^{(l)}(0)\right\|_1 &\le \begin{cases}
      \infty & \text{if } l = 1,
      \\
      \omega(k_0) \sqrt{d} \log d & \text{if } 2 \le l \le L,
      \\
      \omega(k_0) \log d & \text{if } l = L+1,
    \end{cases}
  \end{align*}
  holds.
\end{lemma}
\begin{proof}
  Combining the linear convergence assumption with the assumption on $\mathcal{L}(\theta_0)$, we can see that the following holds for all $0 \le k < k_0$,
  \[
    \mathcal{L}(\theta_{k}) \le \mathcal{L}(\theta_0) \left( 1 - \frac{\eta \kappa^2 \lambda \lambda_0}{4M} \right)^{k}.
  \]
  Then the result is obtained via triangle inequality with Lemma~\ref{lem:weightdiff}.
\end{proof}

\begin{lemma} \label{lem:smalltermbound}
  Fix the fail probability $\delta > 0$ and network depth $L$. 
  Suppose that the network width $d$ and $\omega(\infty)$ satisfies
  \begin{align*}
    d &\ge \Omega(N d_0 \exp(L) / \delta),
    \\
    d &\ge \Omega(\omega^{-2}),
    \\
    \omega &\le O\left( \frac{\lambda_0 \lambda}{N \exp(\Omega(L)) (\log d)^{3L}} \right),
    \\
    \omega &\le (\log d)^{-2L}.
  \end{align*}

  With probability at least $1-\delta$ over the random initialisation, if all the weights stay close to their initialisation, i.e.,
  \begin{align*}
    \left\|W^{(l)}(k) - W^{(l)}(0)\right\| &\le \omega(k) \sqrt{d},
    \\
    \left\|W^{(l)}(k) - W^{(l)}(0)\right\|_\infty &\le \begin{cases}
      \infty & \text{if } l = L+1,
      \\
      \omega(k) \sqrt{d} \log d & \text{if } 2 \le l \le L,
      \\
      \omega(k) \log d & \text{if } l = 1,
    \end{cases}
    \\
    \left\|W^{(l)}(k) - W^{(l)}(0)\right\|_1 &\le \begin{cases}
      \infty & \text{if } l = 1,
      \\
      \omega(k) \sqrt{d} \log d & \text{if } 2 \le l \le L,
      \\
      \omega(k) \log d & \text{if } l = L+1,
    \end{cases}
  \end{align*}
  then
  \[
    \left\|\unn(\theta_{k+1}) - \unn(\theta_k)\right\|^2 \le \eta^2 \kappa^4 \exp(O(L)) N d_0^2 \|\unn(\theta_k) - \mathbf{y}\|^2.
  \]
\end{lemma}
\begin{proof}
  Most of the technical computations are done in the proof of Lemma~\ref{lem:weightdiff}, and the remaining job is using those bounds on matrix differences to bound the output differences.
  The followings hold for $1 \le l \le L$:
  \begin{align*}
    &\left\| g_{k+1}^{(1)}(x^{(i)}) - g_k^{(1)}(x^{(i)}) \right\| 
    \\
    \le &\left\| W^{(1)}(k+1) - W^{(1)}(k) \right\| \left\|x^{(i)} \right\|
    \\
    \le &\eta \kappa \exp(O(L)) \sqrt{N d_0 \mathcal{L}(\theta_k)},
    \\
    &\left\| h_{k+1}^{(l)}(x^{(i)}) - h_k^{(l)}(x^{(i)}) \right\|
    \\
    \le &M_1 \left\| g_{k+1}^{(l)}(x^{(i)}) - g_k^{(l)}(x^{(i)}) \right\|,
    \\
    &\left\| g_{k+1}^{(l+1)}(x^{(i)}) - g_k^{(l+1)}(x^{(i)}) \right\|
    \\
    \le &\frac{1}{\sqrt{d}} \left\| W^{(l+1)}(k+1) - W^{(l+1)}(k) \right\| \left\| h_{k+1}^{(l)}(x^{(i)}) \right\| + \frac{1}{\sqrt{d}} \left\|W^{(l+1)}(k)\right\| \left\| h_{k+1}^{(l)}(x^{(i)}) - h_k^{(l)}(x^{(i)}) \right\|
    \\
    \le &\eta \kappa \exp(O(L)) \sqrt{N d_0 \mathcal{L}(\theta_k)} + \exp(O(L))\left\| h_{k+1}^{(l)}(x^{(i)}) - h_k^{(l)}(x^{(i)}) \right\|,
    \\
    &\left|f_{d,\theta_{k+1}}(x^{(i)}) - f_{d,\theta_k}(x^{(i)})\right|
    \\
    \le &\kappa \left| g_{k+1}^{(L+1)}(x^{(i)}) - g_k^{(L+1)}(x^{(i)}) \right|
    \\
    \le &\eta \kappa^2 \exp(O(L)) \sqrt{N d_0 \mathcal{L}(\theta_k)}.
  \end{align*}
  Similarly, the following hold for $1 \le l \le L$:
  \begin{align*}
    &\left\| \frac{\partial g_{k+1}^{(1)}(x^{(i)})}{\partial x_{\alpha-1}^{(i)}} - \frac{\partial g_{k}^{(1)}(x^{(i)})}{\partial x_{\alpha-1}^{(i)}} \right\| 
    \\
    \le &\left\| W^{(1)}(k+1) - W^{(1)}(k) \right\| \left\| e_{\alpha-1} \right\|
    \\
    \le &\eta \kappa \exp(O(L)) \sqrt{N d_0 \mathcal{L}(\theta_k)},
    \\
    &\left\| \frac{\partial h_{k+1}^{(l)}(x^{(i)})}{\partial x_{\alpha-1}^{(i)}} - \frac{\partial h_{k}^{(l)}(x^{(i)})}{\partial x_{\alpha-1}^{(i)}} \right\|
    \\
    \le &M_2 \left\| g_{k+1}^{(l)}(x^{(i)}) - g_k^{(l)}(x^{(i)}) \right\| \left\| \frac{\partial g_{k+1}^{(l)}(x^{(i)})}{\partial x_{\alpha-1}^{(i)}} \right\|_\infty + M_1 \left\| \frac{\partial g_{k+1}^{(l)}(x^{(i)})}{\partial x_{\alpha-1}^{(i)}} - \frac{\partial g_{k}^{(l)}(x^{(i)})}{\partial x_{\alpha-1}^{(i)}} \right\|
    \\
    \le & \eta \kappa \exp(O(L)) \sqrt{N d_0 \mathcal{L}(\theta_k)} + M_1 \left\| \frac{\partial g_{k+1}^{(l)}(x^{(i)})}{\partial x_{\alpha-1}^{(i)}} - \frac{\partial g_{k}^{(l)}(x^{(i)})}{\partial x_{\alpha-1}^{(i)}} \right\|,
    \\
    &\left\| \frac{\partial g_{k+1}^{(l+1)}(x^{(i)})}{\partial x_{\alpha-1}^{(i)}} - \frac{\partial g_{k}^{(l+1)}(x^{(i)})}{\partial x_{\alpha-1}^{(i)}} \right\| 
    \\
    \le &\frac{1}{\sqrt{d}} \left\| W^{(l+1)}(k+1) - W^{(l+1)}(k) \right\| \left\| \frac{\partial h_{k+1}^{(l)}(x^{(i)})}{\partial x_{\alpha-1}^{(i)}} \right\| + \frac{1}{\sqrt{d}} \left\| W^{(l+1)}(k) \right\| \left\| \frac{\partial h_{k+1}^{(l)}(x^{(i)})}{\partial x_{\alpha-1}^{(i)}} - \frac{\partial h_{k}^{(l)}(x^{(i)})}{\partial x_{\alpha-1}^{(i)}} \right\|
    \\
    \le &\eta \kappa \exp(O(L)) \sqrt{N d_0 \mathcal{L}(\theta_k)} + \exp(O(L))\left\| \frac{\partial h_{k+1}^{(l)}(x^{(i)})}{\partial x_{\alpha-1}^{(i)}} - \frac{\partial h_{k}^{(l)}(x^{(i)})}{\partial x_{\alpha-1}^{(i)}} \right\|,
    \\
    &\left| J(f_{d,\theta_{k+1}})(x^{(i)})_{\alpha-1} - J(f_{d,\theta_{k}})(x^{(i)})_{\alpha-1} \right|
    \\
    \le &\kappa \left| \frac{\partial g_{k+1}^{(L+1)}(x^{(i)})}{\partial x_{\alpha-1}^{(i)}} - \frac{\partial g_k^{(L+1)}(x^{(i)})}{\partial x_{\alpha-1}^{(i)}} \right|
    \\
    \le &\eta \kappa^2 \exp(O(L)) \sqrt{N d_0 \mathcal{L}(\theta_k)}.
  \end{align*}
  Combining these, we obtain
  \[
    \left\|\unn(\theta_{k+1}) - \unn(\theta_k)\right\|^2 \le \eta^2 \kappa^4 \exp(O(L)) N d_0^2 \|\unn(\theta_k) - \mathbf{y}\|^2.
  \]
\end{proof}

\begin{lemma} \label{lem:linearconvergencegd}
  Fix the fail probability $\delta > 0$ and network depth $L$. 
  Suppose that the network width $d$, learning rate $\eta$, and $\omega(\infty)$ satisfies
  \begin{align*}
    d &\ge \Omega(N d_0 \exp(L) / \delta),
    \\
    d &\ge \Omega(\omega^{-2}),
    \\
    \omega &\le O\left( \frac{\lambda_0 \lambda}{N \exp(\Omega(L)) (\log d)^{3L}} \right),
    \\
    \omega &\le (\log d)^{-2L},
    \\
    \eta &\le \frac{\lambda \lambda_0}{4 N^2 d_0^2} \exp(-\Omega(L)).
  \end{align*}

  With probability at least $1-\delta$ over the random initialisation, if all the weights stay close to their initialisation, i.e.,
  \begin{align*}
    \left\|W^{(l)}(k) - W^{(l)}(0)\right\| &\le \omega(k) \sqrt{d},
    \\
    \left\|W^{(l)}(k) - W^{(l)}(0)\right\|_\infty &\le \begin{cases}
      \infty & \text{if } l = L+1,
      \\
      \omega(k) \sqrt{d} \log d & \text{if } 2 \le l \le L,
      \\
      \omega(k) \log d & \text{if } l = 1,
    \end{cases}
    \\
    \left\|W^{(l)}(k) - W^{(l)}(0)\right\|_1 &\le \begin{cases}
      \infty & \text{if } l = 1,
      \\
      \omega(k) \sqrt{d} \log d & \text{if } 2 \le l \le L,
      \\
      \omega(k) \log d & \text{if } l = L+1,
    \end{cases}
  \end{align*}
  then
  \[
    \mathcal{L}(\theta_{k+1}) \le \left(1 - \frac{\eta \kappa^2 \lambda \lambda_0}{4}\right) \mathcal{L}(\theta_k).
  \]
\end{lemma}
\begin{proof}
  From the Lemma~\ref{lem:weightinductiongd}, we can see that $W^{(l)}(k+1)$ also satisfies similar bound, but with $\omega(k+1)$.
  Using the convexity of matrix norms, we can see that
  \begin{align*}
    \left\|W_{k,s}^{(l)} - W^{(l)}(0)\right\| &\le \omega(k+1) \sqrt{d},
    \\
    \left\|W_{k,s}^{(l)} - W^{(l)}(0)\right\|_\infty &\le \begin{cases}
      \infty & \text{if } l = L+1,
      \\
      \omega(k+1) \sqrt{d} \log d & \text{if } 2 \le l \le L,
      \\
      \omega(k+1) \log d & \text{if } l = 1,
    \end{cases}
    \\
    \left\|W_{k,s}^{(l)} - W^{(l)}(0)\right\|_1 &\le \begin{cases}
      \infty & \text{if } l = 1,
      \\
      \omega(k+1) \sqrt{d} \log d & \text{if } 2 \le l \le L,
      \\
      \omega(k+1) \log d & \text{if } l = L+1,
    \end{cases}
  \end{align*}
  for all $s \in [0, \eta]$, since $W_{k,s}^{(l)}$ linearly interpolates $W^{(l)}(k)$ and $W^{(l)}(k+1)$.

  Using Theorem~\ref{thm:JNTKCloseToInit}, we can show that
  \[
    \lambda_{\min} \left( \Theta_{d,\theta_k,\theta_{k,s}}(x^{(1:N)}, x^{(1:N)}) \right) \ge \frac{\kappa^2 \lambda \lambda_0}{2}
  \]
  which shows that 
  \begin{align*}
    &\left\langle \mathbf{y} - \unn(\theta_k), \unn(\theta_{k+1}) - \unn(\theta_k) \right\rangle
    \\
    = &\frac{1}{N} \int_0^s (\unn(\theta_k) - \mathbf{y})^\intercal \Lambda^{\oplus N} \Theta_{d,\theta_k,\theta_{k,s}}(x^{(1:N)}, x^{(1:N)}) \Lambda^{\oplus N} (\unn(\theta_k) - \mathbf{y}) ds
    \\
    \ge &\frac{1}{N} \int_0^s (\unn(\theta_k) - \mathbf{y})^\intercal \lambda_{\min} \left(\Lambda^{\oplus N} \Theta_{d,\theta_k,\theta_{k,s}}(x^{(1:N)}, x^{(1:N)}) \Lambda^{\oplus N}\right) (\unn(\theta_k) - \mathbf{y}) ds
    \\
    \ge &\frac{\eta \kappa^2 \lambda \lambda_0}{2N} \left\|\unn(\theta_k) - \mathbf{y}\right\|^2.
  \end{align*}
  Also from Lemma~\ref{lem:smalltermbound}, we have
  \[
    \left\|\unn(\theta_{k+1}) - \unn(\theta_k)\right\|^2 \le \eta^2 \kappa^4 \exp(O(L)) N d_0^2 \|\unn(\theta_k) - \mathbf{y}\|^2.
  \]
  From the assumption on $\eta$, this term can be further bounded by
  \[
    \left\|\unn(\theta_{k+1}) - \unn(\theta_k)\right\|^2 \le \frac{\eta \kappa^2 \lambda \lambda_0}{4N} \left\|\unn(\theta_k) - \mathbf{y}\right\|^2.
  \]

  Combining these two inequalities, we obtain
  \[
    \|\unn(\theta_{k+1}) - \mathbf{y}\|^2 \le \left(1 - \frac{\eta \kappa^2 \lambda \lambda_0}{4N}\right) \| \unn(\theta_k) - \mathbf{y} \|^2.
  \]
\end{proof}

Now we have all the ingredients to prove Theorem~\ref{thm:JacobianNTKTraining}.

\begin{theorem} \label{thm:JacobianNTKTrainingGD}
  Fix the fail probability $\delta > 0$ and network depth $L$. 
  Suppose that the network width $d$, learning rate $\eta$, and $\omega(\infty)$ satisfies
  \begin{align*}
    d &\ge \Omega(N d_0 \exp(L) / \delta),
    \\
    d &\ge \Omega(\omega^{-2}),
    \\
    \omega &\le O\left( \frac{\lambda_0 \lambda}{N \exp(\Omega(L)) (\log d)^{3L}} \right),
    \\
    \omega &\le (\log d)^{-2L},
    \\
    \eta &\le \frac{\lambda \lambda_0}{4 N^2 d_0^2} \exp(-\Omega(L)).
  \end{align*}
  Then, with probability at least $1-\delta$ over the random initialisation, the following holds for all $k \ge 0$:
  \begin{itemize}
    \item The weights stay close to their initialisation:
    \begin{align*}
      \left\|W^{(l)}(k) - W^{(l)}(0)\right\| &\le \omega(k) \sqrt{d},
      \\
      \left\|W^{(l)}(k) - W^{(l)}(0)\right\|_\infty &\le \begin{cases}
        \infty & \text{if } l = L+1,
        \\
        \omega(k) \sqrt{d} \log d & \text{if } 2 \le l \le L,
        \\
        \omega(k) \log d & \text{if } l = 1,
      \end{cases}
      \\
      \left\|W^{(l)}(k) - W^{(l)}(0)\right\|_1 &\le \begin{cases}
        \infty & \text{if } l = 1,
        \\
        \omega(k) \sqrt{d} \log d & \text{if } 2 \le l \le L,
        \\
        \omega(k) \log d & \text{if } l = L+1,
      \end{cases}
    \end{align*}
    \item Linear convergence happens:
    \[
      \mathcal{L}(\theta_k) \le \left( 1 - \frac{\eta \kappa^2 \lambda \lambda_0}{4N} \right)^k \mathcal{L}(\theta_0).
    \]
  \end{itemize}
\end{theorem}
\begin{proof}
  The proof is immediate with induction on $k$, and the application of Lemma~\ref{lem:linearconvergencegd} and Lemma~\ref{lem:weightinductiongd}.
\end{proof}

\end{document}